\documentclass[twoside,11pt]{article}

\usepackage{jmlr2e}

\usepackage{booktabs}       
\usepackage{amsfonts}       
\usepackage{nicefrac}       
\usepackage{microtype}      
\usepackage{amsmath}
\usepackage{caption}
\usepackage{graphicx}
\usepackage{subcaption}
\usepackage{amssymb}
\usepackage{bm}
\usepackage{xr}
\usepackage{hyperref}
\usepackage{url}
\usepackage{tabularx}
\usepackage{placeins}
\usepackage{minitoc}
\usepackage{indentfirst}
\usepackage{thm-restate}
\usepackage{float}

\newenvironment{psmallmatrix}{\left(\begin{smallmatrix}}{\end{smallmatrix}\right)}

\ShortHeadings{Failures of VAEs and Their Effects on Downstream Tasks}{Yacoby, Pan \& Doshi-Velez}
\firstpageno{1}

\usepackage[toc,page,header]{appendix}
\usepackage{minitoc}

\begin{document}

\title{Failure Modes of Variational Autoencoders \\
and Their Effects on Downstream Tasks}

\author{\name Yaniv Yacoby \email yanivyacoby@g.harvard.edu \\
	\name Weiwei Pan \email weiweipan@g.harvard.edu \\
	\name Finale Doshi-Velez \email finale@seas.harvard.edu \\
       \addr John A. Paulson School of Engineering and Applied Sciences\\
       Harvard University\\
       Cambridge, MA 02138, USA
       }

\editor{}

\maketitle

\doparttoc 
\faketableofcontents 
\part{} 

\begin{abstract}
Variational Auto-encoders (VAEs) are deep generative latent variable models that are widely used for a number of downstream tasks. While it has been demonstrated that VAE training can suffer from a number of pathologies, existing literature lacks characterizations of exactly \emph{when} these pathologies occur and \emph{how} they impact downstream task performance. In this paper, we concretely characterize conditions under which VAE training exhibits pathologies and connect these failure modes to undesirable effects on specific downstream tasks, such as learning compressed and disentangled representations, adversarial robustness, and semi-supervised learning.
\end{abstract}


\section{Introduction}
Variational Auto-encoders (VAEs) are deep generative latent variable models that transform simple distributions over a latent space to model complex data distributions~\citep{Kingma2013}. They have been used for a wide range of downstream tasks, including, generating realistic-looking synthetic data (e.g ~\cite{Pu2016}), learning compressed representations (e.g., ~\cite{Alemi2017}), adversarial defense using de-noising \citep{Ghosh2018}, and, when expert knowledge is available, generating counter-factual data using weak or semi-supervision 
(e.g., \cite{Kingma2014,Siddharth2017,Klys2018}).
Variational auto-encoders are widely used by practitioners due to the ease of their implementation and simplicity of their training. In particular, the common choice of mean-field Gaussian (MFG) approximate posteriors for VAEs (MFG-VAE) results in an inference procedure that is straightforward to implement and stable in training.

Unfortunately, a growing body of work has demonstrated that MFG-VAEs suffer from a variety of pathologies, including learning un-informative latent codes (e.g.\cite{Oord2017,Kim2018}) and unrealistic data distributions (e.g., \cite{Tomczak2017}). 
When the data consists of images or text, rather than evaluating the model based on metrics alone, we often rely on ``gut checks" to make sure that the quality of the latent representations the model learns and the synthetic data (as well as counterfactual data) generated by the model is high (e.g., by reading generated text / inspecting generated images visually \citep{TCVAE, Klys2018}). However, as VAEs are increasingly being used in applications where the data is numeric, e.g., in medical or financial domains \citep{pfohl2019counterfactual}, these intuitive qualitative checks no longer apply. For example, in many medical applications, the original data features themselves (e.g., biometric reading) are difficult to analyze by human experts in raw form. In these cases, in which the application touches human lives and potential model errors are particularly consequential, we need to have a clear understanding of the failure modes of our models as well as the potential negative consequences on downstream tasks. 

Recent work \citep{Yacoby2020,Stuhmer2019} attributes a number of the pathologies of MFG-VAEs to properties of the training objective: 
the objective may compromise learning a good generative model in order to learn a good inference model -- or in other words, the inference model ``over-regularizes'' the generative model.
While this pathology has been noted in literature \citep{burda_importance_2016,zhao_towards_2017,cremer_inference_2018}, 
no prior work characterizes \emph{conditions} under which the MFG-VAE objective compromises learning a good generative model in order to learn a good inference model; moreover, no prior work relates MFG-VAE pathologies with the performance on \emph{downstream tasks}. 
Rather, existing literature focuses on mitigating the regularizing effect of the inference model on the VAE generative model 
by using richer variational families (e.g., \cite{Kingma2016,nowozin_debiasing_2018,luo_sumo_2020}).
While promising, these methods potentially introduce significant additional computational costs to training~\citep{agrawal2020advances},
as well as new training issues (e.g., noisy gradients ~\cite{Roeder2017,tucker_doubly_2018,rainforth_tighter_2019}).
As such, it is important to understand precisely when MFG-VAEs exhibit pathologies and when alternative training methods are worth the computational trade-off. 
In this paper, we characterize the conditions under which MFG-VAEs perform poorly and link these failures to effects on a range of downstream tasks. 
While we might expect that methods designed to mitigate VAE training pathologies (e.g., methods with richer variational families~\citep{Kingma2016}), will also alleviate the negative downstream effects, we find that this is not always so. Our observations point to reasons for further studying the performance of VAE alternatives in these applications. 
Our contributions are:

I. \emph{When} VAE pathologies occur: (1) We characterize concrete conditions under which learning the inference model will compromise learning the generative model for MFG-VAEs. More problematically, we show that these bad solutions are \emph{globally optimal} for the training objective, the ELBO. (2) We demonstrate that using the ELBO to select the output noise variance and the latent dimension results in biased estimates. (3) We propose synthetic data-sets that trigger these two pathologies and can be used to test future proposed inference methods.

II. \emph{Effects} on tasks: (4) Using novel synthetic data-sets, we demonstrate ways in which these pathologies affect key downstream tasks, including learning compressed, disentangled representations, adversarial robustness, and semi-supervised learning. In semi-supervised learning, we are the first to document the instance of ``functional collapse'', in which the data conditionals problematically collapse to the same distribution. (5) We show that while the use of richer variational families alleviate VAE pathologies on unsupervised learning tasks, they introduce new ones in the semi-supervised tasks.
Lastly, (6) we provide guidelines to avoid these failure modes in practice. 

These contributions help identify when MFG-VAEs suffice, and when advanced methods are needed.

\section{Related Work}

\paragraph{Characterizing pathologies of MFG-VAEs occurring at local optima of the training objective.}
Existing works that characterize MFG-VAEs pathologies largely focus on relating \emph{local optima} 
of the training objective to a single pathology: 
the un-informativeness of the learned latent codes (posterior collapse) ~\citep{he_lagging_2019,Lucas2019,Dai2019}. 
In contrast, there has been little work to characterize pathologies 
at the \emph{global optima} of the MFG-VAE's training objective. 
\cite{Yacoby2020} show that, when the decoder's capacity is restricted, posterior collapse 
and the mismatch between aggregated posterior and prior can occur at the global optima of the training objective.
In contrast, we focus on \emph{global optima} of the MFG-VAE objective in a \emph{fully general} setting: 
with fully flexible generative and inference models, as well as with and without learned observation noise.
In this work, we therefore do not discuss posterior collapse, since as a global optima,
posterior collapse only occurs under restricted conditions, 
in which the true posterior equals the prior and is thus perfectly modeled by an MFG~\citep{zhao_towards_2017,he_lagging_2019,Dai2019}.
For this condition to occur as a global optima, the likelihood must completely ignore the latent code,
and must therefore use the observation noise (assumed to be Gaussian) to explain the data distribution. 
Thus, posterior collapse cannot occur as a global optima on non-Gaussian data (see Appendix \ref{sec:posterior_collapse}). 

\paragraph{Characterizing pathologies of MFG-VAEs occurring at global optima of the variational training objective.}
Our work follows a long line of research that examines whether the global optima of the variational
training objective retains properties necessary for good downstream performance.
For example, previous work shows that, with an inflexible variational family,
variational inference is inconsistent for the drift coefficient in linear state-space models~\citep{Wang2004},
that it fails to propagate uncertainty through time in time-series models~\citep{turner_sahani_2011},
that it underestimates uncertainty in data-scarce regions of the input space in Bayesian Neural Networks~\citep{Foong2020},
and that it inappropriately prunes additional degrees of freedom in mixture models~\citep{MacKay2001}. 

Whereas previous work has focused on other probabilistic models, in this work, we focus on VAEs;
previous works on VAEs (e.g., \cite{Yacoby2020,Stuhmer2019}) have connected pathologies, such as posterior collapse, 
to the over-regularizing effect of the variational family on the generative model. 
However while there are many works that note / mitigate the over-regularization issue (e.g.~\cite{burda_importance_2016,zhao_towards_2017,cremer_inference_2018,shu_amortized_2018}), 
none have given a full characterization of \emph{when} the learned generative model is over-regularized, 
and few related the quality of the learned model to its performance on \emph{downstream tasks}.
In particular, many of these works have shown that their proposed methods have higher test log-likelihood 
relative to an MFG-VAEs, but as we show in this paper, 
high test log-likelihood is not the only property needed for good performance on downstream tasks. 
Lastly, these works propose fixes that require a potentially significant computational overhead.
For instance, works that use complex variational families, such as normalizing flows~\citep{Kingma2016}, 
require a significant number of parameters to scale~\citep{kingma_glow_2018}.
In the case of the Importance Weighted Autoencoder (IWAE) objective~\citep{burda_importance_2016}, 
which can be interpreted as having a more complex variational family~\citep{cremer_reinterpreting_2017,domke_importance_2018},
the complexity of the posterior scales with the number of importance samples used. 
Lastly, works that de-bias existing bounds~\citep{nowozin_debiasing_2018,luo_sumo_2020} require several evaluations of the objective.

Given that MFG-VAEs remain popular today due to the ease of their implementation, 
speed of training, and their connections to other dimensionality 
reduction approaches like probabilistic PCA~\citep{Stuhmer2019,rolinek_variational_2019,dai_connections_nodate,Lucas2019}, 
it is important to characterize the training pathologies of MFG-VAE, 
as well as the concrete connections between these pathologies and downstream tasks.
More importantly, this characterization will help clarify for which tasks and data-sets an MFG-VAE suffices and for which the computational tradeoffs are worth it.


\section{Background} \label{sec:background}

\paragraph{Unsupervised VAEs \citep{Kingma2013}.} A VAE assumes the following generative process: 
\begin{align}
p(z) &= \mathcal{N}(0, I),  \quad p_\theta(x | z) = \mathcal{N}(f_\theta(z), \sigma^2_\epsilon \cdot I)
\end{align}
where  $x$ in $\mathbb{R}^D$, $z\in \mathbb{R}^K$ is a latent variable and $f_\theta$ is a neural network parametrized by $\theta$. We learn the likelihood parameters $\theta$ while jointly approximating the posterior $p_\theta(z | x)$ with $q_\phi(z  | x)$:
\begin{align}
\begin{split}
\max_\theta \mathbb{E}_{p(x)} \left\lbrack \log p_\theta(x) \right\rbrack &\geq \max_{\theta, \phi} \underbrace{\mathbb{E}_{p(x)} \left\lbrack \mathbb{E}_{q_\phi(z | x)} \left\lbrack \log \frac{p_\theta(x | z) p(z)}{q_\phi(z | x)} \right\rbrack \right\rbrack}_{\text{ELBO}(\theta, \phi)}
\end{split}
\label{eq:elbo}
\end{align}
where $p(x)$ is the true data distribution, $p_\theta(x)$ is the learned data distribution,
and $q_\phi(z | x)$ is an MFG with mean and variance $\mu_\phi(x), \sigma^2_\phi(x)$, parameterized by neural network with parameters $\phi$. 
The VAE ELBO can alternately be written as a sum of two objectives --  
the ``MLE objective'' (MLEO), which maximizes the $p_\theta(x)$, 
and the ``posterior matching objective'' (PMO), 
which encourages variational posteriors to match posteriors of the generative model. 
That is, we can write $\text{argmin}_{\theta, \phi} -\text{ELBO}(\theta, \phi)$ as follows~\citep{zhao_towards_2017}:
\begin{align} 
\begin{split}
\text{argmin}_{\theta, \phi} ( \underbrace{D_{\text{KL}} \lbrack p(x) || p_\theta(x) \rbrack}_{\text{MLEO}} + \underbrace{\mathbb{E}_{p(x)} \left\lbrack D_{\text{KL}} \lbrack q_\phi(z | x) || p_\theta(z | x) \rbrack \right\rbrack}_{\text{PMO}})
\end{split}
\label{eq:vae-obj}
\end{align}
This decomposition allows for a more intuitive interpretation of VAE training and illustrates the tension between approximating the true posteriors and approximating $p(x)$.

\paragraph{Semi-Supervised VAEs.}
We extend the VAE model and inference to incorporate partial labels, allowing for some supervision of the latent space dimensions. For this, we use the semi-supervised model introduced by \cite{Kingma2014} as the ``M2 model'',
which assumes the generative process,
\begin{align}
\begin{split}
z \sim \mathcal{N}(0, I), \quad y \sim p(y), \quad \epsilon \sim \mathcal{N}(0, \sigma^2_\epsilon \cdot I), 
\quad x | z, y = f_\theta(z, y) + \epsilon,
\end{split}
\end{align}
where $y$ is observed only a portion of the time. 
The inference objective for this model is typically written as a sum of three objectives:
a lower bound for the likelihood of $M$ labeled observations, a lower bound for the likelihood for $N$ unlabeled observations, and a term encouraging the discriminative powers of the variational posterior:
\begin{align}
\begin{split}
\mathcal{J}(\theta, \phi) = \sum\limits_{n=1}^N \mathcal{U}(x_n; \theta, \phi)
+ \gamma \cdot \sum\limits_{m=1}^M \mathcal{L}(x_m, y_m; \theta, \phi) 
+ \alpha \cdot \sum\limits_{m=1}^M \log q_\phi(y_m | x_m)
\end{split}
\label{eq:ss-m2-objective-pre}
\end{align}
where the $\mathcal{U}$ and $\mathcal{L}$ lower bound $p_\theta(x)$ and $p_\theta(x, y)$, respectively (see Appendix \ref{sec:semisup_det}); the last term in the sum is included to explicitly increase discriminative power of the posteriors $q_\phi(y_m|x_m)$ (\cite{Kingma2014} and \cite{Siddharth2017}); $\alpha$, $\gamma$ controls the relative weights of the last two terms. Note that $\mathcal{J}(\theta, \phi)$ is only a lower bound of the observed data log-likelihood only when $\gamma=1, \alpha=0$, but in practice, $\gamma, \alpha$ are tuned as hyper-parameters. 
Following \cite{Kingma2014}, we assume an MFG variational family for each of the unlabeled and labeled objectives.

\paragraph{Additional Notation.} Let $\theta_\text{GT}$ denote the ground-truth generative model parameters
(for which the MLEO is $0$),
and let $\phi_\text{GT} = \mathrm{argmin}_{\phi} -\text{ELBO}(\theta_\text{GT}, \phi)$ be the best corresponding
variational posterior. 
Let $\theta^*, \phi^* = \mathrm{argmin}_{\theta, \phi} -\text{ELBO}(\theta, \phi)$ be the global optima of the ELBO.
Lastly, let $\mathcal{F}$ represent all functions $f_\theta$ realizable under the generative model's function class.

\section{Methodology} \label{sec:methodology}

In this section, we provide an overview of the experimental philosophy and setup used in the remainder of the paper.

\paragraph{When VAEs fail on downstream tasks, at a high-level.}
Previous work has shown that, at the global optima of the VAE training objective,
failures to perform well on downstream tasks generally occur due to 
the non-identifiability of the generative model, and/or due to the inductive bias of the variational family~\citep{Yacoby2020,Stuhmer2019}. 
Specifically, \cite{Yacoby2020} argue that when the ELBO is tight,
all generative models that explain the observed data well are preferred equally by the ELBO (non-identifiability).
As such, it is highly unlikely that the learned model retains any properties of the ground-truth model 
(like the meaning behind each latent dimension).
When the ELBO is not tight, \cite{Yacoby2020} argue that the ELBO will select one of the many generative models
that explain the observed data well -- specifically, it will select the one for which the posterior is most easily captured by the variational family (the inductive bias of the variational family). 
In this case, it is again unlikely that the model selected by the ELBO will be the one appropriate for the downstream task.

In this work, we extend this analysis, but unlike existing work, 
we focus our analysis on the case in which $f_\theta$ is arbitrarily flexible,
and on shedding intuition on the types of data-sets that trigger these failure modes 
on a large variety of common downstream tasks. 
Unless noted otherwise, in the examples considered in this paper, 
we fix a set of realizable likelihood functions $\mathcal{F}$, implied by our choice of the generative model network architecture. We choose $\mathcal{F}$ to be significantly more expressive than necessary to contain any smooth function, 
including the ground-truth generating function.
Thus, we emphasize that all results demonstrating failure modes of VAE in this paper
are not due to a lack of generative model capacity. 

\paragraph{Roadmap.} 
We begin by showing when and how the inductive bias of the variational family
biases the learned VAE towards models that misestimate the target data density,
as well as how this affects downstream tasks that require accurate density estimation (Section \ref{sec:misestimate-px}).
While the goal of VAE inference is, by design, to maximize the likelihood of the observed data,
many downstream tasks do not use the estimated data density at all,
and instead use the latent codes.
We next show when and how non-identifiability and the inductive bias of the variational family
compromise performance on tasks that require the learned latent space dimensions to 
retain the same meaning as those of the ground-truth model (Section \ref{sec:bad-latent-space}). 
While one may be tempted to incorporate additional knowledge (or inductive bias), e.g. in the form of partial labels,
to guide latent space to have its desired properties,
in Section \ref{sec:semi-supervision} we show how, when this additional information conflicts with the 
inductive bias of the variational family, performance on downstream tasks is similarly compromised.
Lastly, we unfold the consequence of using the ELBO for model selection on tasks requiring an accurate decomposition
between signal and noise (Section \ref{sec:misestimate-signal-vs-noise}).
Each of these sections shares a common approach to designing datasets, evaluation, and experimental setup. 
We describe these shared elements below before describing each failure mode. 
Following the sections about failure modes, we highlight implications for best practice (Section \ref{sec:discussion}).

\paragraph{Evaluation.}
Every downstream task requires the learned model to recover some property of the ground-truth model that generated the observed data. 
Some tasks, like learning disentangled representations, require the learned model's latent dimensions to align with those of the ground-truth model;
tasks like defenses against adversarial perturbations, require the learned model to exhibit the ground-truth model's decomposition between ``signal'' and ``noise'' (or between $f_\theta(z)$ and $\epsilon$);
and tasks like out-of-distribution (OOD) detection, simply require the learned model to accurately model ground-truth data distribution $p(x)$. 
As such, success in every task will be measured according to task-specific metrics (described along with each downstream task in the following sections of the paper). 

\paragraph{Data-sets.}
Ideally, we want to compare the learned models with the ground-truth models to measure how inference affects downstream tasks. Thus, in this work, in addition to using real-world data-sets, we rely heavily on synthetic data 
(for which we know the ground-truth model). 
Using knowledge of the ground-truth model, we can benchmark when traditional inference successfully is able to recover the necessary property of the ground-truth model, and when it cannot, how the downstream task's performance will be impacted.
Furthermore, although many works in literature observe various failures of VAEs, 
there does not exist a collection of benchmark data-sets in current literature known to \emph{causally} 
trigger failure modes of VAEs on common downstream tasks. 
That is, when VAE failures are observed in practice, it is often unknown whether it is the fault of the 
inference method (e.g. choice of gradient estimator), parameters of the optimizer, initialization, etc. 
In this work, we propose a collection of such synthetic benchmark data-sets 
for which we know that the global optima of the ELBO correspond to models with undesirable properties. 
For each data-set we propose, we additionally provide intuition explaining what properties of this data-set 
trigger the failure on the downstream task, as well as how to construct additional data-sets with similar properties. 
We consider our proposed data-sets (as well as the intuition for constructing these data-sets) as a contribution of this paper; 
these data-sets can help future researchers and practitioners to test their proposed inference methods with concrete downstream tasks in mind, to sanity-check their implementations, etc., and the intuition behind the construction of these data-sets can further help future researchers construct new benchmark data-sets that trigger VAE failure modes on new tasks. 
We summarize all proposed benchmark data-sets in Appendix \ref{sec:unsup-examples} (unsupervised) and Appendix \ref{sec:ss-examples} (semi-supervised), as well as explain why we expect them to trigger a failure mode. 

In addition to proposing synthetic benchmarks that trigger VAE failures on common downstream tasks,
we also describe how these benchmarks typify real classes of data-sets.
For some tasks (e.g. semi-supervised learning), we can even develop metrics that will a priori 
tell us whether a failure mode is likely to occur.
However, for many of the other tasks, this is not possible,
since the characterization of the failure mode is a function of the ground-truth model and not the data.
While it would be ideal to have such metrics for all tasks, 
we emphasize that knowledge of the general characteristics of a data-set that may trigger a failure mode is still useful.
Much like incorporating our beliefs about the data-set into our probabilistic model,
such knowledge (of how our approximate inference assumptions affect our downstream tasks)
needs to be accounted for with equal care. 
For example, as we show in Section \ref{sec:misestimate-px}, MFG-VAEs will misestimate the true $p(x)$
given noisy samples off of a curvy manifold. 
Since estimating the curvature of the manifold may be as expensive as training an MFG-VAE,
it may not be possible to know a priori whether an MFG-VAE will fail to estimate $p(x)$ well;
however, if you believe your data lies on such a curvy manifold,
it is best to use a richer variational family.

\paragraph{Experimental setup.}
In this paper, we are concerned with failures that occur at the global optima of the VAE training objective. However, proving properties about the global optima of the ELBO evaluated on arbitrary data-sets is difficult. We thus take an empirical approach.
To ensure that the solutions recovered via traditional inference are as close as possible to the global optima, we take the following measures: 
\begin{itemize}
\item We give the optimizer access to the ground-truth parameters $\theta_\text{GT}, \phi_\text{GT}$ (described in Section \ref{sec:background}). That is, we use 10 random restarts -- 5 initialized at $\theta_\text{GT}, \phi_\text{GT}$ and 5 initialized randomly -- and select the model with the highest ELBO on a validation set. 
When the highest ELBO does not correspond to the ground-truth model, we know that the model corresponding to the global optima of the ELBO may not retain the desired properties needed for the downstream task. 
\item Unless otherwise specified, we fix the model's hyper-parameters -- the observation noise variation $\sigma^2_\epsilon$ and the latent space's dimensionality $K$ -- at the ground-truth. 
\item Unless noted otherwise, we ensure that the generative model's architecture is significantly more expressive than necessary to contain any smooth function, including the ground-truth generating function. In this way, we ensure that the failures we observe are not caused by a lack of generative model capacity. 
\item We compare traditional VAE inference against two baselines. First, we compare against Lagging Inference Networks (LIN) training~\citep{he_lagging_2019}, an alternative training method for VAEs designed to better escape problematic local optima. By showing that even when initialized at the ground-truth, LIN does not improve on traditional training, we confirm that our initialization scheme gets us close to the global optima. 
Second, we compare against IWAE as a representative example of a VAE with a more complex variational family.
By showing that IWAE does not suffer from the same failures as an MFG-VAE, we can attribute the failure to our choice of variational family. We specifically choose IWAE as a baseline since we can easily control the complexity of the implied variational family using the number of importance samples $S$, 
and in doing so we encompass other types of variational families. 
\item Lastly, we ensure that all baselines do not suffer from training issues by constructing all synthetic data to be low dimensional (most have a 1D latent space and a 2D data space). 
\end{itemize}
All details about the experimental setup are in Appendix \ref{sec:exp-details}.

\section{When VAEs explain the observed data poorly} \label{sec:misestimate-px}

Since the goal of VAE inference is to maximize the likelihood of the observed data,
we first examine tasks that require accurate density estimation:
generating realistic-looking synthetic data (which requires sampling from $p(x)$),
and OOD detection (which requires finding samples for which $p(x)$ is relatively small).
In Section \ref{sec:misestimate-px-conditions} we characterize two conditions of the ground-truth generative model
under which the global optima of the ELBO prefers models that significantly compromise learning $p(x)$.
In Section \ref{sec:intuition} we then provide intuition for these two conditions.
Next (in Section \ref{sec:ds-construction}), 
since in practice one does not know the ground-truth generative model (and thus one cannot verify these two conditions),
we demonstrate how to construct synthetic data-sets that trigger this failure.
Using the insights used to construct these synthetic data-sets,
we translate the original two conditions into properties of real data that we would expect to trigger this failure.
Lastly, we unfold the consequences of misestimating $p(x)$ on the two aforementioned tasks.

\subsection{Conditions under which VAE inference will meaningfully compromise learning $p(x)$} \label{sec:misestimate-px-conditions}

Intuitively, the global optima of the ELBO correspond to incorrect generative models under two conditions (that must both be satisfied):
\begin{itemize}
\item[] \textbf{Condition 1:} The true posterior for the ground-truth $f_{\theta_\text{GT}}$ is difficult to approximate by an MFG for a large portion of $x$'s.
\item[] \textbf{Condition 2:} There does not exist a likelihood function $f_\theta \in \mathcal{F}$ with simpler posteriors that approximates $p(x)$ well.
\end{itemize}
For completeness, we formalize these conditions in a theorem in Appendix \ref{sec:path_1}
and prove that under these conditions, the global optima of the VAE corresponds to a model that misestimates $p(x)$.
However, we emphasize that just because $p(x)$ is learned incorrectly, it does not mean that the quality of the generative model is \emph{meaningfully} compromised; in fact, there are many conditions that can lead to the ELBO learning a model  that does not recover $p(x)$ exactly, but for which the compromise in the quality of the learned $p(x)$ is imperceptible in downstream tasks. 
Here we conjecture that conditions (1) and (2) are necessary for \emph{significant} compromises in the quality of the learned $p(x)$. While we do not provide a proof of this conjecture, we verify it both quantitatively and qualitatively in this section. We furthermore provide examples where only one of the conditions is met and show, as a result, that the learned $p(x)$ differs in non-significant ways from the true data distribution, again by qualitative evaluations. 
Before empirically verifying that under these conditions, the quality of $p(x)$ is meaningfully compromised, however,
we provide the intuition behind these conditions.

\subsection{The Intuition Behind the Two Conditions} \label{sec:intuition}

\begin{figure*}[t!]
    \centering

    \includegraphics[width=0.7\textwidth]{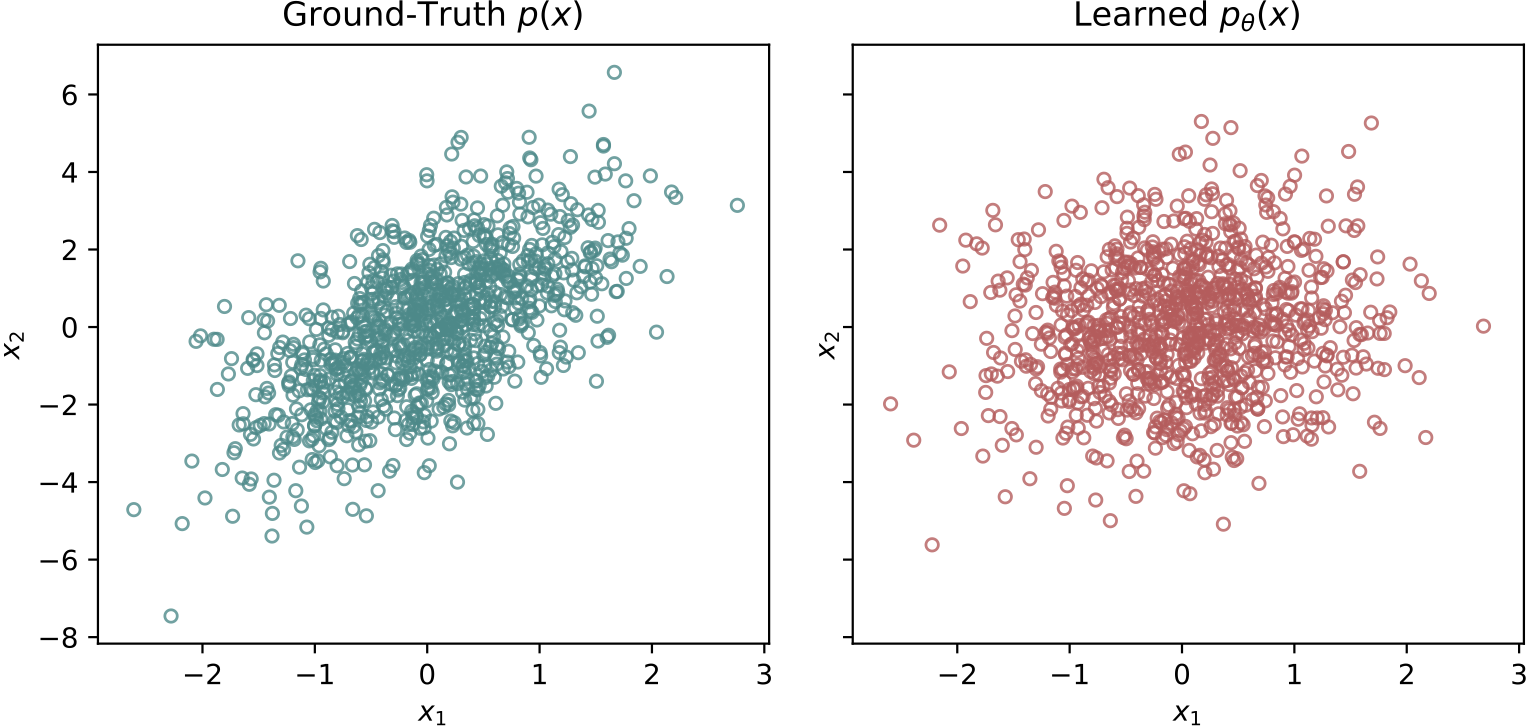}

    \caption{\textbf{The model at the global optima of the ELBO fails to capture the most salient feature of the true data distribution -- the correlation between the covariates.} Left: samples from the ground-truth model in Equation \ref{eq:aabi1} (left). Right: samples from the model at the global optima of the ELBO (right).}
    \label{fig:aabi1-px}
\end{figure*}

While the over-regularizing effect of the posterior on the generative model has served as a motivation to develop
more sophisticated inference methods (e.g.~\cite{burda_importance_2016,zhao_towards_2017,cremer_inference_2018}),
there has been few synthetic examples and data-sets constructed to trigger this pathology. 
In this section, we will use simple, low-dimensional examples to build intuition why the two conditions are necessary,
as well as for the remainder of the section,
which describes when (or on what types of data-sets) VAEs fail to estimate $p(x)$ accurately. 
We begin by reviewing an example constructed by \cite{Yacoby2020}
to exploit the form of the ELBO in Equation \ref{eq:vae-obj} so that at the global optima of the ELBO,
the MLEO is high and the PMO is low (i.e. $p(x)$ was approximated poorly in order to learn simpler posteriors).

\paragraph{Trade-off between the MLEO and PMO when $\bm{f_\theta}$ is inflexible.}
Consider the following model~\citep{Yacoby2020}:
\begin{equation} \label{eq:aabi1}
  x | z = \underbrace{\text{Cholesky} \left( A A^\intercal + B \right) \cdot z}_{f_\theta(z)} + \epsilon, \quad z \sim \mathcal{N} \left( 0, I \right),\quad \epsilon \sim \mathcal{N} \left( 0, I \cdot \sigma^2_\epsilon - B \right) 
\end{equation}
where $B, \sigma^2_\epsilon$ are fixed\footnote{In this example, $\sigma^2_\epsilon = 0.01$, $B = \left\lbrack \begin{smallmatrix} 0.006 & 0 \\ 0 & 0.006 \end{smallmatrix} \right\rbrack$ are fixed.},
and $\theta = A$ is learned\footnote{The ground-truth model parameter $\theta_\text{GT} = A_\text{GT} = \left\lbrack \begin{smallmatrix} 0.75 & 0.25 \\ 1.5 & -1.0 \end{smallmatrix} \right\rbrack$}.
In this example, the global optima of the ELBO corresponds to a model that explains the true data distribution poorly:
at the ground-truth model, $-\text{ELBO}(\theta_\text{GT}, \phi_\text{GT}) = 0.532$, 
while at the global optima, $-\text{ELBO}(\theta^*, \phi^*) = 0.196$,
corresponding to a model that explains $p(x)$ poorly.
Namely, as shown in Figure \ref{fig:aabi1-px}, 
the learned $p_\theta(x)$ misses the only salient feature of the true $p(x)$ -- the correlation between the two covariates.
Why did the global optima of the ELBO prefer a model that explains the data so poorly?
Because, as we will show, for this model both conditions from Section \ref{sec:misestimate-px-conditions} hold. 

The model in Equation \ref{eq:aabi1} was specifically constructed such that it cannot simultaneously have a low MLEO and PMO.
At the ground-truth, we have,
\begin{align*}
-\text{ELBO}(\theta_\text{GT}, \phi_\text{GT}) + \text{const} &= \underbrace{\text{MLEO}}_{0} + \underbrace{\text{PMO}}_{\text{large}} = 0.532.
\end{align*}
This is because, for this model, $f_\theta(z)$ is linear in $z$;
as such, for every $x$, the true posterior $p_{\theta_\text{GT}}(z | x)$ is a full-covariance Gaussian 
(poorly approximated by an MFG), causing PMO at the ground-truth to be high. 
On the other hand, at the global optima of the ELBO, we have,
\begin{align*}
-\text{ELBO}(\theta^*, \phi^*) + \text{const} &= \underbrace{\text{MLEO}}_{\text{small}} + \underbrace{\text{PMO}}_{\text{small}} = 0.196.
\end{align*}
Why did the ELBO compromise the MLEO to lower the PMO?
Recall that in general for MFG-VAEs with a linear $f_\theta$, 
the ELBO recovers a likelihood $f_{\theta^*}(z)$ that is a rotation of the ground-truth likelihood $f_{\theta_\text{GT}}(z)$
~\citep{tipping1999probabilistic,Lucas2019};
that is, $f_{\theta^*}(z) = f_{\theta_\text{GT}}(R \cdot z)$, where $R$ is selected 
so that the best-fitting approximate posterior $q_{\phi^*}(z | x)$ is an MFG.
For linear MFG-VAEs, at the global optima of the ELBO, the MLEO is $0$, since the prior is rotationally invariant, 
and the PMO is also $0$, since there always exists a rotation for which $q_{\phi^*}(z | x)$ is an MFG.
However, in this example, since $f_\theta(z)$ is restricted to be a triangular matrix by the Cholesky decomposition,
not all rotations of the ground-truth function $f_{\theta_\text{GT}}(z)$ can be represented by the likelihood:
$f_{\theta_\text{GT}}(R \cdot z) \notin \mathcal{F}$.
As such, both conditions from Section \ref{sec:misestimate-px-conditions} are satisfied:
for every $x$, the posterior (full-covariance Gaussian) is poorly approximated by an MFG (satisfying condition 1),
and since $f_\theta$ is restricted to be a triangular matrix,
there does not exist an alternative $f_\theta \in \mathcal{F}$ that explains the observed data
and has easy-to-approximate posteriors (satisfying condition 2). 

This example assumes a linear and restricted $f_\theta$ to force a tradeoff between the MLEO and the PMO; 
however, when $f_\theta$ is arbitrarily complex (e.g., a neural network), as is most commonly found in VAE applications,
does the ELBO still need to compromise the MLEO in order to have a lower PMO?
And if so, under what conditions?
This is the subject of Section \ref{sec:ds-construction}.
Before diving in, however, we provide one more toy example to illustrate why it is so difficult to characterize 
when (or on what types of data) an MFG-VAE will fail to approximate the data density well: 
generative model non-identifiability.

\paragraph{A trade-off between the MLEO and PMO is not always necessary due to non-identifiability.}
While the model in Equation \ref{eq:aabi1} was constructed so that having a low MLEO and a low PMO 
is mutually exclusive, when $f_\theta$ can express any likelihood function,
are there data-sets for which the same happens? 
As we show here, this is a difficult question to answer due to non-identifiability in the generative model.
When $f_\theta$ is fully flexible, there may exist an alternative model that explains the observed data equally well
as the ground-truth model \emph{and} has simpler posteriors.
However, it is non-trivial to determine when there exists such an alternative model, 
and whether for this alternative model the ``simpler'' posteriors are still poorly approximated by an MFG,
causing the learned $p_\theta(x)$ to be significantly compromised. 

\begin{figure*}[t!]
    \centering

    \begin{subfigure}[t]{0.3075\textwidth}
    \includegraphics[width=1.0\textwidth]{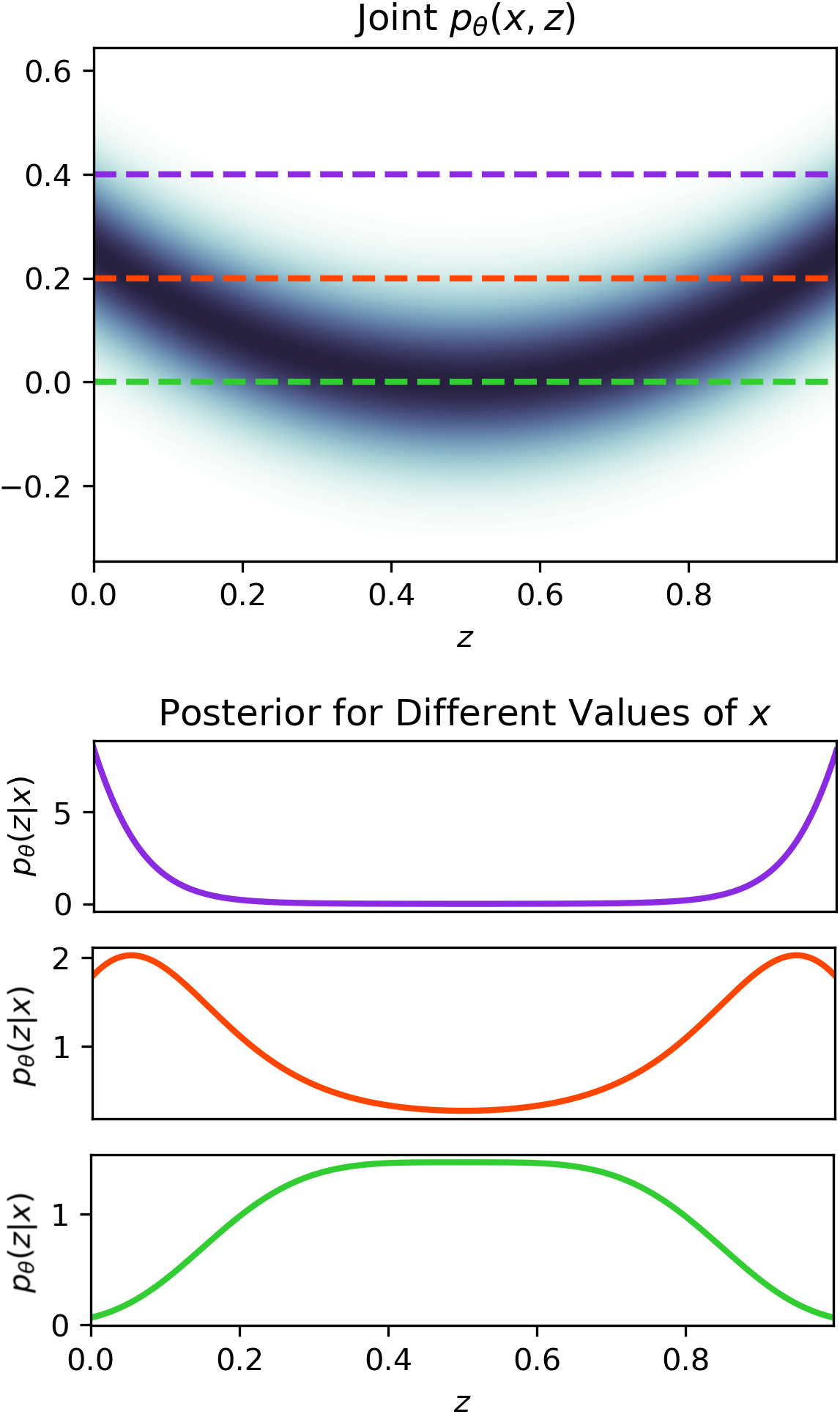}
    \caption{\scriptsize \textbf{Variant 1}}
    \label{fig:quad-true-fn}
    \end{subfigure}
    \begin{subfigure}[t]{0.2835\textwidth}
    \includegraphics[width=1.0\textwidth]{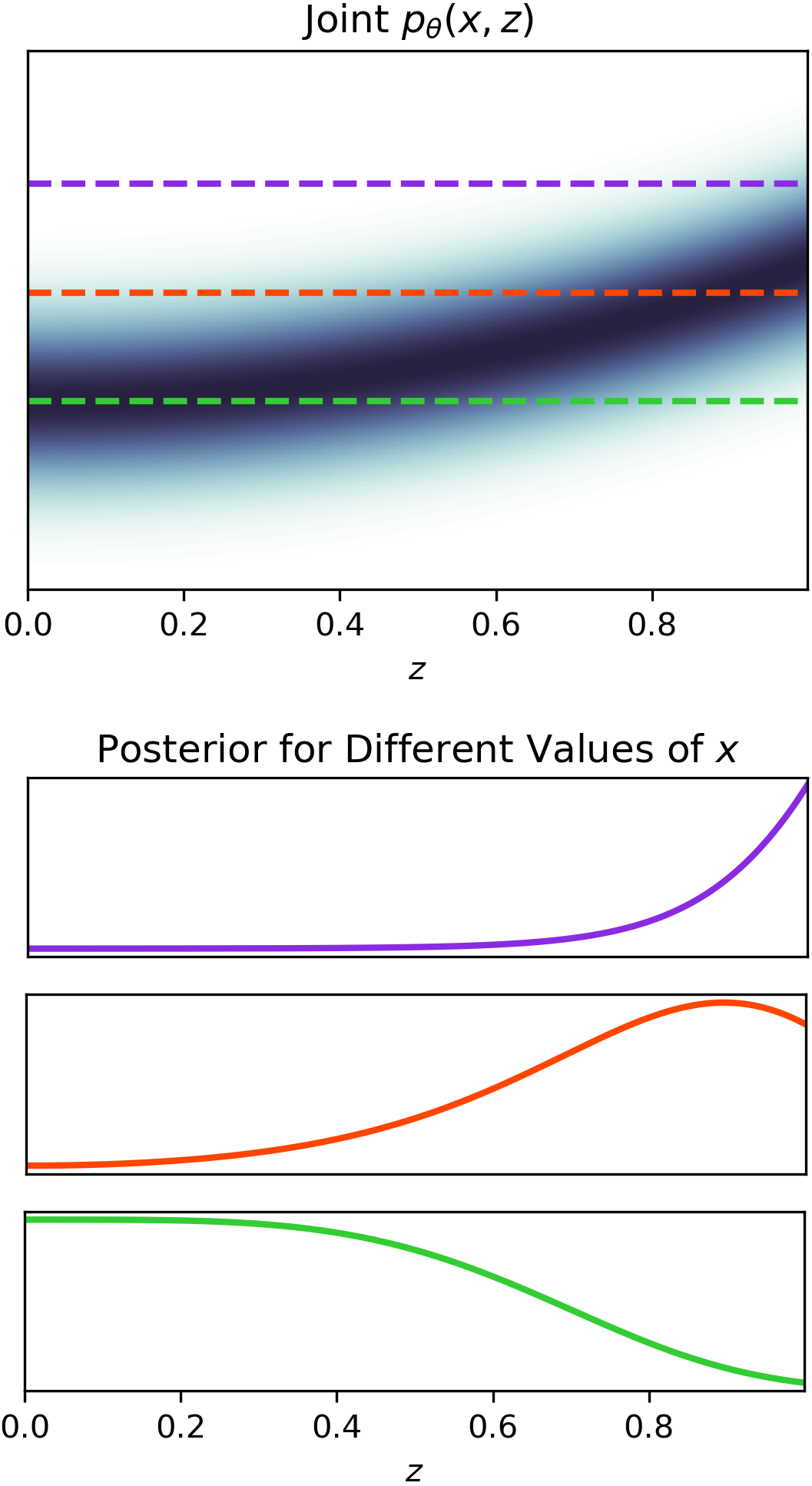}
    \caption{\scriptsize \textbf{Variant 2}}
    \label{fig:quad-variant1-fn}
    \end{subfigure}
    \begin{subfigure}[t]{0.2835\textwidth}
    \includegraphics[width=1.0\textwidth]{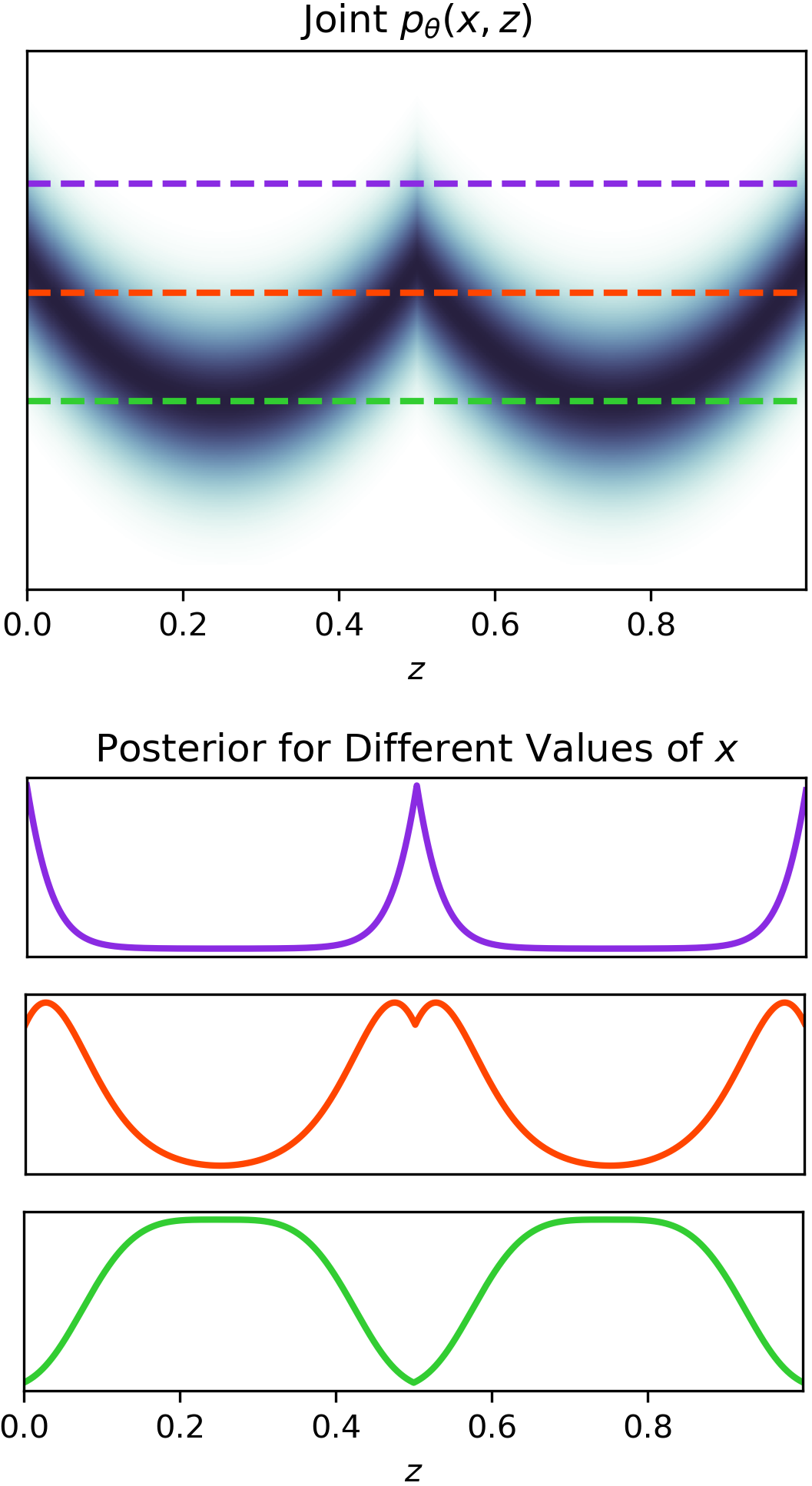}
    \caption{\scriptsize \textbf{Variant 3}}
    \label{fig:quad-variant2-fn}
    \end{subfigure}
    \begin{subfigure}[t]{0.0925\textwidth}
    \includegraphics[width=1.0\textwidth]{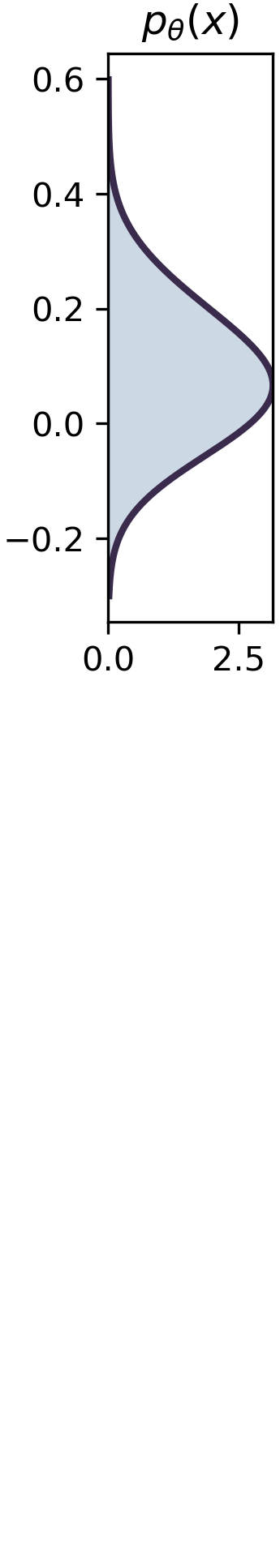}
    \end{subfigure}
    
    \begin{subfigure}[t]{0.3075\textwidth} \scriptsize
    \vspace{-10pt}
    \begin{align*}
    f_\theta(z) = (z - 0.5)^2
    \end{align*}
    \end{subfigure}
    \begin{subfigure}[t]{0.2835\textwidth} \scriptsize
    \vspace{-10pt}
    \begin{align*}
    f_\theta(z) =  0.25 \cdot z^2
    \end{align*}
    \end{subfigure}
    \begin{subfigure}[t]{0.2835\textwidth} \scriptsize
    \vspace{-10pt}
    \begin{align*}
    f_\theta(z) = \mathbb{I}&(z < 0.5) \cdot (2 z - 0.5)^2 \\
			 &+ \mathbb{I}(z \geq 0.5) \cdot (2 z - 1.5)^2
    \end{align*}    
    \end{subfigure}
    \begin{subfigure}[t]{0.0925\textwidth}
    \end{subfigure}
    
    \caption{\textbf{VAE non-identifiability: three models with the same $\bm{p_\theta(x)}$ and $\bm{p(z)}$ but different posteriors $\bm{p_\theta(z | x)}$.} Consider three models with the prior $z \sim \mathcal{U}(0, 1)$ and likelihood $x | z \sim \mathcal{N}(f_\theta(z), \sigma^2_\epsilon)$, but different $f_\theta(z)$. The top-row visualizes the joint distributions $p_\theta(x, z)$, in which each dotted cross-section represents a posterior for a different value of $x$, visualized on the bottom row. All models have the same marginal $p_\theta(x)$ (top-right), but have very different posteriors.}
    \label{fig:quad}
\end{figure*}

We illustrate this using the three models depicted in Figure \ref{fig:quad}, all assuming a model of the form,
\begin{align*}
x|z \sim \mathcal{N}(f_\theta(z), \sigma^2_\epsilon), \quad z \sim \mathcal{U}(0, 1),
\end{align*}
but each with a different $f_\theta$.
For all three models, summing the joint density along a horizontal cross-section
(or ``projecting onto the vertical axis'') yields the \emph{same data marginal} $p_\theta(x)$, 
depicted in the top right corner.
As such, one may intuitively assume that the ELBO would select any one of them arbitrarily 
(depending on the random initialization of the optimizer); 
however, as we show next, this is not the case.
Plotting the joint density along horizontal cross-sections (depicted in the dotted lines)
reveals that each variant has \emph{different posteriors} $p_\theta(z | x)$ (visualized in the bottom row). 

Looking at the shapes of the posteriors for Variant 1, many are multi-modal and thus poorly approximated by an MFG.
Variant 1 will thereby have an MLEO of $0$ but a high PMO. 
Intuitively, one might therefore think that, like in the example in Equation \ref{eq:aabi1},
in order to lower the PMO, the MLEO must increase, causing the VAE objective to compromise learning the data distribution.
But this is actually not the case:
by looking at the geometry of $f_\theta$, notice that both low and high values of $z$ 
yield the same likelihood (e.g., $p_\theta(x | z = 0) = p_\theta(x | z = 1)$).
This causes multimodality in the posterior, since cross-sections (e.g., purple and orange lines)
intersect the joint density in multiple separate regions of high-density.
Since $p_\theta(x)$ is computed via projection onto the vertical axis,
it is not affected by the location of the high-mass regions in the joint density,
so long as summing the joint along all horizontal cross-sections yields the same values of $p_\theta(x)$. 
We can therefore construct a joint density that ``consolidates'' 
the areas of high mass for large and small $z$'s into one side of the plot,
so that any cross-section / posterior will intersect the high-mass regions of the joint density only once. 
Variant 2 is exactly such a function
(e.g., instead of $p_\theta(x | z = 0)$ and $p_\theta(x | z = 1)$ each contributing equal mass to $p_\theta(x=0.4)$, 
in Variant 2 $p_\theta(x=0.4 | z = 1)$ contributes all of the mass). 
In comparison to Variant 1, Variant 2 has the same data marginal $p_\theta(x)$ (so its MLEO will be $0$),
and has posteriors that are significantly more MFG-looking (and thus have a low PMO).
As such, given data generated by Variant 1, the VAE objective will always prefer Variant 2.

Even though we have established that the ELBO will prefer Variant 2 over the other variants,
the posteriors of Variant 2 are still not exactly MFG; 
that is for Variant 2, the MLEO is $0$ and the PMO is non-zero (but lower than it is for the other variants). 
For Variant 2, it is not immediately obvious whether the PMO is high enough 
to cause the ELBO to tradeoff the quality of the learned $p(x)$ for even simpler posteriors.


To summarize, as shown in this example, due to generative model non-identifiability, 
for some data-sets it may possible to learn an alternative non ground-truth model that 
recovers $p(x)$ almost perfectly and has simpler posteriors.
This observation necessitates condition (2), and generally renders the theoretical analysis of VAE failure modes difficult. 
So on what types of data-sets, even with a fully-flexible generative model,
the VAE objective will be forced to compromise learning $p(x)$ in order to learn a model with simpler posteriors? 
And how do we construct such benchmark data-sets to test VAE inference methods?
We discuss this next.

\subsection{Constructing synthetic data for which MFG-VAEs misestimate $p(x)$ significantly} \label{sec:ds-construction}

\begin{figure*}[t!]
    \centering

    \begin{subfigure}[t]{0.45\textwidth}
    \includegraphics[width=1.0\textwidth]{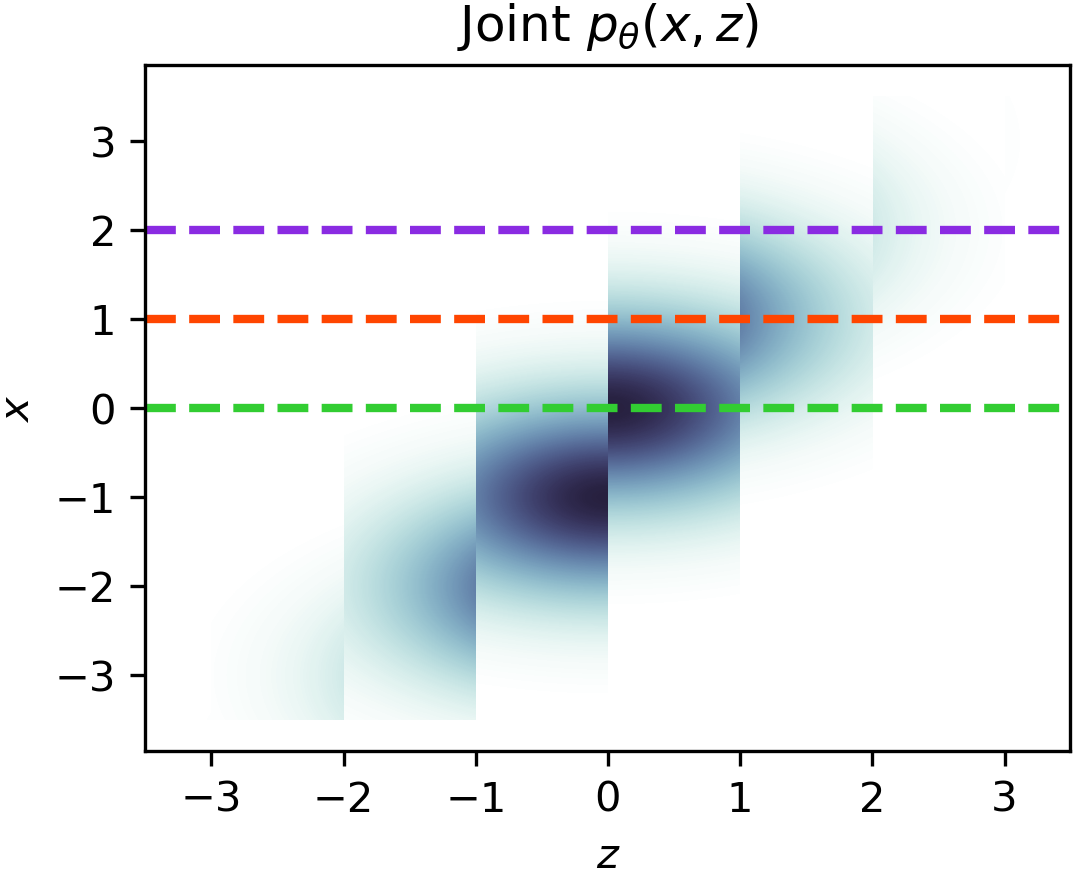}
    \end{subfigure}
    \hspace{5mm}
    \begin{subfigure}[t]{0.41\textwidth}
    \includegraphics[width=1.0\textwidth]{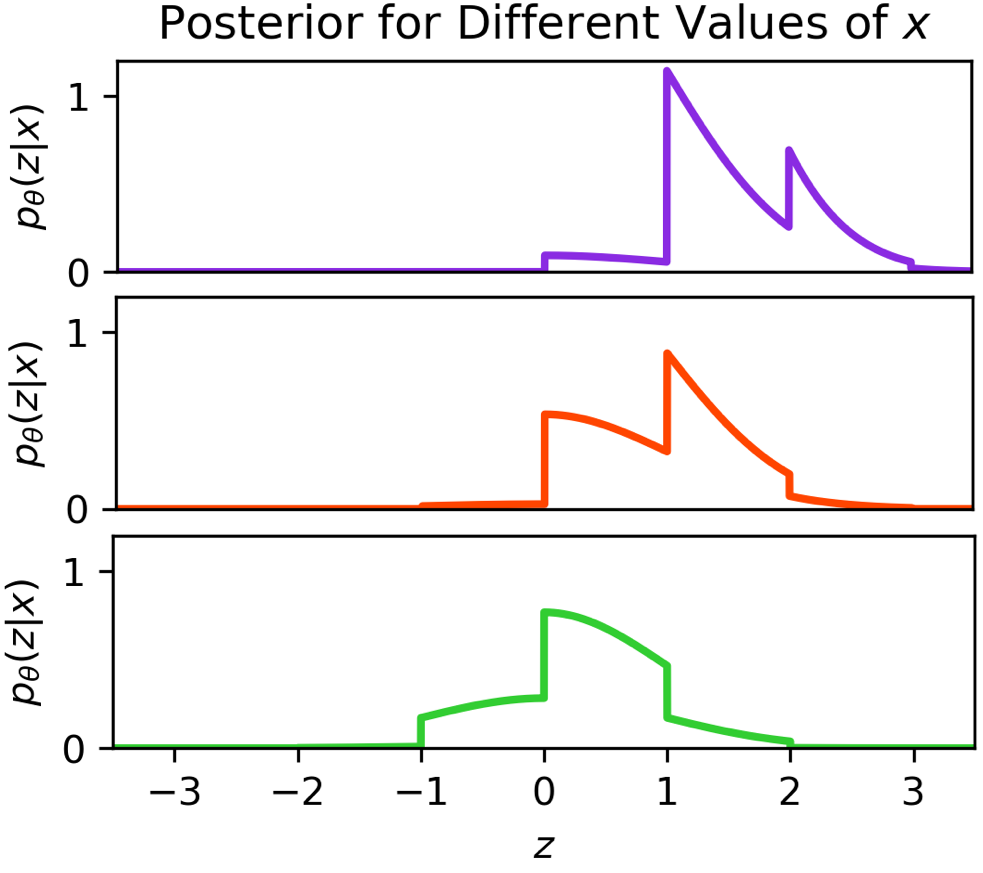}
    \end{subfigure}
    
    \caption{\textbf{Rapid changes in $\bm{f_\theta}$ lead to multi-modal posterior.} Consider the model in Equation \ref{eq:stepfn}. Left: joint distribution. Right: posterior for different values of $x$. Posteriors are multimodal since horizontal cross-sections of the joint density alternate passing through low and high mass regions.}
    \label{fig:stepfn}
\end{figure*}

\begin{figure*}[t!]
    \centering

    \begin{subfigure}[t]{0.02\textwidth}
        \centering
        \small
        \rotatebox[origin=l]{90}{\hspace{31pt}\textbf{Clusters Example}}
    \end{subfigure}
    \begin{subfigure}[t]{0.46\textwidth}
    \includegraphics[width=1.0\textwidth]{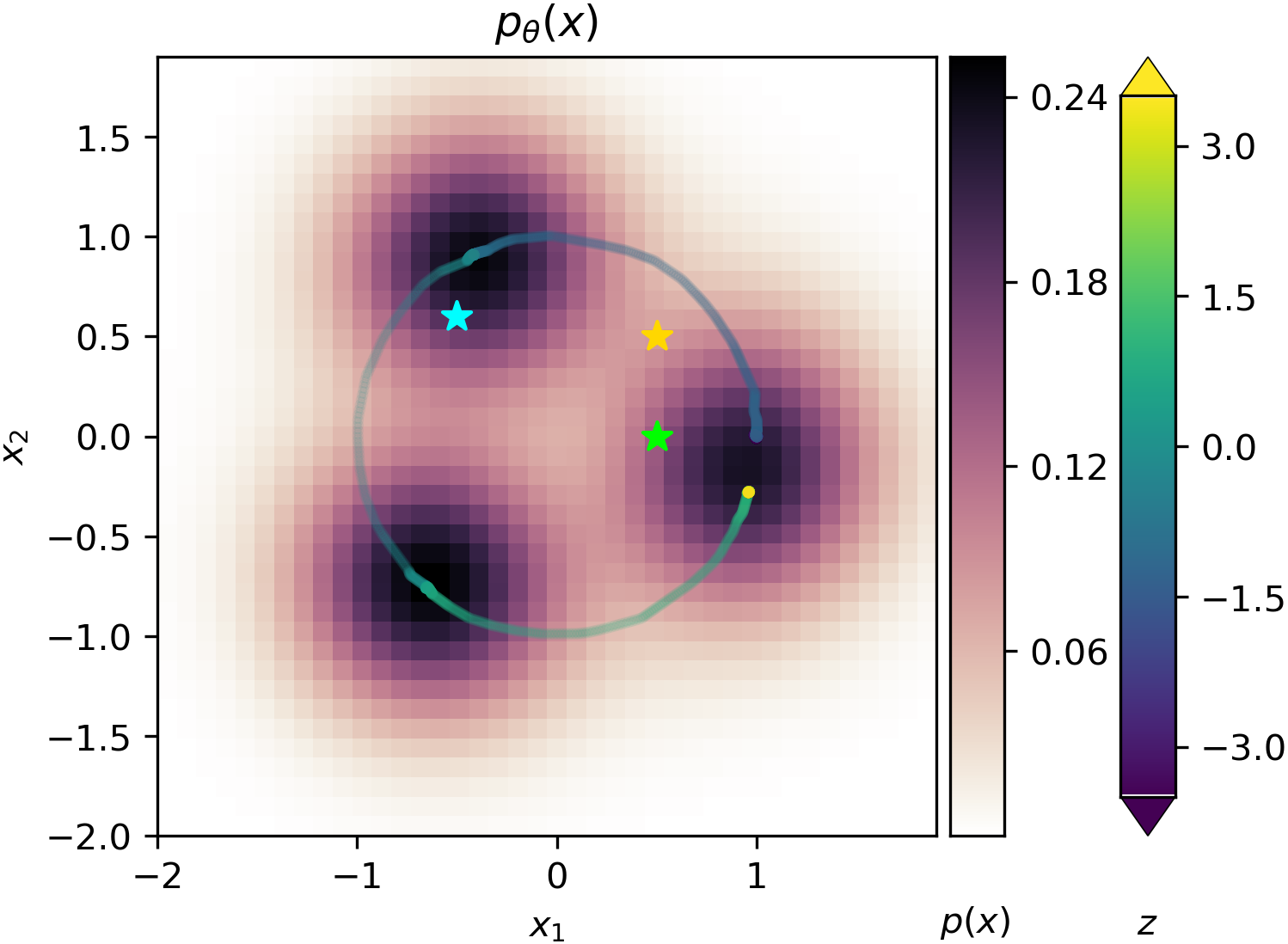}
    \end{subfigure}
    \hspace{1mm}
    \begin{subfigure}[t]{0.46\textwidth}
    \includegraphics[width=1.0\textwidth]{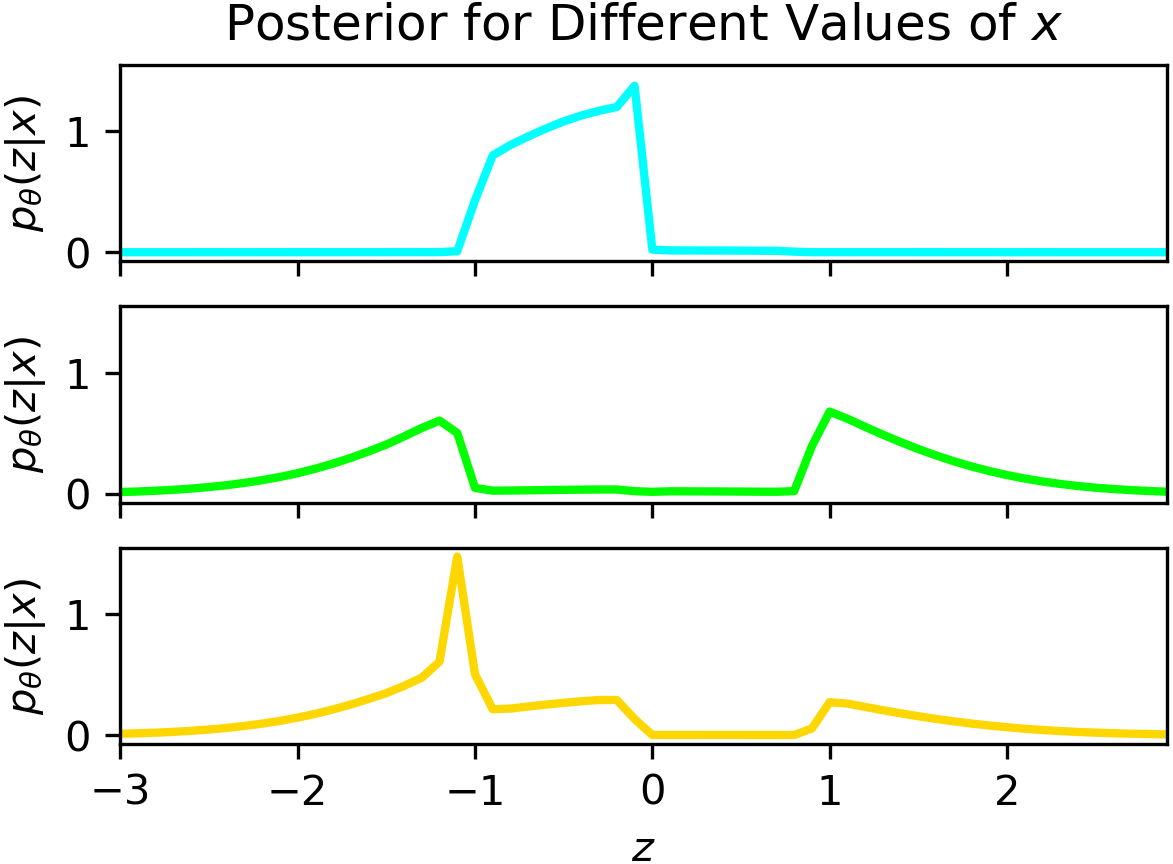}
    \end{subfigure}
    
    \begin{subfigure}[t]{0.02\textwidth}
        \centering
        \small
        \rotatebox[origin=l]{90}{\hspace{26pt}\textbf{Figure-8 Example}}
    \end{subfigure}
    \begin{subfigure}[t]{0.46\textwidth}
    \includegraphics[width=1.0\textwidth]{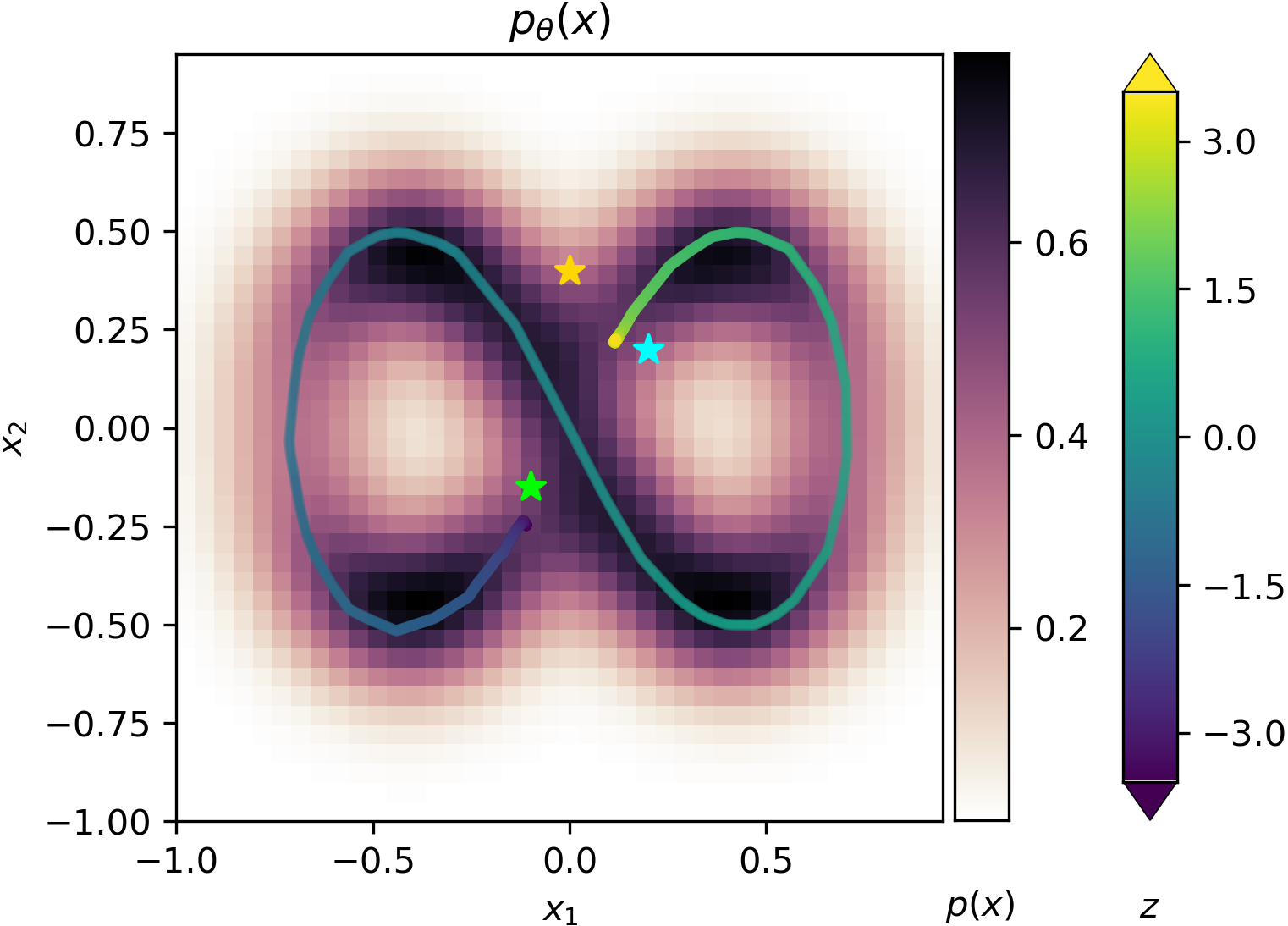}
    \end{subfigure}
    \hspace{1mm}
    \begin{subfigure}[t]{0.46\textwidth}
    \includegraphics[width=1.0\textwidth]{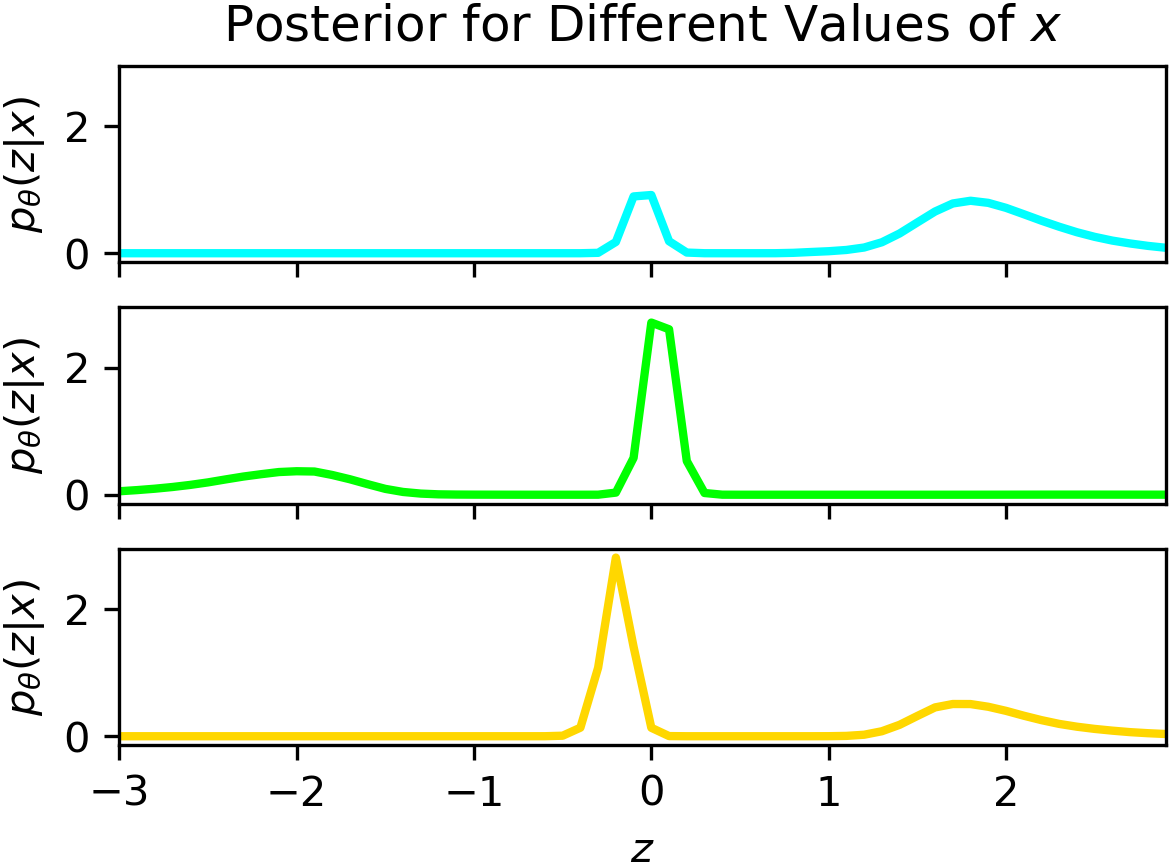}
    \end{subfigure}
    
    \caption{\textbf{Novel benchmark data-sets for which MFG-VAE's ELBO meaningful compromises learning $\bm{p(x)}$.} Top-row: the ``Figure-8 Example''. Bottom-row: ``Clusters Example''. Left-column: data marginal $p_\theta(x)$, with $f_\theta(z)$ overlaid on-top, colored by its corresponding value of $z$. Right-column: posteriors corresponding to the $x$'s starred on the density plots. At these starred points, the posterior is multi-modal and therefore poorly approximated by an MFG.}
    \label{fig:fig-8-and-clusters}
\end{figure*}

Based on the analysis of model non-identifiability in Figure \ref{fig:quad}, 
we argue that even in the ``simplest parameterization'' of $f_\theta$,
the posteriors need to be approximated poorly by an MFG for the VAE 
to significantly misestimate $p(x)$.
That is, if we consider all likelihood functions $f_\theta$ that capture the true data distribution $p(x)$
(e.g. if we consider all variants in Figure \ref{fig:quad})
and select the one that has the easiest-to-approximate posteriors,
when these posteriors are still sufficiently difficult to approximate, 
the learned $p(x)$ will be compromised. 

We create such data-sets in two ways.
Our first way to create such a data-set is to choose a likelihood function $f_\theta$ that interleaves 
areas of high slope and low slope; for example, consider,
\begin{align}
z \sim \mathcal{N}(0, 1), \quad x | z \sim \mathcal{N}(f_\theta(z), \sigma^2_\epsilon), \quad f_\theta(z) = \lfloor z \rfloor,
\label{eq:stepfn}
\end{align}
visualized in Figure \ref{fig:stepfn}. 
The figure shows precisely how the posterior, computed as horizontal cross-sections of $p_\theta(x, z)$ (on the left),
intersects the joint density $p(x, z)$ in multiple locations, 
and is therefore skewed and multi-modal (on the right). 
Using this mechanism, we propose the ``Clusters Example'' 
(visualized in Figure \ref{fig:fig-8-and-clusters}, described in Appendix \ref{sec:clusters-example}),
which is a smoothed step-function embedded on a circle. 
The flat regions of the step function yield the ``clusters''.
For this data-set, we verify that both conditions from Section \ref{sec:misestimate-px-conditions} hold:
(1) as \ref{fig:fig-8-and-clusters} shows, a large portion of the data points has posteriors that are poorly approximated by an MFG, and (2) there does not exist a simpler parameterization of the same $f_\theta$ (all would have interleaved areas of high and low slopes, resulting in these complex posteriors).

Instead of interleaving areas of high and low gradient in $f_\theta$,
we can alternatively create a ground-truth model that has non-MFG posteriors by selecting a likelihood function
$f_\theta$ that curves close to itself.
The ``Figure-8 Example'' (visualized in Figure \ref{fig:fig-8-and-clusters} described in Appendix \ref{sec:fig-8-example})
is an example of such a model. 
When $f_\theta$ curves close to itself, for may $x$'s there will exist several different $z$ that could have generated it,
leading to a multi-modal posterior.
To see this in the ``Figure-8 Example'', consider the starred points in Figure \ref{fig:fig-8-and-clusters}, 
each of which could have been generated by different areas on $f_\theta(z)$ by different $z$'s, 
and for which the true posterior is therefore multi-modal, satisfying condition (1). 
We also verify condition (2) is satisfied by considering all continuous parameterizations of  the ``Figure-8" curve: 
any such parametrization of $f_\theta$ will curve close to itself and thus result in multi-modal posteriors. 

\paragraph{VAEs approximate $p(x)$ poorly when conditions (1) and (2) of Section \ref{sec:misestimate-px-conditions} hold.}
To show that these issues occur because the MFG variational family over-regularizes the generative model,
we follow the experimental setup described in Section \ref{sec:methodology}:
that is, we give the initialized access to the ground-truth parameters, and we compare VAE inference with LIN and IWAE. 
As expected, IWAE learns $p(x)$ better than LIN, which outperforms the VAE (Figure \ref{fig:fig-8-inline}).
Like the VAE, LIN compromises learning the data distribution in order to learn simpler posteriors,
since it also uses an MFG variational family.
In contrast, IWAE is able to learn more complex posteriors and thus compromises $p(x)$ far less.
However, note that with $S = 20$ importance samples, IWAE still does not learn $p(x)$ perfectly.
For the full analysis (both quantitative and qualitative), see Appendix \ref{sec:verifying_path_cond}.

What happens if the portion of observations with highly non-MFG posterior was small?
or if there exists an alternative function that explains $p(x)$ well?
In Appendix \ref{sec:thm-1-cond-not-sat-quant}, we present two benchmarks --
one for which only condition (1) is satisfied (Figure \ref{fig:abs-inline}), 
and one for which only condition (2) is satisfied  (Figure \ref{fig:circle-inline}) --
and demonstrate that in both cases an MFG-VAE estimates $p(x)$ well.

\begin{figure*}[p]
    \centering
    
    \begin{subfigure}[t]{1.0\textwidth} 
        \begin{subfigure}[t]{0.41\textwidth}
        \centering
        \small
        True $p_{\theta_\text{GT}}(x)$
        \end{subfigure}
        \begin{subfigure}[t]{0.16\textwidth}
        \centering
        \small    
        IWAE $p_\theta(x)$
        \end{subfigure}
        \begin{subfigure}[t]{0.41\textwidth}
        \centering
        \small    
        VAE $p_\theta(x)$
        \end{subfigure}
        
        \begin{subfigure}[t]{0.01\textwidth}
            \centering
            \small
            \rotatebox[origin=l]{90}{\hspace{65pt}$x_2$}
        \end{subfigure}
        \begin{subfigure}[t]{0.98\textwidth}
        \includegraphics[width=1.0\textwidth]{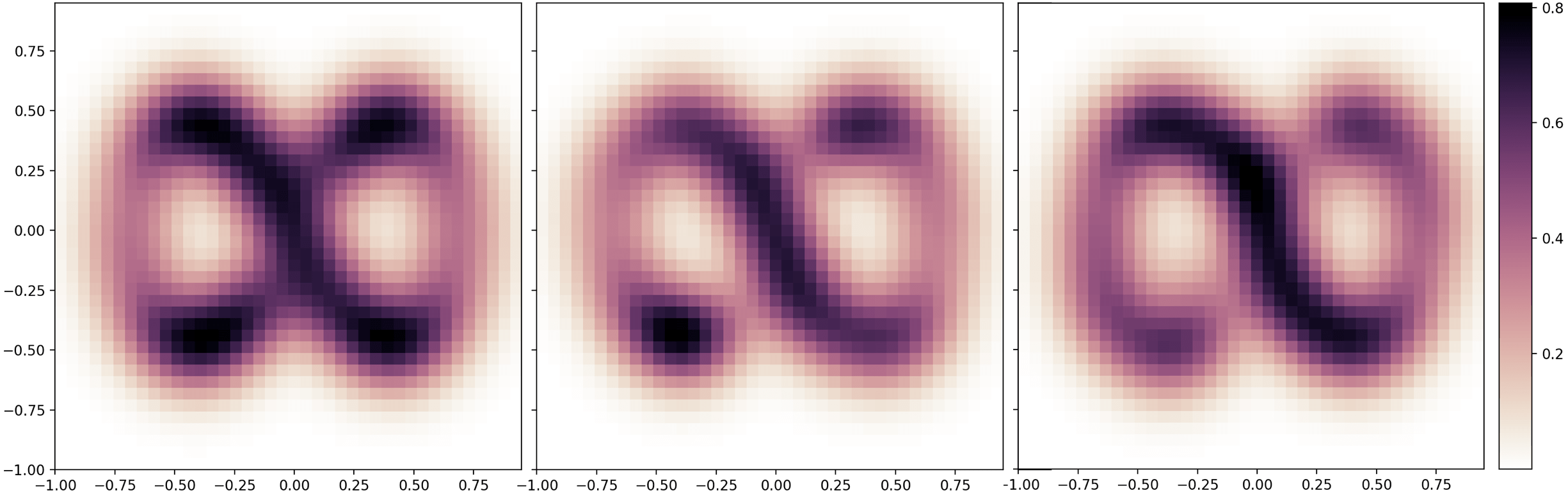}    
        \end{subfigure}
        
        \begin{subfigure}[t]{0.41\textwidth}
        \centering
        \small
        \vspace{-0.3cm}
        $x_1$
        \end{subfigure}
        \begin{subfigure}[t]{0.16\textwidth}
        \centering
        \small    
        \vspace{-0.3cm}
        $x_1$
        \end{subfigure}
        \begin{subfigure}[t]{0.41\textwidth}
        \centering
        \small    
        \vspace{-0.3cm}
        $x_1$
        \end{subfigure}
    
    \caption{\small{\textbf{Figure-8 Example}}}
    \label{fig:fig-8-inline}
    \end{subfigure}
    
    \vspace{0.5cm}
    
    \begin{subfigure}[t]{1.0\textwidth} 
        \begin{subfigure}[t]{0.41\textwidth}
        \centering
        \small
        True $p_{\theta_\text{GT}}(x)$
        \end{subfigure}
        \begin{subfigure}[t]{0.16\textwidth}
        \centering
        \small    
        IWAE $p_\theta(x)$
        \end{subfigure}
        \begin{subfigure}[t]{0.41\textwidth}
        \centering
        \small    
        VAE $p_\theta(x)$
        \end{subfigure}
        
        \begin{subfigure}[t]{0.01\textwidth}
            \centering
            \small
            \rotatebox[origin=l]{90}{\hspace{65pt}$x_2$}
        \end{subfigure}
        \begin{subfigure}[t]{0.98\textwidth}
        \includegraphics[width=1.0\textwidth]{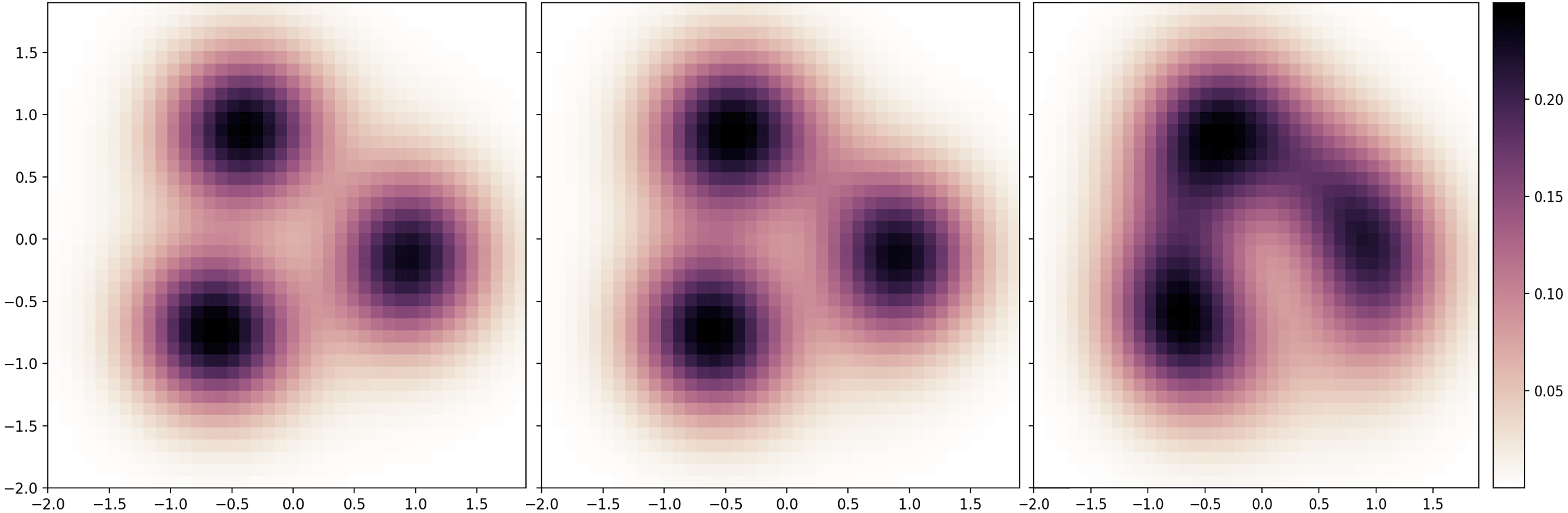}        
        \end{subfigure}
        
        \begin{subfigure}[t]{0.41\textwidth}
        \centering
        \small
        \vspace{-0.3cm}
        $x_1$
        \end{subfigure}
        \begin{subfigure}[t]{0.16\textwidth}
        \centering
        \small    
        \vspace{-0.3cm}
        $x_1$
        \end{subfigure}
        \begin{subfigure}[t]{0.41\textwidth}
        \centering
        \small    
        \vspace{-0.3cm}
        $x_1$
        \end{subfigure}        
    
    \caption{\small{\textbf{Clusters Example}}}
    \label{fig:clusters-inline}
    \end{subfigure}
    
    \vspace{0.5cm}
        
    \begin{subfigure}[t]{0.49\textwidth}
        \begin{subfigure}[t]{0.45\textwidth}
            \centering
            \scriptsize
            \phantom{AAAa}True $p_{\theta_\text{GT}}(x)$
        \end{subfigure}
        \begin{subfigure}[t]{0.45\textwidth}
        	   \centering
            \scriptsize    
            \phantom{Af}IWAE/VAE $p_\theta(x)$
        \end{subfigure}   
         	
    	\begin{subfigure}[t]{0.01\textwidth}
            \centering
            \scriptsize
            \rotatebox[origin=l]{90}{\hspace{45pt}$x_2$}
        \end{subfigure}
    	\begin{subfigure}[t]{0.98\textwidth}
    	\includegraphics[width=1.0\textwidth]{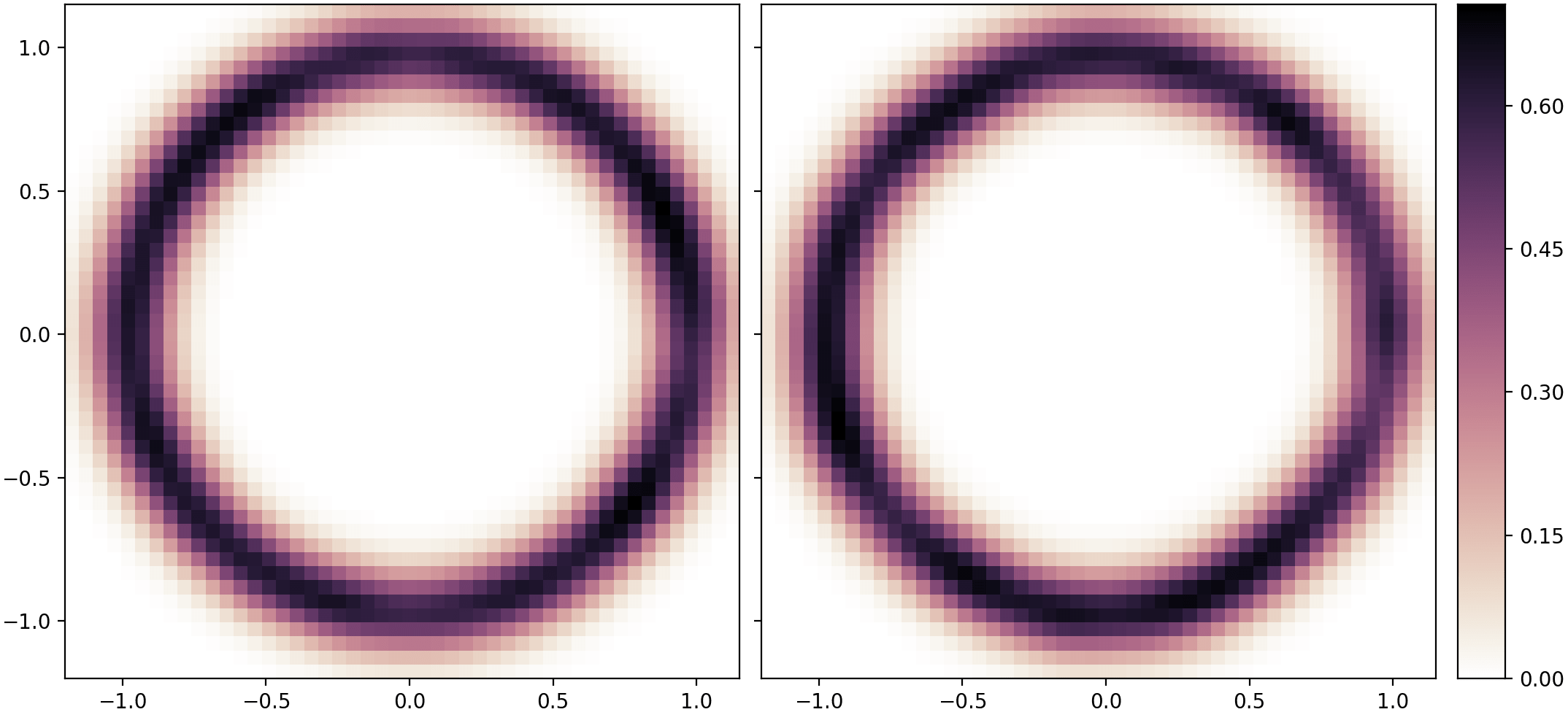}
	\end{subfigure}
	
	\begin{subfigure}[t]{0.45\textwidth}
            \centering
            \scriptsize
            \vspace{-0.4cm}
            \phantom{AAAa}$x_1$
        \end{subfigure}
        \begin{subfigure}[t]{0.45\textwidth}
        	   \centering
            \scriptsize    
            \vspace{-0.4cm}
            \phantom{Aa}$x_1$
        \end{subfigure}   
	
    \caption{\small{\textbf{Circle Example}}}
    \label{fig:circle-inline}
    \end{subfigure}
    \begin{subfigure}[t]{0.49\textwidth}
    	\begin{subfigure}[t]{0.45\textwidth}
            \centering
            \scriptsize
            \phantom{AAAa}True $p_{\theta_\text{GT}}(x)$
        \end{subfigure}
        \begin{subfigure}[t]{0.45\textwidth}
        	   \centering
            \scriptsize    
            \phantom{Af}IWAE/VAE $p_\theta(x)$
        \end{subfigure}  
    
    	\begin{subfigure}[t]{0.01\textwidth}
            \centering
            \scriptsize
            \rotatebox[origin=l]{90}{\hspace{45pt}$x_2$}
        \end{subfigure}
    	\begin{subfigure}[t]{0.98\textwidth}
    	\includegraphics[width=0.99\textwidth]{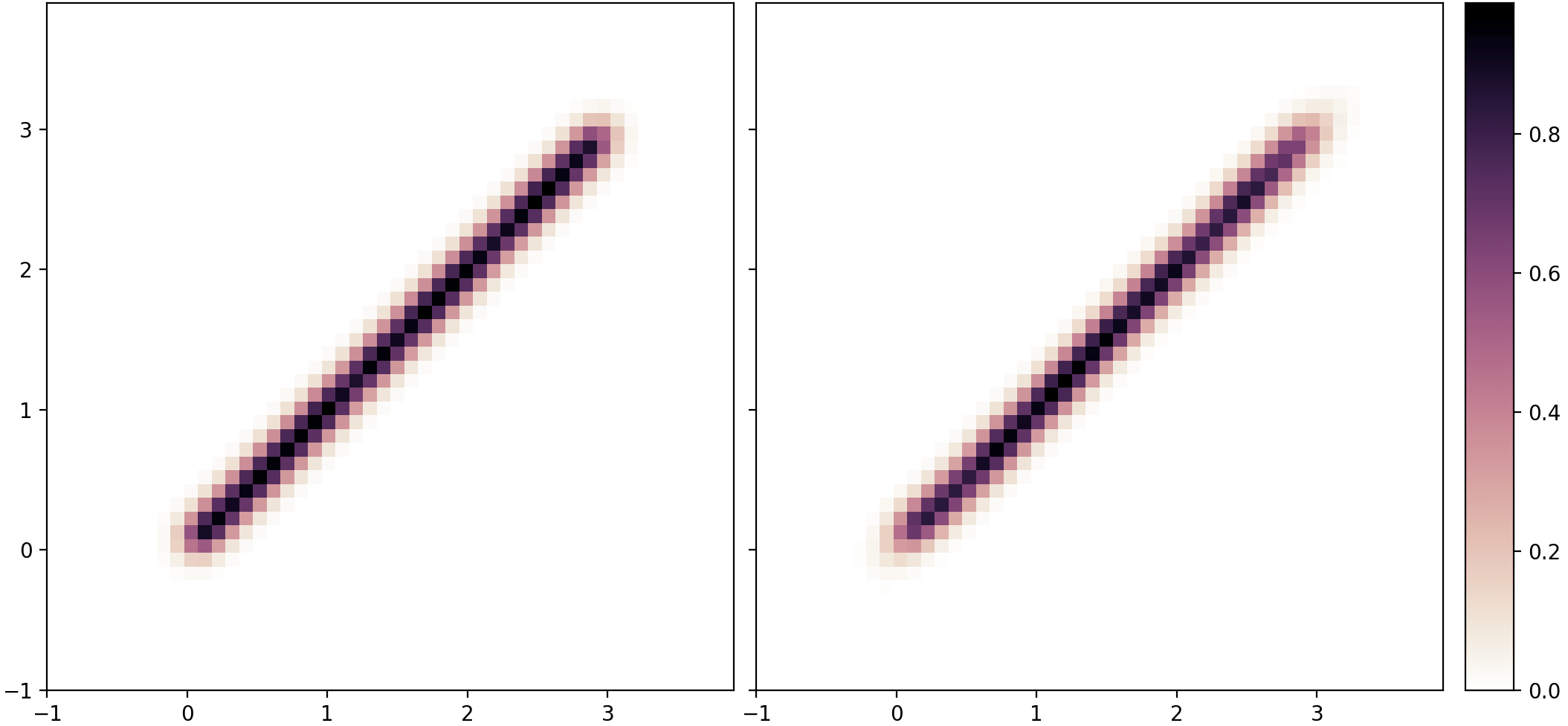}
	\end{subfigure}
	
	\begin{subfigure}[t]{0.45\textwidth}
            \centering
            \scriptsize
            \vspace{-0.45cm}
            \phantom{AAAa}$x_1$
        \end{subfigure}
        \begin{subfigure}[t]{0.45\textwidth}
        	   \centering
            \scriptsize    
            \vspace{-0.45cm}
            \phantom{Aa}$x_1$
        \end{subfigure}   	
	
    \caption{\small{\textbf{Absolute-Value Example}}}
    \label{fig:abs-inline}
    \end{subfigure}
    
    \caption{\textbf{VAE approximates $\bm{p(x)}$ poorly when conditions from Section \ref{sec:misestimate-px-conditions} hold.} This figure comprises the true data distributions with the corresponding distributions learned by the global optima of the VAE and IWAE objectives. When conditions are satisfied, in examples (a) and (b), VAE training approximates $p(x)$ poorly and IWAE performs better. When one of the conditions is not met, in examples (c) and (d), then the VAE can learn $p(x)$ as well as IWAE.}
    \label{fig:inline-unsup}
\end{figure*}

\paragraph{Proposed benchmarks typify classes of real data.} 
The ``Figure-8'' Example generalizes to any data manifold with high curvature 
(e.g., images from videos showing continuous physical transformations of objects), 
i.e. where the Euclidean distance between two points in a high-density region on the manifold is 
(A) less than the length of the geodesic connecting these points and 
(B) within 2 standard deviations of observation noise. 
The ``Clusters" Example in Figure \ref{fig:clusters-inline} generalizes to cases where we are learning very low-dimensional representations of multimodal data distributions (e.g., popular image data-sets where similar images lie in clusters).
On these data-sets, we expect the ELBO to prefer compromising 
the quality of the generative model for posteriors that are easy to approximate. 
We note that it is difficult (and perhaps even impossible) to verify these conditions on real data,
since on real data, the global optima of the ELBO may not be the largest bottleneck for good performance;
in fact, recent work shows that other choices -- such as the choice of gradient estimator and optimizer's step size -- 
may play just as big of a role as the choice of variational family~\citep{agrawal2020advances}. 
However, it is still helpful to keep in mind the types of data-sets mentioned here,
to ensure our modeling assumptions are consistent with the data.
Just like one would not apply a linear model to non-linear data,
it is best not to use a vanilla MFG-VAE on the aforementioned types of data-sets. 

\paragraph{Effect on downstream tasks: generating synthetic data \& OOD detection.} 
The effects of misestimating $p(x)$ on generating synthetic data and on OOD detection are clear;
when the learned model under-estimates $p(x)$ in one region, it must over-estimate it in another region.
For example, in Figure \ref{fig:clusters-inline}, at the global optima of the ELBO,
the corresponding VAE ``smears'' the data distribution across the three modes,
placing mass where it ought not to. 
As a result, samples from the regions between modes will be incorrectly labeled 
as ``within distribution'', as opposed to as OOD. 
Similarly, generating synthetic data from this model will inevitably yield samples
from these problematic regions.

\paragraph{Proposed benchmarks question whether the inflexibility of the prior causes the misestimation of $p(x)$.}
Recent work~\citep{Tomczak2017,falorsi_explorations_2018,davidson_hyperspherical_2018,bauer2019resampled} 
attributes the failure to generate high-quality synthetic data to the inflexibility of the prior.
For example, recent work \citep{falorsi_explorations_2018,davidson_hyperspherical_2018,miao2021incorporating}
argues that when the topology of the data manifold is non-trivial 
(e.g. when the data lies on a figure-8, circle, or cluster shaped manifold),
the Gaussian prior over-regularizes the generative model, causing it to under-fit. 
In contrast to these works, our ``Figure-8'', ``Circle'' and ``Clusters'' examples all tell a different story:
they show that the VAE's generative model is perfectly capable of expressing data distributions with complex topologies,
and that any estimation error is caused by the choice of an inexpressive variational family and/or local optima during inference.
Similarly, recent work (e.g. ~\cite{Tomczak2017,bauer2019resampled}) 
attributes the failure to generate high-quality synthetic data to the inflexibility of the prior,
arguing that the prior creates ``holes'' in the latent space, 
which lead to a mismatch between the aggregated posterior $\frac{1}{N}\sum_{n=1}^N q_\phi(z | x)$ and the prior $p(z)$.
Due to this mismatch, the generative model does not learn to decode samples from the ``holes'',
causing it to generate unrealistic samples.
Here, however, we argue that this mismatch is not the cause, but a symptom of the problem.
That is, when the global optima of the ELBO corresponds to a model that misestimates $p(x)$,
we can no longer expect the aggregated posterior to recover the prior~\citep{Yacoby2020}:
\begin{align}
p(z) &= \mathbb{E}_{p(x)} [ p_{\theta_\text{GT}} (z | x) ] \neq \mathbb{E}_{p(x)} [ p_{\theta^*} (z | x) ] \approx \mathbb{E}_{p(x)}  [q_{\phi^*} (z | x) ] \approx \frac{1}{N} \sum\limits_{n=1}^N q_{\phi^*} (z_n | x_n)
\end{align}
So how do existing works estimate $p(x)$ better using a more sophisticated prior?
A more sophisticated prior allows for a simpler $f_\theta$, thereby simplifying the posterior;
for example, in the ``Clusters Example'', having a mixture-of-Gaussians prior would obviate the need
for $f_\theta$ to alternate between flat and steep regions, thereby allowing for a better MFG approximation of the posterior.

To improve the quality of the learned $p(x)$, practitioners are now faced with a choice:
either to increase the flexibility of the prior or that of the variational family.
Based on our results, we recommend choosing a prior based on domain knowledge 
(e.g. using a mixture-of-Gaussians prior if your data contains several ``subtypes'' of observations),
and increasing the flexibility of the variational family sufficiently to learn $p(x)$ well
(but not too to avoid overfitting and/or increasing the difficulty of optimization~\citep{shu_amortized_2018,rosca2018distribution}).

\section{When VAEs learn models with undesirable latent spaces} \label{sec:bad-latent-space}

Whereas in Section \ref{sec:misestimate-px}, 
we focus on tasks that require accurate estimatation of the data distribution $p(x)$,
in this section, we investigate tasks (like learning disentangled representations)
that require the learned model's latent dimensions to align with those of the ground-truth model. 
In disentangled representation learning, we suppose that each dimension of the latent space 
corresponds to a task-meaningful concept \citep{Ridgeway2016,TCVAE}. 
We aim to infer these meaningful ground-truth latent dimensions. 
It is noted in literature that this inference problem is ill-posed --  
there are an infinite number of likelihood functions (and hence latent codes) 
that can capture $p(x)$ equally well~\citep{locatello_challenging_2019}. 
Previous work therefore advocates that with random restarts, one can hopefully find a model with the correct disentangled representation, and then select that model via human input~\citep{sercu_interactive_2019}, via a new metric~\citep{Duan2019}, or alternatively one can align the latent representation with the desired latent concepts with side-information or inductive bias~\citep{Siddharth2017,locatello_challenging_2019}. 
In contrast to these works, \cite{Stuhmer2019} argue that these approaches will not mitigate the issue;
they show that even when $f_\theta$ is linear, the ELBO already exhibits a specific 
inductive bias towards models that have easy-to-approximate posteriors,
which ``entangle'' the latent representations (i.e. does not recover the ground-truth latent space). 

\paragraph{Conditions under which VAEs will learn models with a non ground-truth latent space.}
In this section, we conjecture that VAE inference will fail to recover the ground-truth model's
latent space under the following two conditions:
\begin{itemize}
\item[] \textbf{Condition 1:} The true posterior for the ground-truth $f_{\theta_\text{GT}}$ is difficult to approximate by an MFG for a large portion of $x$'s.
\item[] \textbf{Condition 2:} There exists a likelihood function $f_\theta \in \mathcal{F}$ with simpler posteriors that approximates $p(x)$ well.
\end{itemize}
Note that condition (1) is the same as the one from Section \ref{sec:misestimate-px-conditions},
while condition (2) is the opposite. 
The intuition behind these two conditions is simple:
when condition (1) is satisfied on the ground-truth model, the MLEO will be $0$ while the PMO will be high,
but when condition (2) is satisfied, VAE inference will recover a model different than that of the ground-truth. 
We now present two illustrative examples -- 
we recap the linear example provided by \cite{Stuhmer2019} and relate it two our two above conditions,
and propose a new non-linear example.

\paragraph{A Linear Example.} Our analysis follows that of \cite{Stuhmer2019}.
Consider data generated by the ground-truth model $f_{\theta_{\text{GT}}}(z) = A z + b$. 
If $A$ is non-diagonal, then the posteriors of this model are correlated Gaussians (poorly approximated by MFGs). 
Since for every $x$ the true posterior $p_\theta(z | x)$ the same non-diagonal covariance matrix, condition (1) is satisfied: 
for \emph{all} $x$'s, the true posterior cannot be well-captured by an MFG.
Let $A' = AR$, where we define $R = (\Sigma V^\top)^{-1}(\Lambda - \sigma^2 I)^{1/2}$ 
with an arbitrary diagonal matrix $\Lambda$ and matrices $\Sigma, V$ taken from the SVD of $A$, 
$A = U\Sigma V^\top$. In this case, $f_{\theta} = A'z + b$ has the same marginal likelihood as 
$f_{\theta_{\text{GT}}}$, that is, 
$p_\theta(x) =p_{\theta_{\text{GT}}}(x) = \mathcal{N} (b, \sigma^2_\epsilon \cdot I + A A^\intercal)$. 
Condition (2) is thereby also satisfied. 
As a result, since the posteriors of $f_{\theta}$ are uncorrelated, 
the ELBO will prefer $f_{\theta}$ over $f_{\theta_{\text{GT}}}$.
However, in the latent space corresponding to $f_{\theta}$ , 
the original \emph{interpretations} of the latent dimensions are now entangled. 

\paragraph{A Non-Linear Example.}
As in the linear example, for more complicated likelihood functions, 
we expect the ELBO to prefer learning models with simpler posteriors,
which are not necessarily ones that are useful for learning disentangled representations. 
Suppose, for instance, that we observe data from ``Variant 1'' in Figure \ref{fig:quad}.
Variant 1 has two regions of the latent space that both map to the same region of the input space.
If $x$ represents patient symptoms and $z$ represents the underlying condition of the patient,
Variant 1 describes a model in which two patients with two different underlying conditions exhibit the same symptoms. 
As we already show in Section \ref{sec:intuition},
if our goal is to recover a model with the same latent space as Variant 1,
when using a uni-modal variational family, the VAE will actually prefer learning Variant 2.
For the medical task in question, however, Variant 2 will not suffice,
since in the latent space of Variant 2, 
all patients that exhibit the same symptoms are mapped to the same region of the latent space,
preventing us from characterizing the relationship between underlying condition and symptoms.
Therefore, just as in the linear example, the inductive bias in our choice of variational family
may cause us to recover models in which the latent space is not meaningful for our downstream task. 
Even if this inductive bias is reduced by using a more flexible variational family,
we are still left with the issue of non-identifiability,
in which we leave which model we recover in practice up to chance
(e.g. we recover Variant 3, which has even more complex posteriors). 

\paragraph{Proposed benchmark on which MFG-VAEs will learn an undesirable latent space.} 
Based on these insights, we propose the ``Absolute Value Example'' (described in Appendix \ref{sec:unsup-examples}).
Since in this example, $f_\theta(z) = | z |$,
there are two values of $z$ that could have generated every $x$.
As such, the ground-truth posteriors are all bimodal.
However, for this model, there also exists a simpler function that explains $p(x)$ and has unimodal posteriors
(see Figure \ref{fig:vae-abs-f}). 
We empirically verify that, indeed, traditional inference recovers a model with unimodal posteriors
while still explaining $p(x)$ well (see Figure \ref{fig:vae-abs-post-learned}). 

\paragraph{Proposed benchmarks typify classes of real data.}
The linear example above trivially generalizes to all data-sets, since the MFG-VAE prior is rotationally invariant. 
The Non-linear example (as well as the ``Absolute Value Example'') both generalize to data-sets 
generated by models for which $f_\theta$ is two-to-one, or alternatively,
to data-sets in which we a priori believe several regions of the latent space all map to the same regions of the input space.
We note that while a priori it is not possible to know which data-set was generated by such a function, 
one can treat the choice of MFG variational family (an inference assumption) as 
a generative model assumption that $f_\theta$ is not two-to-one. 
If one believes this assumption is inappropriate for a task, do not use an MFG-VAE.

\section{When VAEs ignore additional inductive bias from semi-supervision} \label{sec:semi-supervision}

\begin{figure*}[p]
    \vspace{-0.75cm}
    \centering
    \small
    
    \begin{subfigure}[t]{0.288\textwidth}
    \centering
    \normalsize
    \textbf{True}
    \end{subfigure}
    \begin{subfigure}[t]{0.288\textwidth}
    \centering
    \normalsize
    \textbf{IWAE}
    \end{subfigure}
    \begin{subfigure}[t]{0.288\textwidth}
    \centering
    \normalsize
    \textbf{VAE}
    \end{subfigure}
    
    \vspace{0.3cm}
    
    \begin{subfigure}[t]{0.05\textwidth}
    \centering
    \small
    \rotatebox[origin=l]{90}{\hspace{0pt}\textbf{$\bm{f_\theta(z, y)}$ for $\bm{y = 0, 1}$}}
    \end{subfigure}
    \begin{subfigure}[t]{0.288\textwidth}
    \includegraphics[width=1.0\textwidth]{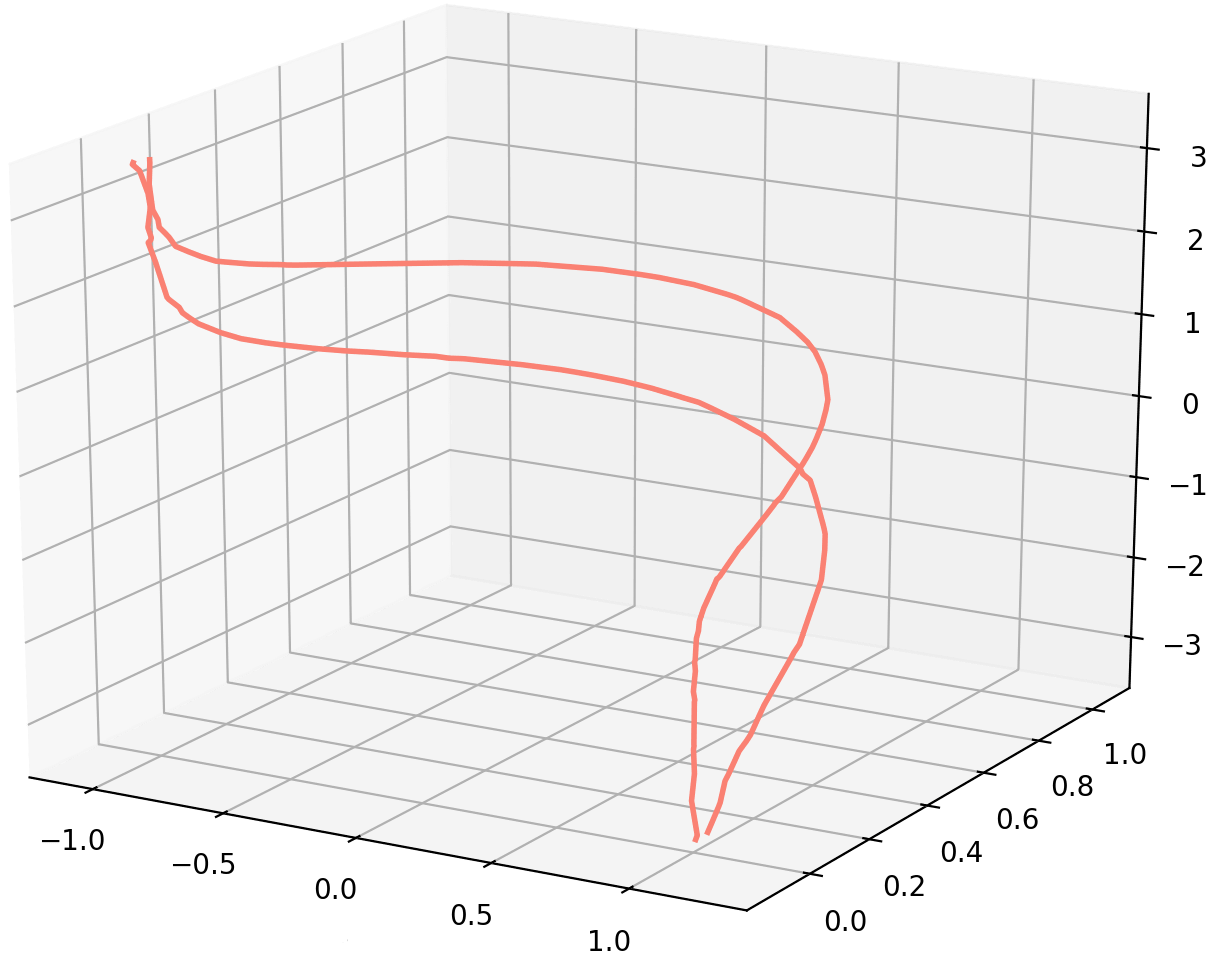} \\
    \scriptsize
    \centering
    \vskip -20pt
    $x_1$ \hspace{70pt} $x_2$
    \caption{}
    \label{fig:discrete-fn-true-inline}
    \end{subfigure}
    \begin{subfigure}[t]{0.288\textwidth}
    \includegraphics[width=1.0\textwidth]{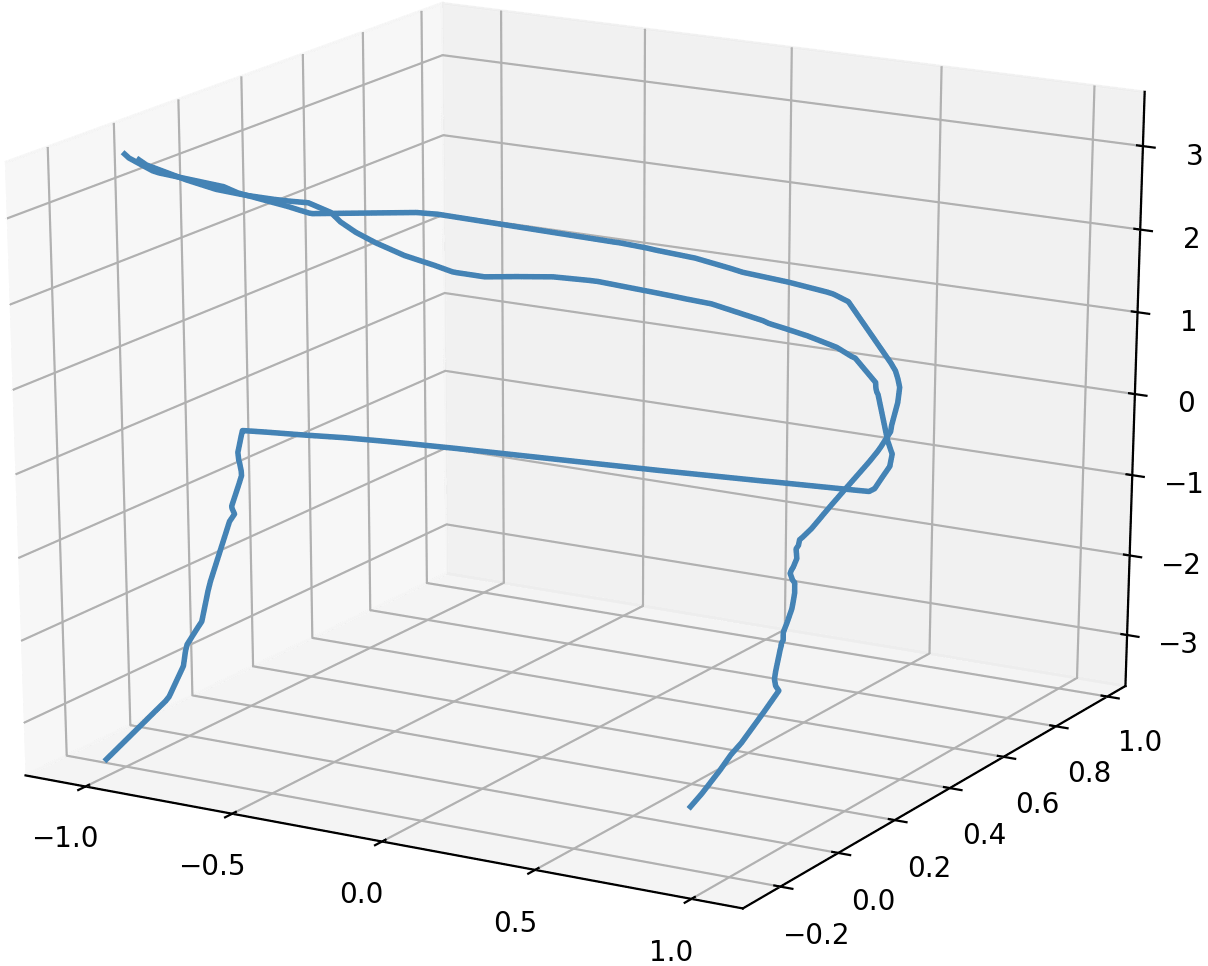}
    \scriptsize
    \centering
    \vskip -10pt
    $x_1$ \hspace{70pt} $x_2$    
    \caption{}
    \label{fig:discrete-fn-iwae-inline}
    \end{subfigure}
    \begin{subfigure}[t]{0.288\textwidth}
    \includegraphics[width=1.0\textwidth]{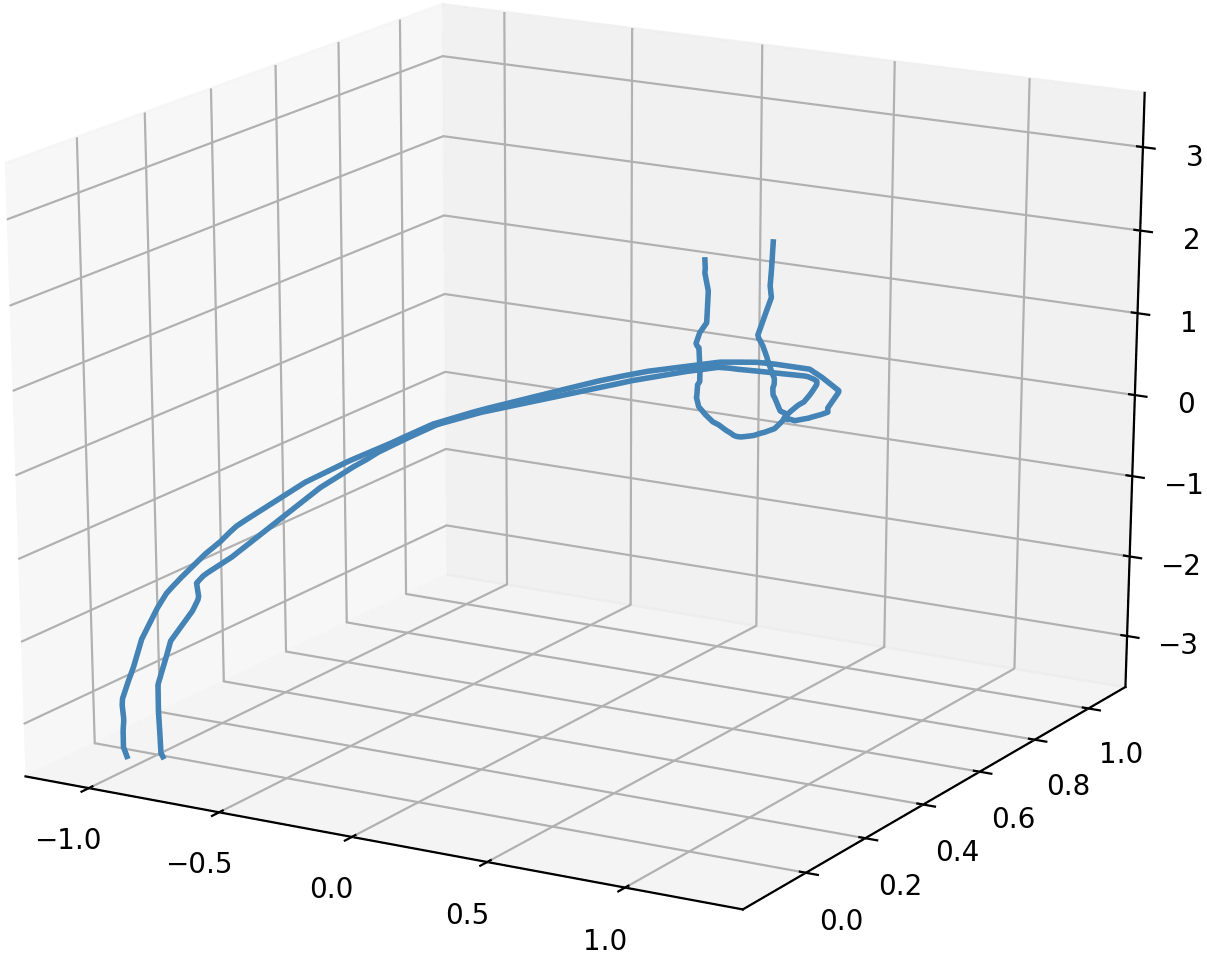}
    \scriptsize
    \centering
    \vskip -10pt
    $x_1$ \hspace{70pt} $x_2$    
    \caption{}
    \label{fig:discrete-fn-vae-inline}
    \end{subfigure}
    \begin{subfigure}[t]{0.01\textwidth}
    \centering
    \footnotesize
    \vspace{-60pt}$z$
    \end{subfigure}
    
    \vskip 0.5em
          
    \begin{subfigure}[t]{0.05\textwidth}
    \centering
    \small
    \rotatebox[origin=l]{90}{\hspace{35pt}\textbf{$\bm{p_\theta(x | y = 1)}$ \phantom{AAAAAA} $\bm{p_\theta(x | y = 0)}$}}
    \end{subfigure}
    \begin{subfigure}[t]{0.01\textwidth}
    \centering
    \footnotesize
    \rotatebox[origin=l]{90}{\hspace{57pt}\textbf{$x_2$ \phantom{AAAAAAAAAAAAA} $x_2$}}
    \end{subfigure}
    \begin{subfigure}[t]{0.286\textwidth}
    \includegraphics[width=1.0\textwidth]{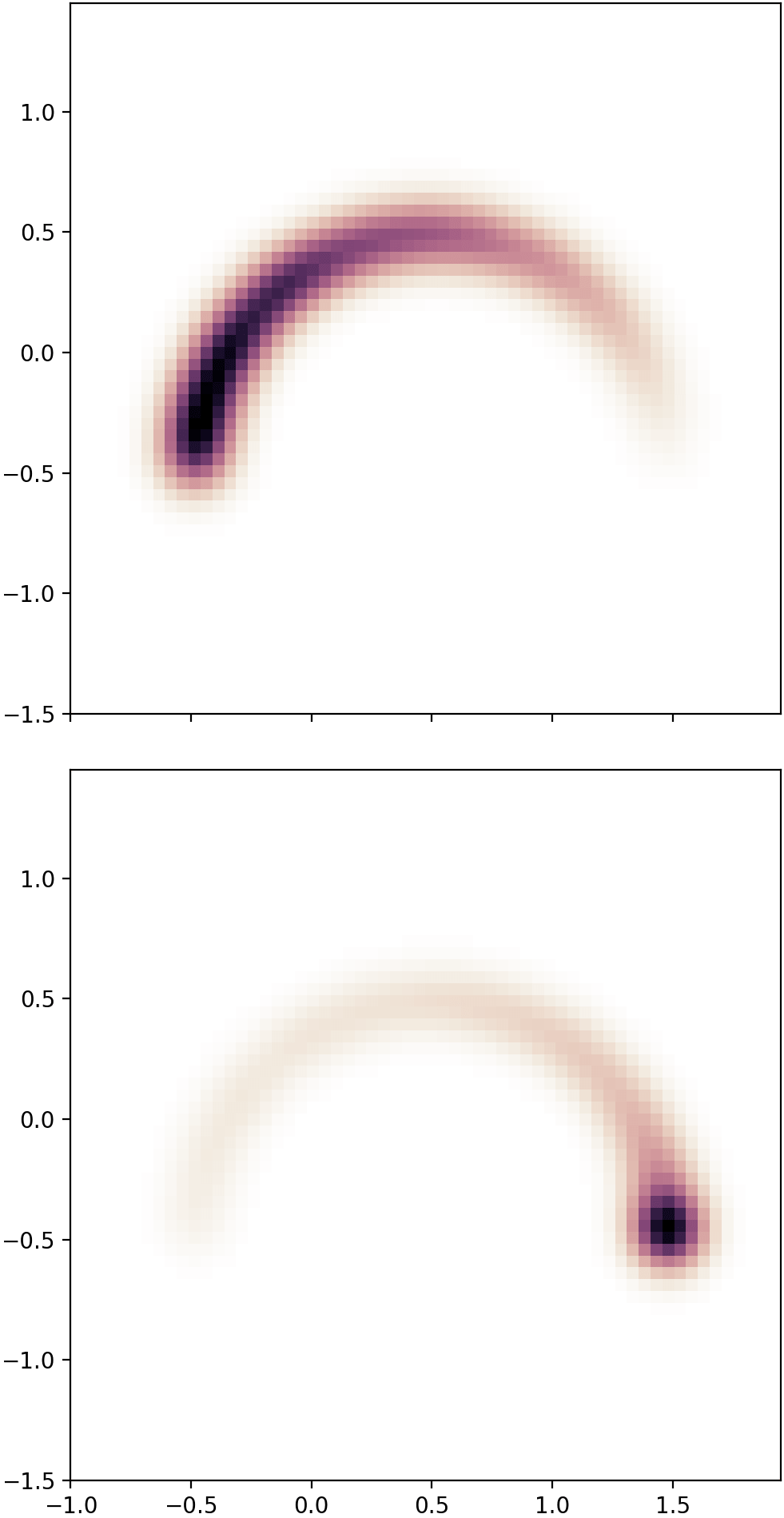}
    \scriptsize
    \centering
    \vskip -3pt
    $x_1$
    \caption{}
    \label{fig:discrete-cond-true-inline}
    \end{subfigure}
    \begin{subfigure}[t]{0.2705\textwidth}
    \includegraphics[width=1.0\textwidth]{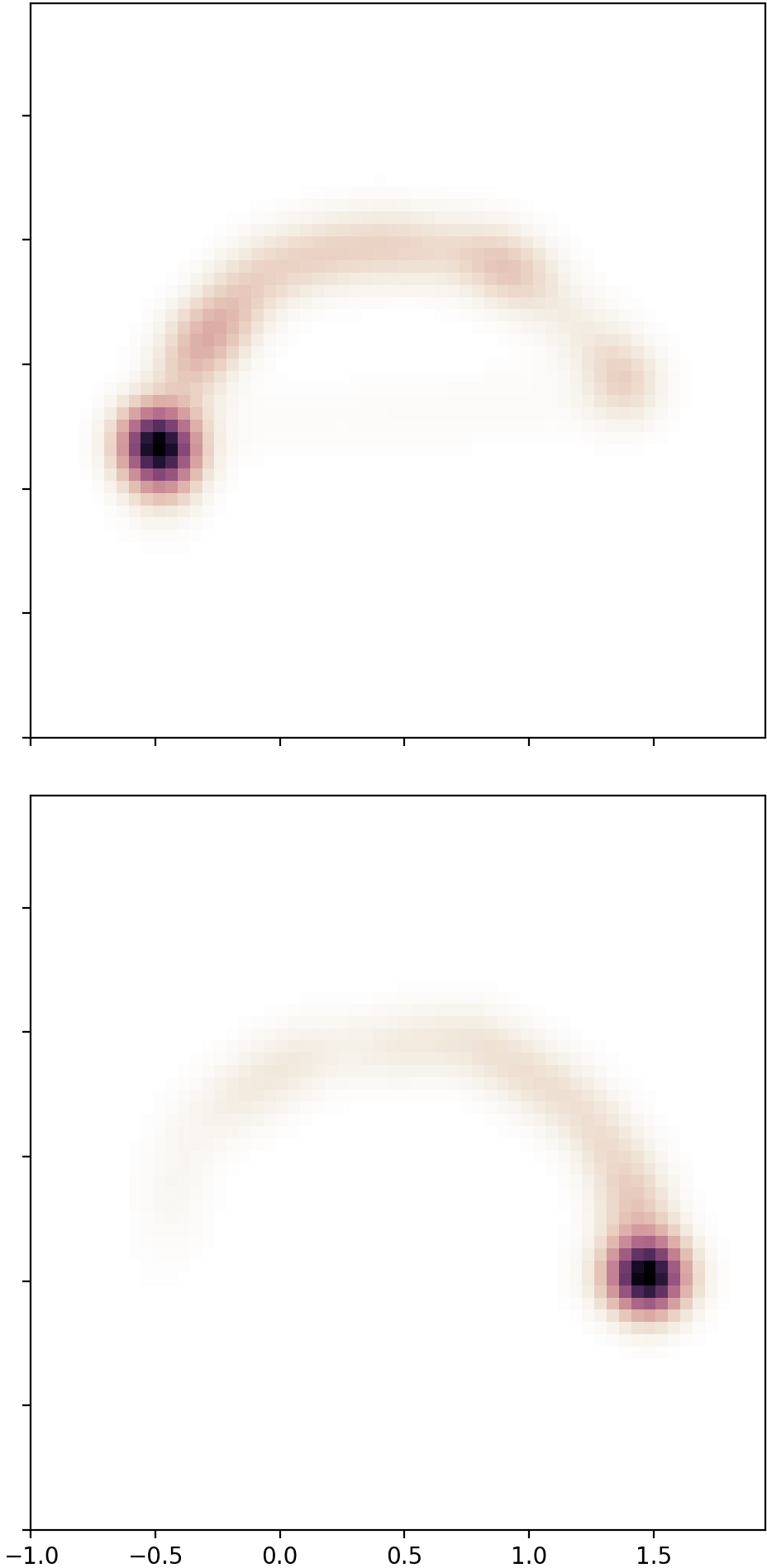}
    \scriptsize
    \centering
    \vskip -3pt
    $x_1$    
    \caption{}
    \label{fig:discrete-cond-iwae-inline}
    \end{subfigure}
    \begin{subfigure}[t]{0.318\textwidth}
    \includegraphics[width=1.0\textwidth]{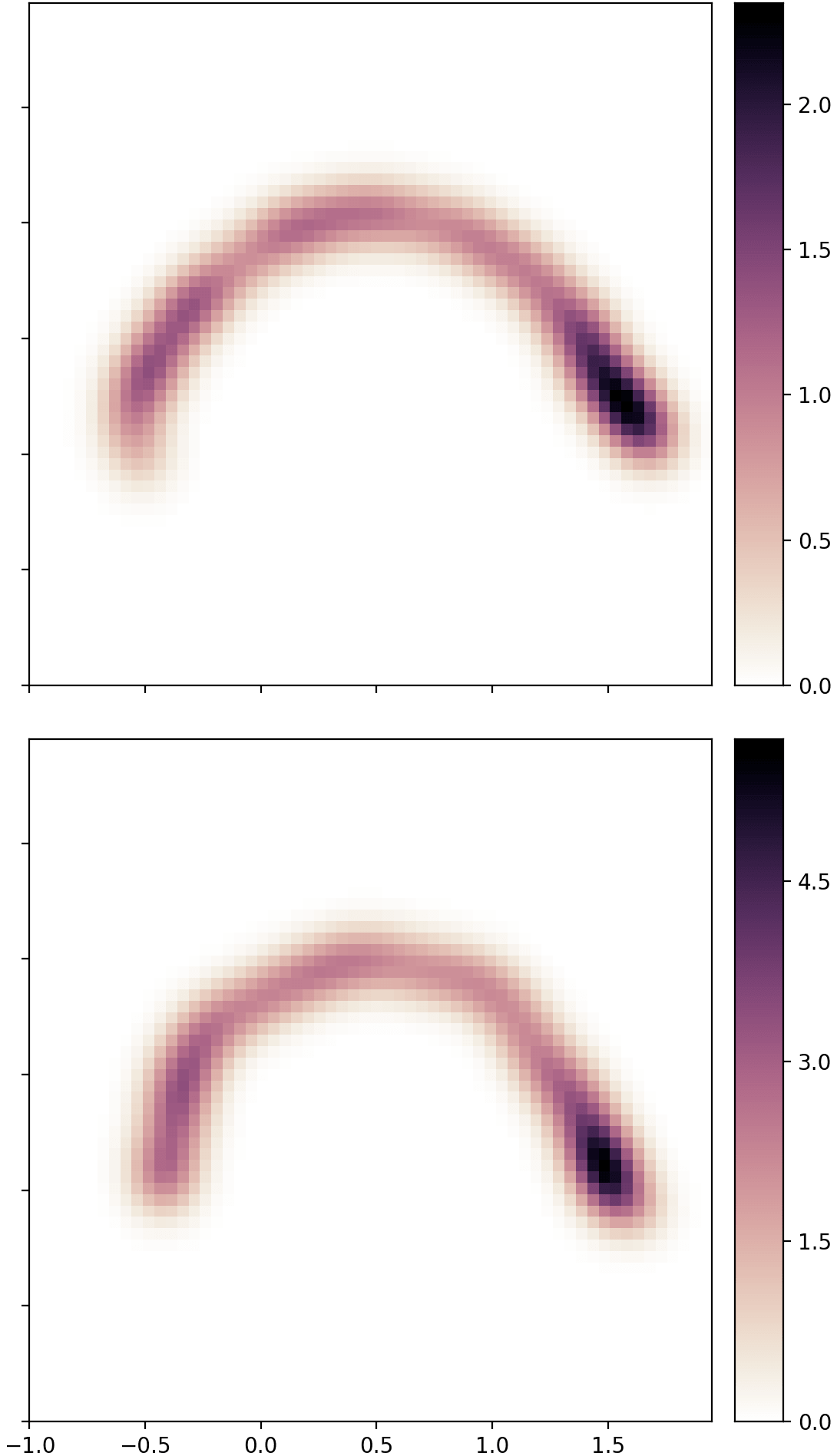}
    \scriptsize
    \centering
    \vskip -3pt
    $x_1$    
    \caption{}
    \label{fig:discrete-cond-vae-inline}
    \end{subfigure}
    \vskip 0.5em
        
    \begin{subfigure}[t]{0.05\textwidth}
    \centering
    \small
    \rotatebox[origin=l]{90}{\hspace{45pt}\textbf{$\bm{p_\theta(x)}$}}
    \end{subfigure}
    \begin{subfigure}[t]{0.01\textwidth}
    \centering
    \footnotesize
    \rotatebox[origin=l]{90}{\hspace{55pt}\textbf{$x_2$}}
    \end{subfigure}
    \begin{subfigure}[t]{0.2875\textwidth}
    \includegraphics[width=1.0\textwidth]{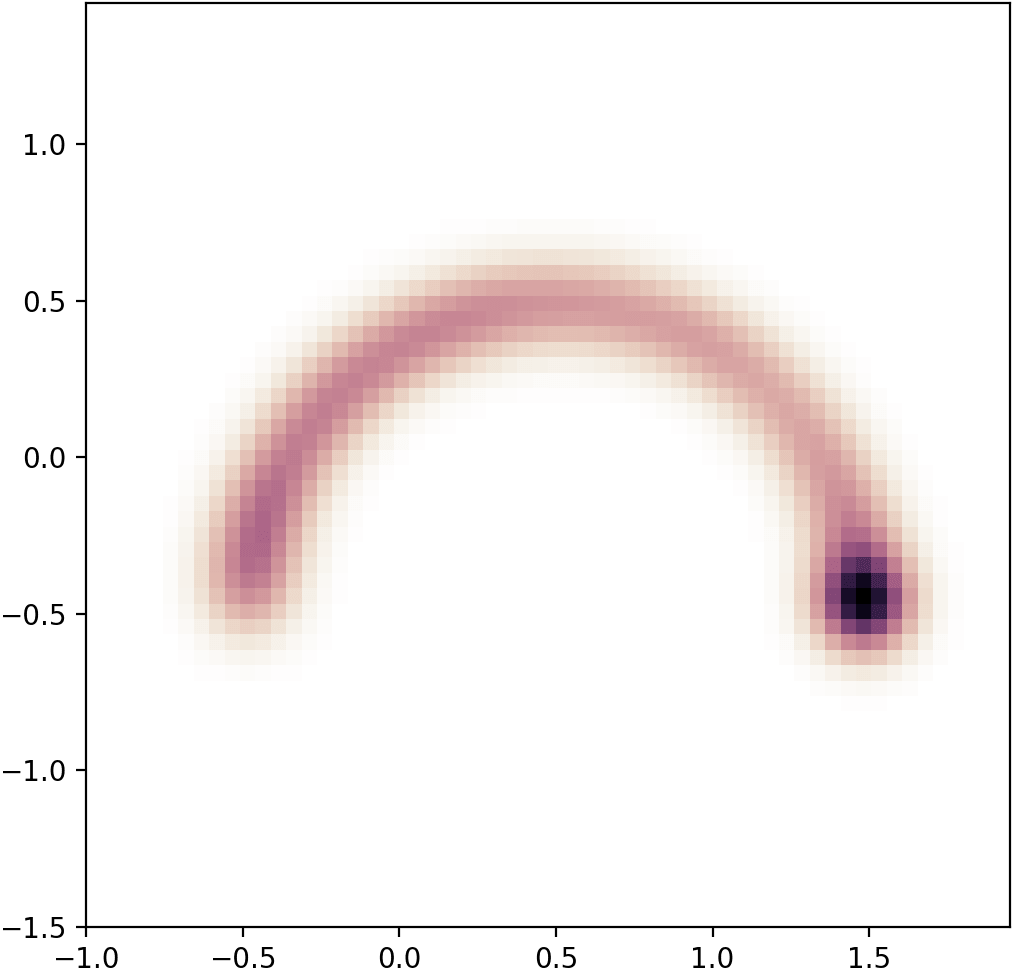}
    \scriptsize
    \centering
    \vskip -3pt
    $x_1$    
    \caption{}
    \label{fig:discrete-px-inline}
    \end{subfigure}
    \begin{subfigure}[t]{0.271\textwidth}
    \includegraphics[width=1.0\textwidth]{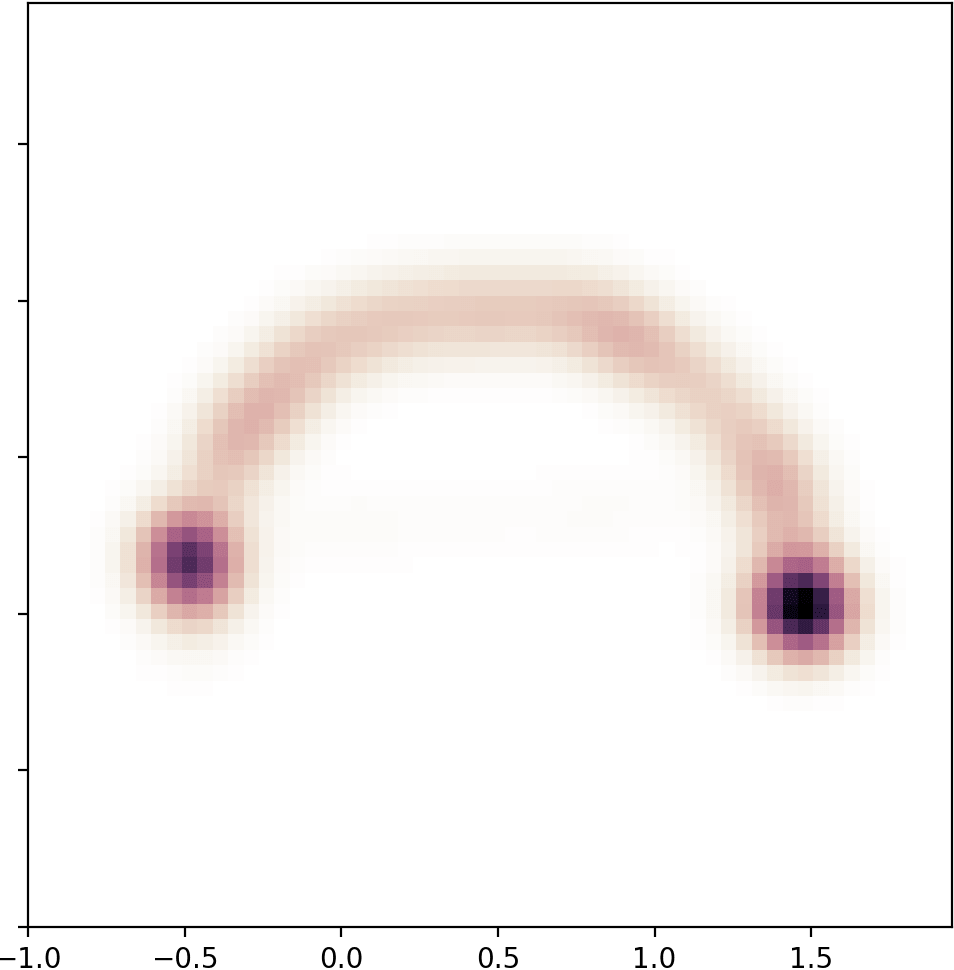}
    \scriptsize
    \centering
    \vskip -3pt    
    $x_1$       
    \caption{}
    \label{fig:discrete-px-iwae-inline}
    \end{subfigure}
    \begin{subfigure}[t]{0.317\textwidth}
    \includegraphics[width=1.0\textwidth]{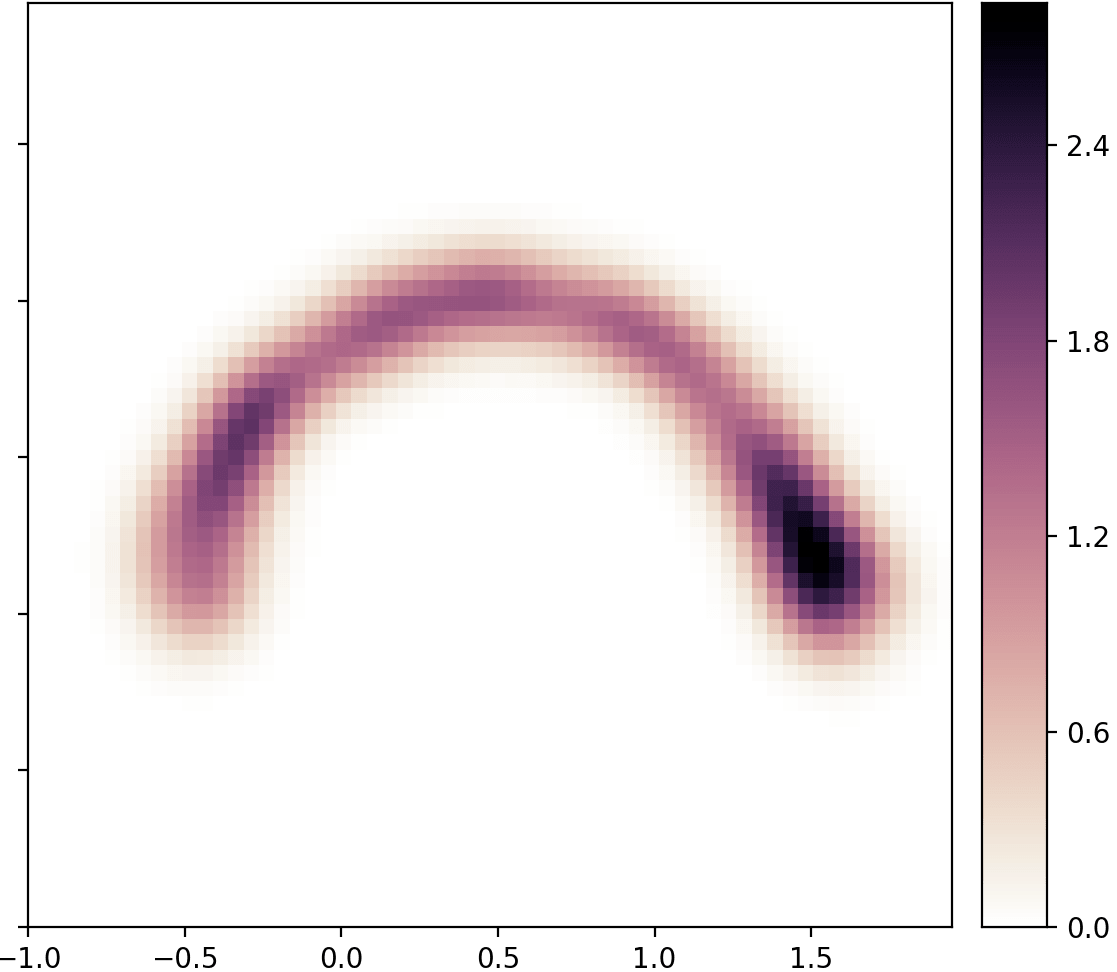}
    \scriptsize
    \centering
    \vskip -3pt    
    $x_1$       
    \caption{}
    \label{fig:discrete-px-vae-inline}
    \end{subfigure}
    
    \caption{\small \textbf{Semi-supervised VAE exhibits ``functional collapse'' on the Discrete Semi-Circle Example.} Comparison of VAE and IWAE on a semi-supervised example 
    (left column: true, middle column: IWAE, right-column: VAE). 
    The ground-truth likelihood function $f_{\theta_\text{GT}}(z, y)$ shows two distinct functions, one for each $y=0,1$. 
    The VAE's $f_\theta(z, y)$ is over-regularized by an MFG variational family and learns two nearly identical functions (``functional collapse''). 
    The IWAE's $f_\theta(z, y)$ function is un-regularized and learns two distinct but overfitted functions.
    As a result, both the VAE and IWAE fail to learn $p(x)$ and $p(x | y)$. 
    The VAE learns $p(x)$ better while IWAE learns $p(x|y)$ better.}
    \label{fig:inline-semisup}
\end{figure*}

Since model non-identifiability and the inductive bias of the variational family both present challenges
for learning disentangled representations (Section \ref{sec:bad-latent-space}), 
one may hope that with additional information about the true latent space, one can recover it.
One way of doing this is using semi-supervision: 
by incorporating partial labels on the latent space, 
we can guide some latent dimensions towards their desired meaning~\citep{Kingma2014,Siddharth2017}.
However, as we show here, when this additional information conflicts with the inductive bias of the variational family,
it may be ignored. 

In semi-supervised VAEs, we partition the latent space into two parts, $y$ and $z$,
and assume that for some portion of the data we observe labels for $y$ (``M2 model'' proposed by ~\cite{Kingma2014}).
These VAEs have been used for tasks such as 
generating synthetic cohorts (sampling from $p(x | y = 1)$, $p(x | y = 0)$ respectively),
and for generating counterfactuals (generating a synthetic data $x'$ with label $y=0$ that is similar to a real observation $x$ with $y=1$).
For these tasks, it is important to accurately model the data conditional $p(x | y)$.
Surprisingly, as we show here, for certain types of data-sets we observe a phenomenon we call ``functional collapse": 
the model ignores the partial labels given by the semi-supervision, 
causing the learned conditionals to collapse onto a single distribution,
$p_\theta(x) \approx p_\theta(x | y=0) \approx p_\theta(x | y=1)$.
While one might expect that increasing the flexibility of the variational family will fix this issue, 
doing so actually causes the model to overfit to the few partial labels, causing it to perform no better than an MFG-VAE.
In this section, we characterize \emph{when} functional collapse occurs,
as well as \textit{how} it impacts the task of generating realistic counterfactuals. 

\paragraph{Conditions under which MFG-VAEs trade-off between generating realistic data and realistic counterfactuals.}
Let $x | y = 0$ represent cohort 0 and let $x | y = 1$ represent cohort 1. 
We identify two conditions under which functional collapse occur: 
\begin{itemize}
\item[] \textbf{Condition 1:} A large portion of the $x$'s in cohort 0 lies on the same manifold as the $x$'s in cohort 1 (or vice versa).
\item[] \textbf{Condition 2:} When the two cohorts lie on the same manifold, they are distributed differently. 
\end{itemize}

\paragraph{Intuition behind conditions.}
When the above conditions hold, the ground-truth model's posterior for the unlabeled data
$p_{\theta_\text{GT}}(z | x) = \int_y p_{\theta_\text{GT}}(z, y | x) dy$ will be multi-modal,
since for each value of $y$ there are a number of different likely $z$'s, each from a different cohort. 
As such, using an MFG variational family for the unlabeled portion of
the semi-supervised objective ($\mathcal{U}$ in Equation \ref{eq:ss-m2-objective-pre}) 
will encourage inference to either compromise learning the data distribution 
in order to better approximate the posterior, 
or to learn the data distribution well but approximate the posterior poorly, depending on our prioritization of the two objectives (indicated by our the choice of the hyperparameter $\gamma$ in Equation \ref{eq:ss-m2-objective-pre}). 
In the first case, data generation will be compromised because the model will overfit to the partial labels;
however, the model will at least be able to generate from two distinct data conditionals $p(x | y = 0)$ and $p(x | y = 1)$.
In contrast, in the latter case, the learned model will be able to generate realistic  
data but not realistic cohorts since the model will over over-regularize
the likelihood function $f_\theta(z, y)$ to collapse to the same function for all values of $y$ (functional collapse).
thereby collapsing the data conditionals $p_\theta(x | y) \approx p(x)$. 
That is, $p(x|y)$ will generate identical-looking cohort regardless of our choice of $y$.

\paragraph{Proposed benchmark on MFG-VAEs exhibit functional collapse.} 
We empirically demonstrate the trade-off between realistic data and realistic counterfactuals generation on the ``Discrete Semi-Circle'' Example in Figure \ref{fig:inline-semisup} (full details in Appendix \ref{sec:discrete-ss-example}).
In this benchmark, the ground-truth functions $f_\theta(z, y = 0)$ and  $f_\theta(z, y = 0)$ both lie on the same manifold (a semi-circle), but each has different slopes. 
As such, $p(x | y = 0)$ and $p(x | y = 1)$ represent two different distributions on the same manifold. 
In this data-set, we show that the MFG-VAE is able to learn the data manifold and distribution well. 
However, the ELBO prefers a model with a simple posterior (in comparison to the true posterior),
causing the learned $f_\theta(z, y)$ to collapse to the same function for all values of $y$ (Figure \ref{fig:discrete-fn-vae-inline}).
As a result, $p_\theta(x | y) \approx p_\theta(x)$ under the learned model.
As expected, functional collapse occurs when training with LIN as well.
In contrast, IWAE is able to learn two distinct data conditionals $p_\theta(x | y=0), p_\theta(x | y=1)$, but it does so at a cost.
\emph{IWAE does not regularize the generative model, and thus overfits to the few partial labels} (notice in Figure \ref{fig:discrete-fn-iwae-inline}, $f_\theta(z, y)$ juts out from one side to the other instantaneously). Lastly, IWAE learns $p(x)$ considerably worse than the VAE, while learning $p(x|y)$ significantly better. In Appendix \ref{sec:semi-sup-quant}, we provide a full quantitative and qualitative analysis of the above on synthetic and real data-sets. 

When $y$ is discrete, we can lower-bound the number of modes of $p_\theta(z | x)$ by the number of distinct values of $y$,
and choose a variational family that is sufficiently expressive.
But when $y$ is continuous, we cannot easily bound the complexity of $p_\theta(z | x)$. 
In this case, an \emph{additional} pathology is introduced by the discriminator 
$q_\phi(y | x)$ (Equation \ref{eq:ss-m2-objective-pre}),
whereby predictive accuracy increases at the cost of collapsing $p_\theta(x | y)$ towards $p_\theta(x)$.
We present the full analysis of this failure using a continuous version of the benchmark proposed here in Appendix \ref{sec:semi-sup-quant}.

\paragraph{Naive adaptation of IWAE for semi-supervision introduces new pathologies.}
The discriminator ensures that the approximate posterior of $y|x$ is predictive,
as it would be under the true posterior.
The approximate posterior implied by the IWAE objective, however,
is not the one given by the IWAE encoder $q_\phi(z | x)$,
and has a rather complex and uninterpretable form, $q_\text{IW}(z | x)$~\citep{cremer_reinterpreting_2017}.
Incorporating the approximate posterior of $y|x$ induced by the IWAE objective into
the semi-supervised objective would require an intractable marginalization of $q_\text{IW}(z | x)$ over $z$. 
Although some work proposes to use with lower bounds \citep{Siddharth2017}
on $q_\phi(z, y | x)$ marginalized over $z$, 
the discriminator in these cases is nonetheless different from the approximate posterior induced by the IWAE objective.  
This may be an additional factor of the poor performance of IWAE in the semi-supervised setting with continuous $y$.

\paragraph{Proposed benchmarks typify classes of real data.}
In real data-sets, we often have samples from multiple cohorts of the population. 
General characteristics of the population hold for all cohorts, but each cohort may have different distributions of
these characteristics~\citep{klys2018learning}. 
Formally, this means that all of the cohorts lie on a shared manifold but each has a different distribution on that manifold
(that is $x | y  = 0$ and $x | y  = 1$ lie on the same manifold but $p(x | y = 0) \neq p(x | y = 1)$).
In order to check whether a data-set contains this structure, one can run a simple check:
we train a classifier to predict $y | x$. 
Difficulty in training a good classifier indicates that $x | y=0$ and $x | y=1$ lie roughly on the same manifold, 
and are difficult to distinguish, leading to the aforementioned failure mode. 
More formally, we expect this failure mode to occur when the classifier's average entropy $\mathbb{E}_{p(x, y)} \left[ \mathbb{H}[p(y | x) \right]$ is high.
In Appendix \ref{sec:semi-sup-quant}, we selected 3 UCI data-sets --
Diabetic Retinopathy Debrecen~\citep{UCIDiabetic}, Contraceptive Method Choice~\citep{Keel,UCI}
and the Titanic~\citep{Keel,Titanic} data-sets -- that exhibit high class overlap,
and show that semi-supervised VAE training struggles to accurately recover the data conditionals $p(x | y)$.

\section{When VAEs misestimate the decomposition between signal and noise} \label{sec:misestimate-signal-vs-noise}

Whereas so far, we were concerned with tasks requiring accurate estimation of $p(x)$
and of the ground-truth latent codes,
in this Section, we are concerned with tasks that require an accurate decomposition between
$f_\theta(z)$ and $\epsilon$ (or ``signal'' and ``noise''): learning compressed representations,
and defenses against adversarial perturbations. 

\paragraph{Conditions under which VAEs misestimate decomposition between signal and noise.}
There are two mechanisms that may cause traditional inference to misestimate the decomposition between signal and noise,
both occurring under the same conditions identified in Section \ref{sec:misestimate-px-conditions}:
\begin{itemize}
\item \textbf{Case A:} In practice, it is common to select the observation noise variance $\sigma^2_\epsilon$
that maximizes the ELBO (since selecting it using log-likelihood directly is intractable).
Here, we show that learning $\sigma^2_\epsilon$, 
either via hyper-parameter search or via direct optimization of the ELBO, can be biased;
for an observation set of size $N$, we have that,
\begin{align} \small
\begin{split}
\underset{{\sigma^{(d)}_\epsilon}^2}{\mathrm{argmin}} -\mathrm{ELBO}\left(\theta, \phi, {\sigma^{(d)}_\epsilon}^2\right) 
= \frac{1}{N} \sum\limits_{n=1}^N \mathbb{E}_{q_\phi(z | x_n)} \left\lbrack \left(x_n^{(d)} - f_\theta(z)^{(d)}\right)^2 \right\rbrack,
\label{eq:elbo-optima-obs-noise-var}
\end{split}
\end{align} 
where $d$ is the dimension (derivation in Appendix \ref{sec:path_2}). 
Equation \ref{eq:elbo-optima-obs-noise-var} shows that the variance $\sigma^2_\epsilon$ that minimizes the negative ELBO depends on the approximate posterior $q_\phi(z | x)$, and thus when the learned generative model does not capture $p(x)$ well (due to the conditions from Section \ref{sec:misestimate-px-conditions}), the learned $\sigma^2_\epsilon$ may be biased.
\item \textbf{Case B:} In practice, if the task does not require a specific latent space dimensionality, $K$, 
one chooses a $K$ that maximizes the ELBO or log-likelihood. 
We show that on data generated from a model for which conditions (1) and (2) 
from Section \ref{sec:misestimate-px-conditions} are satisfied,
using a model with a larger $K$ and a smaller $\sigma^2_\epsilon$ than those used by the ground-truth model 
no longer satisfied condition (2).
That is, we can now capture the data distribution with a simpler function $f_\theta(z)$ and hence get simpler posteriors. 
Thus, increasing $K$ and decreasing $\sigma^2_\epsilon$ alleviates the need to compromise the generative model in order 
to improve the inference model and leads to a better approximation of $p(x)$. 
\end{itemize}


\paragraph{Proposed benchmarks on which MFG-VAEs misestimate signal-to-noise decomposition.} 
For Case A, we propose the ``Spiral Dots" Example (details in Appendix \ref{sec:spiral-dots-example})
and perform two experiments.
In the first, we fix the noise variance ground-truth ($\sigma_\epsilon^2= 0.01$),
we initialize and train $\theta, \phi$ following the experimental setup from Section \ref{sec:methodology}, and finally,
we recompute $\sigma_\epsilon^2$ that maximizes the ELBO for the learned $\theta, \phi$. 
In the second experiment, we do the same, but train the ELBO jointly over 
$\sigma_\epsilon^2$, $\theta$ and $\phi$.
Using these two methods of learning the noise, we get $0.014\pm 0.001$ and $0.020\pm0.003$, respectively.
The ELBO therefore over-estimates the noise variance by $50\%$ and $100\%$, respectively.
We will argue later that this behavior is undesirable for tasks such as defenses against
adversarial perturbations. 

For Case B, we propose two benchmarks based on the ``Figure-8'' and ``Clusters'' Examples 
proposed in Section \ref{sec:ds-construction} to show how hyper-parameter selection on 
$\sigma^2_\epsilon$ and $K$ may prevent learning compressed representations.
Specifically, we apply a linear transform to the output of the function $A \cdot f_\theta(z)$ to embed it in 5D
\footnote{We use the following linear transformation: $A = \begin{psmallmatrix} 1.0 & 0.0 & 0.5 & 0.2 & -0.8 \\ 0.0 & 1.0 & -0.5 & 0.3 & -0.1 \\ \end{psmallmatrix}$}. 
We then train a VAE with latent dimensionality $K \in \{1, 2, 3 \}$, with 
$K = 1$ corresponding to the ground-truth model.  
Training for $K=1$ is initialized at the ground-truth model (GT), and for $K>2$ we initialize randomly; 
in each case we optimize $\sigma^2_\epsilon$ per-dimension to minimize the negative $\mathrm{ELBO}$. 
We find that the ELBO prefers models with larger $K$ over the ground-truth model ($K=1$),
and with smaller $\sigma^2_\epsilon$ (Table \ref{tab:model-mismatch}). 
We confirm that the posteriors become simpler
as $K$ increases, lessening the incentive for the VAE to compromise on approximating $p(x)$ (Figure \ref{fig:fig-8-mismatch-5d}). Lastly, we confirm that while LIN also shows a preference for higher $K$'s, IWAE does not (Table \ref{tab:model-mismatch}).
In preferring larger values of $K$, MFG-VAE inference therefore hinders learning compressed representations.
Furthermore, as $K$ increases, we find that the learned latent codes are, on average, less informative
 (Table \ref{tab:model-mismatch} shows that the average mutual information between 
each latent dimension and the inputs decreases).
To accommodate larger values of $K$, MFG-VAE also selects smaller values of $\sigma^2_\epsilon$.
This, again, hinders performance on tasks such as defenses against adversarial perturbations.

\begin{table}[!t]
\centering
\scriptsize

\setlength{\tabcolsep}{2.75pt}

\begin{tabularx}{\columnwidth}{@{}cccc|ccc@{}}
\textbf{VAE}                                              & \multicolumn{3}{c}{Figure-8 Example}                           & \multicolumn{3}{c}{Clusters Example}                        \\ \midrule
\multicolumn{1}{l|}{}                                & $K=1$ (GT) & $K=2$              & $K=3$              & $K=1$ (GT) & $K=2$             & $K=3$             \\ \midrule
\multicolumn{1}{l|}{Test $-\mathrm{ELBO}$}           & $-0.127 \pm 0.057$   & $-0.260 \pm 0.040$ & $\bm{-0.234 \pm 0.050}$ & $4.433 \pm 0.049$    & $4.385 \pm 0.034$ & $\bm{4.377 \pm 0.024}$ \\
\multicolumn{1}{l|}{Test $\mathrm{avg}_i I(x; z_i)$} & $\bm{2.419 \pm 0.027}$    & $1.816 \pm 0.037$  & $1.296 \pm 0.064$  & $\bm{1.530 \pm 0.011}$    & $1.425 \pm 0.019$ & $1.077 \pm 0.105$ \\ \bottomrule
\end{tabularx}

\vskip 2.0em

\begin{tabularx}{\columnwidth}{@{}cccc|ccc@{}}
\textbf{IWAE}                                              & \multicolumn{3}{c}{Figure-8 Example}                           & \multicolumn{3}{c}{Clusters Example}                        \\ \midrule
\multicolumn{1}{l|}{}                                & $K=1$ (GT) & $K=2$              & $K=3$              & $K=1$ (GT) & $K=2$             & $K=3$             \\ \midrule
\multicolumn{1}{l|}{Test $-\mathrm{ELBO}$}           & $\bm{-0.388 \pm 0.044}$ &            $-0.364 \pm 0.051$ &            $-0.351 \pm 0.045$        & $\bm{4.287 \pm 0.047}$ &             $4.298 \pm 0.054$ &             $4.295 \pm 0.049$ \\
\multicolumn{1}{l|}{Test $\mathrm{avg}_i I(x; z_i)$} & $\bm{2.159 \pm 0.088}$ &             $1.910 \pm 0.035$ &             $1.605 \pm 0.087$        & $1.269 \pm 0.052$ &             $\bm{1.321 \pm 0.033}$ &             $1.135 \pm 0.110$ \\ \bottomrule
\end{tabularx}

\vskip 1.0em

\caption{\textbf{The ELBO prefers learning models with more latent dimensions (and smaller $\bm{\sigma^2_\epsilon}$)
over the ground-truth (GT) model ($\bm{k = 1}$).}
Although the models preferred by the ELBO have higher mutual information between the data and learned $z$'s,
the average mutual information between dimensions of $z$ and the data decreases since with more latent dimensions,
the latent space learns $\epsilon$. In contrast, IWAE does not suffer from this pathology.
LIN was not included here because it was not able to minimize the negative ELBO as well as the VAE on these data-sets.
}
\label{tab:model-mismatch}
\end{table}

\paragraph{Intuition behind ELBO's preference for a larger $K$ and smaller $\sigma^2_\epsilon$.} 
So on these data-sets, why does the ELBO prefer models that do not compress the data and are on average less informative?
When increasing the latent dimensionality $K$ and decreasing the observation noise variance $\sigma^2_\epsilon$, 
condition (2) from Section \ref{sec:misestimate-px-conditions} no longer holds, 
since now there exist alternative generative models that explain $p(x)$ well but have simpler posteriors. 
This happens for two different reasons on the two archetypical pathological data-sets 
we identify in Section \ref{sec:ds-construction}. 
On ``Figure-8''-like data, the high $\sigma^2_\epsilon$ causes the posterior for the ground-truth model to be multi-modal; 
an observation $x$ near the crossing of the Figure-8 could have been 
generated by $z$'s from very different regions in the 1-D latent space (Figure \ref{fig:vae-fig-8-px}). 
On the other hand, for a model that captures $p(x)$ equally well but with a smaller $\sigma^2_\epsilon$, 
the posterior will be less multi-modal (the inverse mapping from x to z will be less ill-posed) 
and thus be preferred by the ELBO. 
As the latent dimension $K$ increases, the latent space has more capacity and increasingly models both $f_\theta(z)$ 
as well as observation noise (as the estimated $\sigma^2_\epsilon$ decreases). 
We observe exactly this phenomenon empirically in Figure \ref{fig:fig-8-mismatch-5d}. 
On the other hand, to generate the "Clusters"-like data with a 1D latent space, 
$f_\theta$ contracts regions of the latent space -- mapping many different $z$'s to nearby $x$'s (Figure \ref{fig:vae-clusters-f}). 
In this case, the posteriors have high skew and bi-modality (see Figure \ref{fig:vae-clusters-post-true}). 
By increasing $K$ and decreasing $\sigma^2_\epsilon$, 
one can learn an $f_\theta(z)$ that becomes more distance preserving. 
In this case, the posteriors will be unimodal and without skew (see Figure \ref{fig:clusters-mismatch-5d}), 
i.e. easily approximated with an MFG.

\paragraph{Effect on successful defenses against adversarial perturbations.}
Manifold-based defenses against adversarial attacks (e.g., \citet{Jalal2017,Meng2017,Samangouei2018,Hwang2019,jang_need_2020}) require both accurate estimates of the noise as well as of the intrinsic dimensionality of the data (i.e. the ground-truth latent dimensionality); however, as we show here, since the ELBO is unable to identify the correct $\sigma^2_\epsilon$ and correct latent dimensionality $K$, and incorrect compression may further result in incorrect noise estimates due to incorrect ground-truth latent space dimensionality. See Appendix~\ref{sec:adversarial} for full analysis.

\paragraph{Proposed benchmarks typify classes of real data.}
Since the failures to accurately decompose a given data distribution into signal and noise 
occurs when the ELBO can explain the data better using a non ground-truth setting of $\sigma^2_\epsilon$ and $K$,
the data-sets on which we expect this failure to occur are the same ones described in Section \ref{sec:ds-construction},
i.e. data-sets on which MFG-VAEs misestimate $p(x)$.

\section{Implications for Practice} \label{sec:discussion}

In this paper, we present two contributions that advance our understanding of VAEs: (1) we describe \emph{when} pathologies occur and introduced benchmarks to expose them;  (2) we describe the \emph{impact} of these pathologies on common downstream tasks.  Now, we connect these insights with implications for using VAEs in practice.
We make three simple guidelines for practitioners when using MFG-VAEs 
in order to avoid the pathologies described in this work.
While the guidelines are simple, we provide empirical and formal rationales for \textit{why} these practices matter (and we note that these best practices are not always used -- 
e.g., it is common to set $\sigma^2_\epsilon = 1$ without examining the data-set,
or to learn it by optimizing jointly with model parameters~\citep{Lucas2019}).
Finally, as others have noted~\citep{finke_importance-weighted_2019,agrawal2020advances}, 
a single methodological innovation is unlikely to fix all issues -- each innovation makes a specific tradeoff; 
thus, improvements will need to be task/data specific.
Our guidelines are:
\begin{enumerate}
\item On semi-supervised tasks, before selecting a variational class, 
check to see if classes are not easily separable. 
If they are not (i.e. if a simple neural network predicts $y | x$ with low balanced-accuracy),
use a rich variational family for $q_\phi(z | x)$ 
in the unlabeled data objective ($\mathcal{U}$ in Equation \ref{eq:ss-m2-objective-pre})
(and use an MFG family otherwise).
\item Set the noise variance $\sigma^2_\epsilon$ using domain expertise, 
or by hyper-parameter selection with an unbiased low-variance log-likelihood estimator.
\item Investigate the topology of the data (e.g., using topological data analysis, dimensionality reduction) 
before choosing a variational family.
If the data lies on a manifold in distorted Euclidean space (e.g., ``Figure-8'' Example),
or if the data is clustered (e.g., ``Clusters'' Example), use a rich variational family if you need to learn a very low-dimensional latent space.
\item If performing some topological analysis is intractable, determine if a richer variational family is needed by separately checking if (i) the manifold or (ii) the density was misestimated by an MFG-VAE. For (i), given a learned model, find the $f_\theta(z)$ is closest to each $x$, compute the residual $R = x - f_\theta(z)$, and use a statistical test for normality on $R$. If $p(R)$ is not Gaussian, the manifold may have been misestimated. Checking for (ii) is more difficult: one can do so by checking whether log-likelihood increases significantly when using a rich variational family. These tests are, of course, not perfect, since there are many factors independent of the ELBO's global optima that may lead to poor performance~\citep{agrawal2020advances}. 
\item Whenever using a rich variational family, apply regularization to the decoder network weights to prevent overfitting.
\end{enumerate}

\section{Conclusion}
In this work, we characterize conditions under which global optima of the MFG-VAE objective exhibit pathologies and connect these failure modes to undesirable effects on specific downstream tasks. We find that while performing inference with richer variational families (which increases training time) can alleviate these issue on unsupervised tasks, the use of complex variational families introduce unexpected new pathologies in semi-supervised settings. Finally, we provide a set of synthetic data-sets on which MFG-VAE exhibits pathologies. We hope that these examples contribute to a benchmarking data-set of ``edge-cases" to test future VAE models and inference methods. 


\acks{WP acknowledges support from the Harvard Institute of Applied Computational Sciences. 
YY acknowledges support from NIH 5T32LM012411-04 and from the IBM Faculty Research Award.}

\vskip 0.2in
\bibliography{references}


\clearpage

\addcontentsline{toc}{section}{Appendix} 
\part{Appendix} 
\parttoc 

\appendix

\section{Posterior Collapse as a Global Optima of the ELBO} \label{sec:posterior_collapse}

Posterior collapse occurs when the posteriors under both the generative model and approximate 
posterior learned by the inference model are equal to the prior~\citep{he_lagging_2019};
that is, $p(z) = p_\theta(z | x) = q_\phi(z | x)$.
In this regime, posterior collapse occurs as a global optima of the ELBO,
and surprisingly, the model is still able to generate samples from 
$p(x)$,
despite the latent codes retaining no information about the data (e.g.~\cite{chen_variational_2017,zhao_towards_2017}).
This is often attributed to the fact the generative model is very powerful and is therefore able 
to maximize the log data marginal likelihood without the help of the auxiliary latent codes~\citep{Oord2017}.

However, as we briefly show here, posterior collapse for MFG-VAEs only occurs when $p(x)$ is Gaussian, 
so we do not consider it in the remainder of the paper. 
Specifically, for the prior to equal both the true and approximate posteriors, the decoder must ignore $z$. 
Let $Y = f_\theta(z)$ be a random variable. 
$X$ is thus a sum of two independent random variables: $X = Y + \epsilon$. 
If the decoder ignores $z$, then $Y$ must be constant; thus, the variance of $X$ and $\epsilon$ must be equal.
Since in MFG-VAEs $\epsilon$ is Gaussian, $p(x)$ must also be Gaussian. 
Furthermore, in this case, $p(x)$ will still be learned perfectly.

\section{The Semi-Supervised VAE Training Objective}\label{sec:semisup_det}
We extend VAE model and inference to incorporate partial labels, allowing for some supervision of the latent space dimensions. For this, we use the semi-supervised model first introduced by \cite{Kingma2014} as the ``M2 model''.
We assume the following generative process:
\begin{align}
z &\sim \mathcal{N}(0, I), \quad \epsilon \sim \mathcal{N}(0, \sigma^2_\epsilon \cdot I), \quad y \sim p(y), \quad x | z, y = f_\theta(z, y) + \epsilon
\end{align}
where $y$ is observed only a portion of the time. Inference objective for this model can be written as a sum of two objectives, 
a lower bound for the likelihood of $M$ labeled observations and a lower bound for the likelihood for $N$ unlabeled observations:
\begin{align}
\mathcal{J}(\theta, \phi) &= \sum\limits_{n=1}^N \mathcal{U}(x_n; \theta, \phi)
+ \gamma \cdot \sum\limits_{m=1}^M \mathcal{L}(x_m, y_m; \theta, \phi)
\label{eq:ss-m2-objective-pre}
\end{align}
where $\mathcal{U}$ and $\mathcal{L}$ lower bound $p_\theta(x)$ and $p_\theta(x, y)$, respectively:
\begin{align}
\log p_\theta(x, y) &\geq \underbrace{\mathbb{E}_{q_\phi(z | x, y)} \left\lbrack -\log p_\theta(x | z, y) \right\rbrack  - \log p(y) 
+ D_\text{KL} \left\lbrack q_\phi(z | x, y) || p(z) \right\rbrack}_{ \mathcal{L}(x, y; \theta, \phi) }\\
\log p_\theta(x) &\geq \underbrace{\mathbb{E}_{q_\phi(y | x) q_\phi(z | x)} \left\lbrack -\log p_\theta(x | z, y) \right\rbrack
+ D_\text{KL} \left\lbrack q_\phi(y | x) || p(y) \right\rbrack + D_\text{KL} \left\lbrack q_\phi(z | x) || p(z) \right\rbrack}_{\mathcal{U}(x; \theta, \phi)}
\end{align}
and $\gamma$ controls their relative weight (as done by ~\cite{Siddharth2017}).
When using IWAE, we substitute the IWAE lower bounds for $\mathcal{U}$ and $\mathcal{L}$ as follows:
\begin{align}
\log p_\theta(x, y) &\geq \underbrace{\mathbb{E}_{z_1, \dots, z_S \sim q_\phi(z | x, y)} \left\lbrack \log \frac{1}{S} \frac{p_\theta(x, z_s, y)}{q_\phi(z_s | x, y)} \right\rbrack}_{ \mathcal{L}(x, y; \theta, \phi) }\\
\log p_\theta(x) &\geq \underbrace{\mathbb{E}_{(y_1, z_1), \dots, (y_S, z_S) \sim q_\phi(y | x) q_\phi(z | x)} \left\lbrack \log \frac{1}{S} \sum\limits_{s=1}^S \frac{p_\theta(x, z_s, y_s)}{q_\phi(y_s | x) q_\phi(z_s | x)} \right\rbrack}_{\mathcal{U}(x; \theta, \phi)}
\end{align}

\section{Theorems and Derivations}

\subsection{Proof that under conditions from Section \ref{sec:misestimate-px-conditions}, a VAE will misestimate $p(x)$} \label{sec:path_1}

For completeness, we formalize the intuitive conditions from Section \ref{sec:misestimate-px-conditions} in the theorem below
and prove that under these conditions the global optima of the VAE corresponds to a model that misestimates $p(x)$.
However, we emphasize that just because $p(x)$ is learned incorrectly, it does not mean that the quality of the generative model is \emph{meaningfully} compromised; in fact, there are many conditions that can lead to the ELBO learning a model  that does not recover $p(x)$ exactly, but for which the compromise in the quality of the learned $p(x)$ is imperceptible in downstream tasks. 
In this work, we conjecture that conditions (1) and (2) are necessary for \emph{significant} compromises in the quality of the learned $p(x)$. While we do not provide a proof of this conjecture, in Appendix \ref{sec:verifying_path_cond} we show that  conditions (1) and (2) are satisfied on actual data-sets and provide evidence that when the two conditions are met the learned $p(x)$ differs significantly from the true data distribution, by qualitative and quantitative evaluations. 

To formalize conditions (1) and (2), first recall the decomposition of the negative ELBO in Equation \ref{eq:vae-obj}.
In this discussion, we always set $\phi$ to be optimal for our choice of $\theta$. 
Assuming that $p(x)$ is continuous, then for any $\eta \in \mathbb{R}_{\geq 0}$, we can decompose the PMO as:
\begin{align}
\begin{split}
\mathbb{E}_{p(x)} \left\lbrack D_{\text{KL}} \lbrack q_\phi(z | x) || p_\theta(z | x)\rbrack \right\rbrack =& 
\mathrm{Pr}[\mathcal{X}_{\mathrm{Lo}}(\theta)]\, \mathbb{E}_{p(x)\vert_{\mathcal{X}_{\mathrm{Lo}}}}\left[ D_{\text{KL}} \lbrack q_\phi(z | x) || p_\theta(z | x)\rbrack \right] \\
&+ \mathrm{Pr}[\mathcal{X}_{\mathrm{Hi}}(\theta)]\, \mathbb{E}_{p(x)\vert_{\mathcal{X}_{\mathrm{Hi}}}}\left[ D_{\text{KL}} \lbrack q_\phi(z | x) || p_\theta(z | x)\rbrack \right]
\end{split}\label{eqn:PMO_decomp}
\end{align}
where $D_{\text{KL}} \lbrack q_\phi(z | x) || p_\theta(z | x) \rbrack \leq \eta$ on $\mathcal{X}_{\mathrm{Lo}}(\theta)$, $D_{\text{KL}} \lbrack q_\phi(z | x) || p_\theta(z | x) \rbrack > \eta$ on $\mathcal{X}_{\mathrm{Hi}}(\theta)$, with $\mathcal{X}_i(\theta) \subseteq \mathcal{X}$; where $\mathbb{E}_{p(x)\vert_{\mathcal{X}_{i}}}$ is the expectation over $p(x)$ restricted to $\mathcal{X}_i(\theta)$ and renormalized, and $\mathrm{Pr}[\mathcal{X}_i]$ is the probability of $\mathcal{X}_i(\theta)$ under $p(x)$. Let us denote the expectation in the first term on the right-hand side of Equation \ref{eqn:PMO_decomp} as $D_{\mathrm{Lo}}(\theta)$ and the expectation in the second term as $D_{\mathrm{Hi}}(\theta)$.

Let $f_{\theta_{\mathrm{GT}}} \in \mathcal{F}$ be the ground-truth likelihood function (for which the MLE objective, MLEO, is zero). Now conditions (1) and (2) above may be rewritten more formally as:

\begin{restatable}{theorem}{ThmFailureI} \label{thm:path_1}
Suppose that there is an $\eta \in \mathbb{R}_{\geq 0}$ such that $\mathrm{Pr}[\mathcal{X}_{\mathrm{Hi}}(\theta_{\mathrm{GT}})]\,D_{\mathrm{Hi}}(\theta_{\mathrm{GT}})$ is greater than $\mathrm{Pr}[\mathcal{X}_{\mathrm{Lo}}(\theta_{\mathrm{GT}})]\,D_{\mathrm{Lo}}(\theta_{\mathrm{GT}})$. Suppose the following two conditions: (1) [True posterior often difficult] there exist an $f_\theta\in \mathcal{F}$ with $D_{\mathrm{Lo}}(\theta_{\mathrm{GT}}) \geq D_{\mathrm{Lo}}(\theta)$ and 
\begin{align*}
\begin{split}
\mathrm{Pr}[\mathcal{X}_{\mathrm{Hi}}(\theta_{\mathrm{GT}})]\,(D_{\mathrm{Hi}}(\theta_{\mathrm{GT}}) - D_{\mathrm{Lo}}(\theta_{\mathrm{GT}})) 
> \mathrm{Pr}[\mathcal{X}_{\mathrm{Hi}}(\theta)]\,D_{\mathrm{Hi}}(\theta) + D_{\text{KL}} \lbrack p(x) || p_\theta(x) \rbrack;
\end{split}
\end{align*}
and (2) [No good, simpler alternative] that for no such $f_\theta \in \mathcal{F}$ is the MLEO $D_{\text{KL}} \lbrack p(x) || p_\theta(x) \rbrack$ equal to zero. Then at the global minima $(\theta^*, \phi^*)$ of the negative ELBO, the MLEO will be non-zero. 
\end{restatable}

\begin{proof}

The proof is straightforward. Condition (1) of the theorem implies that the negative ELBO of $f_\theta$ will be lower than that of $f_{\theta_{\mathrm{GT}}}$. That is, we can write:
\begin{align}
-\mathrm{ELBO}(\theta_{\mathrm{GT}}, \phi_{\mathrm{GT}}) &= \mathrm{Pr}[\mathcal{X}_{\mathrm{Hi}}(\theta_{\mathrm{GT}})]\, D_{\mathrm{Hi}}(\theta_{\mathrm{GT}}) + \mathrm{Pr}[\mathcal{X}_{\mathrm{Lo}}(\theta_{\mathrm{GT}})]\, D_{\mathrm{Lo}}(\theta_{\mathrm{GT}}) \\
&= \mathrm{Pr}[\mathcal{X}_{\mathrm{Hi}}(\theta_{\mathrm{GT}})]\, D_{\mathrm{Hi}}(\theta_{\mathrm{GT}}) + (1 - \mathrm{Pr}[\mathcal{X}_{\mathrm{Hi}}(\theta_{\mathrm{GT}})])\, D_{\mathrm{Lo}}(\theta_{\mathrm{GT}}) \\
&= \mathrm{Pr}[\mathcal{X}_{\mathrm{Hi}}(\theta_{\mathrm{GT}})]\ (D_{\mathrm{Hi}}(\theta_{\mathrm{GT}}) -D_{\mathrm{Lo}}(\theta_{\mathrm{GT}})) + D_{\mathrm{Lo}}(\theta_{\mathrm{GT}})\\
&>\underbrace{\mathrm{Pr}[\mathcal{X}_{\mathrm{Hi}}(\theta)]\,D_{\mathrm{Hi}}(\theta) + \mathrm{Pr}[\mathcal{X}_{\mathrm{Lo}}(\theta)]\,D_{\mathrm{Lo}}(\theta) + D_{\text{KL}} \lbrack p(x) || p_\theta(x)}_{-\mathrm{ELBO}(\theta, \phi)}\rbrack
\end{align}
So we have that $-\mathrm{ELBO}(\theta_{\mathrm{GT}}, \phi_{\mathrm{GT}}) > -\mathrm{ELBO}(\theta, \phi)$. Note again that by construction $\phi_{\mathrm{GT}}$ and $\phi$ are both optimal for $\theta_{\mathrm{GT}}$ and $\theta$, respectively. 

Furthermore, if there is an $f_{\theta'}\in \mathcal{F}$ such that $ -\mathrm{ELBO}(\theta', \phi') < -\mathrm{ELBO}(\theta, \phi)$, then it must also satisfy the conditions in assumption (1) and, hence, the global minima of the negative ELBO satisfy the conditions in assumption (1). By assumption (2), at the global minima of the negative ELBO, the MLEO $D_{\text{KL}} \lbrack p(x) || p_\theta(x) \rbrack$ cannot be equal to zero.

\end{proof}

Theorem \ref{thm:path_1} shows that under conditions (1) and (2) 
the ELBO can prefer learning likelihood functions $f_\theta$ that reconstruct $p(x)$ incorrectly, 
even when learning the ground-truth likelihood is possible. 
We again emphasize that this theorem is included for completeness, 
and that the main focus of the paper is to empirically demonstrate that under these two conditions, 
the VAE's performance \emph{significantly} degrades on a variety of downstream tasks,
as well as to provide novel benchmark data-sets that trigger this pathology.

\subsection{Derivation of the observation noise variation that maximizes the ELBO} \label{sec:path_2}
In practice, the noise variance of the data-set is unknown and it is common to estimate the variance as a hyper-parameter. Here, we show that 
learning the variance of $\epsilon$ either via hyper-parameter search or via direct optimization of the ELBO
can be biased. 
We rewrite the negative ELBO:
\begin{align}
\underset{{\sigma^{(d)}_\epsilon}^2}{\mathrm{argmin}} -&\mathrm{ELBO}(\theta, \phi, \sigma^2_\epsilon) \\
&=
\underset{{\sigma^{(d)}_\epsilon}^2}{\mathrm{argmin}} \quad
\mathbb{E}_{p(x)} \left\lbrack \mathbb{E}_{q_\phi(z | x)} \left\lbrack -\log p_\theta(x | z) \right\rbrack 
+ D_\text{KL} \left\lbrack q_\phi(z | x) || p(z) \right\rbrack \right\rbrack \\
&= 
\underset{{\sigma^{(d)}_\epsilon}^2}{\mathrm{argmin}} \quad
\mathbb{E}_{p(x)} \left\lbrack \mathbb{E}_{q_\phi(z | x)} \left\lbrack -\log p_\theta(x | z) \right\rbrack \right\rbrack \\
&= 
\underset{{\sigma^{(d)}_\epsilon}^2}{\mathrm{argmin}} \quad
\mathbb{E}_{p(x)} \left\lbrack \mathbb{E}_{q_\phi(z | x)} \left\lbrack -\sum\limits_{d=1}^D \log \left( \frac{1}{\sqrt{2 \pi {\sigma^{(d)}_\epsilon}^2}} \cdot \exp\left( \frac{-(x^{(d)} - f_\theta(z)^{(d)})^2}{2 {\sigma^{(d)}_\epsilon}^2} \right) \right) \right\rbrack \right\rbrack \\
&= \underset{{\sigma^{(d)}_\epsilon}^2}{\mathrm{argmin}} \quad
 \sum\limits_{d=1}^D 
\mathbb{E}_{p(x)} \left\lbrack \mathbb{E}_{q_\phi(z | x)} \left\lbrack \log \left( \sqrt{2 \pi {\sigma^{(d)}_\epsilon}^2} \right) + \frac{(x^{(d)} - f_\theta(z)^{(d)})^2}{2 {\sigma^{(d)}_\epsilon}^2} \right\rbrack \right\rbrack \\
&=  \underset{{\sigma^{(d)}_\epsilon}^2}{\mathrm{argmin}} \quad
\sum\limits_{d=1}^D 
\mathbb{E}_{p(x)} \left\lbrack \mathbb{E}_{q_\phi(z | x)} \left\lbrack \log \left( {\sigma^{(d)}}_\epsilon \right) + \frac{(x^{(d)} - f_\theta(z)^{(d)})^2}{2 {\sigma^{(d)}_\epsilon}^2} \right\rbrack \right\rbrack \\
&= \underset{{\sigma^{(d)}_\epsilon}^2}{\mathrm{argmin}} \quad
 \sum\limits_{d=1}^D \log \left( {\sigma^{(d)}}_\epsilon \right) + 
\frac{1}{2 {\sigma^{(d)}_\epsilon}^2} \cdot \underbrace{\mathbb{E}_{p(x)} \left\lbrack \mathbb{E}_{q_\phi(z | x)} \left\lbrack (x^{(d)} - f_\theta(z)^{(d)})^2 \right\rbrack \right\rbrack}_{C(\theta, \phi, d)} 
\end{align}
Setting the gradient of the above with respect to $\sigma^2_\epsilon$ equal to zero yields the following:
\begin{align}
0 &= -\frac{\partial}{\partial \sigma^{(d)}_\epsilon} \mathrm{ELBO}(\theta, \phi, \sigma^{(d)}_\epsilon) \\
&= \frac{ {\sigma^{(d)}_\epsilon}^2 - C(\theta, \phi, d)}{{\sigma^{(d)}}^3_\epsilon} .
\end{align}
Thus, we can write,
\begin{align}
{\sigma^{(d)}_\epsilon}^2 = C(\theta, \phi, d) &= \mathbb{E}_{p(x)} \left\lbrack \mathbb{E}_{q_\phi(z | x)} \left\lbrack (x^{(d)} - f_\theta(z)^{(d)})^2 \right\rbrack \right\rbrack \\
&\approx \frac{1}{N} \sum\limits_{n=1}^N \mathbb{E}_{q_\phi(z | x_n)} \left\lbrack (x_n^{(d)} - f_\theta(z)^{(d)})^2 \right\rbrack.
\end{align}

\section{Experimental Details} \label{sec:exp-details}

\paragraph{Initialization at Global Optima of the VAE Objective.}
The decoder function $f_\theta$ is initialized to the ground-truth using full supervision given the
ground-truth $z$'s and $f_{\theta_\text{GT}}$.
The encoder is initialized to $\phi_\text{GT}$ by fixing the decoder at the ground-truth
and maximizing the ELBO (with the 10 random restarts).
We fix the observation error $\sigma^2_\epsilon$ to that of the ground-truth model,
and we fix a sufficiently flexible architecture -- one that is significantly more expressive than needed to capture 
$f_{\theta_{\text{GT}}}$ -- to ensure that, if there exists an $f_\theta$ with simpler posteriors, 
it would be included in our feasible set $\mathcal{F}$.
Lastly, we select the restart that yields the lowest value of the objective function.

\paragraph{Synthetic Data-sets.}
We use 4 synthetic data-sets for unsupervised VAEs (described in Appendix \ref{sec:unsup-examples}),
and 2 synthetic data-sets for semi-supervised VAEs (described in Appendix \ref{sec:ss-examples}),
and generate 5 versions of each data-set (each with $5000/2000/2000$ train/validation/test points).
We use 3 real semi-supervised data-sets: 
Diabetic Retinopathy Debrecen~\citep{UCIDiabetic}, Contraceptive Method Choice~\citep{Keel,UCI}
and the Titanic~\citep{Keel,Titanic} data-sets, each with $10\%$ observed labels,
split in 5 different ways equally into train/validation/test.

\paragraph{Real Data-sets.}
We consider 3 UCI data-sets: 
Diabetic Retinopathy Debrecen~\citep{UCIDiabetic}, Contraceptive Method Choice~\citep{Keel,UCI}
and the Titanic~\citep{Keel,Titanic} data-sets. 
In these, we treat the outcome as a partially observed label (observed $10\%$ of the time).
We split the data 5 different ways into equally sized train/validation/test.
On each split of the data, we run 5 random restarts and select the run that yielded the best
value on the training objective, computed on the validation set. 

\paragraph{Evaluation Metrics.}
To evaluate the quality of the generative model, 
we use the smooth $k$NN test statistic ~\citep{Djolonga2017} on samples from the learned model vs.
samples from the training set / ground-truth model as an alternative to log-likelihood,
since log-likelihood has been shown to be problematic for evaluation because of its numerical instability / high variance~\citep{theis_note_2016,wu_quantitative_2017}.
In the semi-supervised case, we also use the smooth $k$NN test statistic to compare $p(x | y)$ 
with the learned $p_\theta(x | y)$.
Finally, in cases where we may have model mismatch, we also evaluate the mutual information between 
$x$ and each dimension of the latent space $z$,
using the estimator presented in~\citep{Kraskov2004}.

\paragraph{Architectures.}
On the synthetic data-sets, we use a leaky-ReLU encoder/decoder with 3 hidden layers, each 50 nodes.
On the UCI data-sets, we use a leaky-ReLU encoder/decoder with 3 hidden layers, each 100 nodes.

\paragraph{Optimization.}
For optimization, we use the Adam optimizer ~\citep{Adam} with a learning rate of $0.001$
and a mini-batch size of 100.
We train for 100 epochs on synthetic data and for 20000 on real data (and verified convergence).
We trained 5 random restarts on each of the split of the data.
For semi-supervised data-sets with discrete labels,
we used continuous relaxations of the categorical distribution with temperature $2.2$~\citep{Jang2016}
as the variational family in order to use the reparameterization trick~\citep{Kingma2013}. 

\paragraph{Baselines.} 
For our baselines, we compare the performance of aan MFG MFG-VAE
with that of a VAE trained with the Lagging Inference Networks (LIN) algorithm 
(still with an MFG variational family), 
since the algorithm claims to be able to escape local optima in training.
Since the pathologies we describe are global optima, we do not expect LIN to mitigate the issues.
We use Importance Weighted Autoencoders (IWAE) as an example of an inference algorithm 
that uses a more complex variational family.
Since the pathologies described are exacerbated by a limited variational family, 
we expect IWAE to out-perform the other two approaches.
For each method, we select the hyper-parameters for which the best restart
yields the best log-likelihood (using the smooth $k$NN test-statistic, described below).

\paragraph{Hyper-parameters.}
When using IWAE, let $S$ be the number of importance samples used.
When using the Lagging Inference Networks, 
let $T$ be the threshold for determining whether the inference network objective has converged,
and let $R$ be the number of training iterations for which the loss is averaged before comparing with the threshold.
When using semi-supervision, $\alpha$ determines the weight of the discriminator,
and $\gamma$ determines the weight of the labeled objective, $\mathcal{L}$.
We grid-searched over all combination of the following sets of parameters:

\paragraph{Unsupervised data-sets:}
\begin{itemize}
\item IWAE: $S \in \{ 3, 10, 20 \}$
\item Lagging Inference Networks: $T \in \{ 0.05, 0.1 \}, R \in \{ 5, 10 \}$
\end{itemize}

\paragraph{Semi-supervised synthetic data-sets:}
\begin{itemize}
\item IWAE: $S \in \{ 3, 10, 20 \}$
\item Lagging Inference Networks: $T \in \{ 0.05, 0.1 \}, R \in \{ 5, 10 \}$
\item All methods: $\alpha \in \{ 0.0, 0.1, 1.0 \}, \gamma \in \{ 0.5, 1.0, 2.0, 5.0 \}$
\end{itemize}

\paragraph{Semi-supervised real data-sets:}
\begin{itemize}
\item IWAE: $S \in \{ 3, 10, 20 \}$
\item Lagging Inference Networks: $T \in \{ 0.05, 0.1 \}, R \in \{ 5, 10 \}$
\item All methods: $\alpha \in \{ 0.0, 0.1, 1.0 \}$, $\gamma \in \{ 0.5, 1.0, 2.0, 5.0 \}$,
$\sigma^2_\epsilon \in \{ 0.01, 0.5 \}$. On Titanic dimensionality of $z$ is $\in \{ 1, 2 \}$,
on Contraceptive and Diabetic Retinopathy $\in \{ 2, 5 \}$.
\end{itemize}

\paragraph{Hyper-parameter Selection.}
For each method, we selected the hyper-parameters that yielded the smallest
value of the smooth $k$NN test statistic (indicating that they learned the $p(x)$ best).

\section{Quantitative Results}

In this section, we present additional quantitative results for the paper,
following the methodology described in Section \ref{sec:methodology} and experimental setup described in Appendix \ref{sec:exp-details}.

\subsection{Approximation of $p(x)$ is poor when both conditions from Section \ref{sec:misestimate-px-conditions} hold} \label{sec:verifying_path_cond}

Here we show that on data-sets for which the conditions from Section \ref{sec:misestimate-px-conditions} hold, 
VAEs approximate $p(x)$ poorly. 
First, consider the ``Figure-8'' Example in Figure \ref{fig:fig-8-inline} (described in Appendix \ref{sec:fig-8-example}). 
For this data-set, values of $z$ in $[-\infty, -3.0] \cup [3.0, \infty]$ map to similar values of $x$ near $(0,0)$, where $p(x)$ is high. 
We verify that, near $x =(0,0)$, the posteriors $p_{\theta_\text{GT}}(z | x)$ are multi-modal, satisfying condition (1).
We verify condition (2) is satisfied by considering all continuous parameterizations of  the ``Figure-8" curve: 
any such parametrization will result in a function $f_\theta$ for which distant values of $z$ map to similar values near $(0,0)$ and thus the posterior matching objective (PMO) will be high. 
As predicted, the learned generative model approximates $p(x)$ poorly, 
learning posteriors that are simpler than those of the ground-truth model (see Table \ref{tab:unsupervised-sknn}).
Moreover, Figure \ref{fig:vae-fig-8-f} shows exactly how $f_\theta$ was regularized 
to induce simpler posteriors, by curling away from itself so to reduce the number of regions in latent 
space that decodes to the same neighborhood of $x$. 

Next, consider the ``Clusters'' Example in Figure \ref{fig:clusters-inline} (described in Appendix \ref{sec:clusters-example}).
For this data-set, $f_{\theta_\text{GT}}$ is a smooth step-function embedded on a circle.
Regions in which $\d f^{-1}_{\theta_\text{GT}} / dx $ is high (i.e. the steps) correspond to regions in which $p(x)$ is high. 
The interleaving of high-density and low-density regions on the manifold yield a multi-modal posterior
(see Figure \ref{fig:vae-clusters-post-true}). 
Since the majority of points lie in the clusters (and have a multi-modal posterior), condition (1) is satisfied,
and since there does not exist an alternative parameterization for a step-function on a circle,
condition (2) is satisfied. 
As predicted, the learned generative model approximates $p(x)$ poorly (see Table \ref{tab:unsupervised-sknn}).
Figure \ref{fig:vae-clusters-post-learned} shows the learned model reduces the slope of the steps
in order to learn simpler posteriors, thus compromising the learned $p(x)$. 

To show that these issues occur because the MFG variational family over-regularizes the generative model,
we compare VAE with LIN and IWAE (Table \ref{tab:unsupervised-sknn}). 
As expected, IWAE learns $p(x)$ better than LIN, which outperforms the VAE 
(Figure \ref{fig:fig-8-inline}).
Like the VAE, LIN compromises learning the data distribution in order to learn simpler posteriors, since it also uses an MFG variational family.
In contrast, IWAE is able to learn more complex posteriors and thus 
compromises $p(x)$ far less.
However, note that with 20 importance samples, IWAE still does not learn $p(x)$ perfectly.

\begin{table}[!t]
\centering
\begin{tabular}{l|lll}
\hline
Data     & IWAE                       & LIN & VAE               \\ \hline
Clusters & $\bm{0.057 \pm 0.028}$ & $0.347 \pm 0.057$              & $0.361 \pm 0.083$ \\
Figure-8    & $\bm{0.036 \pm 0.013}$ & $0.040 \pm 0.081$              & $0.066 \pm 0.014$ \\ \hline
\end{tabular}
\caption{Comparison unsupervised learned vs. true data distributions via the smooth $k$NN test (lower is better). Hyper-parameters selected via smaller value of the loss function on the validation set.}
\label{tab:unsupervised-sknn}
\end{table}

\subsection{Approximation of $p(x)$ is un-compromised when only one condition from Section \ref{sec:misestimate-px-conditions} holds} 
\label{sec:thm-1-cond-not-sat-quant}

What happens if the portion of observations with highly non-Gaussian posterior is small?
or if there exists an alternative function that explains $p(x)$ well?
Here, we present two benchmarks --
one for which condition (1) is not satisfied (Figure \ref{fig:circle-inline}) and one for which condition (2) is not satisfied (Figure \ref{fig:abs-inline}) --
and demonstrate that in both cases an MFG-VAE estimates $p(x)$ well.

\paragraph{Benchmark: approximation of $p(x)$ may be fine when only condition (2) holds.}
What happens if the observations with highly non-Gaussian posterior were few in number?
Consider the ``Circle'' Example in Figure \ref{fig:circle-inline} (described in Appendix \ref{sec:circle-example}).
Here, the regions that have non-Gaussian posteriors are near $x \approx (1.0, 0.0)$,
since  $z \in [-\infty, -3.0] \cup[3.0, \infty]$ map to points near $(1.0, 0.0)$.
However, since the overall number of such points is small,
the VAE objective does not trade-off capturing $p(x)$ for easy posterior approximation.
Indeed, we see that VAE training is capable of recovering $p(x)$, 
regardless of whether training was initialized randomly or at the ground-truth.

\paragraph{Benchmark: approximation of $p(x)$ may be fine when only condition (1) holds.}
We now study the case where the true posterior has a high PMO for a large portion of $x$'s, 
but there exists an $f_\theta$ in our realizable set $\mathcal{F}$ that approximates $p(x)$ well and has simple posteriors. 
Consider the ``Absolute-Value'' Example visualized in Figure \ref{fig:abs-inline}.
Although the posteriors under the ground-truth generative model are complex, there is an alternative likelihood $f_\theta(z)$ that models $p(x)$ equally well and has simpler posteriors, and this is the model 
selected by the VAE objective, regardless of whether training was initialized randomly or at the ground-truth. 
Details in Appendix \ref{sec:abs-value-example}.

\subsection{VAEs trade-off between generating realistic data and realistic counterfactuals in semi-supervision} \label{sec:semi-sup-quant}

\paragraph{Trade-offs when labels are discrete.}
The trade-off between realistic data and realistic counterfactuals generation is demonstrated in the ``Discrete Semi-Circle'' Example, visualized in Figure \ref{fig:inline-semisup} (details in Appendix \ref{sec:discrete-ss-example}).
The VAE is able to learn the data manifold and distribution well (Figure \ref{fig:discrete-px-vae-inline}).
However, the learned model has a simple posterior
in comparison to the true posterior (Figure \ref{fig:vae-ss-discrete-post-learned}). 
In fact, the learned $f_\theta(z, y)$ is collapsed to the same function for all values of $y$ (Figure \ref{fig:discrete-fn-vae-inline}).
As a result, $p_\theta(x | y) \approx p_\theta(x)$ under the learned model (Figure \ref{fig:discrete-cond-vae-inline}).
We call this phenomenon ``functional collapse''. 
As expected, functional collapse occurs when training with LIN as well (Figure \ref{fig:lin-ss-discrete}).
In contrast, IWAE is able to learn two distinct data conditionals $p_\theta(x | y)$, but it does so at a cost.
\emph{\textbf{Since IWAE does not regularize the generative model, it overfits}} (Figure \ref{fig:discrete-fn-iwae-inline}).
Table \ref{tab:semi-supervised-sknn} shows that IWAE learns $p(x)$ worse than the VAE, 
while Table \ref{tab:semi-supervised-conditional-sknn} shows that it learns $p(x|y)$ significantly better.
We see a similar pattern in the real data-sets (see Tables \ref{tab:real-ss-sknn} and \ref{tab:real-ss-cond-sknn}).

\paragraph{Trade-offs when labels are continuous.}
When $y$ is discrete, we can lower-bound the number of modes of $p_\theta(z | x)$ by the number of distinct values of $y$,
and choose a variational family that is sufficiently expressive.
But when $y$ is continuous, we cannot easily bound the complexity of $p_\theta(z | x)$. 
In this case, we show that the same trade-off between realistic data and realistic counterfactuals exists, 
and that there is an \emph{additional} pathology
introduced by the discriminator $q_\phi(y | x)$ (Equation \ref{eq:ss-m2-objective-pre}).
Consider the ``Continuous Semi-Circle'' Example, visualized in Figure \ref{fig:vae-ss-continuous-fn} (details in Appendix \ref{sec:continuous-ss-example}).
Here, since the posterior $p_\theta(y | x)$ is bimodal,
encouraging the MFG discriminator $q_\phi(y | x)$ to be predictive will 
collapse $f_\theta(z, y)$ to the same function for all $y$ (Figure \ref{fig:vae-ss-continuous-fn}). 
So as we increase $\alpha$ (the priority placed on prediction), our predictive accuracy increases at the cost of collapsing $p_\theta(x | y)$ towards $p_\theta(x)$. The latter will result in low-quality counterfactuals (see Figure \ref{fig:vae-ss-continuous-px-given-y}).
Like in the discrete case, $\gamma$ still controls the tradeoff between realistic data and realistic counterfactuals;
in the continuous case, $\alpha$ \emph{additionally} controls the tradeoff between realistic 
counterfactuals and predictive accuracy.
Table \ref{tab:semi-supervised-conditional-sknn} shows that IWAE is able to learn $p(x)$ better than VAE and LIN, as expected, but \emph{\textbf{the naive addition of the discriminator to IWAE means that it learns $p(x|y)$ no better than the other two models}} (see below for an explanation); that is, with the naive discriminator, just like the VAE and LIN, 
IWAE suffers from functional collapse (see Figure \ref{fig:iwae-ss-continuous}).

\paragraph{Naive adaptation of IWAE for semi-supervision introduces new pathologies.}
The goal of the discriminator is to ensure that the approximate posterior of $y|x$ is predictive,
as it would be under the true posterior.
The approximate posterior implied by the IWAE objective, however,
is not the one given by the IWAE encoder $q_\phi(z | x)$,
and has a rather complex and uninterpretable form, $q_\text{IW}(z | x)$~\citep{cremer_reinterpreting_2017}.
Incorporating the approximate posterior of $y|x$ induced by the IWAE objective into
the semi-supervised objective would require an intractable marginalization of $q_\text{IW}(z | x)$ over $z$. 
Although some work proposes to use with lower bounds \citep{Siddharth2017}
on $q_\phi(z, y | x)$ marginalized over $z$, 
the discriminator in these cases is nonetheless different from the approximate posterior induced by the IWAE objective.  
This may be an additional factor of the poor performance of IWAE in the semi-supervised setting with continuous $y$.

\begin{table}[!h]
\centering
\begin{tabular}{l|lll}
\hline
Data                   & IWAE                       & LIN & VAE                        \\ \hline
Discrete Semi-Circle   & $0.694 \pm 0.096$          & $0.703 \pm 0.315$              & $\bm{0.196 \pm 0.078}$ \\
Continuous Semi-Circle & $\bm{0.015 \pm 0.011}$ & $0.128 \pm 0.094$              & $0.024 \pm 0.014$          \\ \hline
\end{tabular}
\caption{Comparison of semi-supervised learned vs. true data distributions via the smooth $k$NN test (lower is better). Hyper-parameters selected via the smooth $k$NN test-statistic computed on the data marginals.}
\label{tab:semi-supervised-sknn}
\end{table}

\begin{table}[!h]
\centering
\scriptsize
\setlength{\tabcolsep}{2.75pt}
\begin{tabular}{@{}l|llllll@{}}
\toprule
                       & \multicolumn{2}{c}{IWAE}                                & \multicolumn{2}{c}{LIN} & \multicolumn{2}{c}{VAE}                          \\
Data                   & \multicolumn{1}{c}{Cohort 1}                   & \multicolumn{1}{c}{Cohort 2}                    & \multicolumn{1}{c}{Cohort 1}                & \multicolumn{1}{c}{Cohort 2}                & \multicolumn{1}{c}{Cohort 1}                    & \multicolumn{1}{c}{Cohort 2}           \\ \midrule
Discrete Semi-Circle   & $\bm{1.426 \pm 1.261}$ & $\bm{1.698 \pm 0.636}$  & $18.420 \pm 1.220$       & $10.118 \pm 0.996$      & $15.206 \pm 1.200$          & $11.501 \pm 1.300$ \\
Continuous Semi-Circle & $15.951 \pm 3.566$         & $\bm{14.416 \pm 1.402}$ & $15.321 \pm 1.507$       & $17.530 \pm 1.509$      & $\bm{13.128 \pm 0.825}$ & $16.046 \pm 1.019$ \\ \bottomrule
\end{tabular}
\caption{Comparison of semi-supervised learned $p_\theta(x | y)$ with ground-truth $p(x | y)$ via the smooth $k$NN test statistic (smaller is better). Hyper-parameters selected via smallest smooth $k$NN test statistic computed on the data marginals. For the discrete data, the cohorts are $p(x|y=0)$ and $p(x|y=1)$, and for the continuous data, the cohorts are $p(x|y=-3.5)$ and $p(x|y=3.5)$.}
\label{tab:semi-supervised-conditional-sknn}
\end{table}

\begin{table}[!h]
\centering
\begin{tabular}{l|ll}
\hline
                     & IWAE              & VAE               \\ \hline
Diabetic Retinopathy & $3.571 \pm 2.543$ & $6.206 \pm 1.035$ \\
Contraceptive        & $1.740 \pm 0.290$ & $2.147 \pm 0.225$ \\
Titanic              & $2.794 \pm 1.280$ & $1.758 \pm 0.193$ \\ \hline
\end{tabular}
\caption{Comparison of semi-supervised learned vs. true data distributions via the smooth $k$NN test (lower is better). Hyper-parameters selected via the smooth $k$NN test-statistic computed on the data marginals.}
\label{tab:real-ss-sknn}
\end{table}

\begin{table}[]
\centering
\scriptsize
\setlength{\tabcolsep}{2.75pt}
\begin{tabular}{l|lll|lll}
\hline
                     & \multicolumn{3}{c|}{IWAE}                                 & \multicolumn{3}{c}{VAE}                                   \\
                     & Cohort 1          & Cohort 2          & Cohort 3          & Cohort 1          & Cohort 2          & Cohort 3          \\ \hline
Diabetic Retinopathy & $4.240 \pm 1.219$ & $4.357 \pm 3.417$ & N/A               & $5.601 \pm 0.843$ & $8.008 \pm 1.096$ & N/A               \\
Contraceptive        & $7.838 \pm 1.138$ & $5.521 \pm 3.519$ & $6.626 \pm 2.571$ & $5.388 \pm 0.788$ & $4.994 \pm 0.932$ & $3.722 \pm 0.488$ \\
Titanic              & $3.416 \pm 0.965$ & $6.923 \pm 1.924$ & N/A               & $3.730 \pm 0.866$ & $8.572 \pm 1.766$ & N/A               \\ \hline
\end{tabular}
\caption{Comparison of semi-supervised learned vs. true conditional distributions $p(x | y)$ via the smooth $k$NN test (lower is better). Hyper-parameters selected via the smooth $k$NN test-statistic computed on the data marginals. }
\label{tab:real-ss-cond-sknn}
\end{table}

\FloatBarrier
\section{Defense Against Adversarial Perturbations Requires the True Observation Noise and Latent Dimensionality}  \label{sec:adversarial}

\begin{figure*}[h!]
    \centering
    
    \begin{subfigure}[t]{0.01\textwidth}
            \centering
            \small
            \rotatebox[origin=l]{90}{\hspace{90pt}$x_2$}
    \end{subfigure}
    \begin{subfigure}[t]{0.47\textwidth}
    \includegraphics[width=1.0\textwidth]{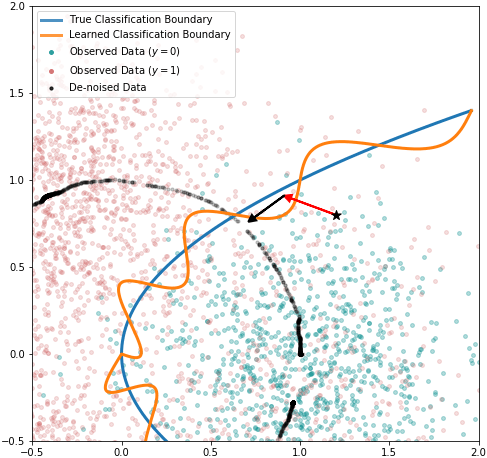} \\
    \centering
    \small
    \vskip -15pt
    $x_1$
    \caption{Projection of adversarial example onto true manifold.}
    \label{fig:adversarial-good}
    \end{subfigure}
    ~
    \begin{subfigure}[t]{0.47\textwidth}
    \includegraphics[width=1.0\textwidth]{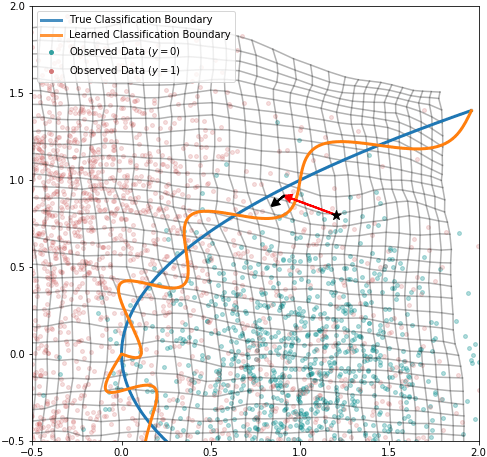}
    \centering
    \small
    \vskip -4pt
    $x_1$
    \caption{Projection of adversarial example onto manifold learned given model mismatch.}
    \label{fig:adversarial-bad}
    \end{subfigure}
    \caption{Comparison of projection of adversarial example onto ground-truth vs. learned manifold. The star represents the original point, perturbed by the red arrow, and then projected onto the manifold by the black arrow.}
    \label{fig:adversarial}
\end{figure*}

As a defense against adversarial attacks, manifold-based approaches de-noise the data before
feeding it into a classifier with the hope that the de-noising will remove the adversarial perturbation from the data
~\citep{Jalal2017,Meng2017,Samangouei2018,Hwang2019,jang_need_2020}.
In this section, we argue that a correct decomposition of the data into $f_\theta(z)$ and $\epsilon$
(or ``signal'' and ``noise'') is necessary to prevent against certain perturbation-based adversarial attacks.

Assume that our data was generated as follows:
\begin{align}
\begin{split}
z &\sim p(z) \\
\epsilon &\sim \mathcal{N}(0, \sigma^2_\epsilon \cdot I) \\
x | z &\sim f_{\theta_\text{GT}}(z) + \epsilon \\
y | z &\sim \text{Cat}\left( g_\psi \circ f_{\theta_\text{GT}}(z) \right)
\end{split}
\end{align}
Let $\mu_\phi(x)$ denote the mean of the encoder and let $M_{\theta, \phi}(x) = f_\theta \circ \mu_\phi(x)$ 
denote a projection onto the manifold.
Our goal is to prevent adversarial attacks on a given discriminative classifier that predicts $y | x$ --
that is, we want to ensure that there does not exist any 
$\eta$ such that $x_n + \eta$ is classified with a different label than $y_n$
by the learned classifier and not by the ground-truth classifier.
Since the labels $y$ are computed as a function of the de-noised data, $f_{\theta_\text{GT}}(z)$,
the true classifier is only defined on the manifold $M$ (marked in blue in Figure \ref{fig:adversarial}).
As such, any learned classifier (in orange) will intersect the true classifier on $M$, 
but may otherwise diverge from it away from the manifold.
This presents a vulnerability against adversarial perturbations, since now any $x$ can be perturbed
to cross the learned classifier's boundary (in orange) to flip its label,
while its true label remains the same, as determined by the true classifier (in blue).
To protect against this vulnerability, existing methods de-noise the data by projecting it onto the manifold
before classifying.
Since the true and learned classifiers intersect on the manifold, 
in order to flip an $x$'s label, the $x$ must be perturbed to cross the true classifier's boundary
(and not just the learned classifier's boundary).
This is illustrated in Figure \ref{fig:adversarial-good}: the black star represents some data point,
perturbed (by the red arrow) by an adversary to cross the learned classifier's boundary
but not the true classifier's boundary.
When projected onto the manifold (by the black arrow), the adversarial attack still falls on the same
side of the true classifier and the learned classifier, rendering the attack unsuccessful
and this method successful.

However, if the manifold is not estimated correctly from the data (i.e. if the ground-truth dimensionality of the latent space and the observation noise $\sigma^2_\epsilon$ are poorly estimated), this defense may fail.
Consider, for example, the case in which $f_\theta(z)$ is modeled with a VAE with
a larger dimensional latent space and a smaller observation noise than the ground-truth model.
Figure \ref{fig:adversarial-bad} shows a uniform grid in $x$'s space projected onto the manifold
learned by this mismatched model.
The figure shows that the learned manifold barely differs from the original space,
since the latent space of the VAE compensates for the observation noise $\epsilon$
and thus does not de-noise the observation. 
When the adversarial attack is projected onto the manifold, it barely moves and is thus left perturbed.
As the figure shows, the attack crosses the learned classifier's boundary but not the true boundary
and is therefore successful.

\FloatBarrier
\section{Unsupervised Pedagogical Examples} \label{sec:unsup-examples}

In this section, we describe in detail the unsupervised pedagogical examples used in the paper
and the properties that cause them to trigger the VAE pathologies. 
For each one of these example decoder functions, we fit a surrogate neural network $f_\theta$
using full supervision (ensuring that the $\mathrm{MSE} < 1\mathrm{e}-4$ 
and use that $f_\theta$ to generate the actual data used in the experiments.

\subsection{Figure-8 Example} \label{sec:fig-8-example}

\paragraph{Generative Process:}
\begin{align}
\begin{split}
z &\sim \mathcal{N}(0, 1) \\
\epsilon &\sim \mathcal{N}(0, \sigma^2_\epsilon \cdot I) \\
u(z) &= \left( 0.6 + 1.8 \cdot \Phi(z) \right) \pi \\
x | z &= \underbrace{
\begin{bmatrix}
\frac{\sqrt{2}}{2} \cdot \frac{\cos(u(z))}{\sin(u(z))^2 + 1} \\
\sqrt{2} \cdot \frac{\cos(u(z)) \sin(u(z))}{\sin(u(z))^2 + 1} \\
\end{bmatrix}
}_{f_{\theta_\text{GT}}(z)} + \epsilon
\end{split}
\label{eq:fig8}
\end{align}
where $\Phi(z)$ is the Gaussian CDF and $\sigma^2_\epsilon = 0.02$ (see Figure \ref{fig:vae-fig-8}). 

\paragraph{Properties:}
In this example, values of $z$ on $[-\infty, -3.0]$, $[3.0, \infty]$ and in small neighborhoods of $z=0$ all produce
similar values of $x$, namely $x\approx 0$; as such, the true posterior $p_{\theta_\text{GT}}(z | x)$ is multi-modal 
in the neighborhood of $x = 0$ (see Figure \ref{fig:vae-fig-8-post-true}), leading to high PMO.
Additionally, in the neighborhood of $x \approx 0$, $p(x)$ is high. 
Thus, condition (1) from Section \ref{sec:misestimate-px-conditions} is satisfied. 
One can verify condition (2) is satisfied by considering all continuous parameterizations of a figure-8 curve. 
Any such parametrization will result in an $f_\theta$ for which far-away values of $z$ lead to nearby values of $x$ and thus in high PMO value for points near $x=0$.

\subsection{Circle Example} \label{sec:circle-example}

\paragraph{Generative Process:}
\begin{align}
\begin{split}
z &\sim \mathcal{N}(0, 1) \\
\epsilon &\sim \mathcal{N}(0, \sigma^2_\epsilon \cdot I) \\
x | z &= \underbrace{
\begin{bmatrix}
\cos (2 \pi \cdot \Phi(z)) \\
\sin (2 \pi \cdot \Phi(z)) \\
\end{bmatrix}
}_{f_{\theta_\text{GT}}(z)} + \epsilon
\end{split}
\label{eq:circle-example}
\end{align}
where $\Phi(z)$ is the Gaussian CDF and $\sigma^2_\epsilon = 0.01$ (see Figure \ref{fig:vae-circle}).

\paragraph{Properties:}
In this example, the regions of the data-space that have a non-Gaussian posterior are near $x \approx [1.0, 0.0]$,
since in that neighborhood, $z \in [-\infty, -3.0]$ and $z \in [3.0, \infty]$ both generate nearby values of $x$.
Thus, this model only satisfies condition (2) from Section \ref{sec:misestimate-px-conditions}.
However, since overall the number of $x$'s for which the posterior is non-Gaussian are few,
the VAE objective does not need to trade-off capturing $p(x)$ for easy posterior approximation.
We see that traditional training is capable of recovering $p(x)$, 
regardless of whether training was initialized randomly or at the ground-truth (see Figure \ref{fig:vae-circle}).

\subsection{Absolute-Value Example} \label{sec:abs-value-example}

\paragraph{Generative Process:}
\begin{align}
\begin{split}
z &\sim \mathcal{N}(0, 1) \\
\epsilon &\sim \mathcal{N}(0, \sigma^2_\epsilon \cdot I) \\
x | z &= \underbrace{
\begin{bmatrix}
|\Phi(z)| \\
|\Phi(z)| \\
\end{bmatrix}
}_{f_{\theta_\text{GT}}(z)} + \epsilon
\end{split}
\label{eq:abs-value-example}
\end{align}
where $\Phi(z)$ is the Gaussian CDF and $\sigma^2_\epsilon = 0.01$ (see Figure \ref{fig:vae-abs}).

\paragraph{Properties:}
In this example, the posterior under $f_{\theta_\text{GT}}$ cannot be well approximated using
an MFG variational family (see Figure \ref{fig:vae-abs-post-true}).
However, there does exist an alternative likelihood function $f_\theta(z)$ (see \ref{fig:vae-abs-f})
that explains $p(x)$ equally well and has simpler posterior \ref{fig:vae-abs-post-learned}.
As such, this model only satisfies condition (1) from Section \ref{sec:misestimate-px-conditions}.

\subsection{Clusters Example} \label{sec:clusters-example}

\paragraph{Generative Process:}
\begin{align}
\begin{split}
z &\sim \mathcal{N}(0, 1) \\
\epsilon &\sim \mathcal{N}(0, \sigma^2_\epsilon \cdot I) \\
u(z) &= \frac{2 \pi}{1 + e^{-\frac{1}{2} \pi z}} \\
t(u) &= 2 \cdot \tanh\left( 10 \cdot u - 20 \cdot \lfloor u / 2 \rfloor - 10 \right) + 4 \cdot \lfloor u / 2 \rfloor + 2 \\
x | z &= \underbrace{
\begin{bmatrix}
\cos(t(u(z))) \\
 \sin(t(u(z))) \\
\end{bmatrix}
}_{f_{\theta_\text{GT}}(z)} + \epsilon
\end{split}
\label{eq:clusters-example}
\end{align}
where $\sigma^2_\epsilon = 0.2$.

\paragraph{Properties:}
In this example, $f_{\theta_\text{GT}}$ is a step function embedded on a circle.
Regions in which $\frac{d f^{-1}_{\theta_\text{GT}}}{dx}$ is high (i.e. the steps) correspond to regions in which $p(x)$ is high. 
The interleaving of high-density and low-density regions on the manifold yield a multi-modal posterior
(see Figure \ref{fig:vae-clusters-post-true}).
For this model, both conditions from Section \ref{sec:misestimate-px-conditions} hold.
In this example, we again see that the VAE objective learns a model with a simpler posterior 
(see Figure \ref{fig:vae-clusters-post-learned}) at the cost of approximating $p(x)$ well
(see Figure \ref{fig:vae-clusters-px}).

%

\subsection{Spiral Dots Example} \label{sec:spiral-dots-example}

\paragraph{Generative Model:}
\begin{align}
\begin{split}
z &\sim \mathcal{N}(0, 1) \\
\epsilon &\sim \mathcal{N}(0, \sigma^2_\epsilon \cdot I) \\
u(z) &= \frac{4 \pi}{1 + e^{-\frac{1}{2} \pi z}} \\
t(u) &= \tanh\left( 10 \cdot u - 20 \cdot \lfloor u / 2 \rfloor - 10 \right) + 2 \cdot \lfloor u / 2 \rfloor + 1 \\
x | z &= \underbrace{
\begin{bmatrix}
t(u(z)) \cdot \cos(t(u(z))) \\
t(u(z)) \cdot \sin(t(u(z))) \\
\end{bmatrix}
}_{f_{\theta_\text{GT}}(z)} + \epsilon
\end{split}
\label{eq:clusters-example}
\end{align}
where $\sigma^2_\epsilon = 0.01$.

\paragraph{Properties:}
In this example, $f_{\theta_\text{GT}}$ a step function embedded on a spiral.
Regions in which $\frac{d f^{-1}_{\theta_\text{GT}}}{dx}$ is high (i.e. the steps) correspond to regions in which $p(x)$ is high. 
The interleaving of high-density and low-density regions on the manifold yield a multi-modal posterior
(see Figure \ref{fig:vae-spiral-dots-post-true}).
In this example, we again see that the VAE objective learns a model with a simpler posterior 
(see Figure \ref{fig:vae-spiral-dots-post-learned}) at the cost of approximating $p(x)$ well
(see Figure \ref{fig:vae-spiral-dots-px}).
Furthermore, for this model the VAE objective highly misestimates the observation noise.

\section{Semi-Supervised Pedagogical Examples} \label{sec:ss-examples}

In this section, we describe in detail the semi-supervised pedagogical examples used in the paper
and the properties that cause them to trigger the VAE pathologies. 
For each one of these example decoder functions, we fit a surrogate neural network $f_\theta$
using full supervision (ensuring that the $\mathrm{MSE} < 1\mathrm{e}-4$ 
and use that $f_\theta$ to generate the actual data used in the experiments.

\subsection{Discrete Semi-Circle Example} \label{sec:discrete-ss-example}

\paragraph{Generative Process:}
\begin{align}
\begin{split}
z &\sim \mathcal{N}(0, 1) \\
y &\sim \text{Bern}\left(\frac{1}{2}\right) \\
\epsilon &\sim \mathcal{N}(0, \sigma^2_\epsilon \cdot I) \\
x | z, y &= \underbrace{
\begin{bmatrix}
\cos\left( \mathbb{I}(y = 0) \cdot \pi \cdot \sqrt{\Phi(z)} + \mathbb{I}(y = 1) \cdot \pi \cdot \Phi(z)^3 \right) \\
\sin\left( \mathbb{I}(y = 0) \cdot \pi \cdot \sqrt{\Phi(z)} + \mathbb{I}(y = 1) \cdot \pi \cdot \Phi(z)^3 \right) \\
\end{bmatrix}
}_{f_{\theta_\text{GT}}(z, y)} + \epsilon
\end{split}
\label{eq:discrete-semi-circle-example}
\end{align}
where $\Phi$ is the CDF of a standard normal and $\sigma^2_\epsilon = 0.01$.

\paragraph{Properties:}
We designed this data-set to specifically showcase issues with the semi-supervised VAE objective.
As such, we made sure that the data marginal $p(x)$ of this example will be
learned well using unsupervised VAE (trained on the $x$'s only) 
This way we can focus on the new issues introduced by the semi-supervised objective.

For this ground-truth model, the posterior of the un-labeled data $p_{\theta_\text{GT}}(z | x)$ is bimodal,
since there are two functions that could have generated each $x$: 
$f_{\theta_\text{GT}}(y = 0, z)$ and $f_{\theta_\text{GT}}(y = 1, z)$.
As such, approximating this posterior with an MFG will encourage
the semi-supervised objective to find a model for which 
$f_{\theta_\text{GT}}(y = 0, z) = f_{\theta_\text{GT}}(y = 1, z)$ (see Figure \ref{fig:vae-ss-discrete-fn}).
When both functions collapse to the same function, $p_\theta(x | y) \approx p_\theta(x)$
(see Figure \ref{fig:vae-ss-discrete-px-given-y}).
This will prevent the learned model from generating realistic counterfactuals.

\subsection{Continuous Semi-Circle Example} \label{sec:continuous-ss-example}

\paragraph{Generative Process:}
\begin{align}
\begin{split}
z &\sim \mathcal{N}(0, 1) \\
y &\sim \mathcal{N}(0, 1) \\
h(y) &= B^{-1}(\Phi(y); 0.2, 0.2) \\
\epsilon &\sim \mathcal{N}(0, \sigma^2_\epsilon \cdot I) \\
x | z, y &= \underbrace{
\begin{bmatrix}
\cos\left( h(y) \cdot \pi \cdot \sqrt{\Phi(z)} + (1 - h(y)) \cdot \pi \cdot \Phi(z)^3 \right) \\
\sin\left( h(y) \cdot \pi \cdot \sqrt{\Phi(z)} + (1 - h(y)) \cdot \pi \cdot \Phi(z)^3 \right) \\
\end{bmatrix}
}_{f_{\theta_\text{GT}}(z, y)} + \epsilon
\end{split}
\label{eq:continuous-semi-circle-example}
\end{align}
where $\Phi$ is the CDF of a standard normal and $B^{-1}(.; \alpha, \beta)$ is the inverse CDF of the beta distribution. 

\paragraph{Properties:} 
As in the ``Discrete Semi-Circle Example'', we designed this data-set to have a $p(x)$ that the VAE objective
would learn well so we can focus on the new issues introduced by the semi-supervised objective.
The data-set demonstrates the same pathologies in the semi-supervised objective
as shown by ``Discrete Semi-Circle Example'' with the addition of yet another pathology:
since the posterior $p_\theta(y | x)$ is bimodal in this example,
encouraging an MFG $q_\phi(y | x)$ discriminator to be predictive will 
collapse $f_\theta(z, y)$ to the same function for all values of $y$ (see Figure \ref{fig:vae-ss-continuous-fn})
As such, as we increase $\alpha$, the better our predictive accuracy will be but
the more $p_\theta(x | y) \rightarrow p_\theta(x)$,
causing the learned model to generate poor quality counterfactuals (see Figure \ref{fig:vae-ss-continuous-px-given-y}).

\section{Qualitative Results} \label{sec:additional-qualitative-results}

\begin{itemize}
\item Qualitative results to support the need for both conditions from Section \ref{sec:misestimate-px-conditions}: Figures \ref{fig:vae-circle}, \ref{fig:vae-abs}.
\item Qualitative demonstration of unsupervised VAE pathologies: Figures \ref{fig:vae-fig-8}, \ref{fig:lin-fig-8}, \ref{fig:iwae-fig-8}, \ref{fig:vae-clusters}, \ref{fig:lin-clusters}, \ref{fig:iwae-clusters}, \ref{fig:vae-spiral-dots}. 
\item Qualitative demonstration of semi-supervised VAE pathologies: Figures \ref{fig:vae-ss-discrete}, \ref{fig:lin-ss-discrete}, \ref{fig:iwae-ss-discrete}, \ref{fig:vae-ss-continuous}, \ref{fig:lin-ss-continuous}, \ref{fig:iwae-ss-continuous}. 
\item When learning compressed representations, posterior is simper for mismatched models: Figures \ref{fig:clusters-mismatch-5d}, \ref{fig:fig-8-mismatch-5d}. 
\end{itemize}

\FloatBarrier

\begin{figure*}[p]
    \centering
    \vspace*{-1cm}
    \tiny
    
    \begin{subfigure}[t]{0.55\textwidth}
    \includegraphics[width=1.0\textwidth]{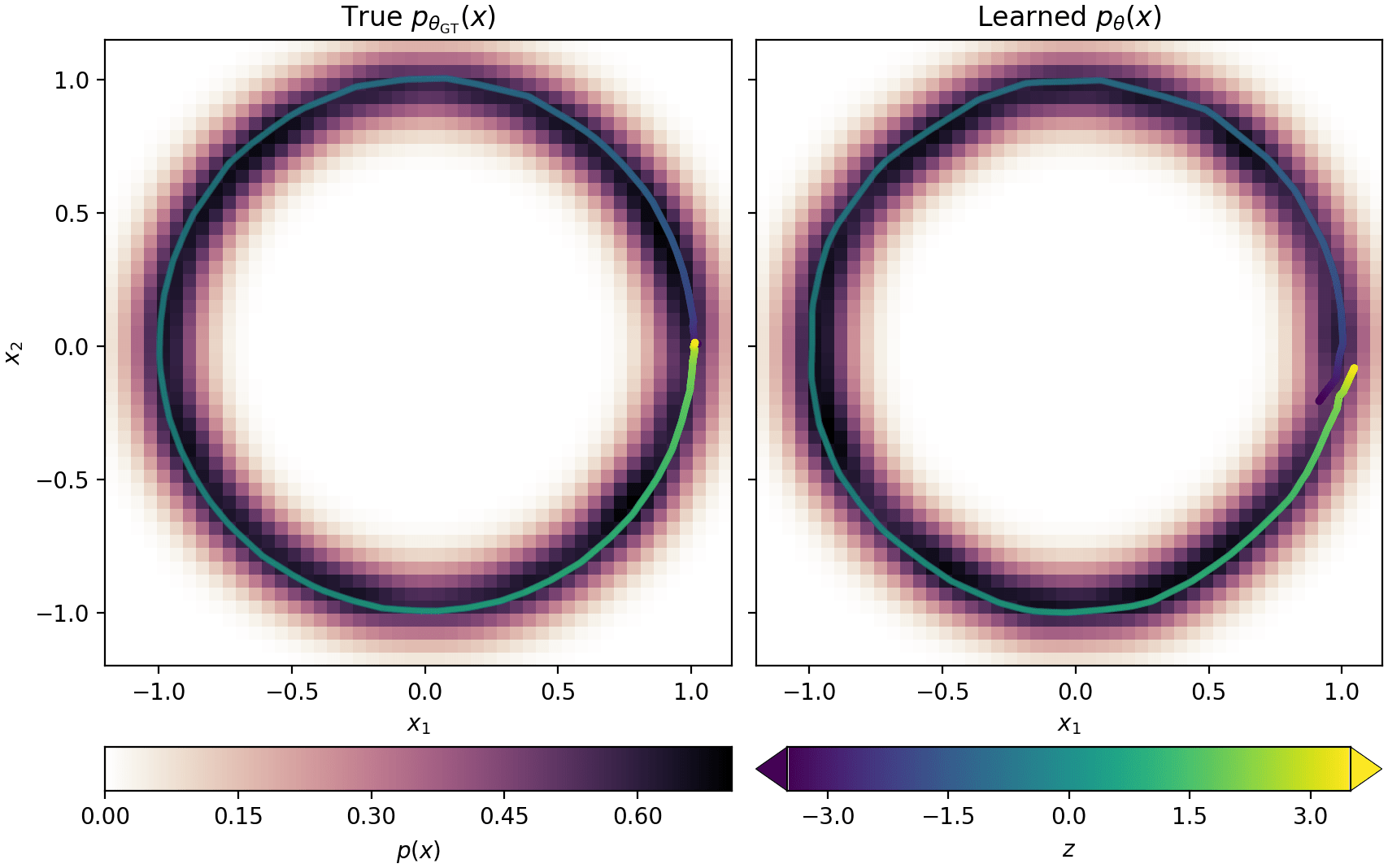}
    \caption{True vs. learned $p_\theta(x)$, and learned vs. true $f_\theta(z)$, colored by the value of $z$.}
    \label{fig:vae-circle-px}
    \end{subfigure}
    
    \begin{subfigure}[t]{0.35\textwidth}
    \includegraphics[width=1.0\textwidth]{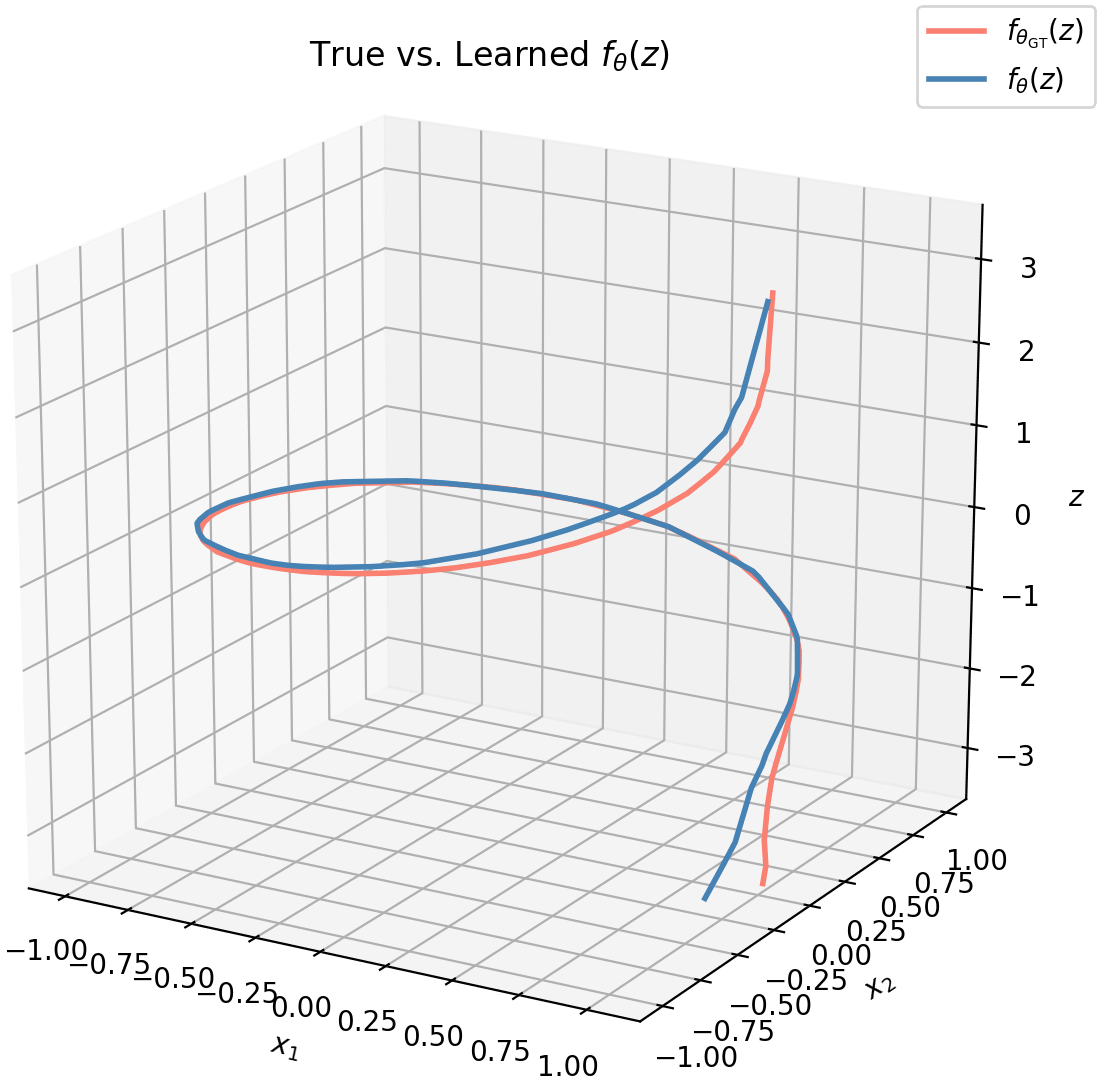}
    \caption{True vs. learned $f_\theta(x)$}
    \label{fig:vae-circle-f}
    \end{subfigure}
    ~
    \begin{subfigure}[t]{0.7\textwidth}
    \includegraphics[width=1.0\textwidth]{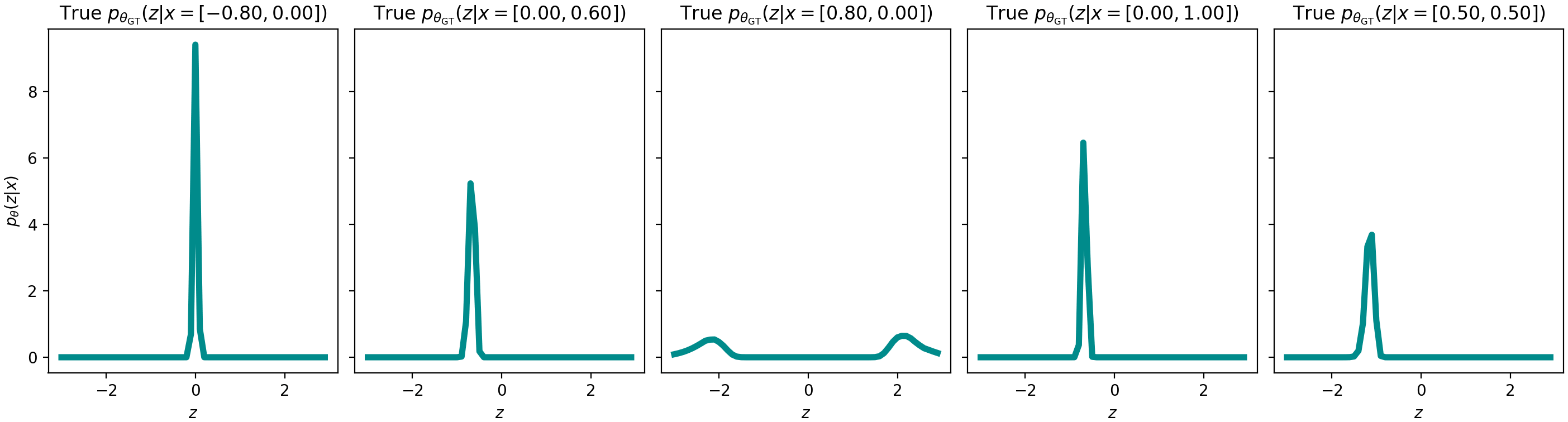}
    \caption{Posteriors under true $f_\theta$}
    \label{fig:vae-circle-post-true}
    \end{subfigure}
    ~
     \begin{subfigure}[t]{0.7\textwidth}
    \includegraphics[width=1.0\textwidth]{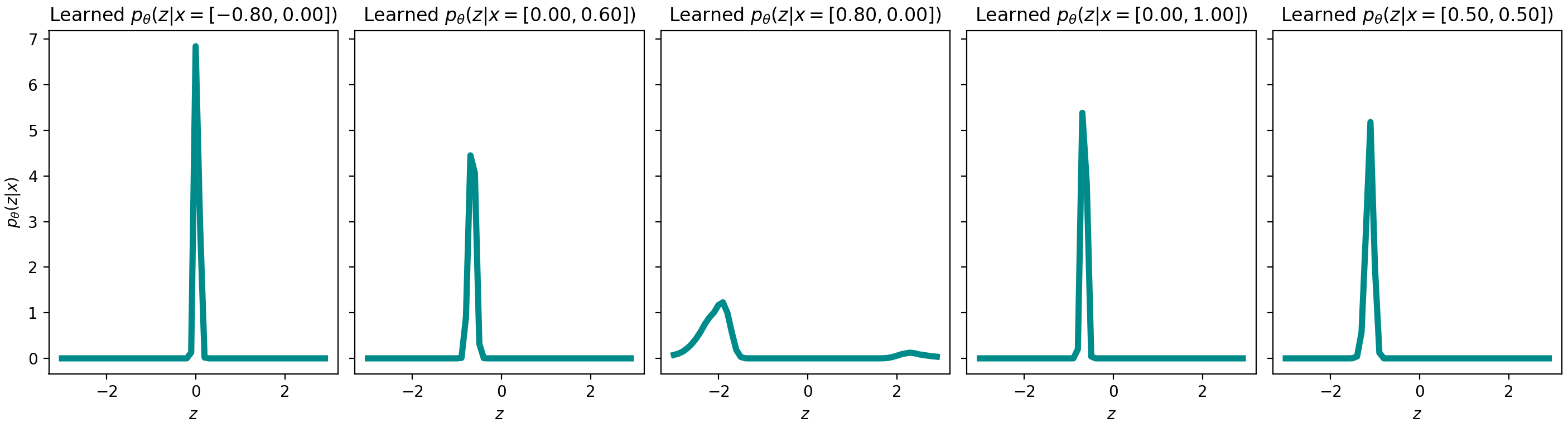}
    \caption{Posteriors under learned $f_\theta$}
    \label{fig:vae-circle-post-learned}
    \end{subfigure}
    \caption{MFG-VAE trained on the Circle Example. In this toy data, condition (2) holds from Section \ref{sec:misestimate-px-conditions} holds and condition (1) does not.
    To see this, notice that most examples of the posteriors are Gaussian-like, with the exception
    of the posteriors near $x = [1.0, 0.0]$, which are bimodal since in that neighborhood, 
    $x$ could have been generated using either $z > 3.0$ or using $z < -3.0$.
    Since only a few training points have a high posterior matching objective,
    a VAE is able to learn the data distribution well.}
    \label{fig:vae-circle}
\end{figure*}

\begin{figure*}[p]
    \centering
    \vspace*{-1cm}
    \tiny
    
    \begin{subfigure}[t]{0.55\textwidth}
    \includegraphics[width=1.0\textwidth]{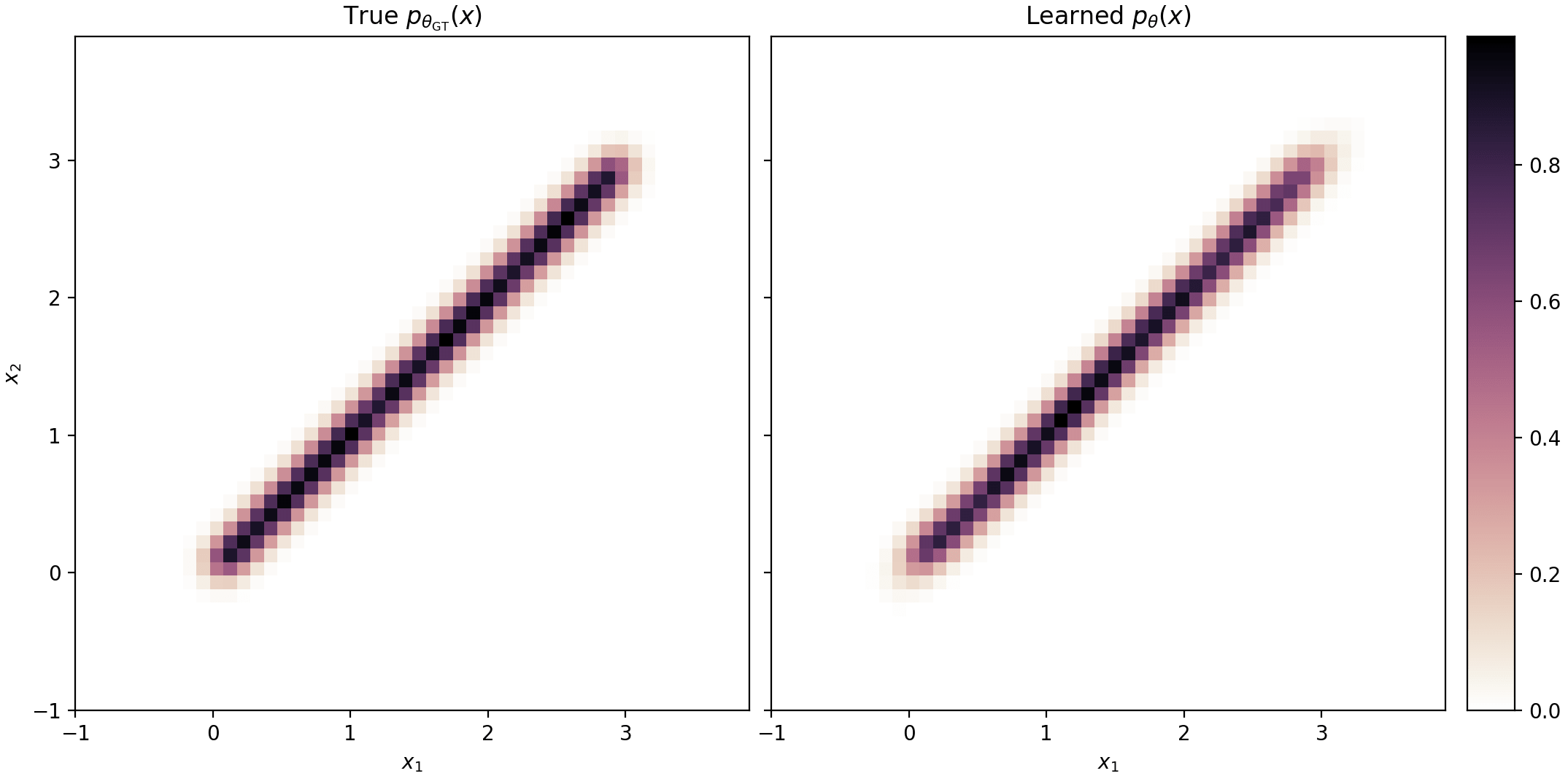}
    \caption{True vs. learned $p_\theta(x)$}
    \label{fig:vae-abs-px}
    \end{subfigure}
    ~
    \begin{subfigure}[t]{0.35\textwidth}
    \includegraphics[width=1.0\textwidth]{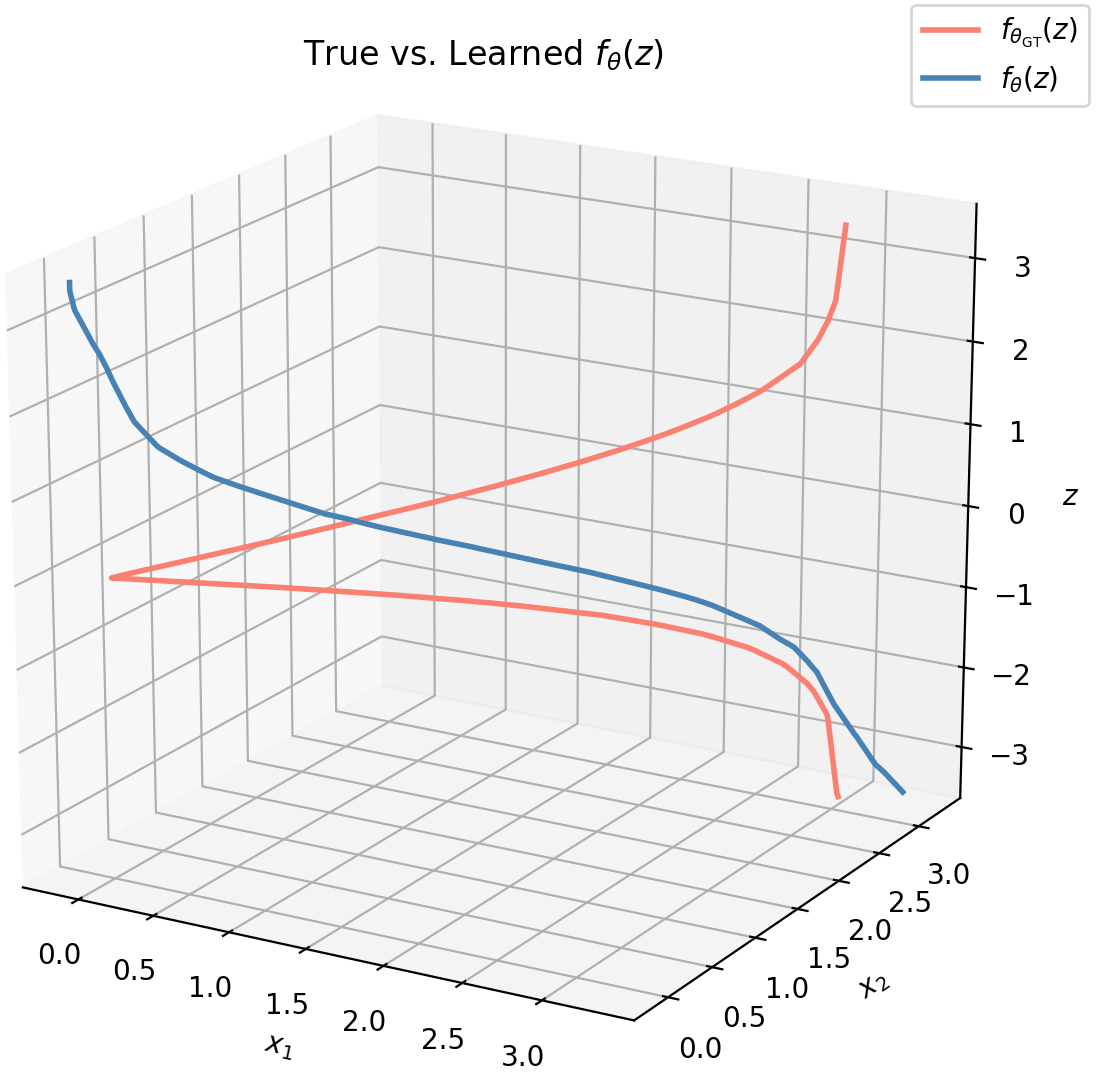}
    \caption{True vs. learned $f_\theta(x)$}
    \label{fig:vae-abs-f}
    \end{subfigure}an MFG
    ~
    \begin{subfigure}[t]{0.7\textwidth}
    \includegraphics[width=1.0\textwidth]{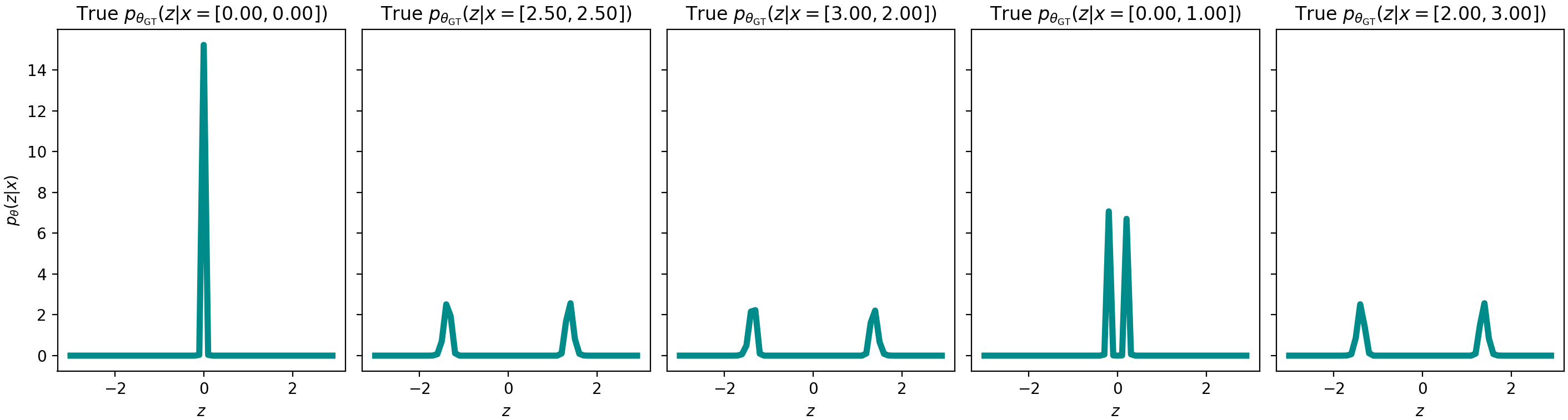}
    \caption{Posteriors under true $f_\theta$}
    \label{fig:vae-abs-post-true}
    \end{subfigure}
    ~
     \begin{subfigure}[t]{0.7\textwidth}
    \includegraphics[width=1.0\textwidth]{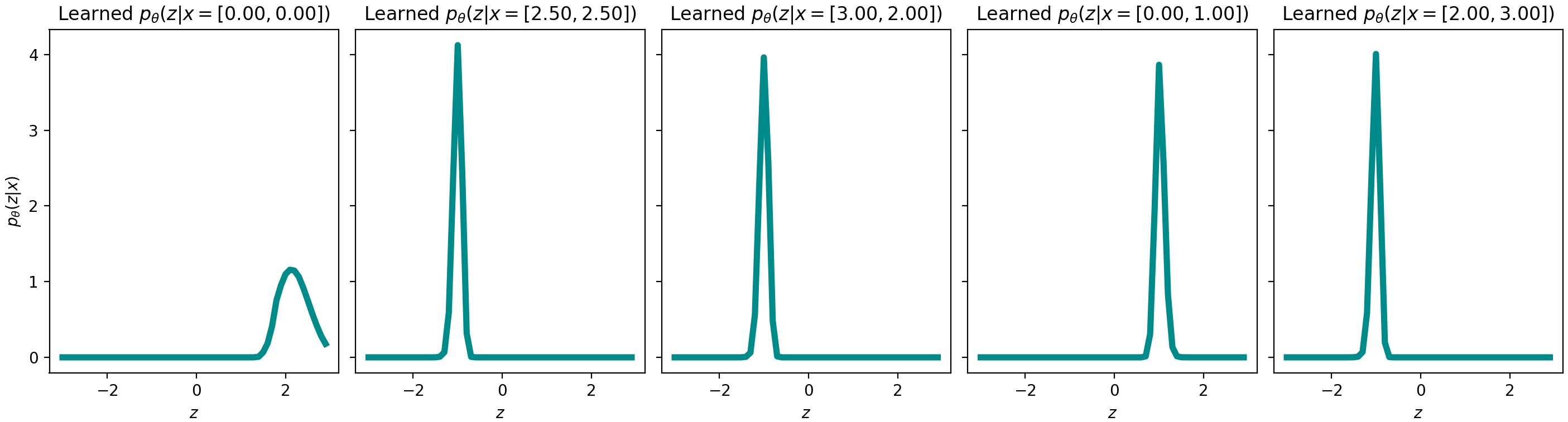}
    \caption{Posteriors under learned $f_\theta$}
    \label{fig:vae-abs-post-learned}
    \end{subfigure}
    \caption{MFG-VAE trained on the Absolute-Value Example. In this toy data, condition (1) from Section \ref{sec:misestimate-px-conditions} holds and condition (2) does not. To see this, notice that the function $f_\theta$ learned with a VAE is completely different than the ground-truth
    $f_\theta$, and unlike the ground-truth $f_\theta$ which has bimodal posteriors, the learned $f_\theta$ has unimodal posteriors (which are easier to approximate with an MFG). As such, a VAE is able to learn the data distribution well.}
    \label{fig:vae-abs}
\end{figure*}

\begin{figure*}[p]
    \centering
    \vspace*{-1cm}
    \tiny
    
    \begin{subfigure}[t]{0.55\textwidth}
    \includegraphics[width=1.0\textwidth]{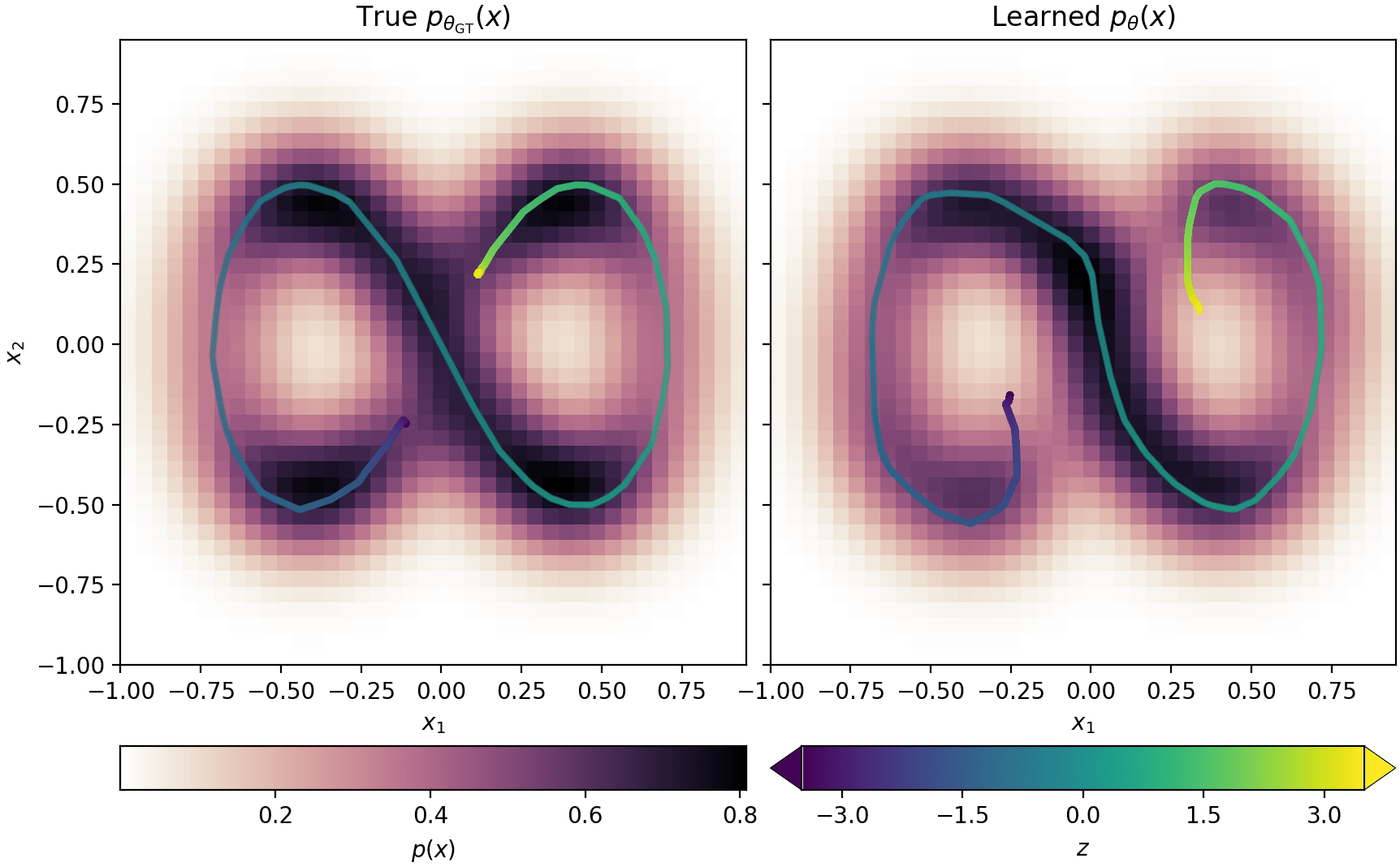}    
    \caption{True vs. learned $p_\theta(x)$, and learned vs. true $f_\theta(z)$, colored by the value of $z$.}
    \label{fig:vae-fig-8-px}
    \end{subfigure} 
    ~
    \begin{subfigure}[t]{0.35\textwidth}
    \includegraphics[width=1.0\textwidth]{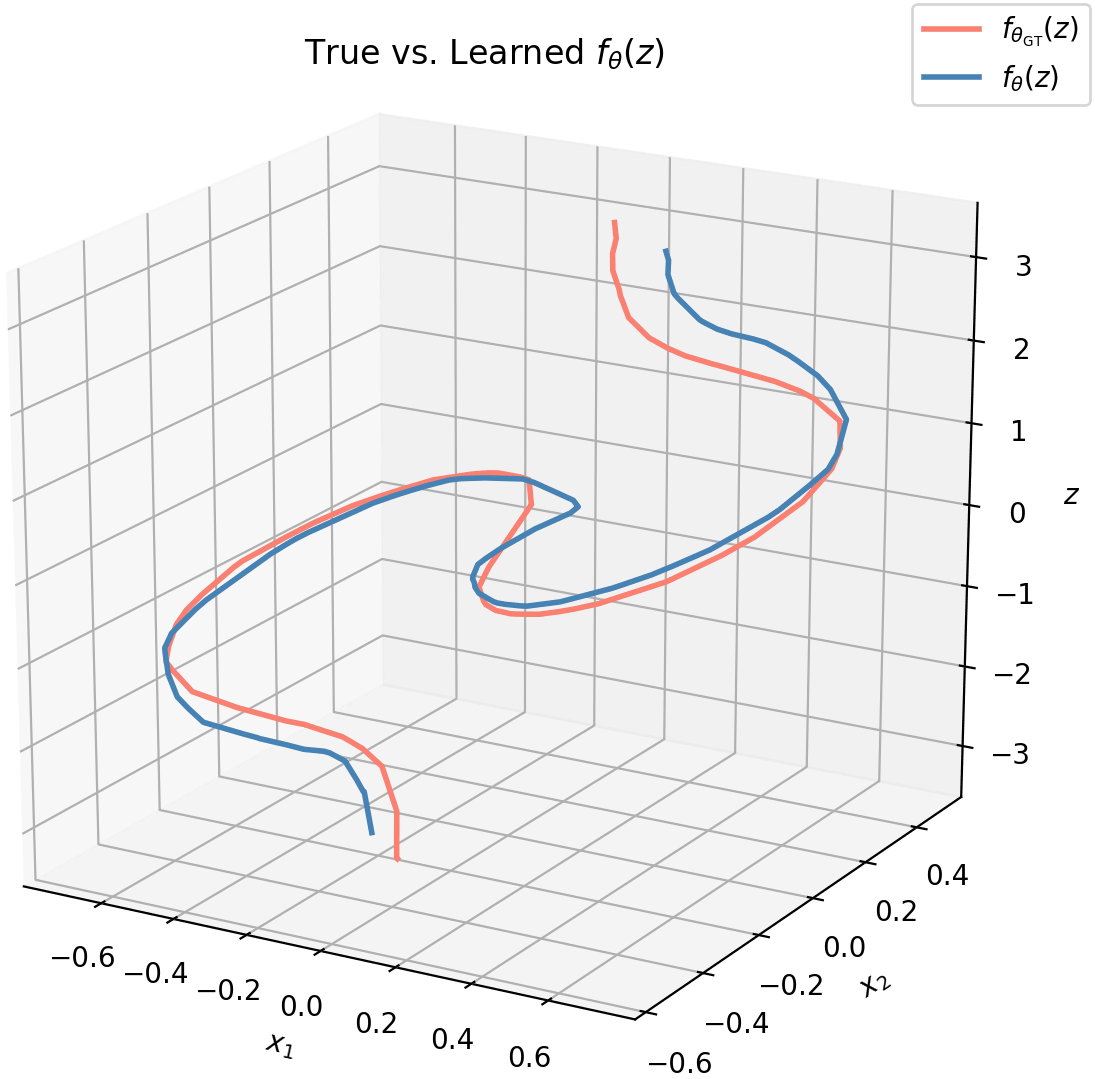}
    \caption{True vs. learned $f_\theta(x)$}
    \label{fig:vae-fig-8-f}
    \end{subfigure}
    ~
    \begin{subfigure}[t]{0.7\textwidth}
    \includegraphics[width=1.0\textwidth]{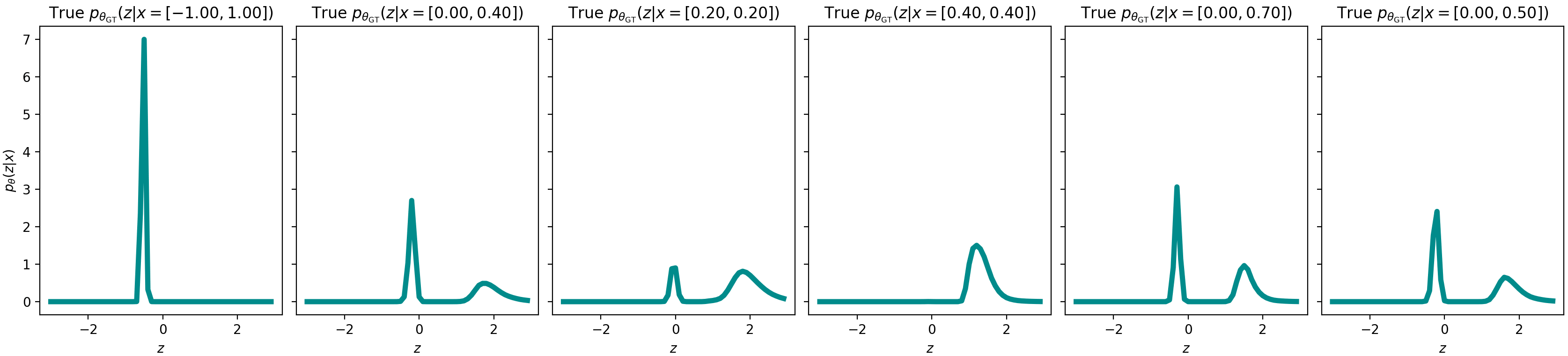}
    \caption{Posteriors under true $f_\theta$}    
    \label{fig:vae-fig-8-post-true}
    \end{subfigure}
    ~
     \begin{subfigure}[t]{0.7\textwidth}
    \includegraphics[width=1.0\textwidth]{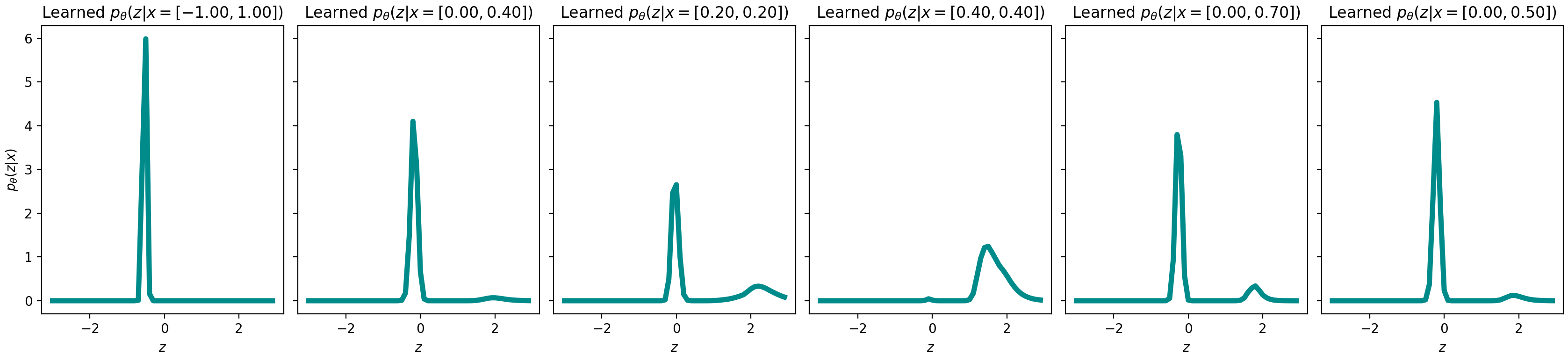}
    \caption{Posteriors under learned $f_\theta$}    
    \label{fig:vae-fig-8-post-learned}
    \end{subfigure}
    \caption{MFG-VAE trained on the Figure-8 Example. In this toy data, both conditions from Section \ref{sec:misestimate-px-conditions} hold.
    The VAE learns a generative model with simpler posterior than that of the ground-truth, 
    though it is unable to completely simplify the posterior as in the Absolute-Value Example.
    To learn a generative model with a simpler posterior, it curves the learned function $f_\theta$ at $z = -3.0$
    and $z = 3.0$ away from the region where $z = 0$.
    This is because under the true generative model, the true posterior $p_\theta(z | x)$ in the neighborhood 
    of $x \approx 0$ has modes around either $z = 0$ and $z = 3.0$, or around $z = 0$ and $z = -3.0$,
    leading to a high posterior matching objective.}
    \label{fig:vae-fig-8}
\end{figure*}

\begin{figure*}[p]
    \centering
    \vspace*{-1cm}
    \tiny
    
    \begin{subfigure}[t]{0.55\textwidth}
    \includegraphics[width=1.0\textwidth]{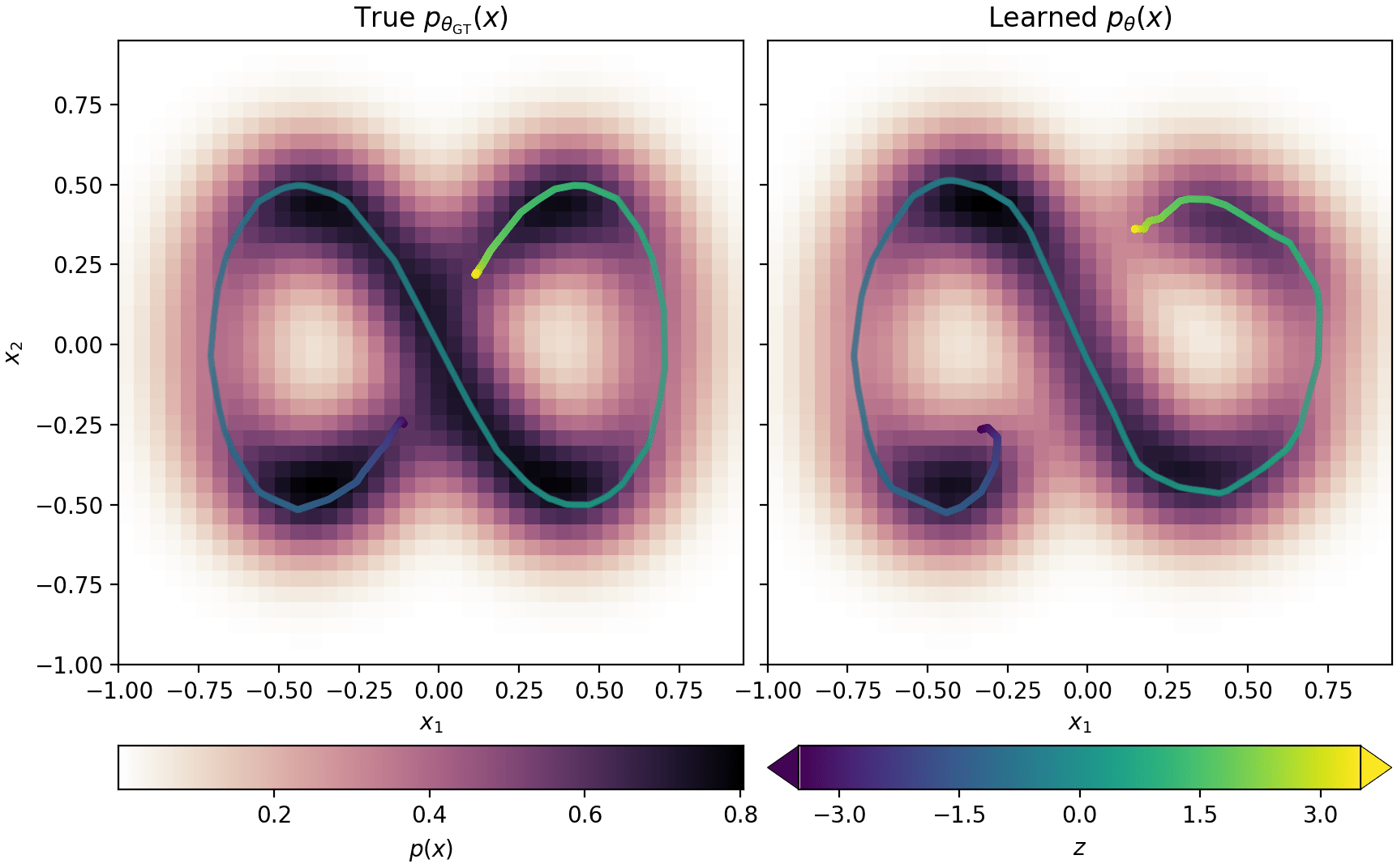}
    \caption{True vs. learned $p_\theta(x)$, and learned vs. true $f_\theta(z)$, colored by the value of $z$.}
    \label{fig:lin-fig-8-px}
    \end{subfigure}
    ~
    \begin{subfigure}[t]{0.35\textwidth}
    \includegraphics[width=1.0\textwidth]{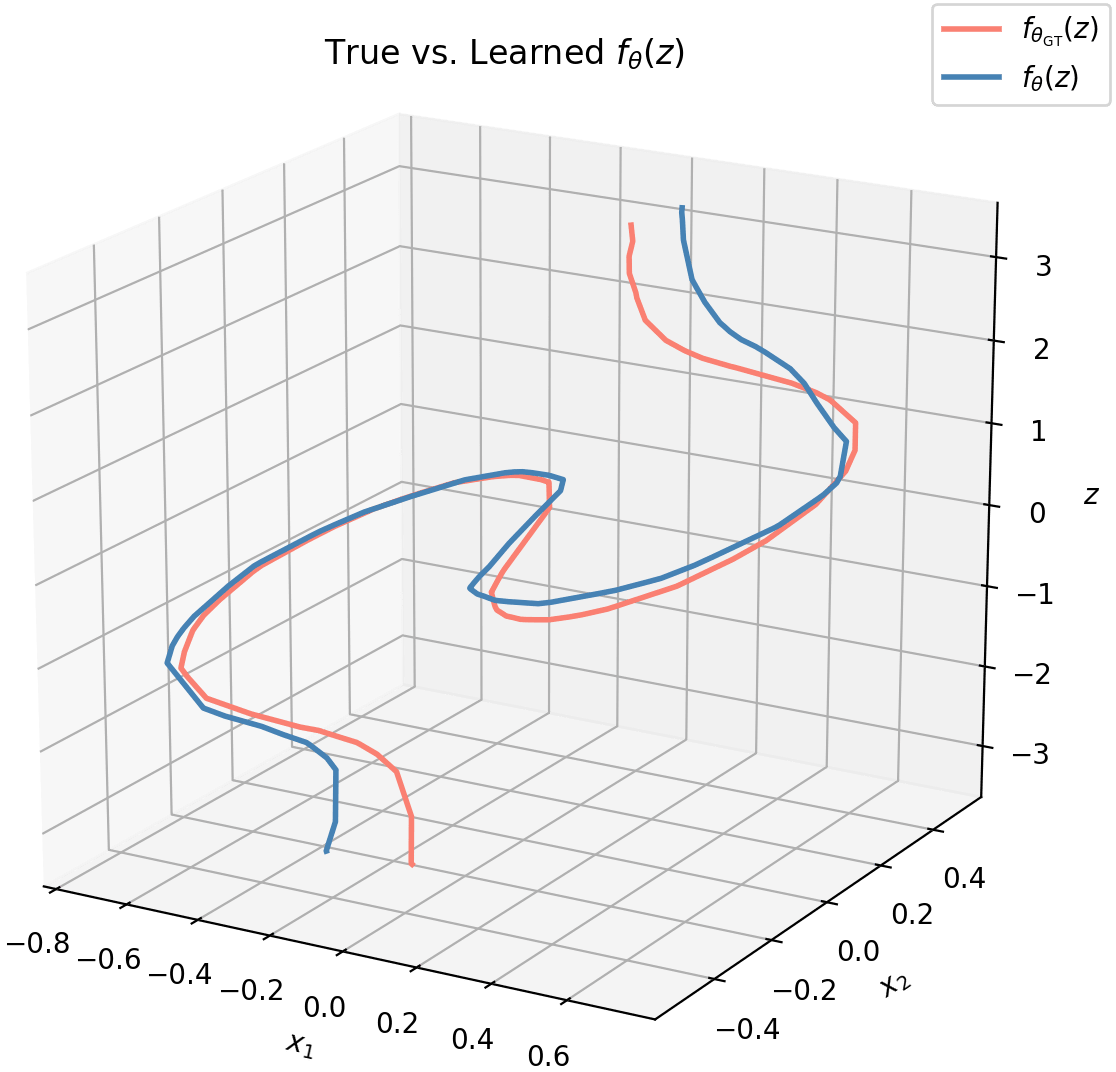}
    \caption{True vs. learned $f_\theta(x)$}
    \label{fig:lin-fig-8-f}
    \end{subfigure}
    ~
    \begin{subfigure}[t]{0.7\textwidth}
    \includegraphics[width=1.0\textwidth]{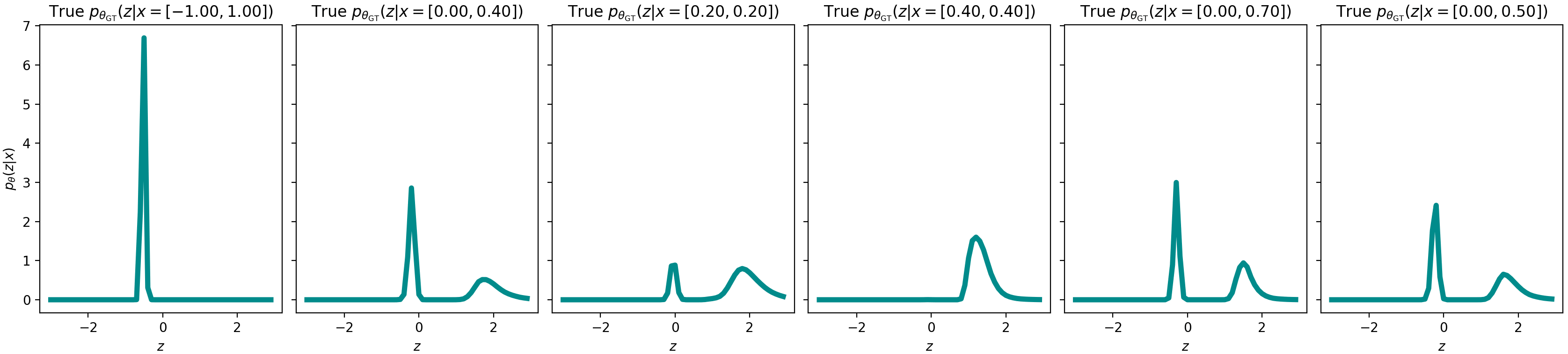}
    \caption{Posteriors under true $f_\theta$}
    \label{fig:lin-fig-8-post-true}
    \end{subfigure}
    ~
     \begin{subfigure}[t]{0.7\textwidth}
    \includegraphics[width=1.0\textwidth]{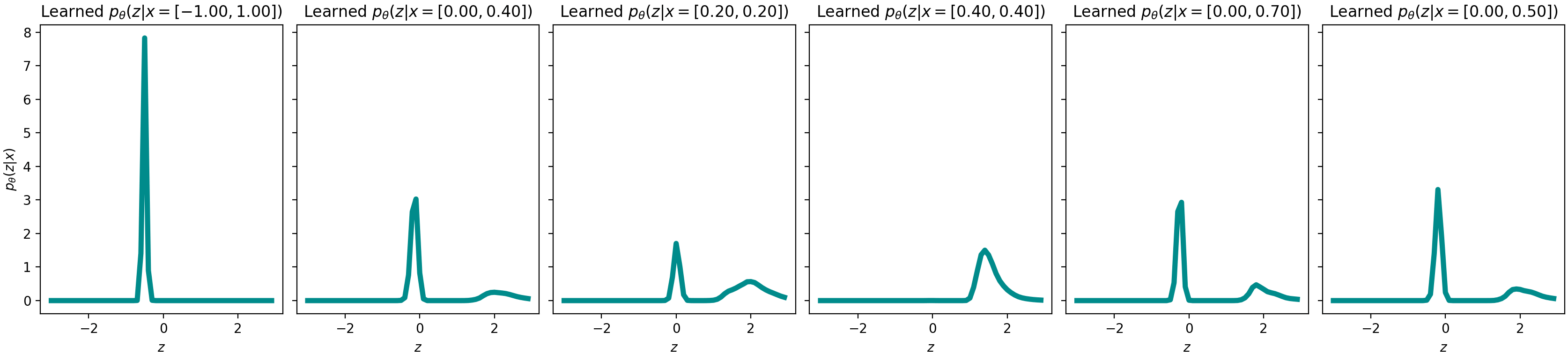}
    \caption{Posteriors under learned $f_\theta$}
    \label{fig:lin-fig-8-post-learned}
    \end{subfigure}
    \caption{VAE with Lagging Inference Networks (LIN) trained on the Figure-8 Example. 
    While LIN may help escape local optima, on this data, the training objective is still biased away
    from learning the true data distribution.
    As such, LIN fails in the same way an MFG-VAE does (see Figure \ref{fig:vae-fig-8}).}
    \label{fig:lin-fig-8}
\end{figure*}

\begin{figure*}[p]
    \centering
    \vspace*{-1cm}
    \tiny
    
    \begin{subfigure}[t]{0.55\textwidth}
    \includegraphics[width=1.0\textwidth]{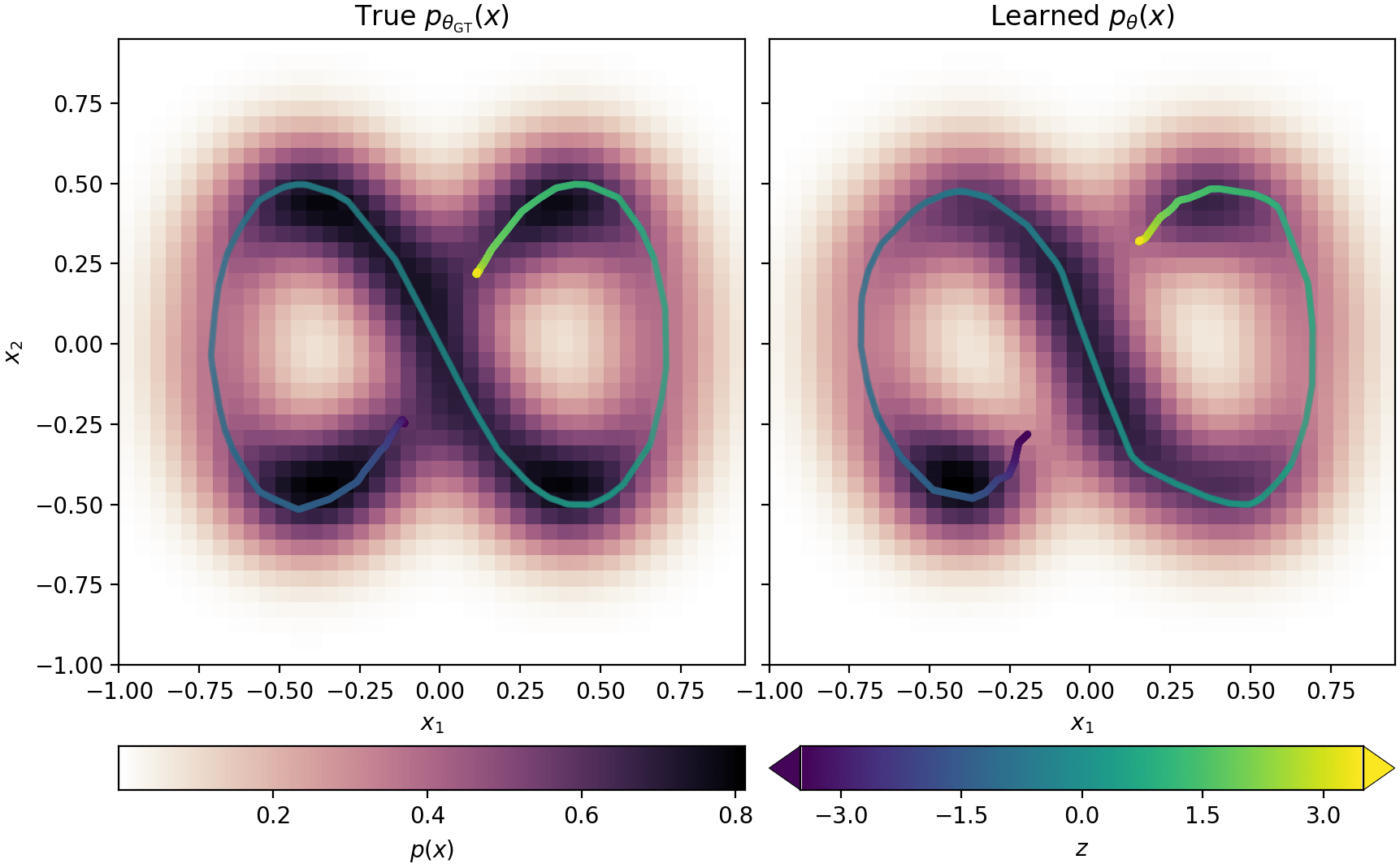}
    \caption{True vs. learned $p_\theta(x)$, and learned vs. true $f_\theta(z)$, colored by the value of $z$.}
    \label{fig:iwae-fig-8-px}
    \end{subfigure}
    ~
    \begin{subfigure}[t]{0.35\textwidth}
    \includegraphics[width=1.0\textwidth]{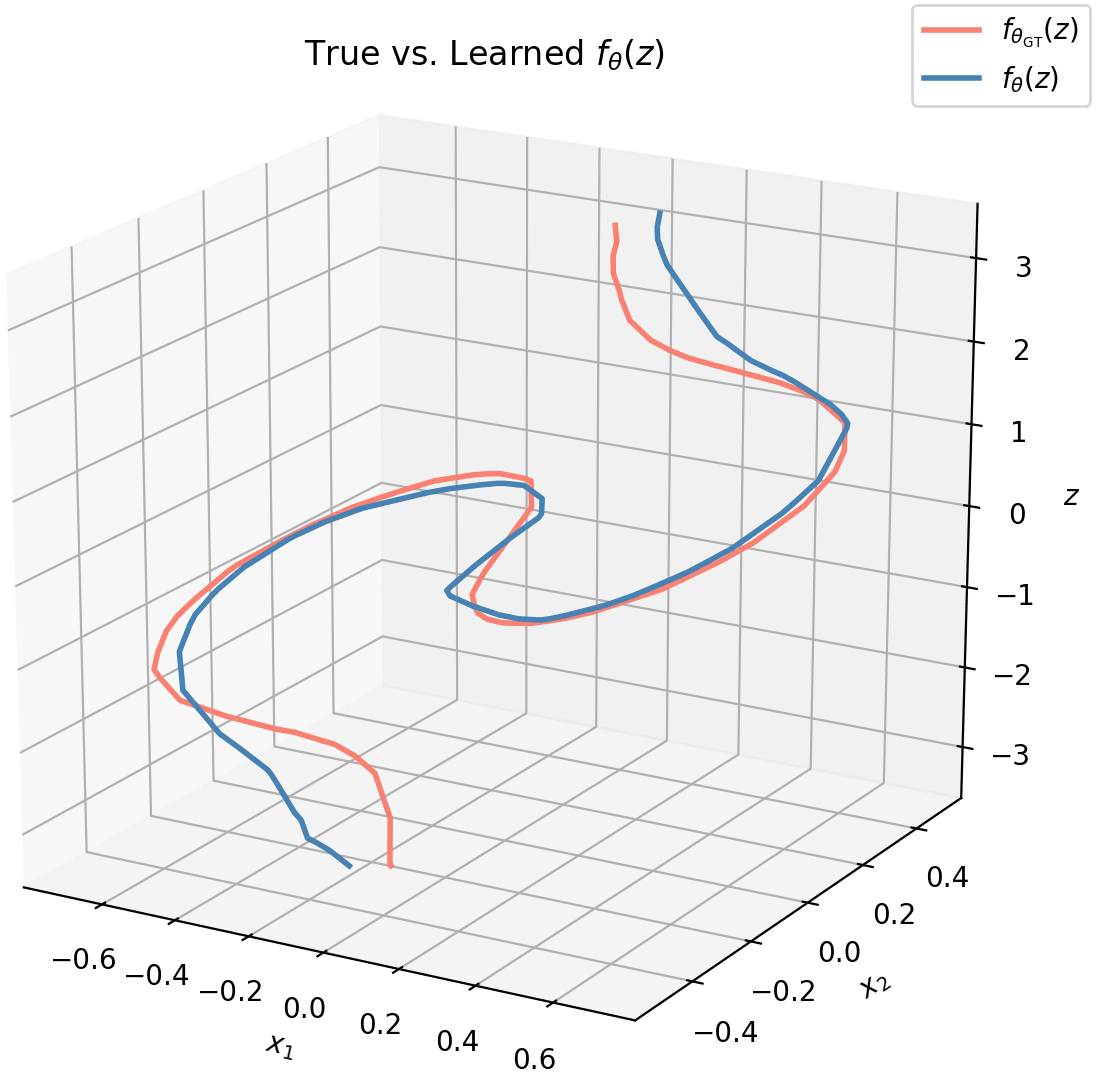}
    \caption{True vs. learned $f_\theta(x)$}
    \label{fig:iwae-fig-8-f}
    \end{subfigure}
    ~
    \begin{subfigure}[t]{0.7\textwidth}
    \includegraphics[width=1.0\textwidth]{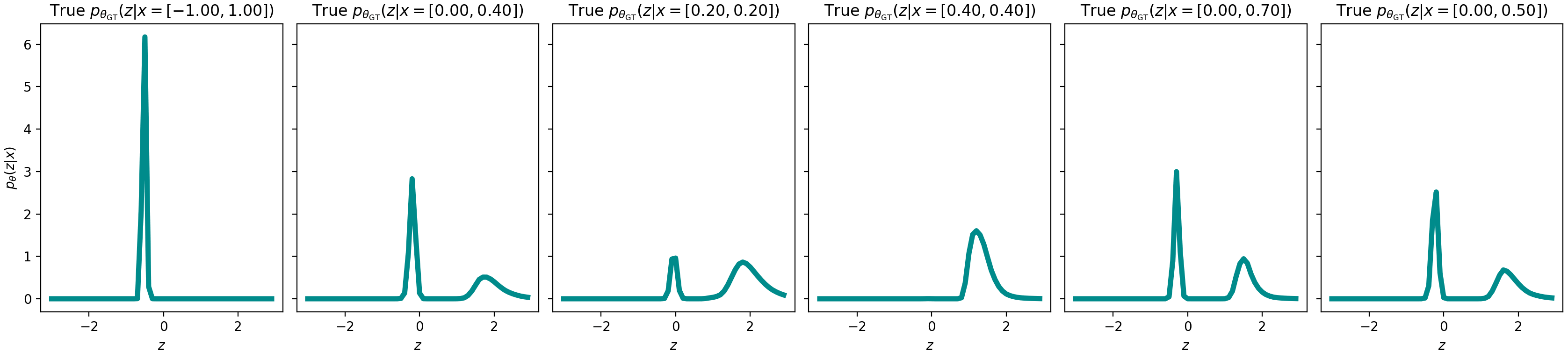}
    \caption{Posteriors under true $f_\theta$}
    \label{fig:iwae-fig-8-post-true}
    \end{subfigure}
    ~
     \begin{subfigure}[t]{0.7\textwidth}
    \includegraphics[width=1.0\textwidth]{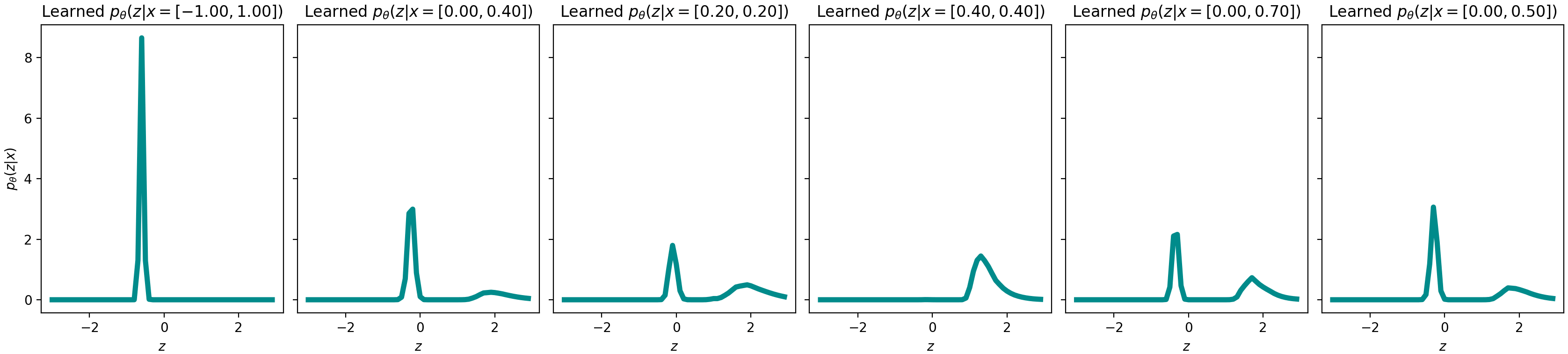}
    \caption{Posteriors under learned $f_\theta$}
    \label{fig:iwae-fig-8-post-learned}
    \end{subfigure}
    \caption{IWAE trained on the Figure-8 Example. In this toy data, both conditions from Section \ref{sec:misestimate-px-conditions} hold.
    The IWAE learns a generative model with a slightly simpler posterior than that of the ground-truth.
    This is because even with the number of importance samples as large as $S = 20$,
    the variational family implied by the IWAE objective is not sufficiently expressive.
    The objective therefore prefers to learn a model with a lower data marginal likelihood.
    While increasing $S \rightarrow \infty$ will resolve this issue, it is not clear how large a $S$ is necessary
    and whether the additional computational overhead is worth it.
    }
    \label{fig:iwae-fig-8}
\end{figure*}

\begin{figure*}[p]
    \centering
    \vspace*{-1cm}
    \tiny
    
    \begin{subfigure}[t]{0.55\textwidth}
    \includegraphics[width=1.0\textwidth]{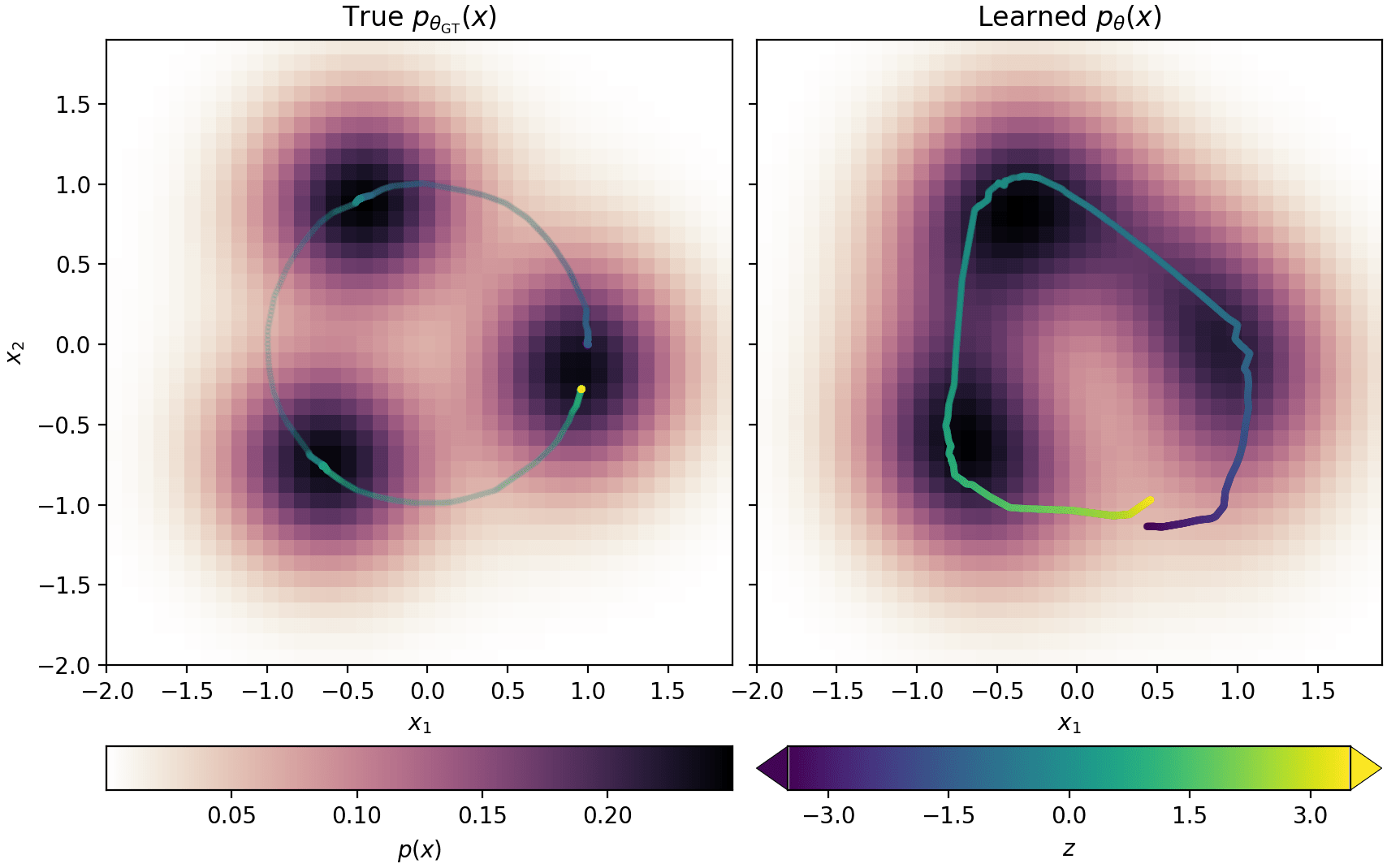}
    \caption{True vs. learned $p_\theta(x)$, and learned vs. true $f_\theta(z)$, colored by the value of $z$.}
    \label{fig:vae-clusters-px}
    \end{subfigure}
    ~
    \begin{subfigure}[t]{0.35\textwidth}
    \includegraphics[width=1.0\textwidth]{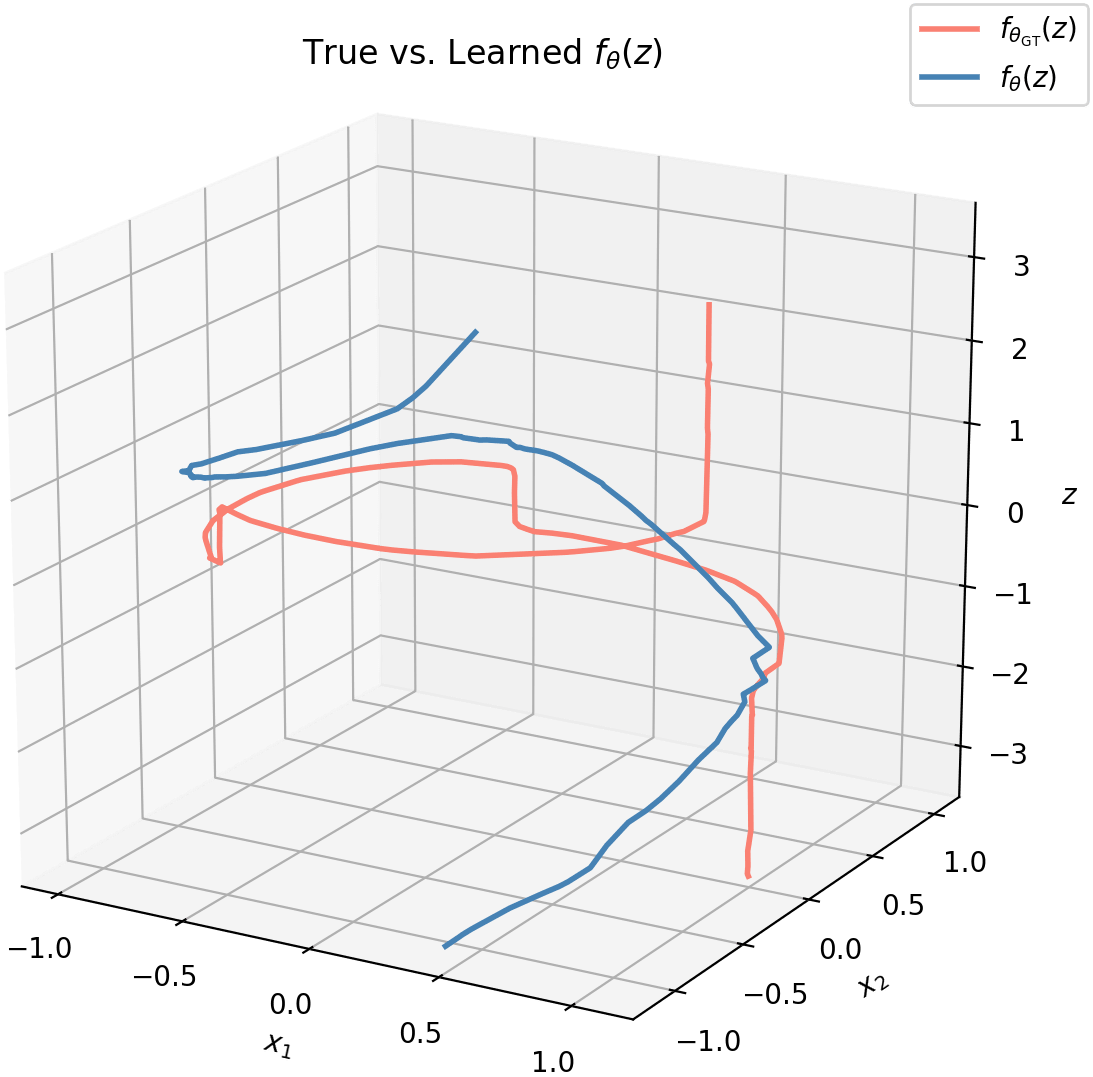}
    \caption{True vs. learned $f_\theta(x)$}
    \label{fig:vae-clusters-f}
    \end{subfigure}
    ~
    \begin{subfigure}[t]{0.7\textwidth}
    \includegraphics[width=1.0\textwidth]{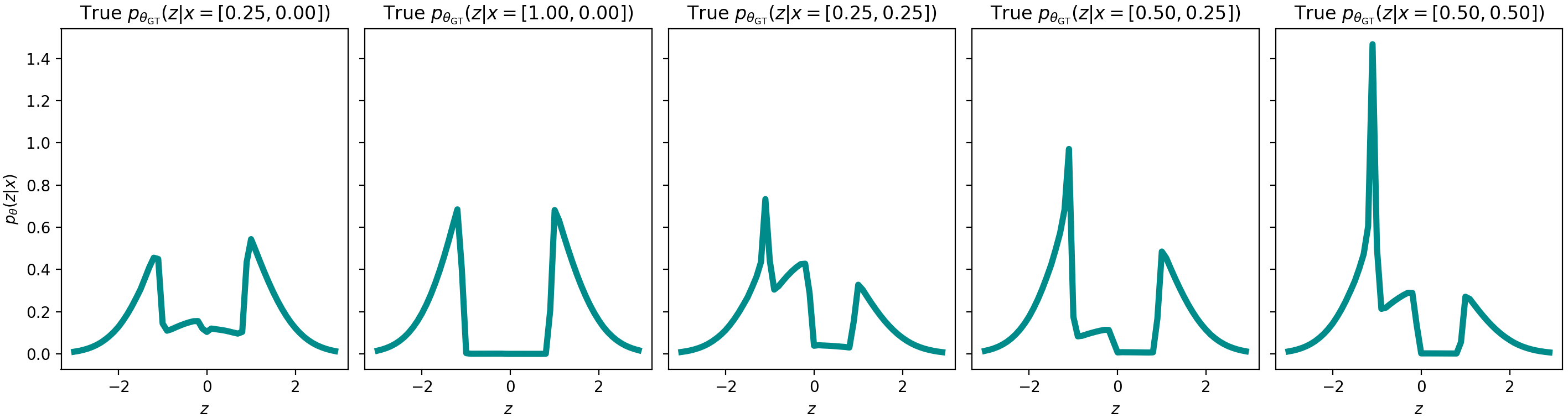}
    \caption{Posteriors under true $f_\theta$}
    \label{fig:vae-clusters-post-true}
    \end{subfigure}
    ~
     \begin{subfigure}[t]{0.7\textwidth}
    \includegraphics[width=1.0\textwidth]{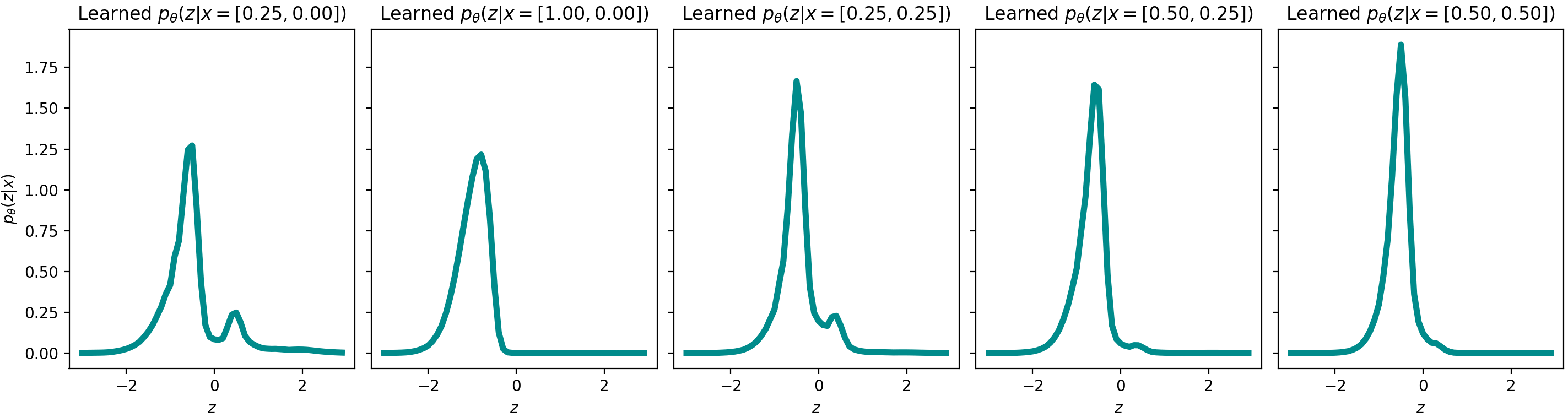}
    \caption{Posteriors under learned $f_\theta$}
    \label{fig:vae-clusters-post-learned}
    \end{subfigure}
    \caption{MFG-VAE trained on the Clusters Example. In this toy data, both conditions from Section \ref{sec:misestimate-px-conditions} hold.
    The VAE learns a generative model with simpler posterior than that of the ground-truth, 
    though it is unable to completely simplify the posterior as in the Absolute-Value Example.
    To learn a generative model with a simpler posterior, it learns a model with a function $f_\theta(z)$ 
    that, unlike the ground-truth function, does not have steep areas interleaved between flat areas.
    As such, the learned model is generally more flat, causing the learned density to be ``smeared'' between the modes.}
    \label{fig:vae-clusters}
\end{figure*}

\begin{figure*}[p]
    \centering
    \vspace*{-1cm}
    \tiny
    
    \begin{subfigure}[t]{0.55\textwidth}
    \includegraphics[width=1.0\textwidth]{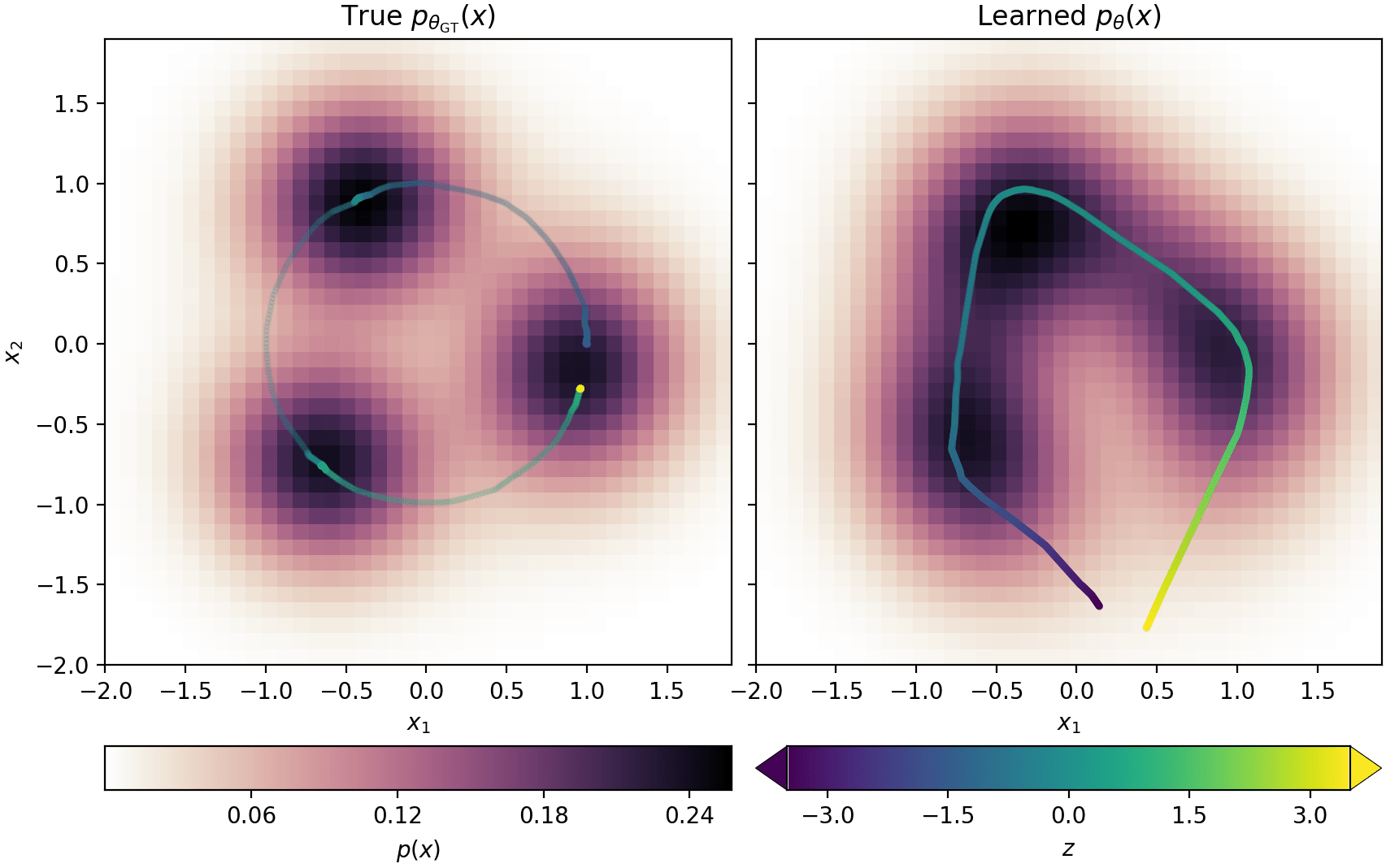}
    \caption{True vs. learned $p_\theta(x)$, and learned vs. true $f_\theta(z)$, colored by the value of $z$.}
    \label{fig:lin-clusters-px}
    \end{subfigure}
    ~
    \begin{subfigure}[t]{0.35\textwidth}
    \includegraphics[width=1.0\textwidth]{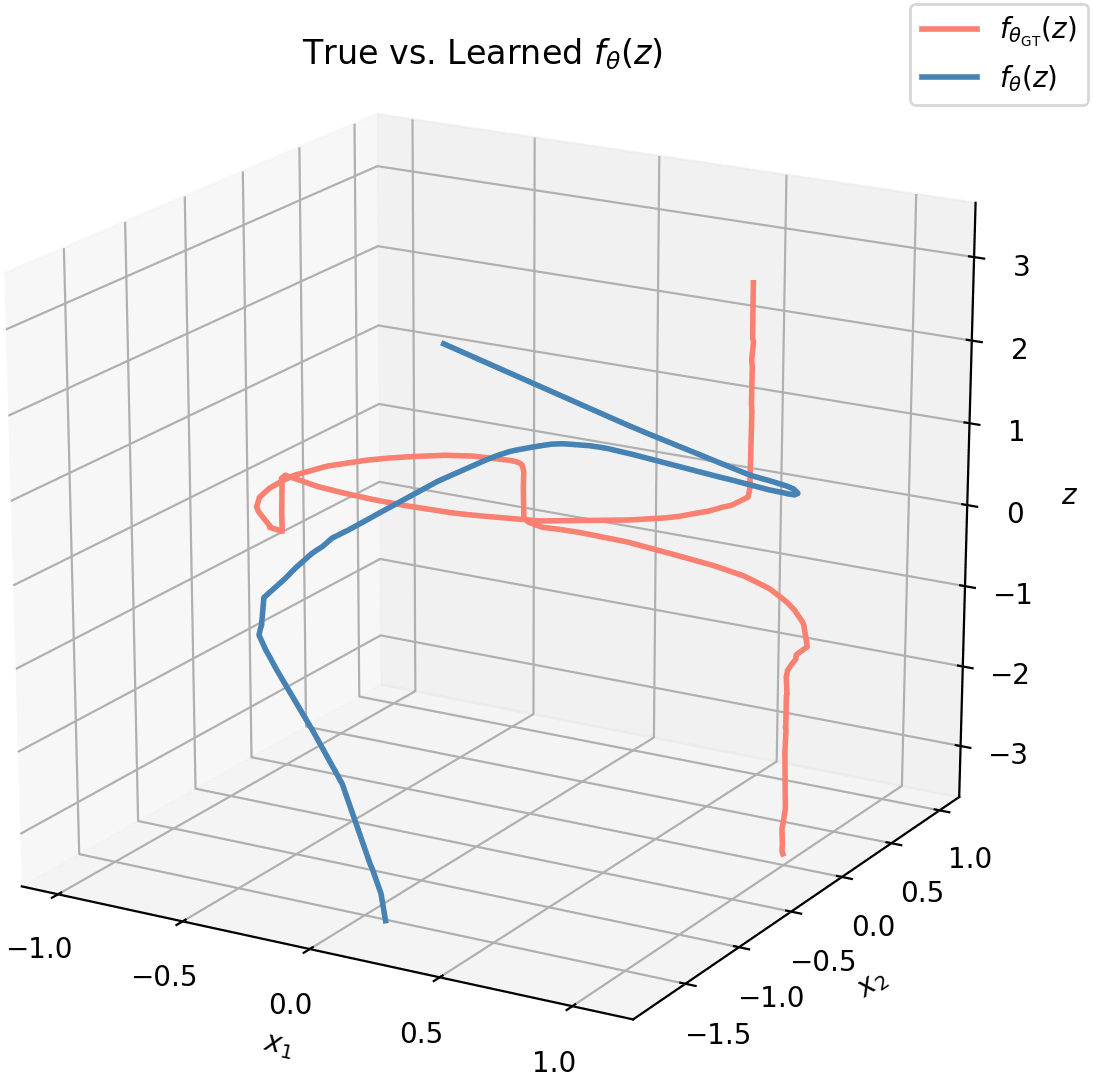}
    \caption{True vs. learned $f_\theta(x)$}
    \label{fig:lin-clusters-f}
    \end{subfigure}
    ~
    \begin{subfigure}[t]{0.7\textwidth}
    \includegraphics[width=1.0\textwidth]{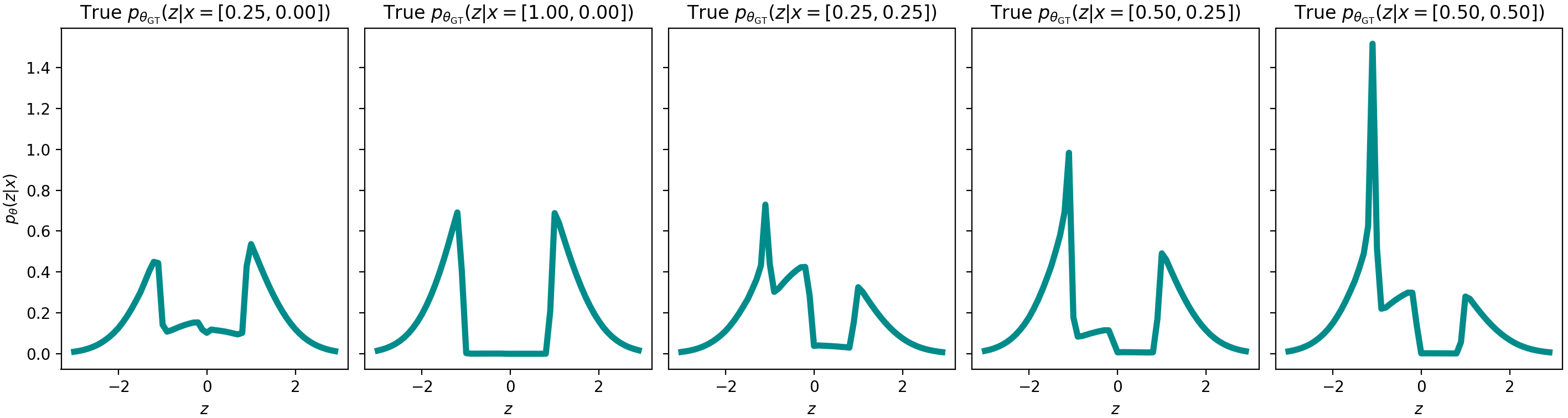}
    \caption{Posteriors under true $f_\theta$}
    \label{fig:lin-clusters-post-true}
    \end{subfigure}
    ~
     \begin{subfigure}[t]{0.7\textwidth}
    \includegraphics[width=1.0\textwidth]{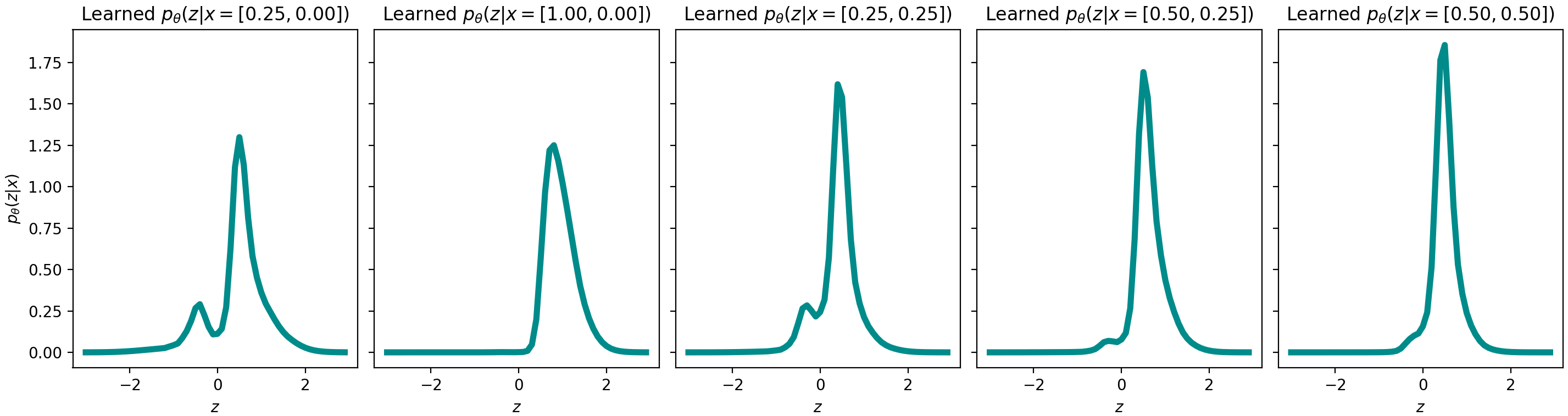}
    \caption{Posteriors under learned $f_\theta$}
    \label{fig:lin-clusters-post-learned}
    \end{subfigure}
    \caption{VAE with Lagging Inference Networks (LIN) trained on the Clusters Example. 
    While LIN may help escape local optima, on this data, the training objective is still biased away
    from learning the true data distribsution.
    As such, LIN fails in the same way an MFG-VAE does (see Figure \ref{fig:vae-clusters}).}
    \label{fig:lin-clusters}
\end{figure*}

\begin{figure*}[p]
    \centering
    \vspace*{-1cm}
    \tiny
    
    \begin{subfigure}[t]{0.55\textwidth}
    \includegraphics[width=1.0\textwidth]{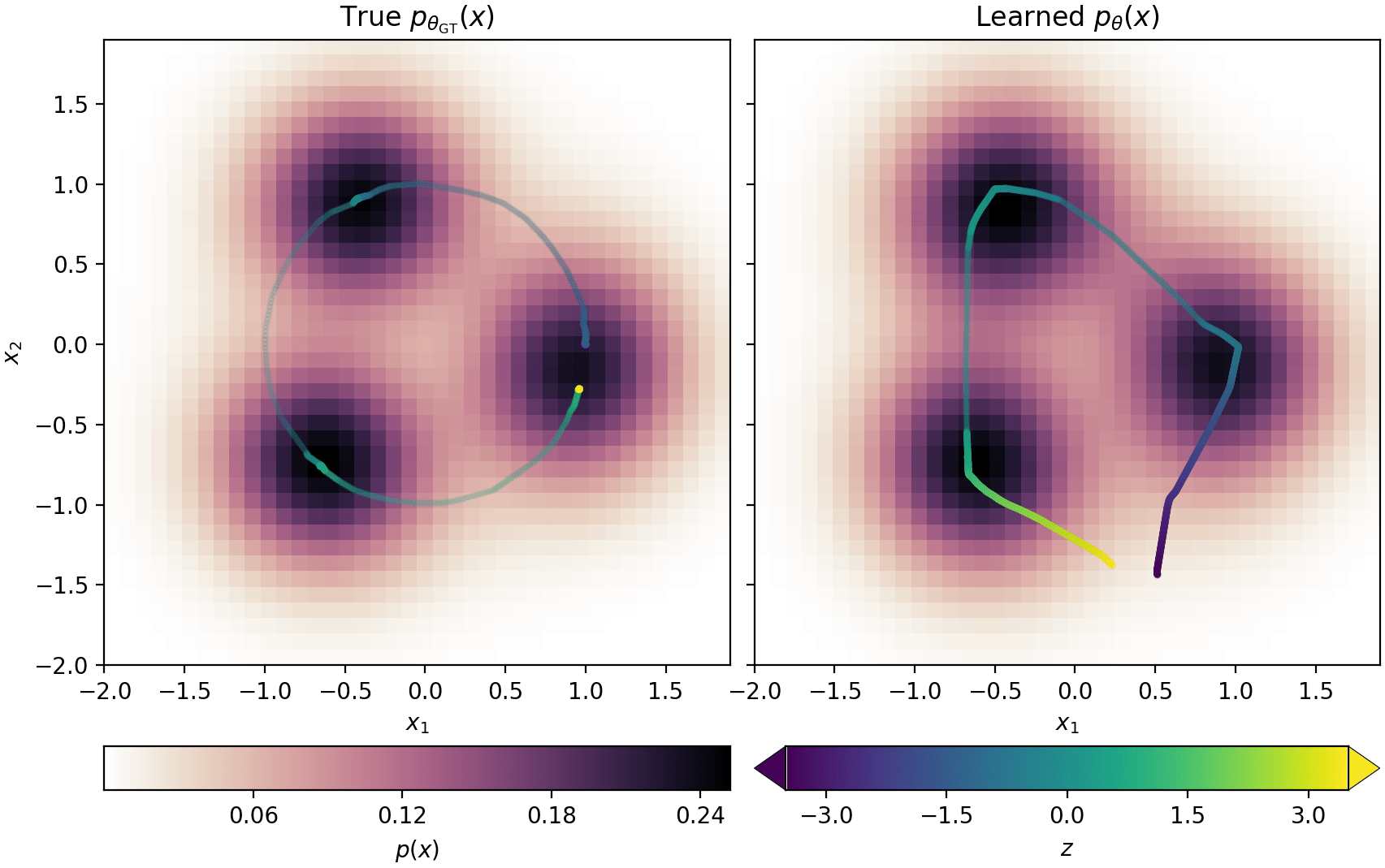}
    \caption{True vs. learned $p_\theta(x)$, and learned vs. true $f_\theta(z)$, colored by the value of $z$.}
    \label{fig:iwae-clusters-px}
    \end{subfigure}
    ~
    \begin{subfigure}[t]{0.35\textwidth}
    \includegraphics[width=1.0\textwidth]{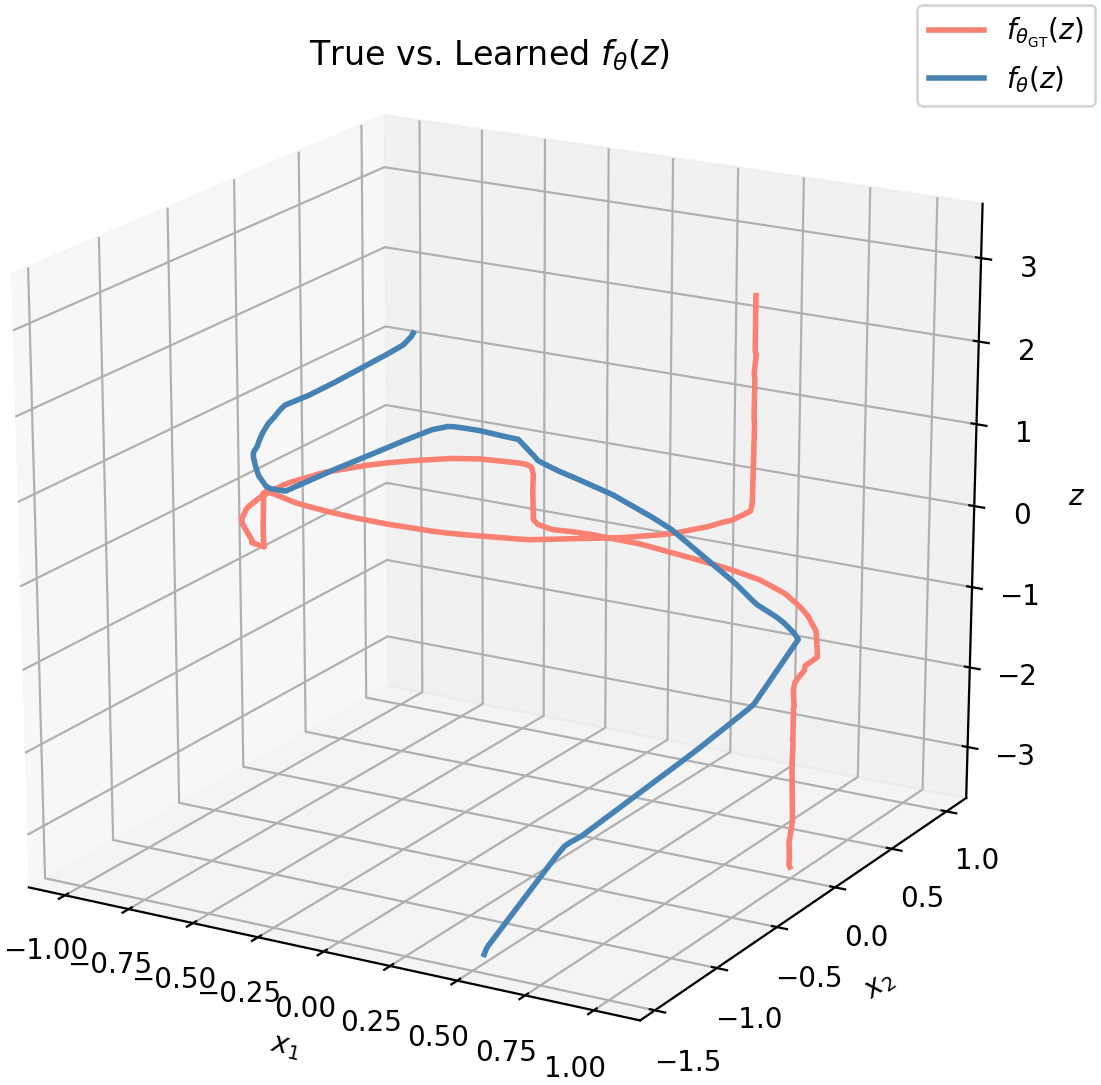}
    \caption{True vs. learned $f_\theta(x)$}
    \label{fig:iwae-clusters-f}
    \end{subfigure}
    ~
    \begin{subfigure}[t]{0.7\textwidth}
    \includegraphics[width=1.0\textwidth]{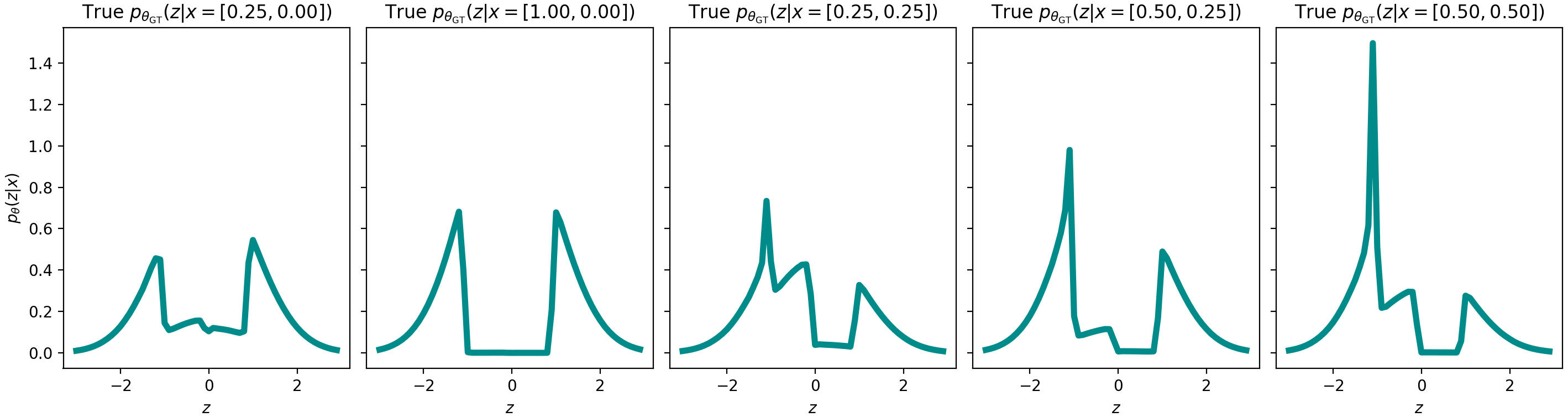}
    \caption{Posteriors under true $f_\theta$}
    \label{fig:iwae-clusters-post-true}
    \end{subfigure}
    ~
     \begin{subfigure}[t]{0.7\textwidth}
    \includegraphics[width=1.0\textwidth]{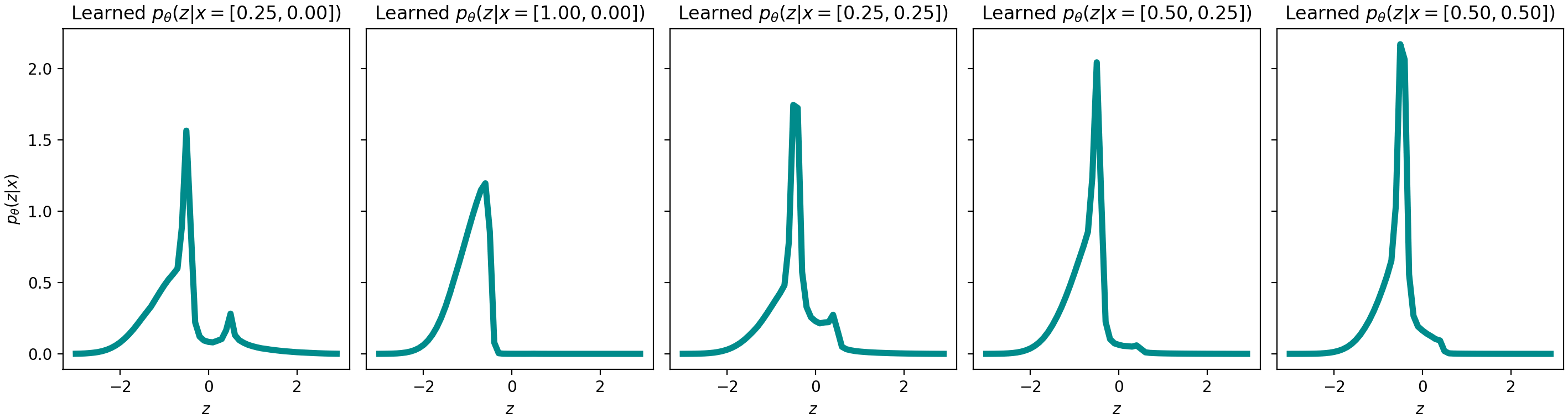}
    \caption{Posteriors under learned $f_\theta$}
    \label{fig:iwae-clusters-post-learned}
    \end{subfigure}
    \caption{IWAE trained on the Clusters Example. In this toy data, both conditions from Section \ref{sec:misestimate-px-conditions} hold.
    IWAE is able to learn the ground-truth data distribution while
    finding a generative model with a simpler posterior than that of the ground-truth model.}
    \label{fig:iwae-clusters}
\end{figure*}

\begin{figure*}[p]
    \centering
    \vspace*{-1cm}
    \tiny
    
    \begin{subfigure}[t]{0.55\textwidth}
    \includegraphics[width=1.0\textwidth]{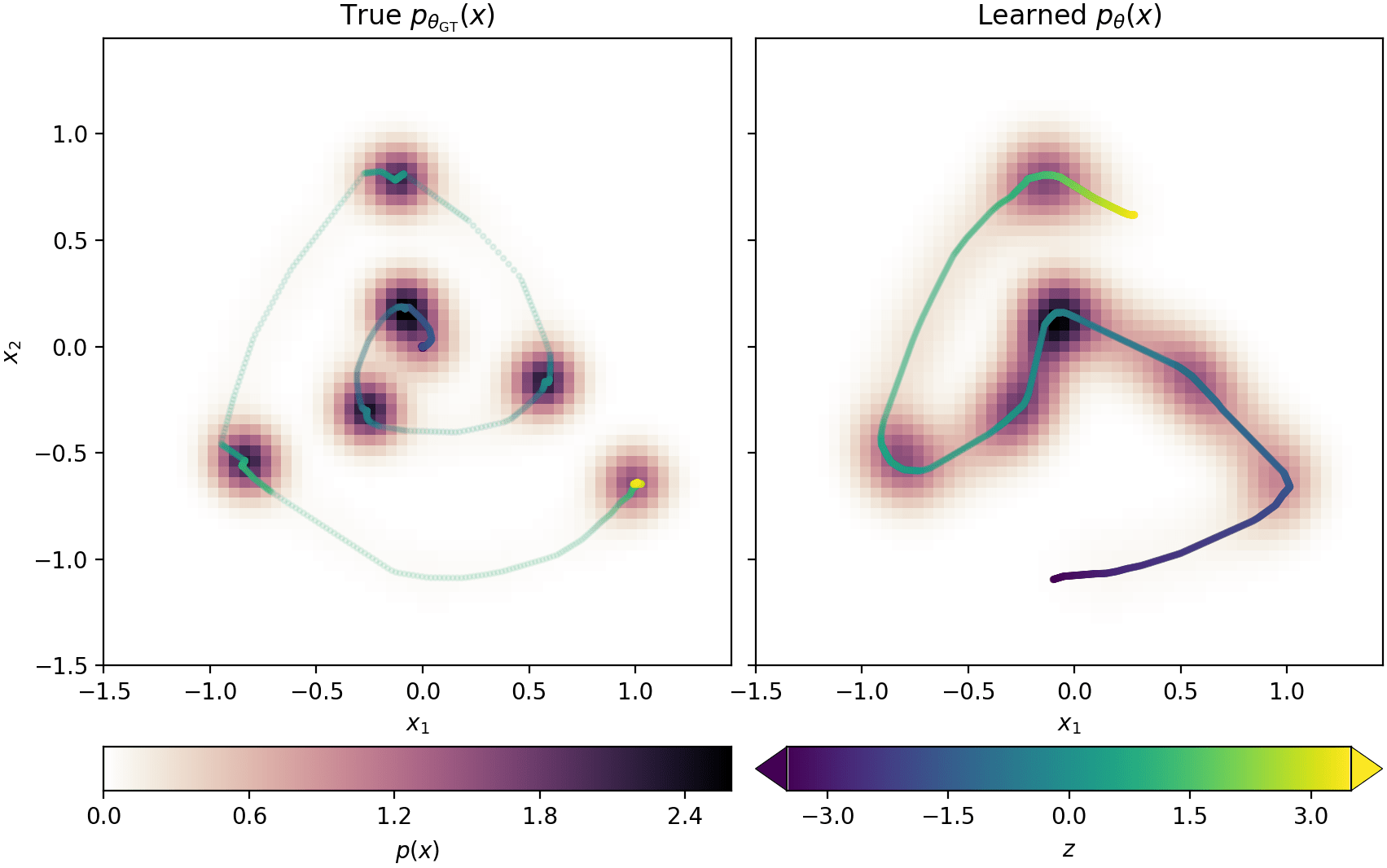}
    \caption{True vs. learned $p_\theta(x)$, and learned vs. true $f_\theta(z)$, colored by the value of $z$.}
    \label{fig:vae-spiral-dots-px}
    \end{subfigure}
    ~
    \begin{subfigure}[t]{0.35\textwidth}
    \includegraphics[width=1.0\textwidth]{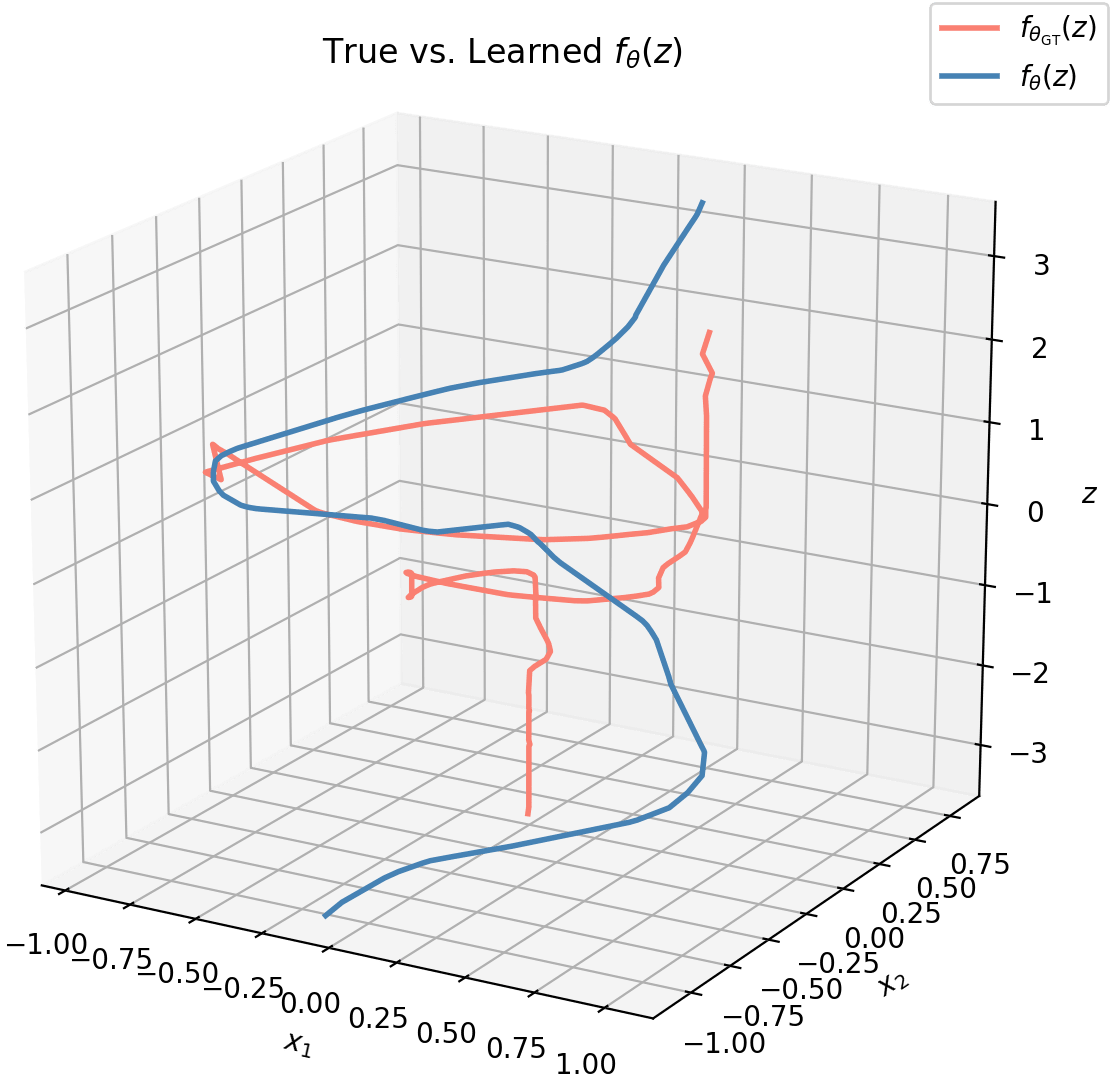}
    \caption{True vs. learned $f_\theta(x)$}
    \label{fig:vae-spiral-dots-f}
    \end{subfigure}
    ~
    \begin{subfigure}[t]{0.7\textwidth}
    \includegraphics[width=1.0\textwidth]{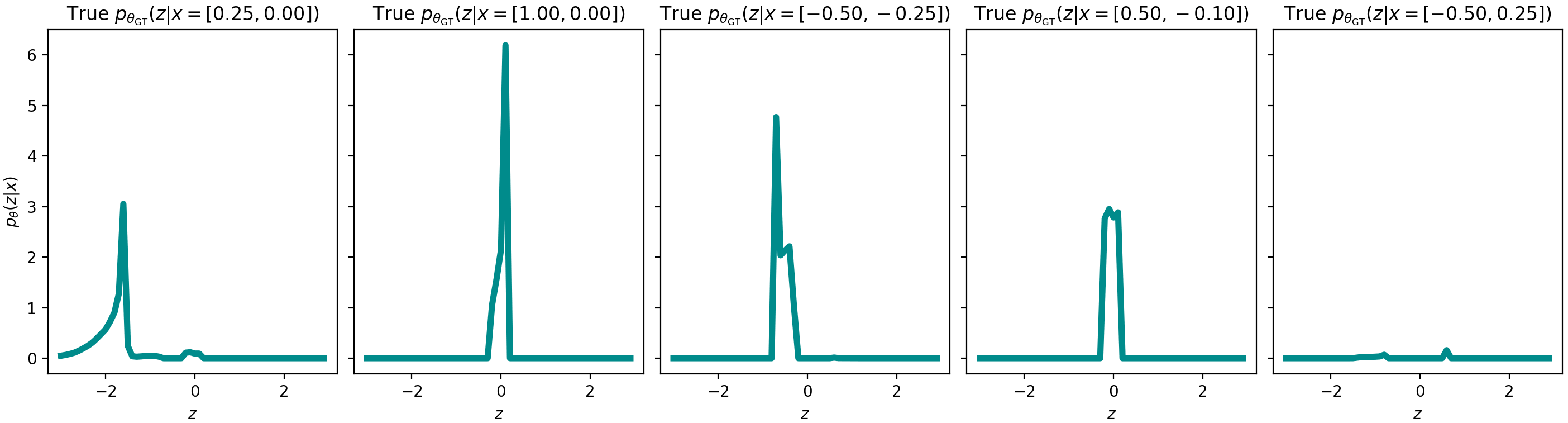}
    \caption{Posteriors under true $f_\theta$}
    \label{fig:vae-spiral-dots-post-true}
    \end{subfigure}
    ~
     \begin{subfigure}[t]{0.7\textwidth}
    \includegraphics[width=1.0\textwidth]{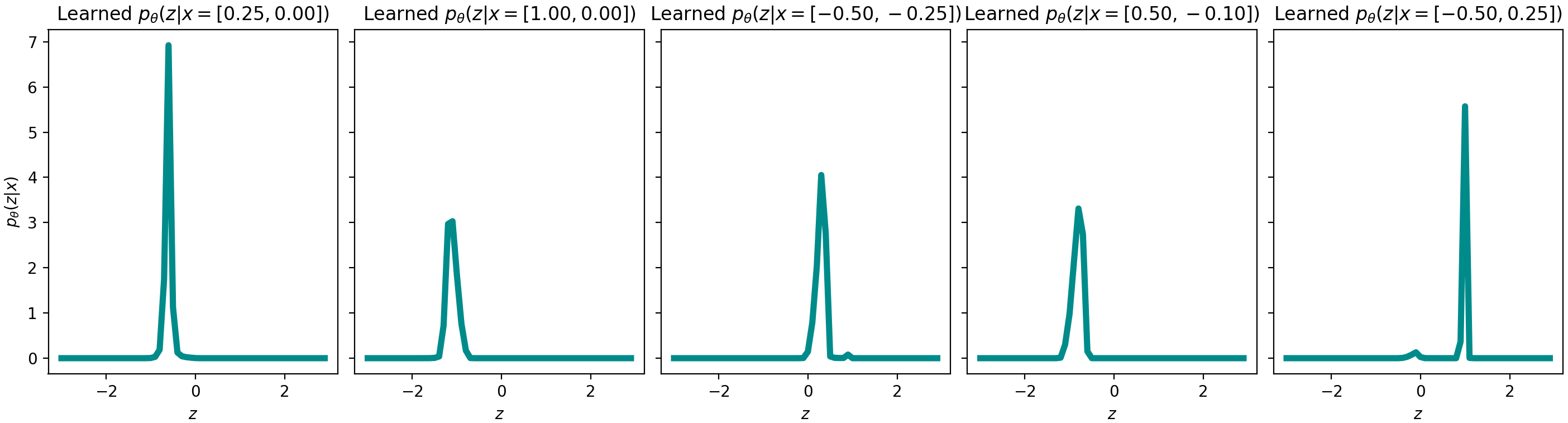}
    \caption{Posteriors under learned $f_\theta$}
    \label{fig:vae-spiral-dots-post-learned}
    \end{subfigure}
    \caption{MFG-VAE trained on the Spiral-Dots Example jointly over $\theta, \phi, \epsilon^2_\epsilon$. 
    In this toy data, the ELBO drastically misestimates the observation noise.
    The VAE learns a generative model with simpler posterior than that of the ground-truth,
    though it is unable to completely simplify the posterior as in the Absolute-Value Example.
    To learn a generative model with a simpler posterior, it learns a model with a function $f_\theta(z)$
    that, unlike the ground-truth function, does not have steep areas interleaved between flat areas.
    As such, the learned model is generally more flat, causing the learned density to be ``smeared'' between the modes.
    Moreover due to the error in approximating the true posterior with an MFG variational family,
    the ELBO misestimates $\sigma^2_\epsilon$.}
    \label{fig:vae-spiral-dots}
\end{figure*}

\begin{figure*}[p]
    \centering
    \vspace*{-1cm}
    \tiny
    
    \begin{subfigure}[t]{0.49\textwidth}
    \includegraphics[width=1.0\textwidth]{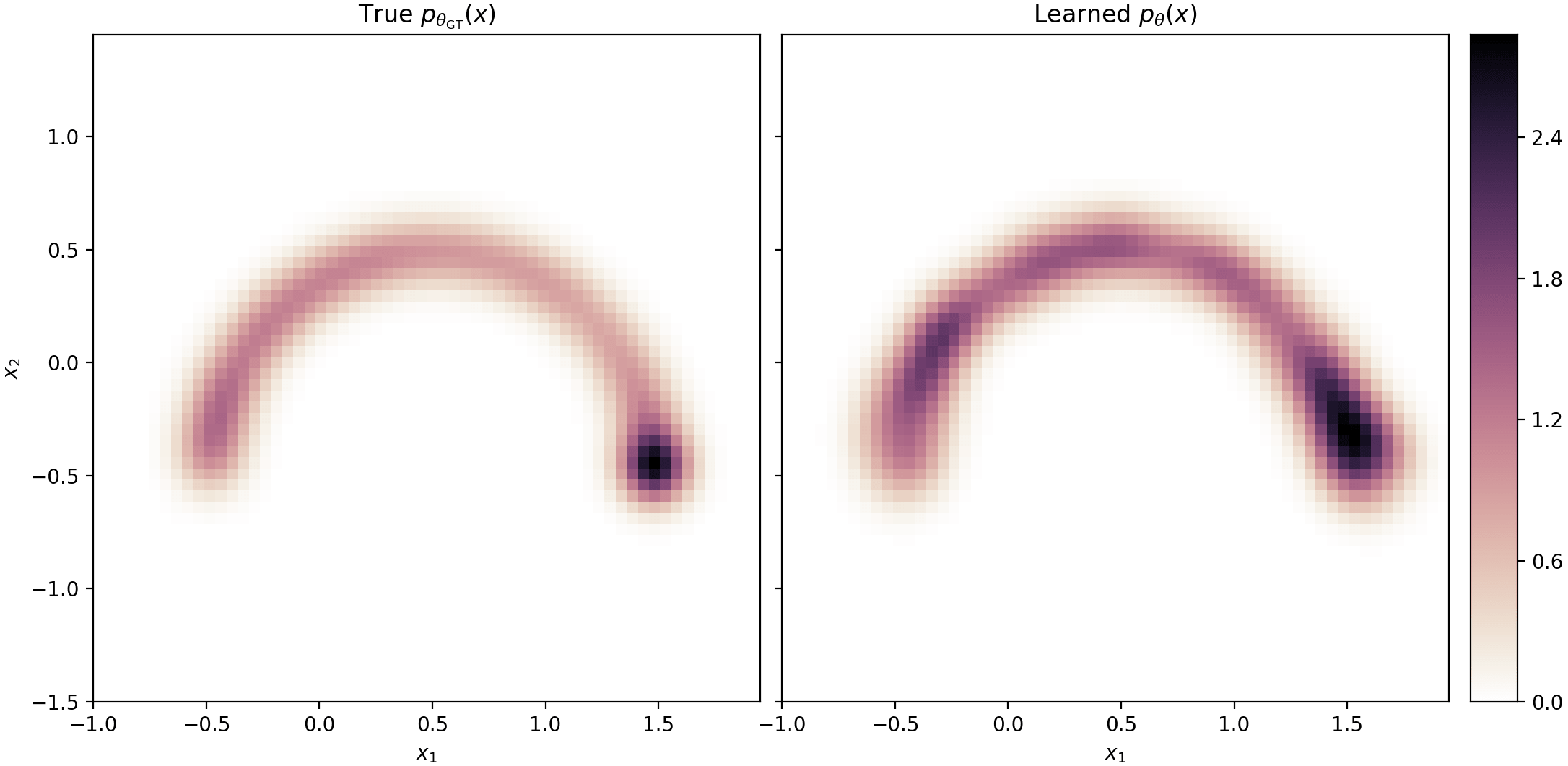}
    \caption{True vs. learned $p_\theta(x)$.}
    \label{fig:vae-ss-discrete-px}
    \end{subfigure}
    ~
    \begin{subfigure}[t]{0.49\textwidth}
    \includegraphics[width=1.0\textwidth]{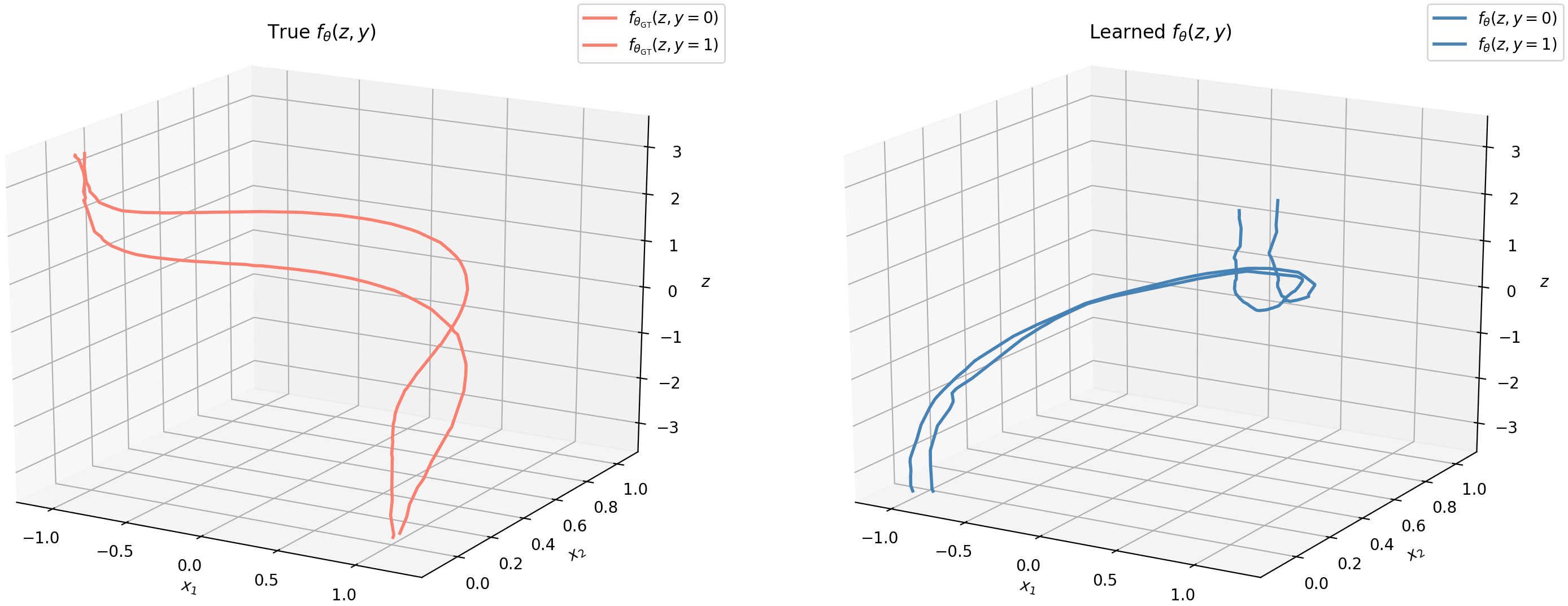}
    \caption{True vs. learned $f_\theta(z, y)$.}
    \label{fig:vae-ss-discrete-fn}
    \end{subfigure}
    ~
    \begin{subfigure}[t]{0.49\textwidth}
    \includegraphics[width=1.0\textwidth]{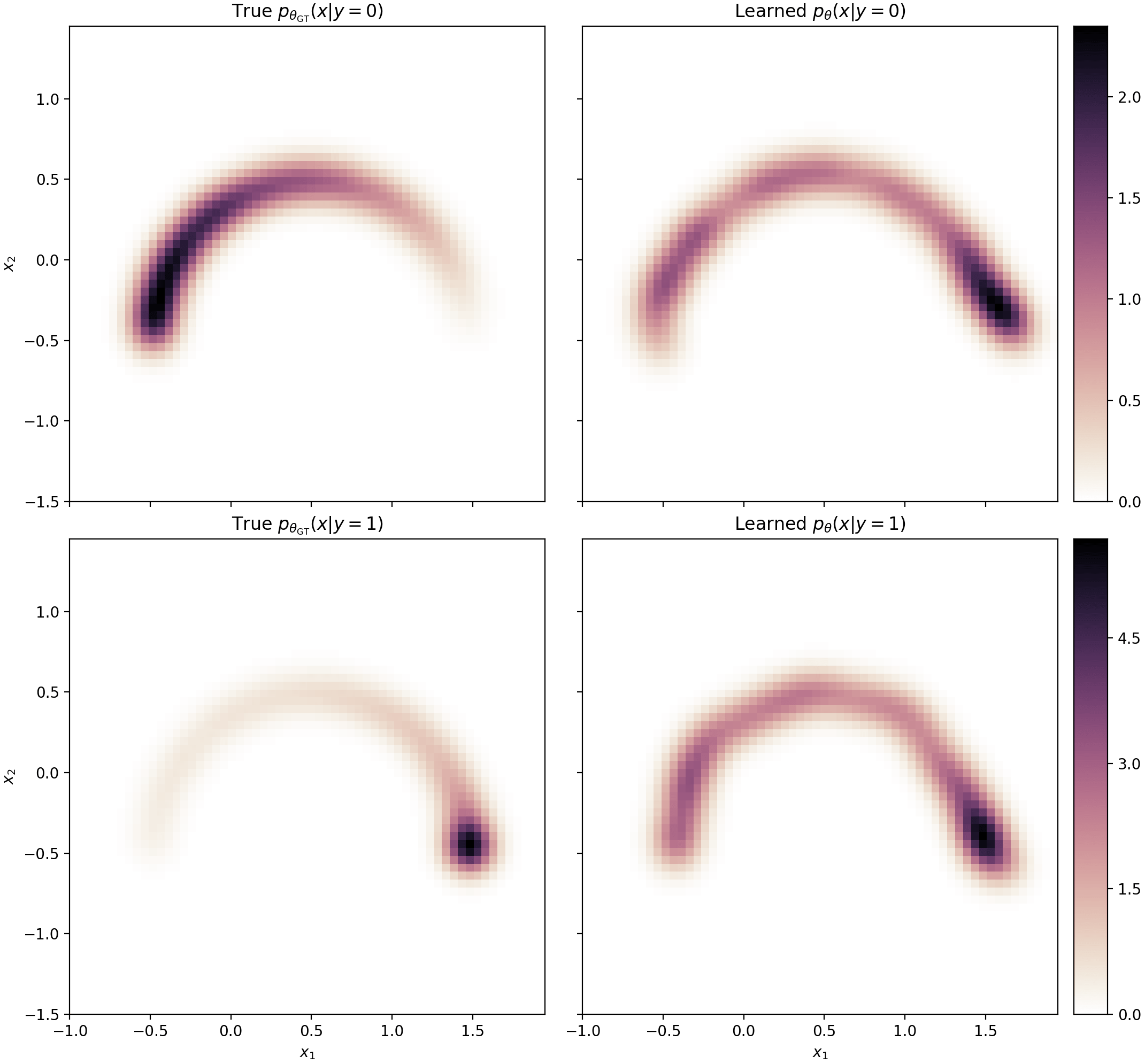}
    \caption{True vs. learned data conditionals $p_\theta(x | y)$.}
    \label{fig:vae-ss-discrete-px-given-y}
    \end{subfigure}
    ~
    \begin{subfigure}[t]{0.7\textwidth}
    \includegraphics[width=1.0\textwidth]{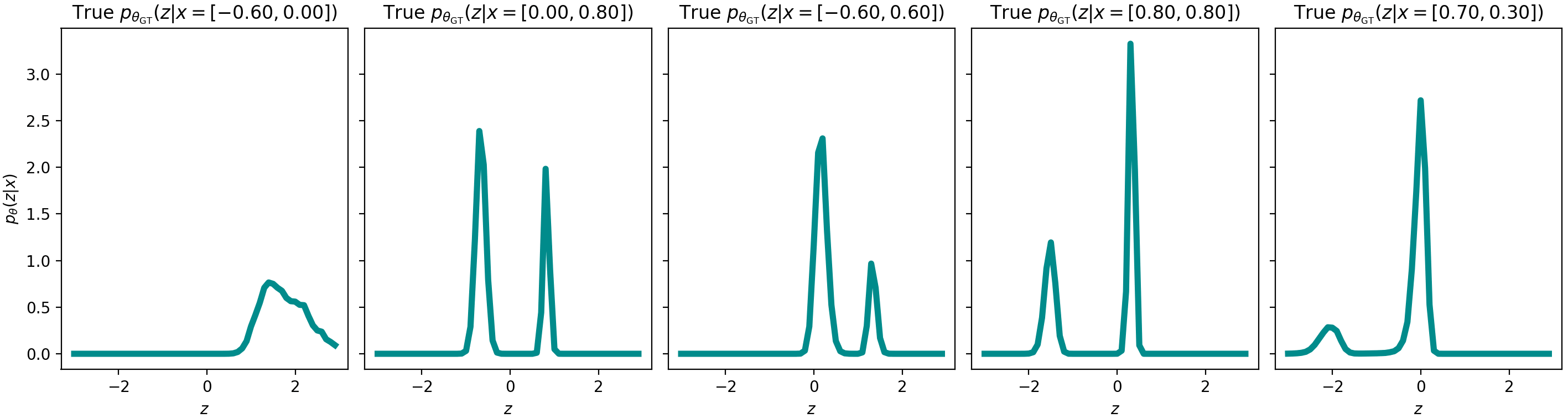}
    \caption{Posteriors under true $f_\theta$}
    \label{fig:vae-ss-discrete-post-true}
    \end{subfigure}
    ~
     \begin{subfigure}[t]{0.7\textwidth}
    \includegraphics[width=1.0\textwidth]{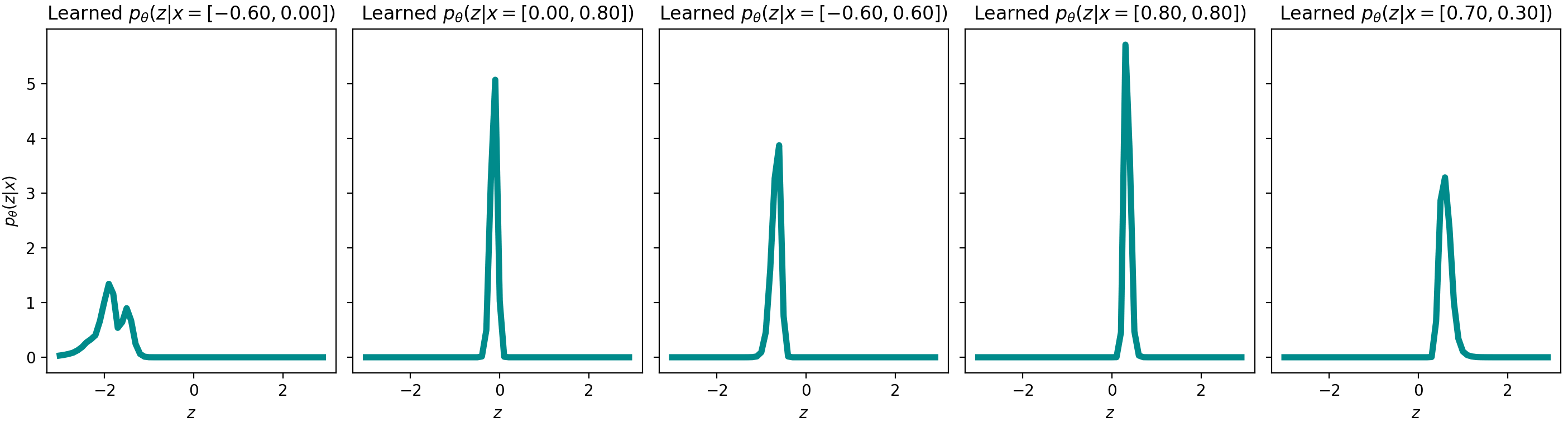}
    \caption{Posteriors under learned $f_\theta$}
    \label{fig:vae-ss-discrete-post-learned}
    \end{subfigure}

    \caption{Semi-Supervised MFG-VAE trained on the Discrete Semi-Circle Example.
    While using semi-supervision, a VAE is still able to learn the $p(x)$ relatively well.
    However, in this example, given $x$ there is uncertainty as to whether 
    it was generated from $f_\theta(y = 0, z)$ or from $f_\theta(y = 1, z)$,
    the posterior $p_\theta(z | x)$ is bimodal and will cause a high posterior matching objective.
    Since semi-supervised VAE objective prefers models with simpler posteriors,
    the VAE learns a unimodal posterior by collapsing $f_\theta(y = 0, z) = f_\theta(y = 1, z)$,
    causing $p(x|y = 0) \approx p(x | y = 1) \approx p(x)$.
    The learned model will therefore generate poor sample quality counterfactuals.}
    \label{fig:vae-ss-discrete}
\end{figure*}

\begin{figure*}[p]
    \centering
    \vspace*{-1cm}
    \tiny
    
    \begin{subfigure}[t]{0.49\textwidth}
    \includegraphics[width=1.0\textwidth]{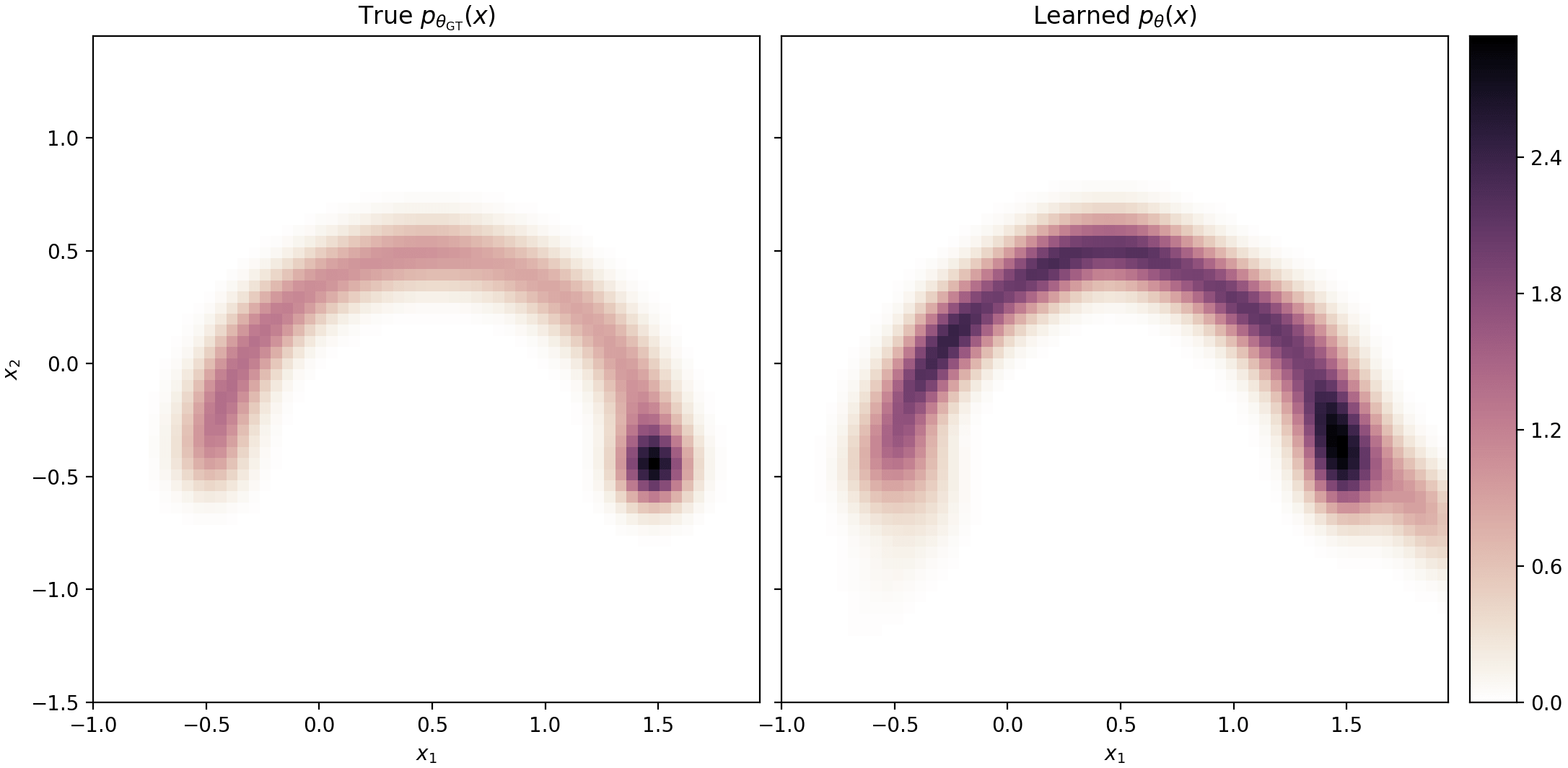}
    \caption{True vs. learned $p_\theta(x)$.}
    \label{fig:lin-ss-discrete-px}
    \end{subfigure}
    ~
    \begin{subfigure}[t]{0.49\textwidth}
    \includegraphics[width=1.0\textwidth]{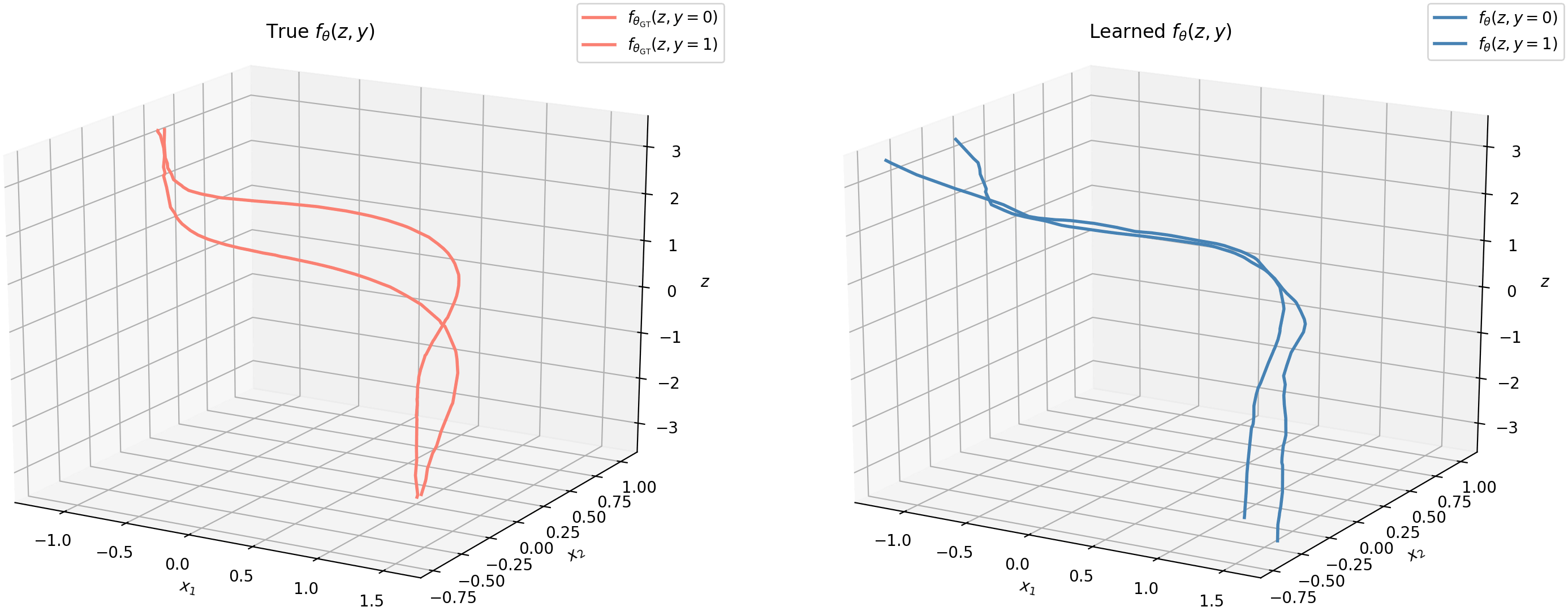}
    \caption{True vs. learned $f_\theta(z, y)$.}
    \label{fig:lin-ss-discrete-fn}
    \end{subfigure}
    ~
    \begin{subfigure}[t]{0.49\textwidth}
    \includegraphics[width=1.0\textwidth]{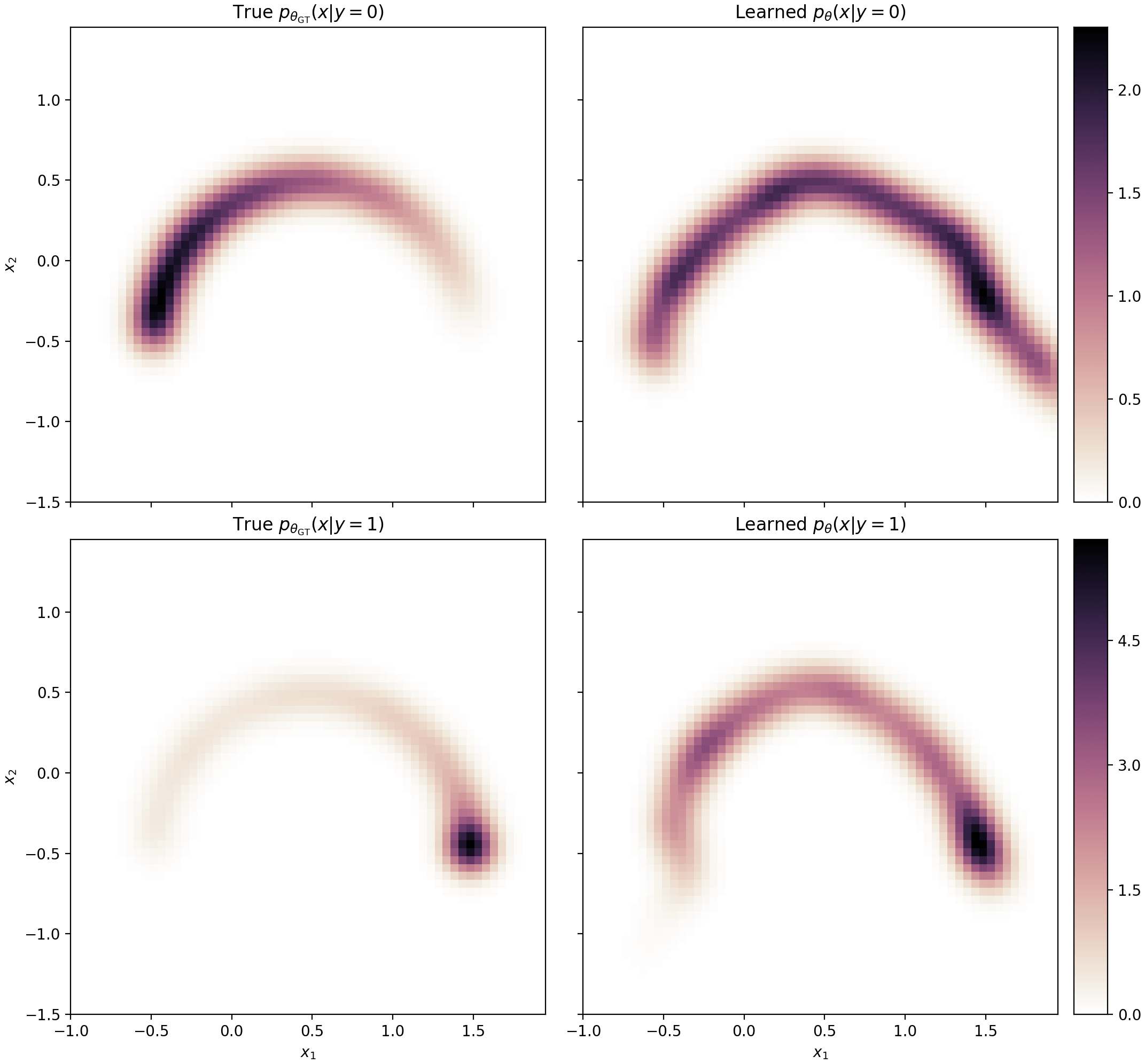}
    \caption{True vs. learned data conditionals $p_\theta(x | y)$.}
    \label{fig:lin-ss-discrete-px-given-y}
    \end{subfigure}
    ~
    \begin{subfigure}[t]{0.7\textwidth}
    \includegraphics[width=1.0\textwidth]{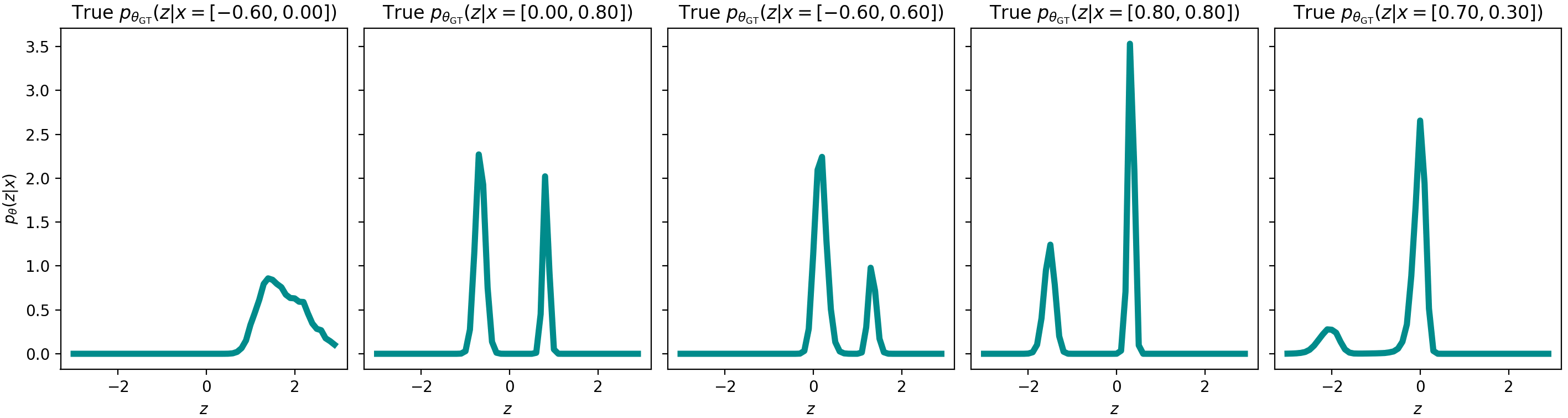}
    \caption{Posteriors under true $f_\theta$}
    \label{fig:lin-ss-discrete-post-true}
    \end{subfigure}
    ~
     \begin{subfigure}[t]{0.7\textwidth}
    \includegraphics[width=1.0\textwidth]{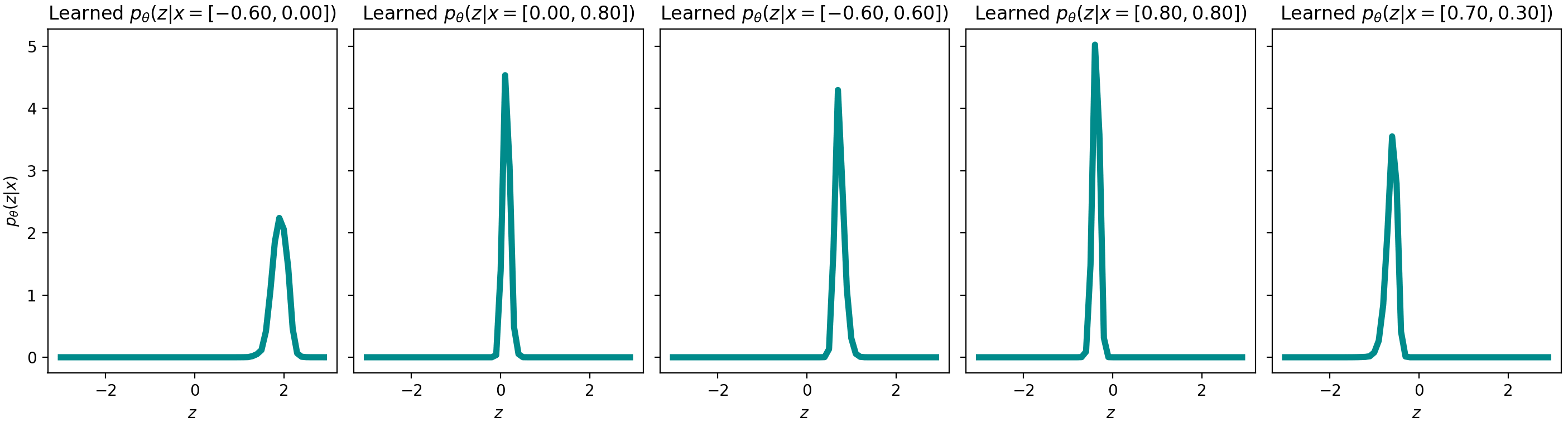}
    \caption{Posteriors under learned $f_\theta$}
    \label{fig:lin-ss-discrete-post-learned}
    \end{subfigure}

    \caption{Semi-Supervised VAE trained with Lagging Inference Networks (LIN) trained on the Discrete Semi-Circle Example.
    While LIN may help escape local optima, on this data, the training objective is still biased away
    from learning the true data distribution.
    As such, LIN fails in the same way an MFG-VAE does (see Figure \ref{fig:vae-ss-discrete}).
    }
    \label{fig:lin-ss-discrete}
\end{figure*}

\begin{figure*}[p]
    \centering
    \vspace*{-1cm}
    \tiny
    
    \begin{subfigure}[t]{0.49\textwidth}
    \includegraphics[width=1.0\textwidth]{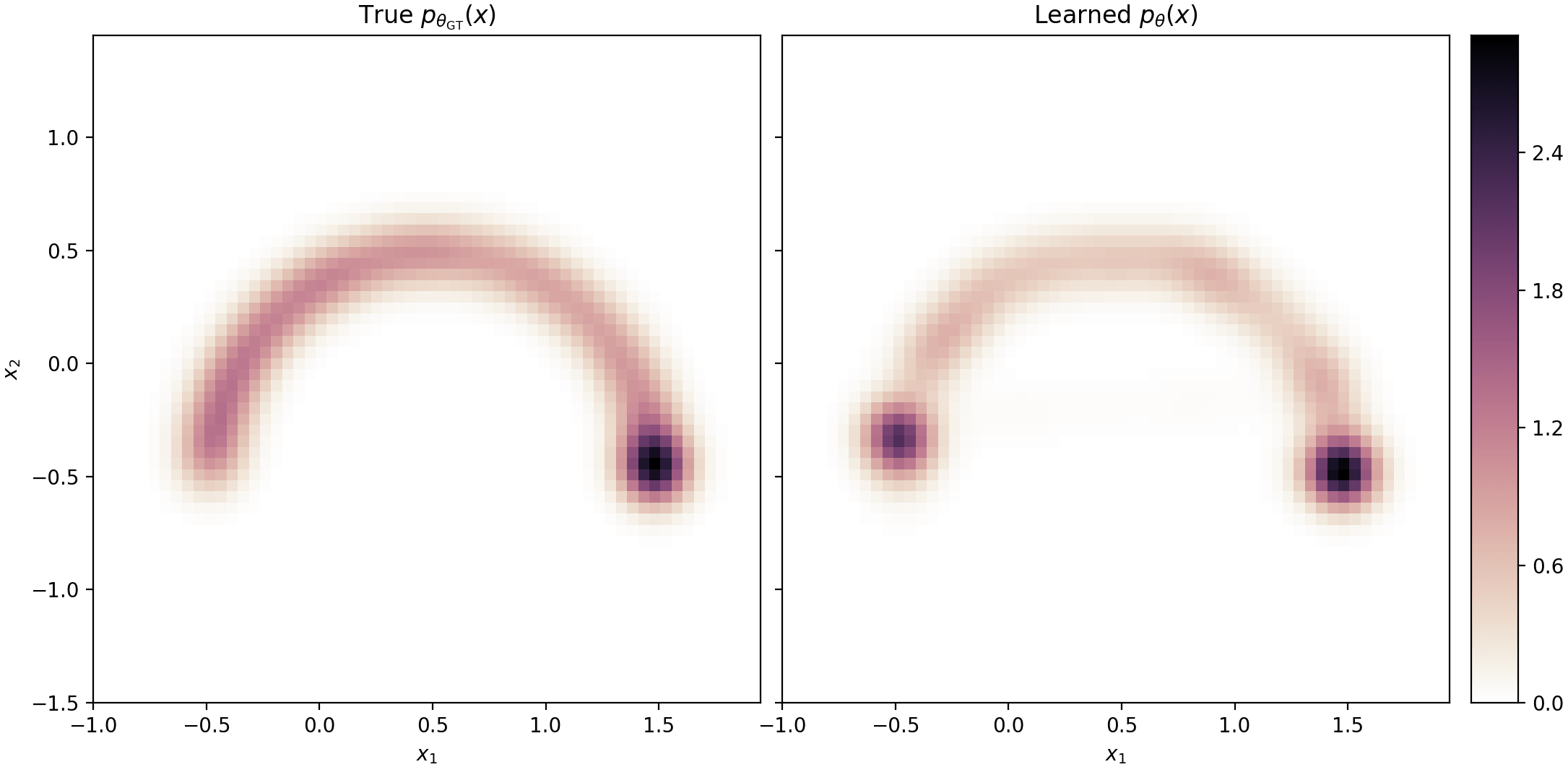}
    \caption{True vs. learned $p_\theta(x)$.}
    \label{fig:iwae-ss-discrete-px}
    \end{subfigure}
    ~
    \begin{subfigure}[t]{0.49\textwidth}
    \includegraphics[width=1.0\textwidth]{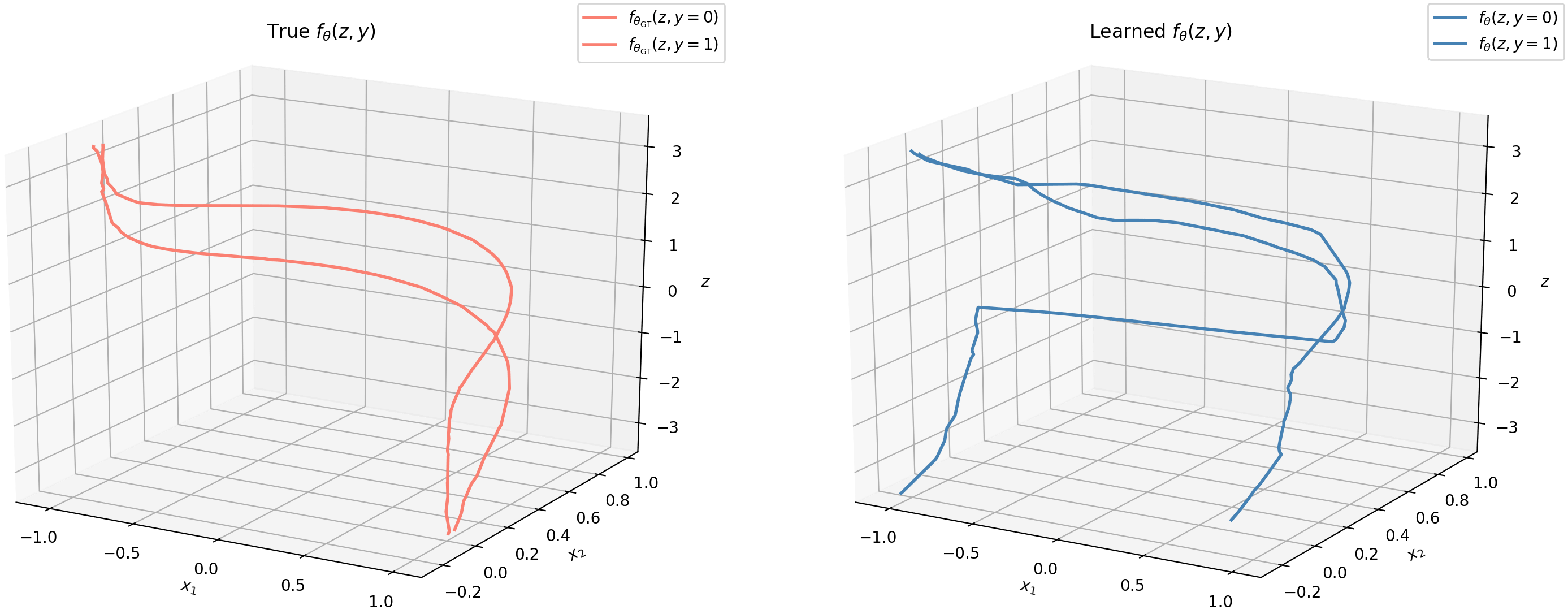}
    \caption{True vs. learned $f_\theta(z, y)$.}
    \label{fig:iwae-ss-discrete-fn}
    \end{subfigure}
    ~
    \begin{subfigure}[t]{0.49\textwidth}
    \includegraphics[width=1.0\textwidth]{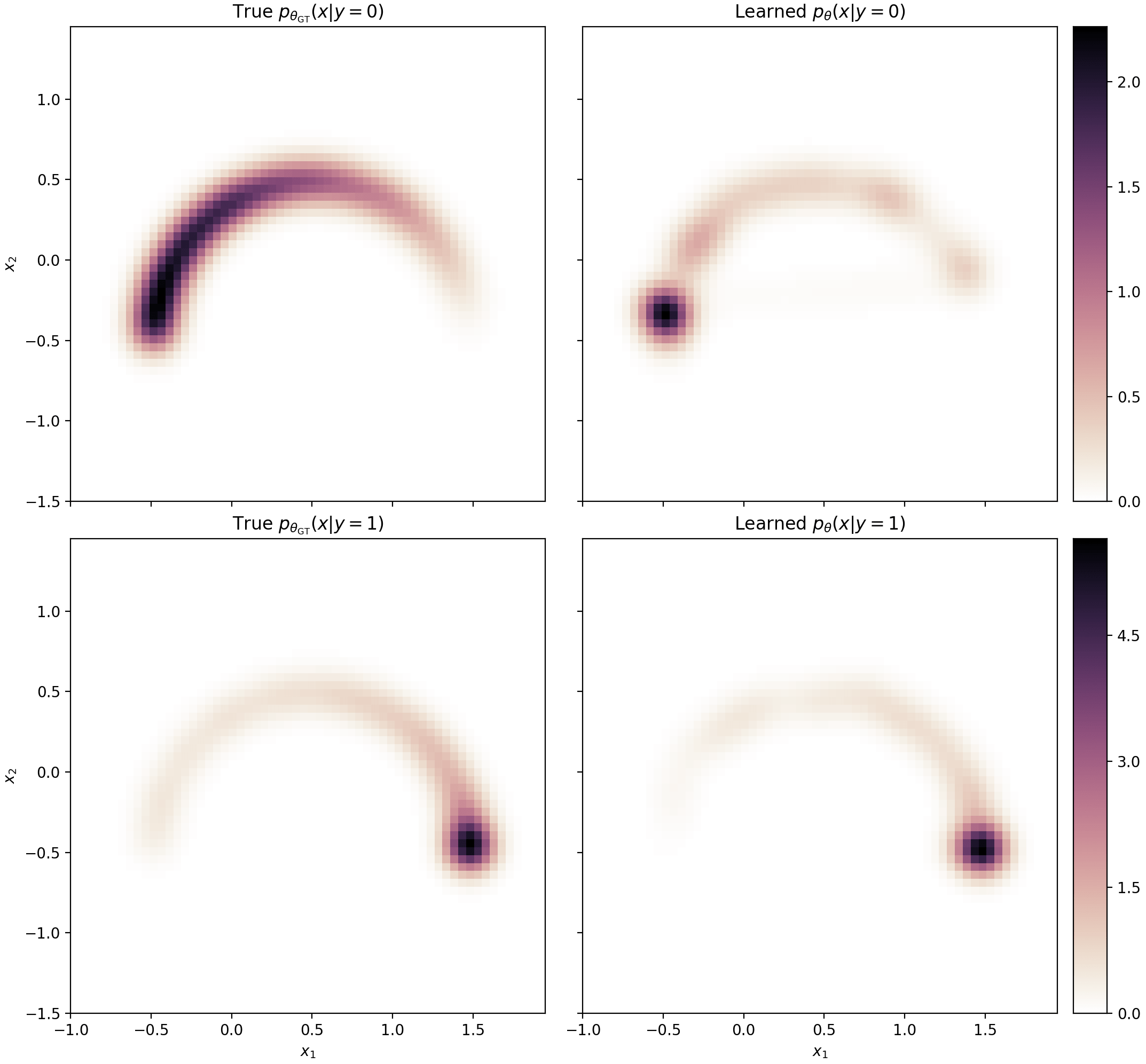}
    \caption{True vs. learned data conditionals $p_\theta(x | y)$.}
    \label{fig:iwae-ss-discrete-px-given-y}
    \end{subfigure}
    ~
    \begin{subfigure}[t]{0.7\textwidth}
    \includegraphics[width=1.0\textwidth]{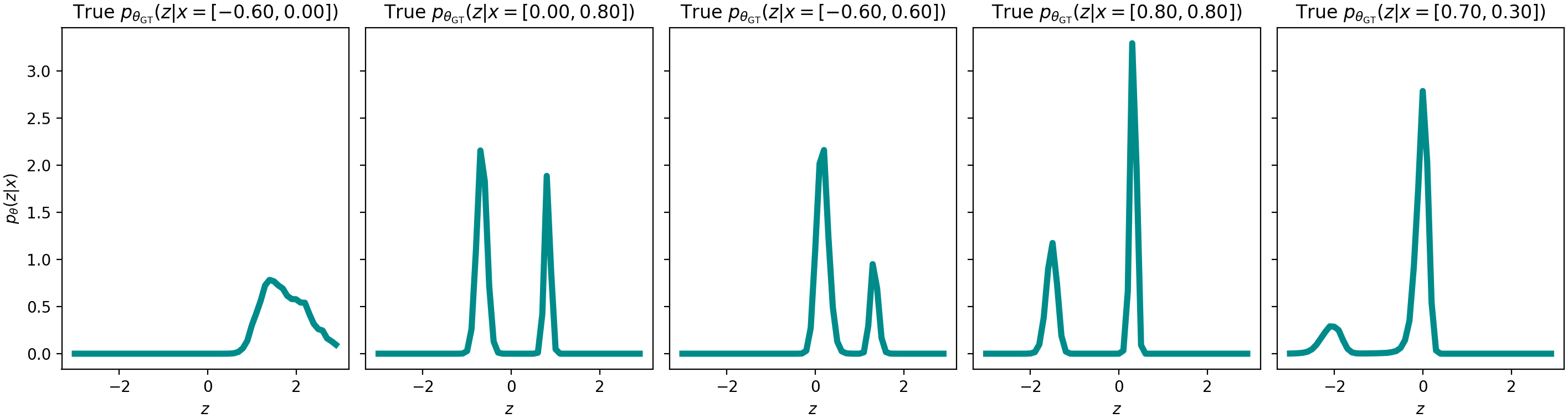}
    \caption{Posteriors under true $f_\theta$}
    \label{fig:iwae-ss-discrete-post-true}
    \end{subfigure}
    ~
     \begin{subfigure}[t]{0.7\textwidth}
    \includegraphics[width=1.0\textwidth]{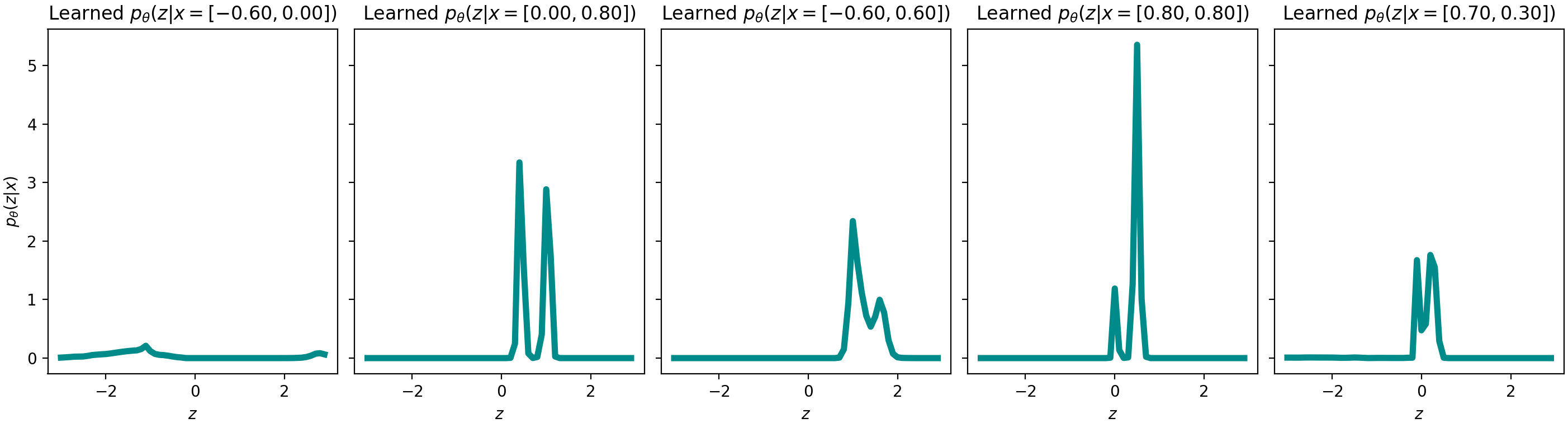}
    \caption{Posteriors under learned $f_\theta$}
    \label{fig:iwae-ss-discrete-post-learned}
    \end{subfigure}

    \caption{Semi-Supervised IWAE trained on the Discrete Semi-Circle Example.
    While using semi-supervision, a IWAE is still able to learn the $p(x)$ and $p(x | y)$ better than a VAE.
    This is because it allows for more complicated posteriors and therefore does  
    not collapse $f_\theta(y = 0, z) = f_\theta(y = 1, z)$.
    However, since IWAE has a more complex variational family, the variational family no longer
    regularizes the function $f_\theta$.
    As such, in order to put enough mass on the left-side of the semi-circle, 
    $f_\theta$ jumps sharply from the right to the left, as opposed to preferring a simpler function
    such as the ground-truth function.
    }
    \label{fig:iwae-ss-discrete}
\end{figure*}

\begin{figure*}[p]
    \centering
    \vspace*{-1cm}
    \tiny
    
    \begin{subfigure}[t]{0.49\textwidth}
    \includegraphics[width=1.0\textwidth]{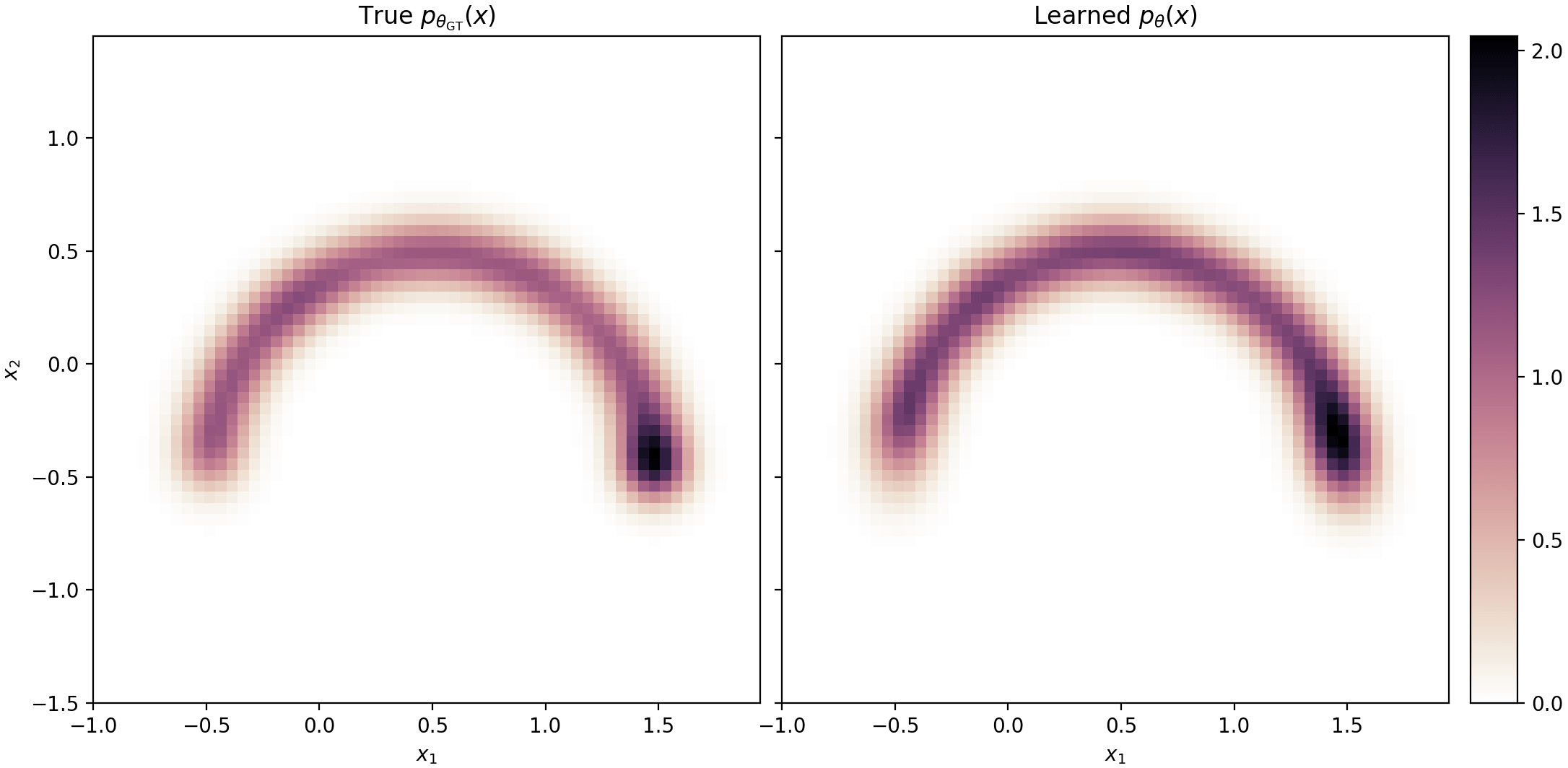}
    \caption{True vs. learned $p_\theta(x)$.}
    \label{fig:vae-ss-continuous-px}
    \end{subfigure}
    ~
    \begin{subfigure}[t]{0.49\textwidth}
    \includegraphics[width=1.0\textwidth]{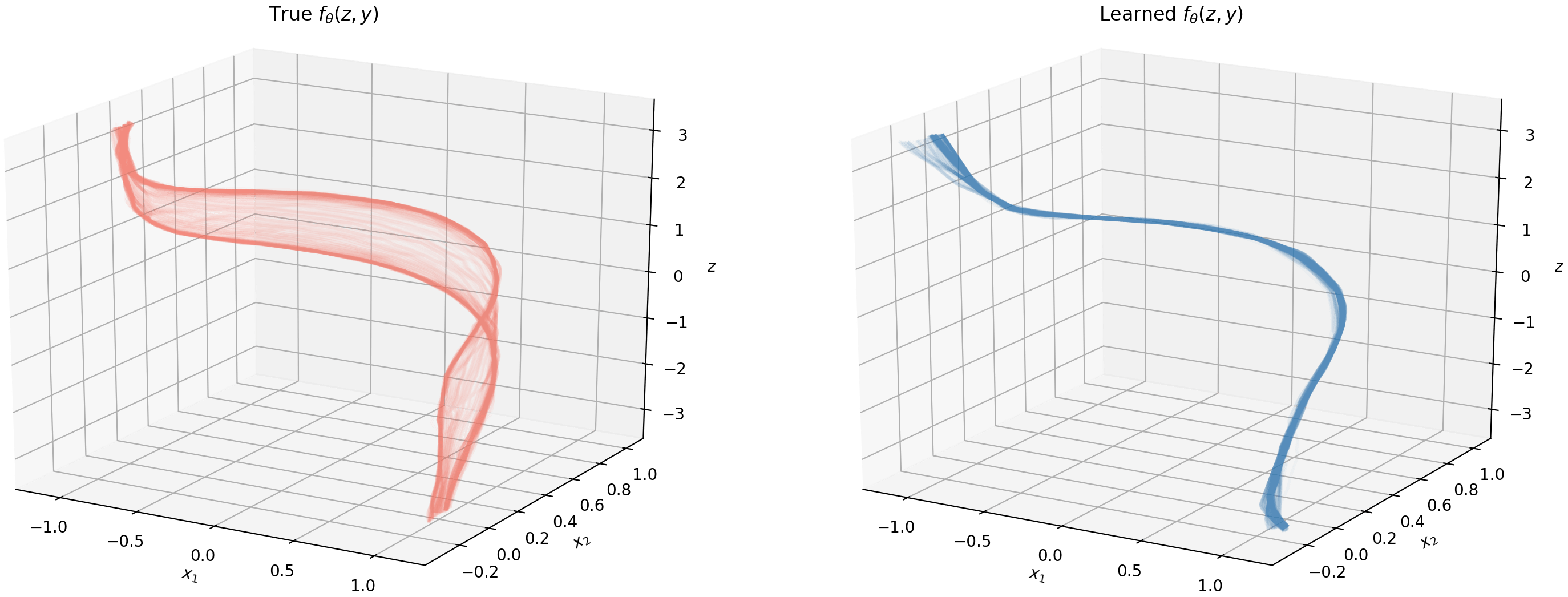}
    \caption{True vs. learned $f_\theta(z, y)$.}
    \label{fig:vae-ss-continuous-fn}
    \end{subfigure}
    ~
    \begin{subfigure}[t]{0.49\textwidth}
    \includegraphics[width=1.0\textwidth]{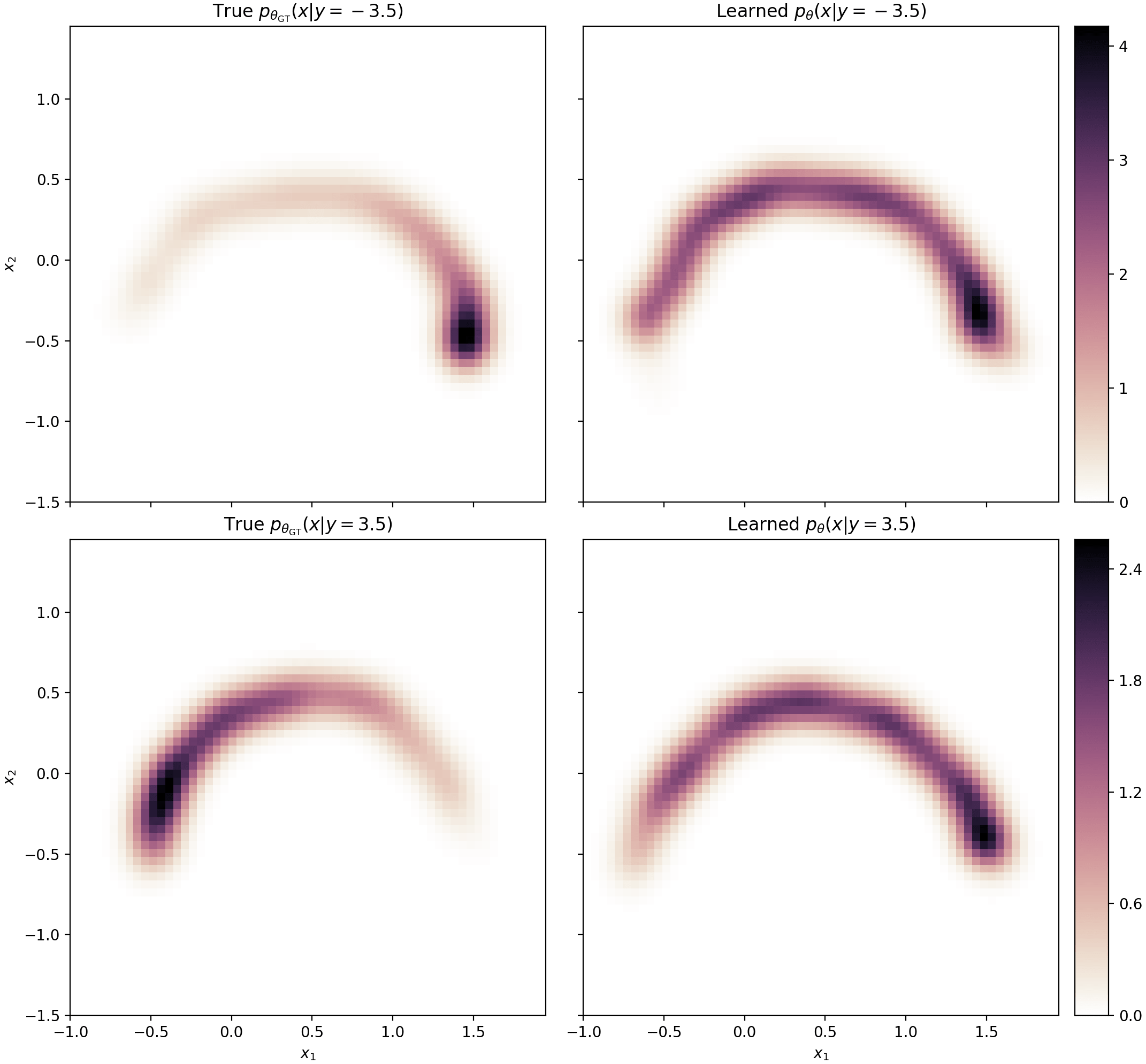}
    \caption{True vs. learned data conditionals $p_\theta(x | y)$.}
    \label{fig:vae-ss-continuous-px-given-y}
    \end{subfigure}
    ~
    \begin{subfigure}[t]{0.7\textwidth}
    \includegraphics[width=1.0\textwidth]{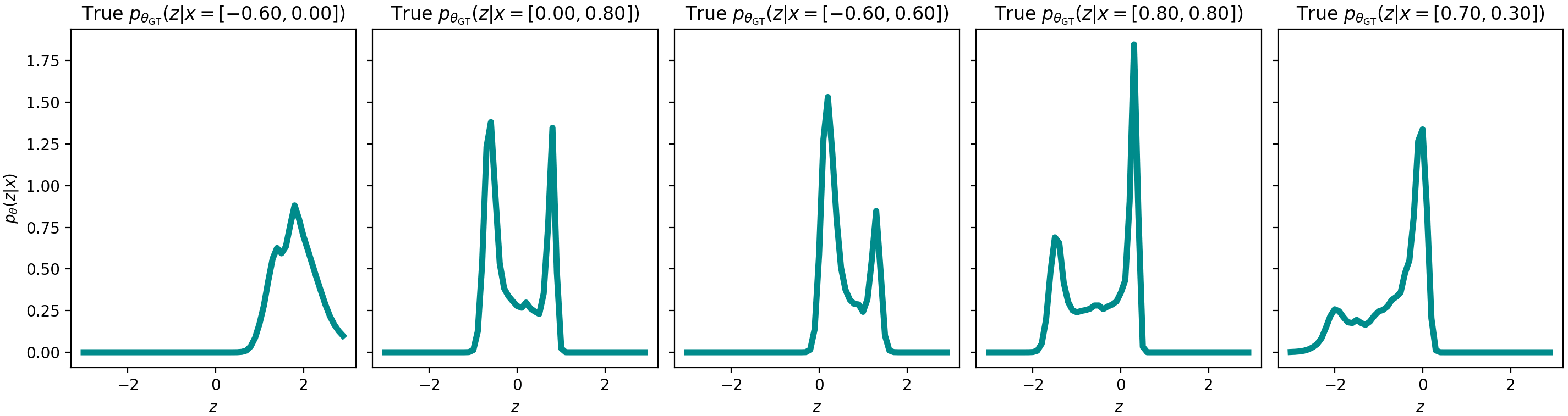}
    \caption{Posteriors under true $f_\theta$}
    \label{fig:vae-ss-continuous-post-true}
    \end{subfigure}
    ~
     \begin{subfigure}[t]{0.7\textwidth}
    \includegraphics[width=1.0\textwidth]{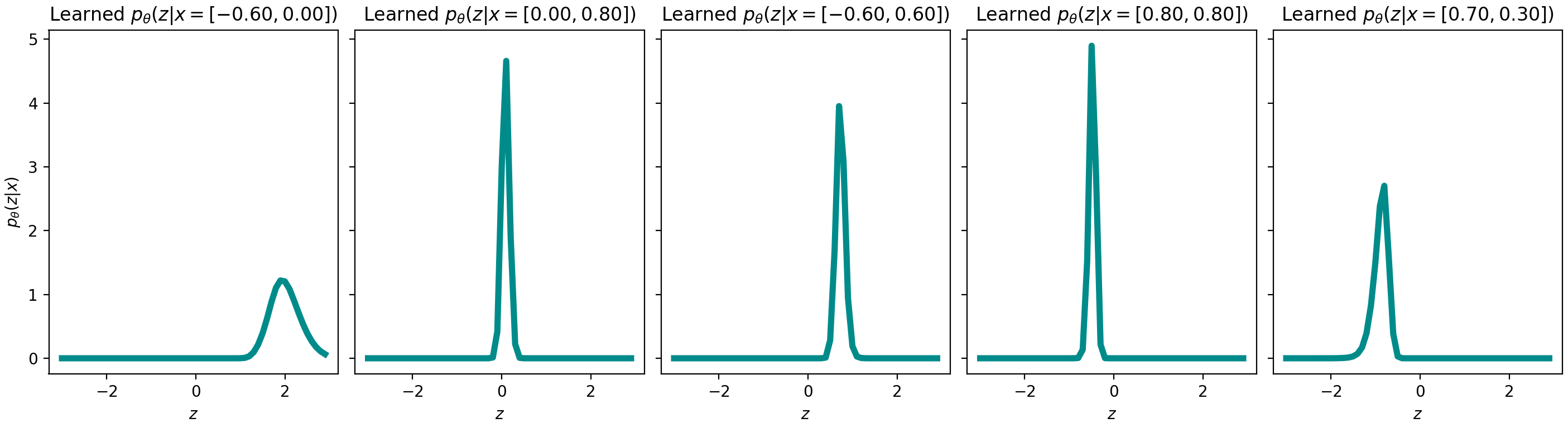}
    \caption{Posteriors under learned $f_\theta$}
    \label{fig:vae-ss-continuous-post-learned}
    \end{subfigure}
    
    \caption{Semi-Supervised MFG-VAE trained on the Continuous Semi-Circle Example.
    In this example, the VAE exhibits the same problems as in the 
    Discrete Semi-Circle Example (Figure \ref{fig:vae-ss-continuous}).
    However, with since $y$ is continuous, this poses an additional issue.
    Since $q_\phi(y | x)$ (the discriminator) in the objective is a Gaussian,
    and the ground-truth $p_\theta(y | x)$ is multi-modal, the objective will select a function 
    $f_\theta$ under which $p_\theta(y | x)$ is an MFG.
    This, again, leads to learning a model in which $f_\theta(y = \cdot, z)$ are the same for all values of $y$,
    causing $p(x|y = 0) \approx p(x | y = 1) \approx p(x)$.
    The learned model will therefore generate poor sample quality counterfactuals.}
    \label{fig:vae-ss-continuous}
\end{figure*}

\begin{figure*}[p]
    \centering
    \vspace*{-1cm}
    \tiny
    
    \begin{subfigure}[t]{0.49\textwidth}
    \includegraphics[width=1.0\textwidth]{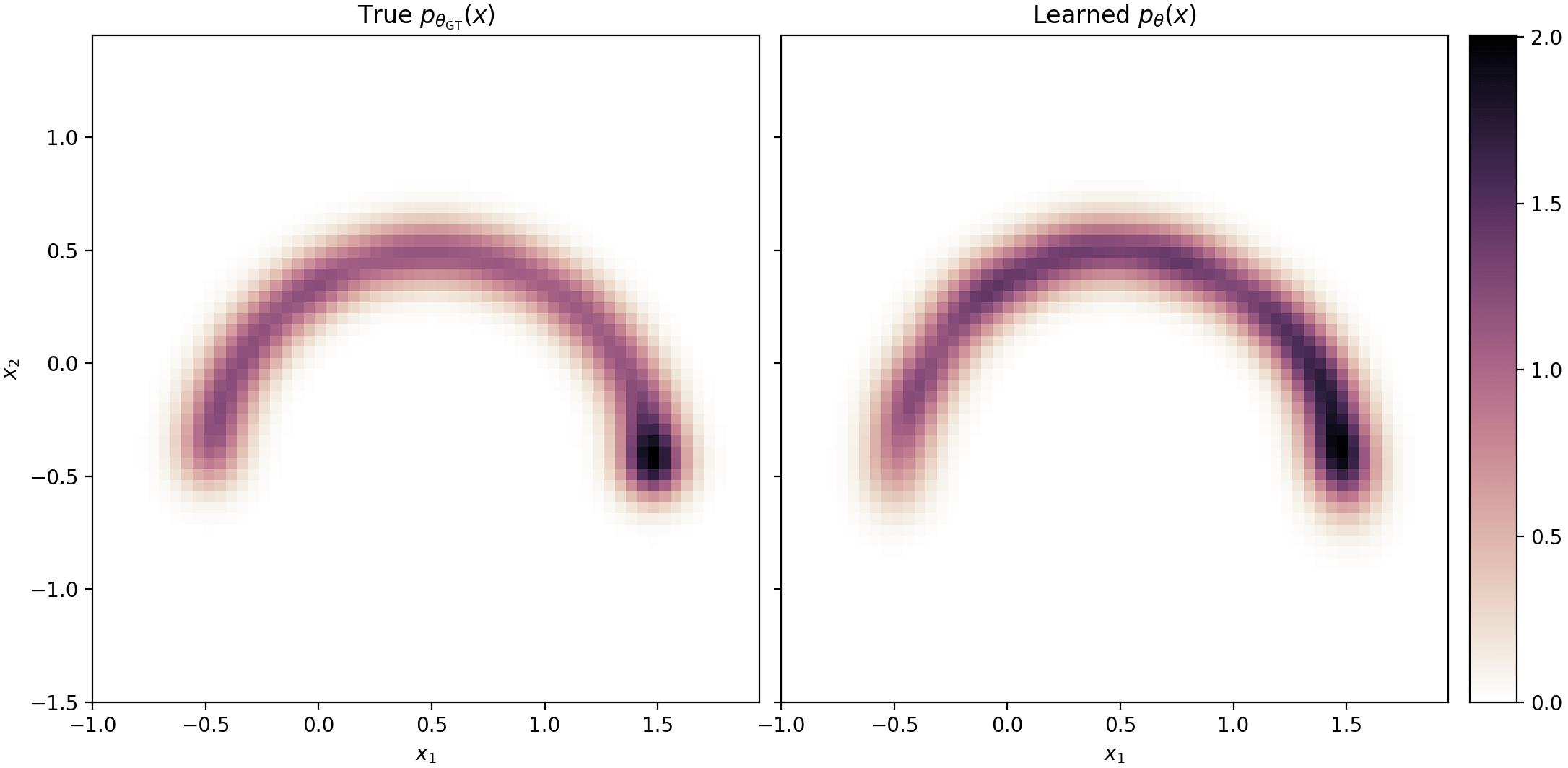}
    \caption{True vs. learned $p_\theta(x)$.}
    \label{fig:lin-ss-continuous-px}
    \end{subfigure}
    ~
    \begin{subfigure}[t]{0.49\textwidth}
    \includegraphics[width=1.0\textwidth]{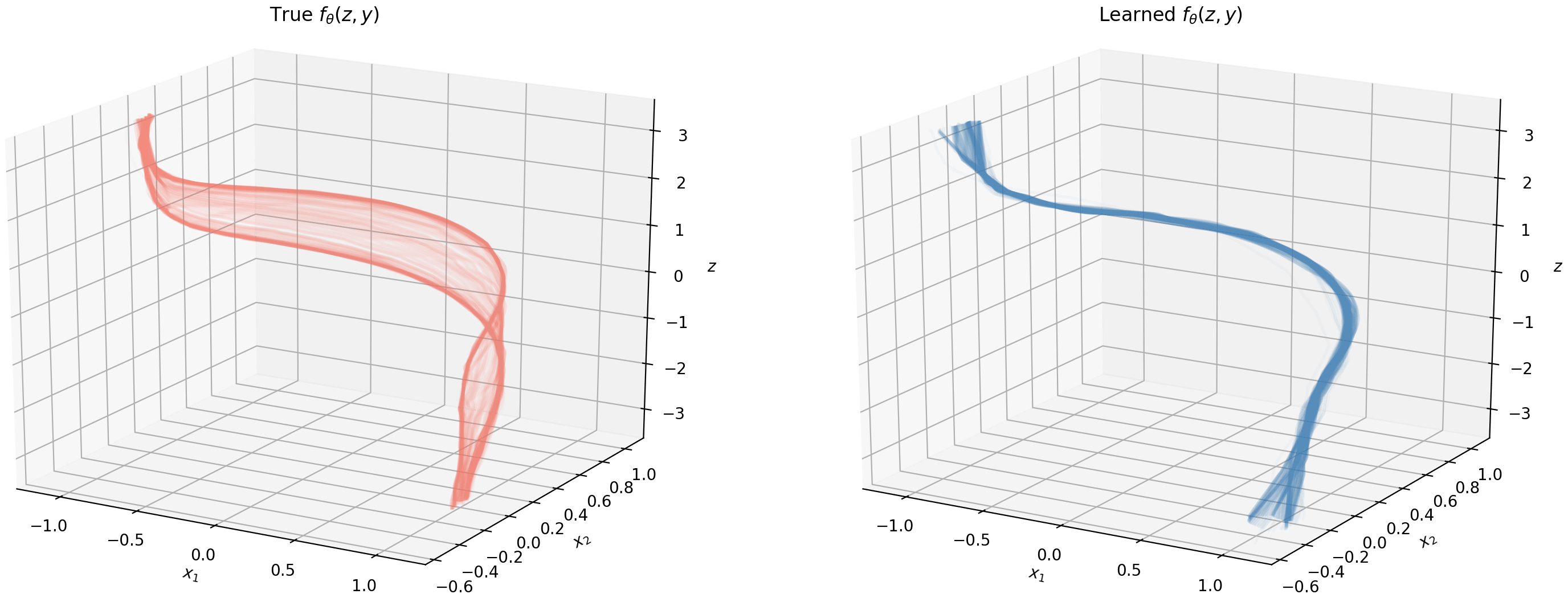}
    \caption{True vs. learned $f_\theta(z, y)$.}
    \label{fig:lin-ss-continuous-fn}
    \end{subfigure}
    ~
    \begin{subfigure}[t]{0.49\textwidth}
    \includegraphics[width=1.0\textwidth]{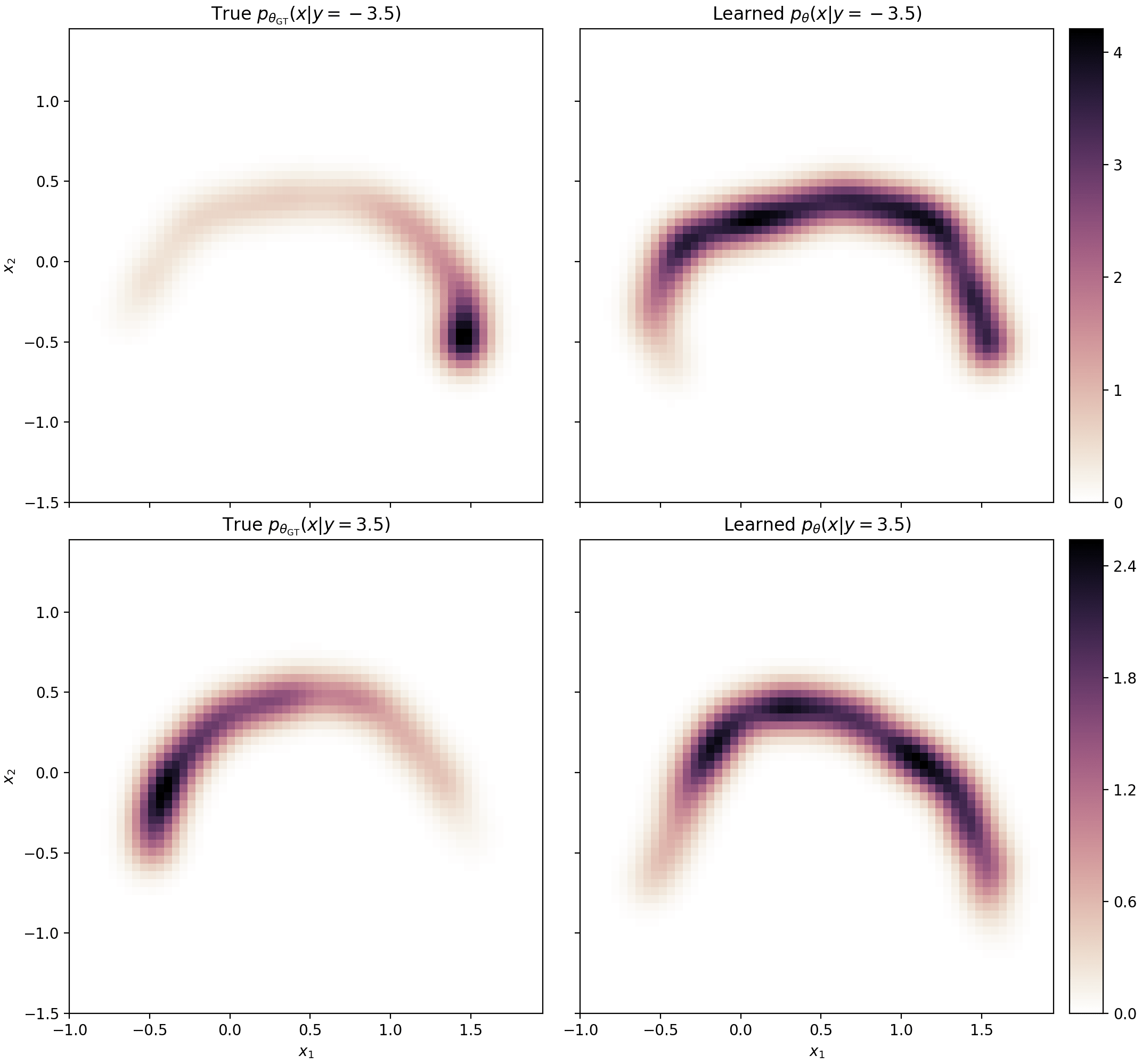}
    \caption{True vs. learned data conditionals $p_\theta(x | y)$.}
    \label{fig:lin-ss-continuous-px-given-y}
    \end{subfigure}
    ~
    \begin{subfigure}[t]{0.7\textwidth}
    \includegraphics[width=1.0\textwidth]{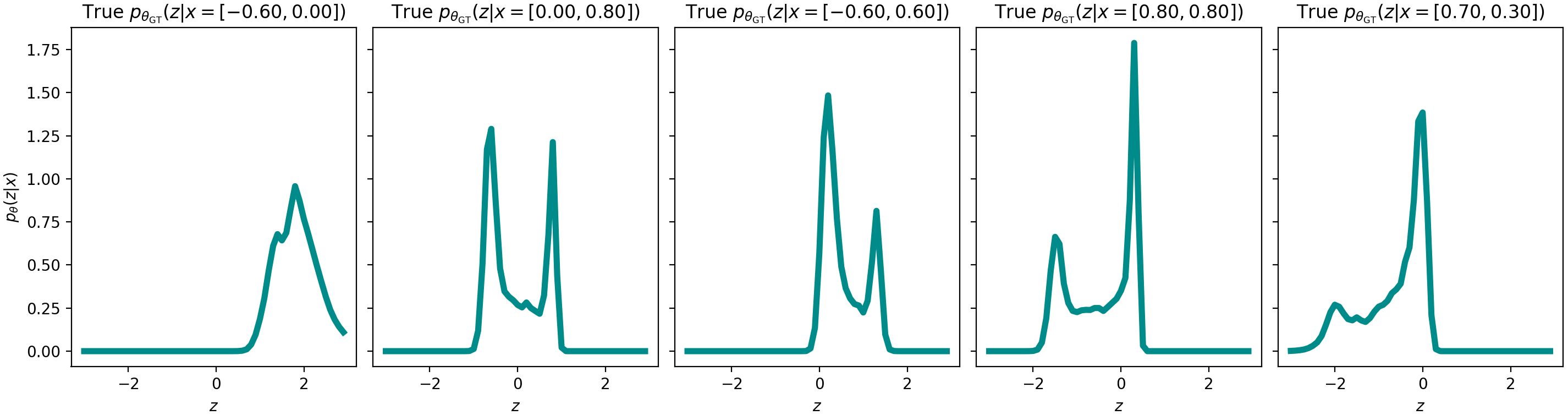}
    \caption{Posteriors under true $f_\theta$}
    \label{fig:lin-ss-continuous-post-true}
    \end{subfigure}
    ~
     \begin{subfigure}[t]{0.7\textwidth}
    \includegraphics[width=1.0\textwidth]{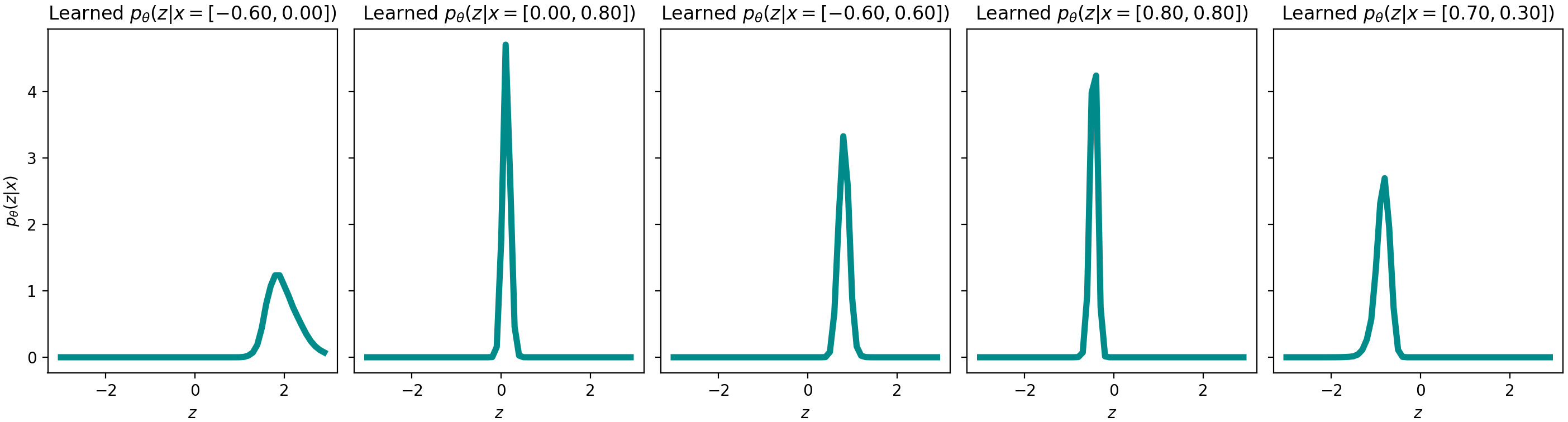}
    \caption{Posteriors under learned $f_\theta$}
    \label{fig:lin-ss-continuous-post-learned}
    \end{subfigure}

    \caption{Semi-Supervised VAE trained with Lagging Inference Networks (LIN) trained on the Continuous Semi-Circle Example.
    While LIN may help escape local optima, on this data, the training objective is still biased away
    from learning the true data distribution.
    As such, LIN fails in the same way an MFG-VAE does (see Figure \ref{fig:vae-ss-continuous}).
    }
    \label{fig:lin-ss-continuous}
\end{figure*}

\begin{figure*}[p]
    \centering
    \vspace*{-1cm}
    \tiny
    
    \begin{subfigure}[t]{0.49\textwidth}
    \includegraphics[width=1.0\textwidth]{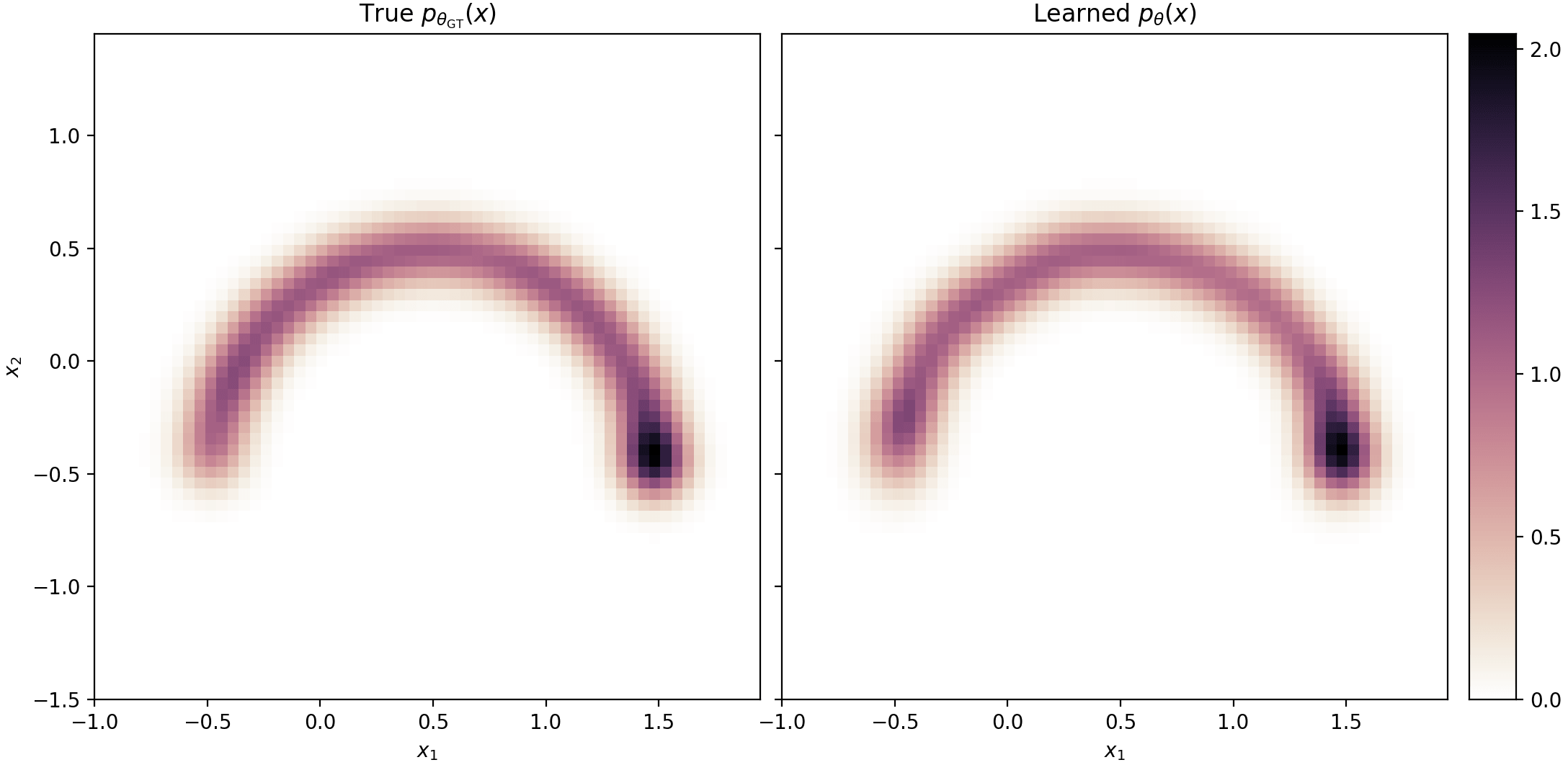}
    \caption{True vs. learned $p_\theta(x)$.}
    \label{fig:iwae-ss-continuous-px}
    \end{subfigure}
    ~
    \begin{subfigure}[t]{0.49\textwidth}
    \includegraphics[width=1.0\textwidth]{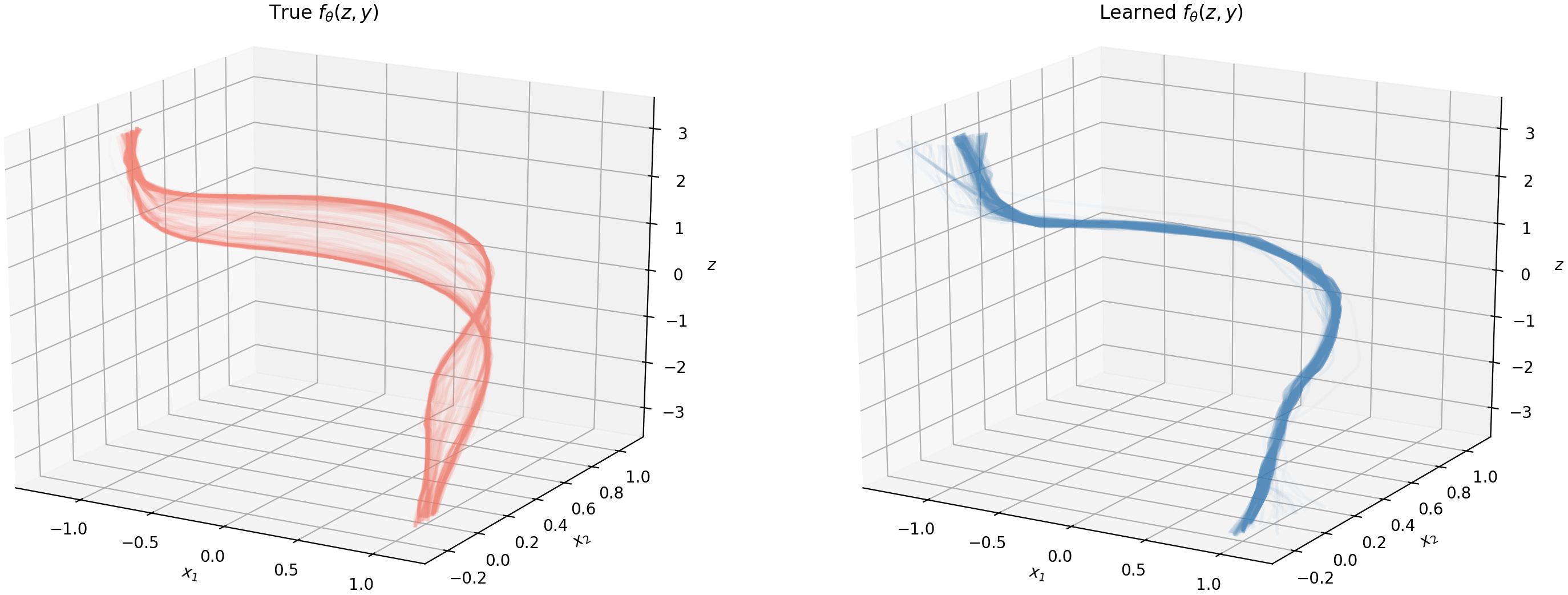}
    \caption{True vs. learned $f_\theta(z, y)$.}
    \label{fig:iwae-ss-continuous-fn}
    \end{subfigure}
    ~
    \begin{subfigure}[t]{0.49\textwidth}
    \includegraphics[width=1.0\textwidth]{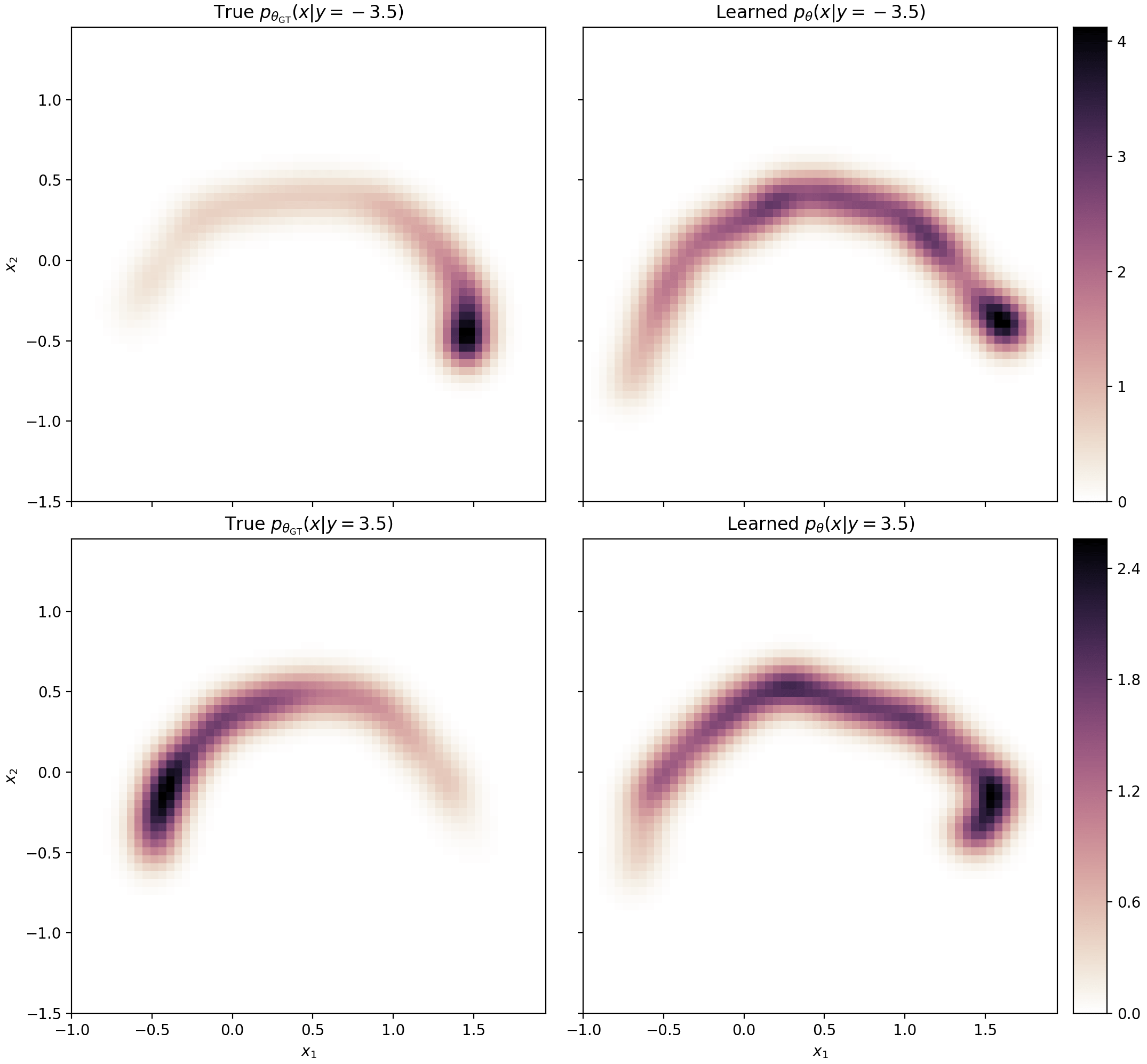}
    \caption{True vs. learned data conditionals $p_\theta(x | y)$.}
    \label{fig:iwae-ss-continuous-px-given-y}
    \end{subfigure}
    ~
    \begin{subfigure}[t]{0.7\textwidth}
    \includegraphics[width=1.0\textwidth]{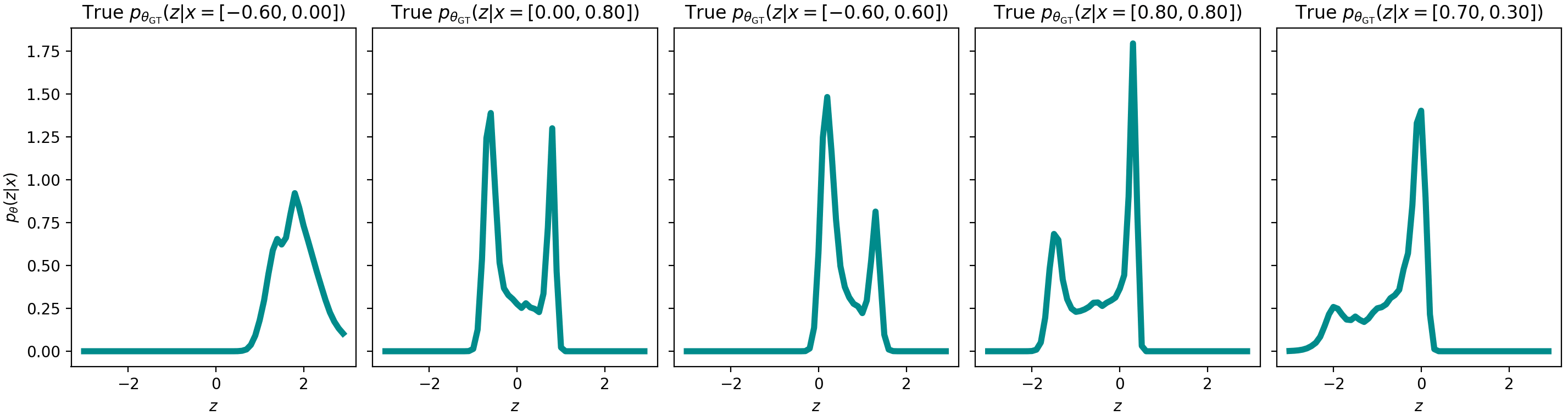}
    \caption{Posteriors under true $f_\theta$}
    \label{fig:iwae-ss-continuous-post-true}
    \end{subfigure}
    ~
     \begin{subfigure}[t]{0.7\textwidth}
    \includegraphics[width=1.0\textwidth]{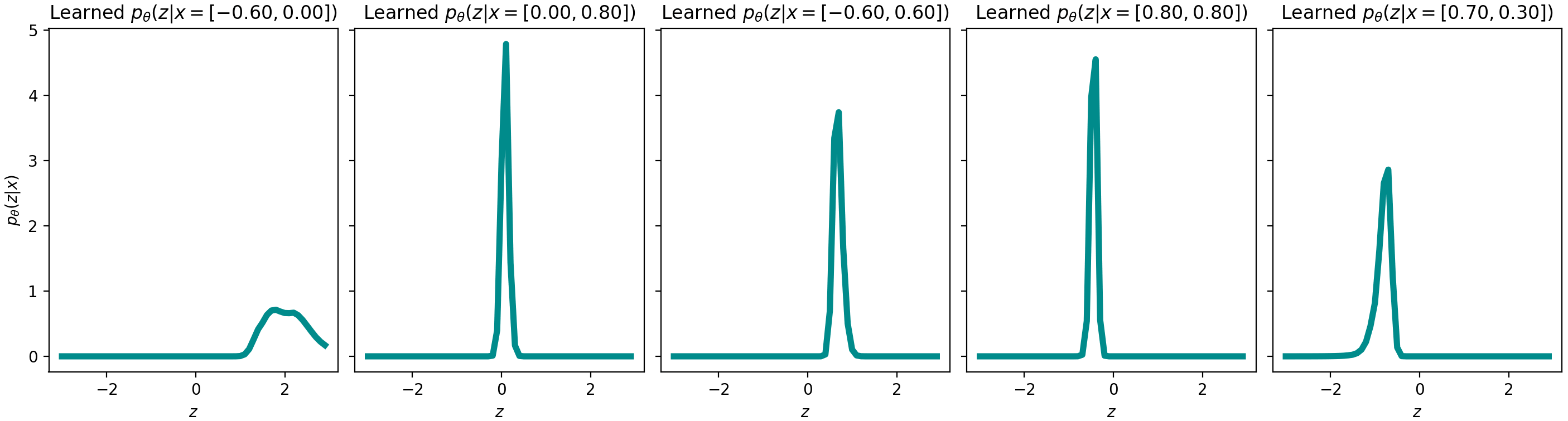}
    \caption{Posteriors under learned $f_\theta$}
    \label{fig:iwae-ss-continuous-post-learned}
    \end{subfigure}

    \caption{
    Semi-Supervised IWAE trained on the Continuous Semi-Circle Example.
    While using semi-supervision, a IWAE is still able to learn the $p(x)$ and $p(x | y)$ better than a VAE.
    However, since $q_\phi(y | x)$ (the discriminator) in the objective is a Gaussian,
    and the ground-truth $p_\theta(y | x)$ is multi-modal, the objective will select a function 
    $f_\theta$ under which $p_\theta(y | x)$ is an MFG.
    This, again, leads to learning a model in which $f_\theta(y = \cdot, z)$ are the same for all values of $y$,
    causing $p(x|y = 0) \approx p(x | y = 1) \approx p(x)$.
    The learned model will therefore generate poor sample quality counterfactuals.
    }
    \label{fig:iwae-ss-continuous}
\end{figure*}

\begin{figure*}[p]
    \centering

    \begin{subfigure}[t]{1.0\textwidth}
    \includegraphics[width=1.0\textwidth]{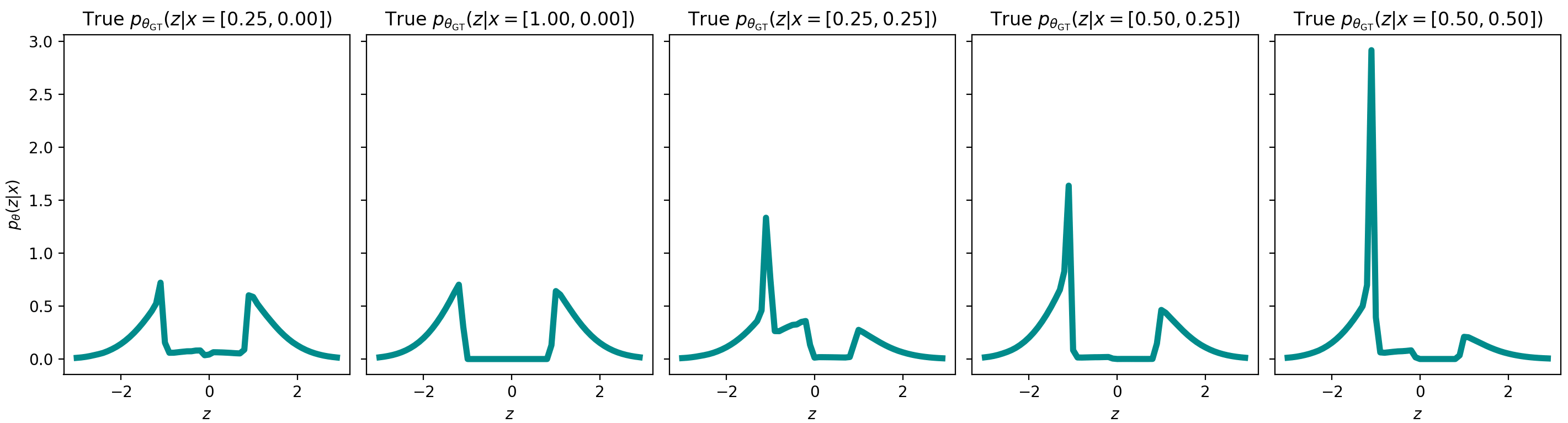}
    \caption{True Posterior $K = 1$}
    \label{fig:clusters-mismatch-5d-post-true}
    \end{subfigure}
    ~
    \begin{subfigure}[t]{1.0\textwidth}
    \includegraphics[width=1.0\textwidth]{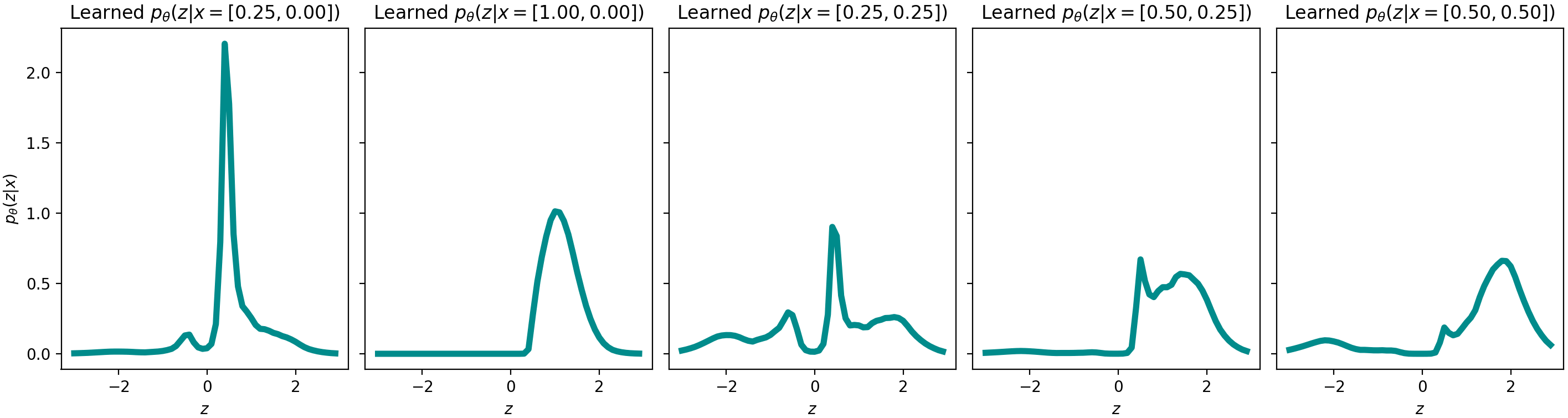}
    \caption{Learned Posterior $K = 1$}
    \label{fig:clusters-mismatch-5d-post-learned-k1}
    \end{subfigure}
    ~
    \begin{subfigure}[t]{1.0\textwidth}
    \includegraphics[width=1.0\textwidth]{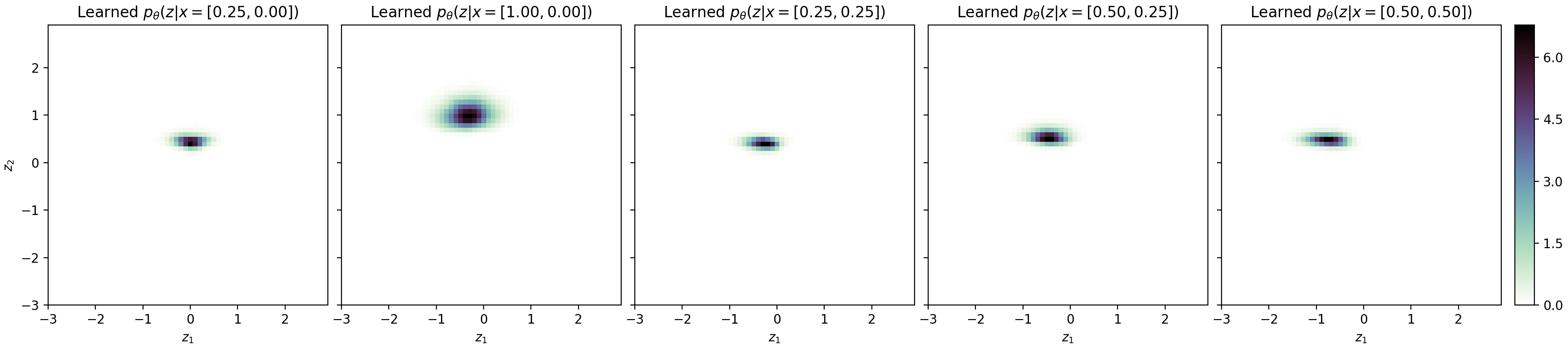}
    \caption{Learned Posterior $K = 2$}
    \label{fig:clusters-mismatch-5d-post-learned-k2}
    \end{subfigure}

    \caption{VAEs learn simpler posteriors as latent dimensionality $K$ increases and
    as the observation noise $\sigma^2_\epsilon$ decreases on ``Clusters Example'' (projected into 5D space).}
    \label{fig:clusters-mismatch-5d}
\end{figure*}

\begin{figure*}[p]
    \centering

    \begin{subfigure}[t]{1.0\textwidth}
    \includegraphics[width=1.0\textwidth]{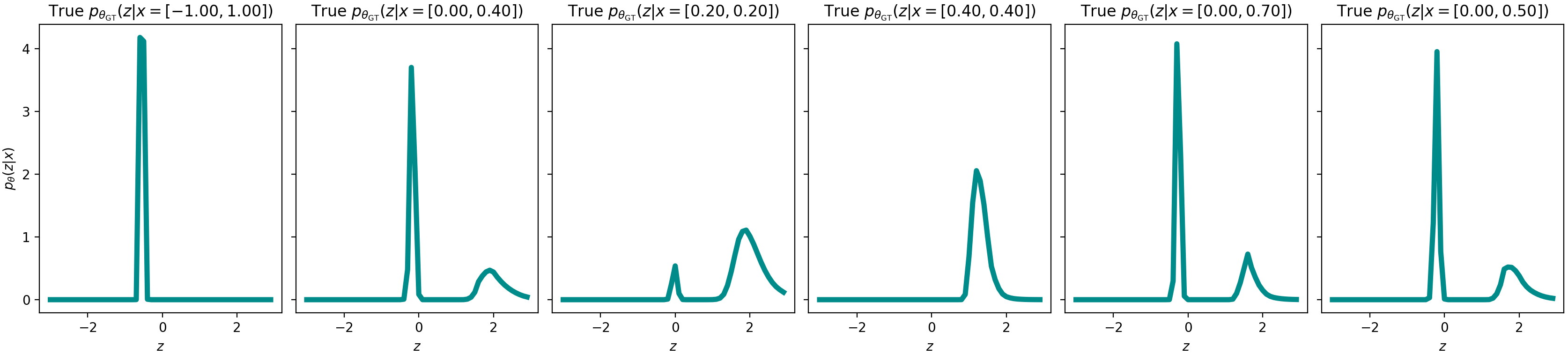}
    \caption{True Posterior $K = 1$}
    \label{fig:fig-8-mismatch-5d-post-true}
    \end{subfigure}
    ~
    \begin{subfigure}[t]{1.0\textwidth}
    \includegraphics[width=1.0\textwidth]{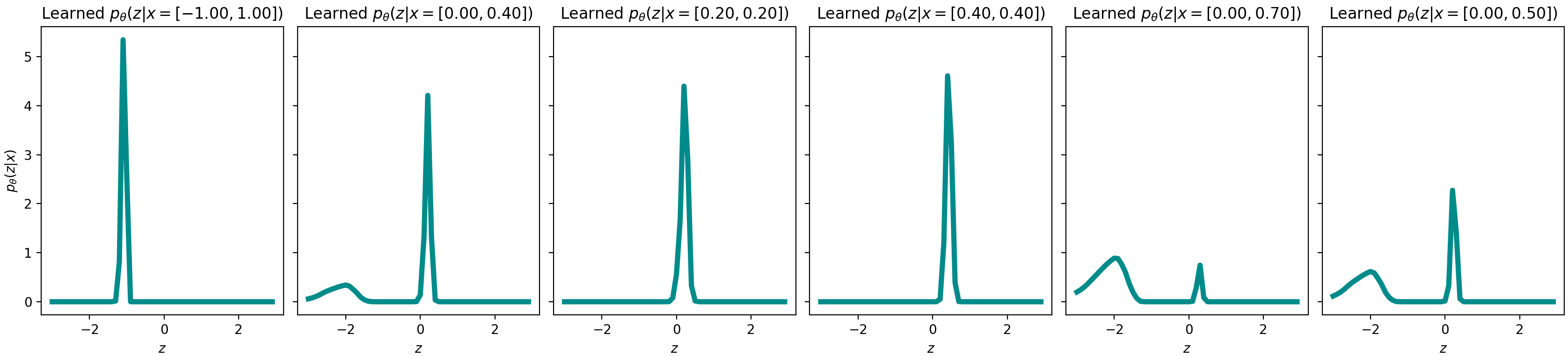}
    \caption{Learned Posterior $K = 1$}
    \label{fig:fig-8-mismatch-5d-post-learned-k1}
    \end{subfigure}
    ~
    \begin{subfigure}[t]{1.0\textwidth}
    \includegraphics[width=1.0\textwidth]{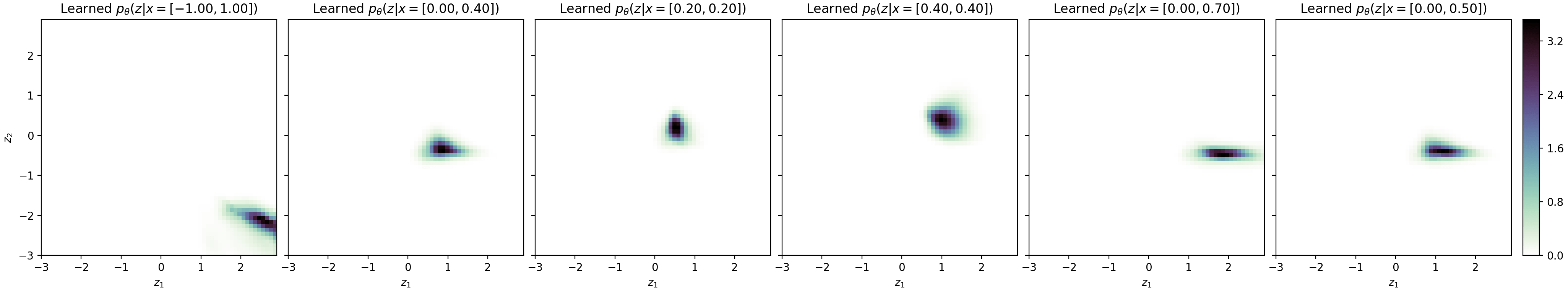}
    \caption{Learned Posterior $K = 2$}
    \label{fig:fig-8-mismatch-5d-post-learned-k2}
    \end{subfigure}

    \caption{VAEs learn simpler posteriors as latent dimensionality $K$ increases and
    as the observation noise $\sigma^2_\epsilon$ decreases on ``Figure-8 Example'' (projected into 5D space).}
    \label{fig:fig-8-mismatch-5d}
\end{figure*}

\end{document}